%% file: thesis.tex
\documentclass[reqno,12pt,oneside]{report} 

\input{macros}

\PassOptionsToPackage{numbers, compress}{natbib}

\begin{document}

\bibliographystyle{agu04}    

\titlepage{Interpretable and Scalable Graphical Models for Complex Spatio-temporal Processes}{Yu Wang}{Doctor of Philosophy}
{Statistics}{2022}
{
 Assistant Professor Yang Chen, Co-Chair\\
 Professor Alfred O. Hero III, Co-Chair\\
 Assistant Professor Walter Dempsey\\
 Dr. Earl Lawrence, Los Alamos National Laboratory}

\initializefrontsections
%
%
%
\mycopyrightpage{Yu Wang}{wayneyw@umich.edu}{ORCID iD: 0000-0002-6287-4710}

\makeatletter
\if@twoside \setcounter{page}{4} \else \setcounter{page}{1} \fi
\makeatother
 


\startacknowledgementspage
\input{Intro/Acknowledgements}
\label{Acknowledgements}

\tableofcontents     

\listoffigures       
\listoftables        
\listofappendices    

\mystartabstractpage
\input{Abstract/Abstract}
\label{Abstract}

\startthechapters 

\chapter{Introduction}
\label{ch:intro}
\input{Intro/Intro}

\chapter{The Sylvester Graphical Lasso}
\label{ch:syglasso}
\input{Chap2/chap2}

\chapter{A Proximal Alternating Linearized Minimization Method for Tensor Graphical Models}
\label{ch:sgpalm}
\input{Chap3/chap3}
\chapter{Multiway Ensemble Kalman Filter}
\label{ch:enkf}
\input{Chap4/chap4}

\chapter{A Geometry-driven Framework for Dynamic Topic Modeling}
\label{ch:gdtm}
\input{Chap5/chap5}

\chapter{Conclusion and Future Work}
\label{ch:conclusion}
\input{Conclusion/conclusion}

\appendix
\chapter{Appendix of Chapter II}
\label{app:syglasso}
\input{Appendices/Appendix_A}

\chapter{Appendix of Chapter III}
\label{app:sgpalm}
\input{Appendices/Appendix_B}  

\chapter{Appendix of Chapter IV}
\label{app:enkf} 
\input{Appendices/Appendix_C}   

\chapter{Appendix of Chapter V}
\label{app:gdtm} 
\input{Appendices/Appendix_D}

\startbibliography
\begin{singlespace} 
\bibliography{thesis}   
\end{singlespace}


\end{document}

%% file: macros.tex
\usepackage{rac}         
\usepackage[intlimits]{amsmath} 
\usepackage{amsxtra}     
\usepackage{amsthm}
\usepackage{amssymb}
\usepackage{amsfonts}
\usepackage{bm}
\usepackage{graphicx}    
\usepackage{rotating}
\usepackage{lscape}
\usepackage{color}
\usepackage{xcolor}
\usepackage[ruled,algochapter]{algorithm2e}
\usepackage{algorithmic}

\usepackage{physics}
\usepackage{epsfig}
\usepackage{enumerate}
\usepackage{enumitem}
\usepackage{subcaption}
\graphicspath{{Intro/Figures/}{Chap2/Figures/}{Chap3/Figures/}{Chap4/Figures/}{Chap5/Figures/}}
\usepackage{verbatim}
\usepackage{natbib}      
\usepackage[printonlyused]{acronym} 
\usepackage[hidelinks]{hyperref}
\usepackage{multirow}
\usepackage{multicol}
\usepackage{subcaption}
\usepackage{booktabs}
\usepackage{tabularx}
\usepackage{makecell}
\usepackage{tikz}
\usepackage{booktabs}
\newcommand\headercell[1]{%
   \smash[b]{\begin{tabular}[t]{@{}c@{}} #1 \end{tabular}}}
\usetikzlibrary{bayesnet}
\usetikzlibrary{positioning}
\usepackage{chngcntr} 
\usepackage[normalem]{ulem} 
\usepackage{setspace}    
\theoremstyle{plain}
\newtheorem{theorem}{Theorem}

\newtheorem{lemma}[theorem]{Lemma}

\newtheorem{remark}{Remark}[chapter]

\theoremstyle{definition}
\newtheorem{definition}[theorem]{Definition}

\newcounter{numexample}[section]

\numberwithin{theorem}{section}     
\numberwithin{figure}{chapter}

\numberwithin{table}{chapter}

\makeatletter
\def\cleardoublepage{\clearpage\if@twoside \ifodd\c@page\else
\hbox{}
\thispagestyle{empty}
\newpage
\if@twocolumn\hbox{}\newpage\fi\fi\fi}
\makeatother

\newcommand{\RNum}[1]{\uppercase\expandafter{\romannumeral #1\relax}}

\newcommand*{\tensor}[1]{\bm{\mathcal{#1}}}
\DeclareMathOperator{\vecto}{vec}
\newcommand*{\mat}[1]{\mathbf{#1}}

\newcommand{\be}{\begin{eqnarray}} \newcommand{\ee}{\end{eqnarray}}
\newcommand{\ben}{\begin{eqnarray*}} \newcommand{\een}{\end{eqnarray*}}
\newcommand{\bc}{\begin{center}} \newcommand{\ec}{\end{center}}
\newcommand{\bt}{\begin{tabbing}} \newcommand{\et}{\end{tabbing}}

\newcommand{\Reals}{\mbox{\rm I\kern-.25em R}}  
\newcommand{\Integers}{\mbox{\rm Z\kern-.25em Z}}  

\newcommand{\Var}{\mbox{\rm Var}}
\newcommand{\cov}{{\mathrm{Cov}}}

\newcommand{\argmin}{{\mbox{\rm argmin}}}

\def\Reals{{\mathbb{R}}}

\def\bSigma{{\boldsymbol{\Sigma}}}

\def\cov{{}\mathrm{cov}}

\def\diag{{\mathrm{diag}}}

\def\bU{{\mathbf{U}}}
\def\bu{{\mathbf{u}}}
\def\bS{{\mathbf{S}}}

\def\bA{{\mathbf{A}}}
\def\bB{{\mathbf{B}}}

\def\bD{{\mathbf{D}}}

\def\bF{{\mathbf{F}}}

\def\bH{{\mathbf{H}}}

\def\bK{{\mathbf{K}}}

\def\bR{{\mathbf{R}}}
\def\bS{{\mathbf{S}}}

\def\bU{{\mathbf{U}}}

\def\bx{{\mathbf{x}}}

\def\by{{\mathbf{y}}}

\def\bff{{\mathbf{f}}}

\def\bfv{{\mathbf{v}}}
\def\bfw{{\mathbf{w}}}

\def\bA{{\mathbf{A}}}
\def\bB{{\mathbf{B}}}

\def\bD{{\mathbf{D}}}

\def\bF{{\mathbf{F}}}

\def\bH{{\mathbf{H}}}

\def\bK{{\mathbf{K}}}

\def\bR{{\mathbf{R}}}
\def\bS{{\mathbf{S}}}

\def\bU{{\mathbf{U}}}

\def\calR{{\mathcal{R}}}

\def\tr{{\mathrm{tr}}}

\def\bR{{\mathbf{R}}}

\def\bff{{\mathbf{f}}}

\def\bfv{{\mathbf{v}}}
\def\bfw{{\mathbf{w}}}

\def\bSigma{{\boldsymbol{\Sigma}}}
\def\bOmega{{\boldsymbol{\Omega}}}

\def\bbE{{\mathbb{E}}}

\def\bbR{{\mathbb{R}}}

\def\calR{{\mathcal{R}}}

%% file: Intro/Acknowledgements.tex

First and foremost, I would like to express my deepest gratitude to my Ph.D. advisors Dr. Alfred Hero and Dr. Yang Chen.  
This dissertation would not have been possible without their continuous guidance, support, and encouragement. 
Al is an exemplary scholar, who is always hardworking and dedicated to research. I am constantly amazed by his sharpness on research and his deep insights about many different topics, ranging from physics, applied math, to statistics, and computer science.
I am deeply indebted to him for devoting so much time and energy to mentoring me and guiding me through the transition from a student to a researcher. 
His passion for research have profoundly influenced me from both professional and personal perspectives. 
I am also fortunate enough to work with Yang, an inspirational advisor, who is incredibly generous with her ideas and time; and a caring mentor, who constantly offers her support and empathy.  
Our meetings and discussions have been a great source of inspiration and greatly shaped my way of approaching research problems.

I am also very grateful to have Dr. Walter Dempsey and Dr. Earl Lawrence serving on my  doctoral dissertation committee and providing me with invaluable comments. 
I was fortunate to collaborate with Walter on topic models and important problems in the public health domain. Walter's  enthusiasm for research has left a lot of positive impacts on me. I first met Earl during his visit to the department as a distinguished alumni speaker. Later, I had the opportunity to work with him at LANL on distributed dimensionality reduction and applications on space weather. His constant support and humor make all the research meetings there and my overall experience at LANL enjoyable. 

My thanks also go to staff members at the University of Michigan, Department of Statistics, who have helped me over the past few years. 
In particular, I want to thank Judy, Jean, Bebe, Virggie, Andrea, Gina, and many others, who always patiently helped with my questions and warmly welcomed me into the office with big smiles. 

Additionally, I would like to express my gratitude and appreciation to Dr. Jim Zidek and Dr. Nhu Le at the University of British Columbia in Canada. My research career started with Jim and Nhu, who are both brilliant researchers and caring mentors. Jim has always been a role model to me, and without him, I would not have gone this far in this journey.

To all my friends that I made and all the people that I met throughout the Ph.D. studies: This journey would not have been so rewarding without you! 
Shout out to everyone in our research labs, especially Byoung and Zeyu from the Hero Group that I was fortunate to collaborate with; Leo who organized those fun board games; and Chengcheng, Cheng, Yangyi, and Ziping in my cohort. It was a great fun to spend time with you, and I have learned a lot from our  interactions.
Last but not least, I would like to thank my parents, Xiaoqing Tan the duck, Kitty \& Bunny the cats, and Larry the chinchilla for their unwavering support and unconditioned love. 

%% file: Abstract/Abstract.tex
This thesis focuses on data that has complex spatio-temporal structure and on probabilistic graphical models that learn the structure in an interpretable and scalable manner. We target two research areas of interest: Gaussian graphical models for tensor-variate data and summarization of complex time-varying texts using topic models. This work advances the state-of-the-art in several directions. First, it introduces a new class of tensor-variate Gaussian graphical models via the Sylvester tensor equation. Second, it develops an optimization technique based on a fast-converging proximal alternating linearized minimization method, which scales tensor-variate Gaussian graphical model estimations to modern big-data settings. Third, it connects Kronecker-structured (inverse) covariance models with spatio-temporal partial differential equations (PDEs) and introduces a new framework for ensemble Kalman filtering that is capable of tracking chaotic physical systems. Fourth, it proposes a modular and interpretable framework for unsupervised and weakly-supervised probabilistic topic modeling of time-varying data that combines generative statistical models with computational geometric methods. Throughout, practical applications of the methodology are considered using real datasets. This includes brain-connectivity analysis using EEG data, space weather forecasting using solar imaging data, longitudinal analysis of public opinions using Twitter data, and mining of mental health related issues using TalkLife data. We show in each case that the graphical modeling framework introduced here leads to improved interpretability, accuracy, and scalability.



%% file: Intro/Intro.tex
Complex, structured data is ubiquitous in both industrial and academic settings and has elicited a commensurate interest in utilizing such information to assist in inference and decision making. Often, there exists simpler and interpretable underlying structure that can be exploited to make inference and summarization procedures more tractable. For large datasets, in particular, it is imperative to consider the data in the context of its structure to develop parsimonious models that represent the intrinsic form of the data well and provide computationally efficient, theoretically grounded inference procedures. On one hand, searching for such structures can help to summarize the data in a more interpretable manner and find relevant attributes of the data of interest that might otherwise go undetected. On the other hand, for some datasets the structure is explicit, and thus requires careful consideration when reasoning about modeling decisions. 

Moreover, despite the fact that machine learning models have recently demonstrated great success in learning the above-mentioned complex structures that enable them to make predictions about unobserved data, the ability to interpret what a model has learned is yet to be determined and has been receiving an increasing amount of attention~\citep{rudin2022interpretable,murdoch2019definitions,du2019techniques,doshi2017towards,rudin2019stop,papernot2018deep, tan2022tree, tan2022rise}. In particular, \citet{murdoch2019definitions} recently introduced a unified PDR (predictive, descriptive, relevant) framework for discussing interpretations of machine learning and statistical models in general, and categorized existing techniques into model-based and post-hoc categories, with subgroups including sparsity, modularity, and simulatability. In this dissertation, the focus is on data that has temporal or spatio-temporal structure and on problems that benefit from the application of spatio-temporal based inference algorithms. In both cases, we target the overarching desiderata described in the PDR interpretability framework and introduce statistical methods that improve the overall (predictive and descriptive) accuracy and relevancy through both model-based and post-hoc approaches. Specifically, we attempt to advance two research areas. First, Gaussian graphical models for tensor-valued data is studied, and we develop a sparse multiway representation of constituent spatial and temporal processes, which enables a decomposable (i.e., spatial and temporal) and scalable framework for analyzing tensor data, especially that generated from complex dynamical systems. Second, a framework for topic modeling of time-varying texts is developed. The framework breaks previously (computationally and statistically) intractable approaches into tractable modules and utilizes computational geometric methods for extracting various (stable) forms of information from the fitted model. Overall, we improve interpretability, scalability, and accuracy throughout the full life cycle of a data science problem with complex structure. Below, these two research areas are briefly introduced that form the backbone of this thesis.


\section{Gaussian Graphical Models for Tensor-valued Data}
Estimating conditional independence patterns of multivariate data has long been a topic of interest for statisticians. In the past decade, researchers have focused on imposing sparsity on the precision matrix (inverse covariance matrix) to develop efficient estimators in the high-dimensional statistics regime where sample size is much less than the dimension of each sample ($N \ll d$). The success of the $\ell_1$-penalized method for estimating conditional dependencies was demonstrated in \citet{meinshausen2006high} and \citet{friedman2008sparse} for the multivariate Gaussian setting. Contributing to this success is the underlying graphical structure (see Figure~\ref{fig:ggm}) that facilitates simple interpretation and ties the statistical model to the mathematical field of graph theory~\citep{lauritzen1996graphical}.

\begin{figure}[!tbh]
    \centering
    \includegraphics[width=\textwidth]{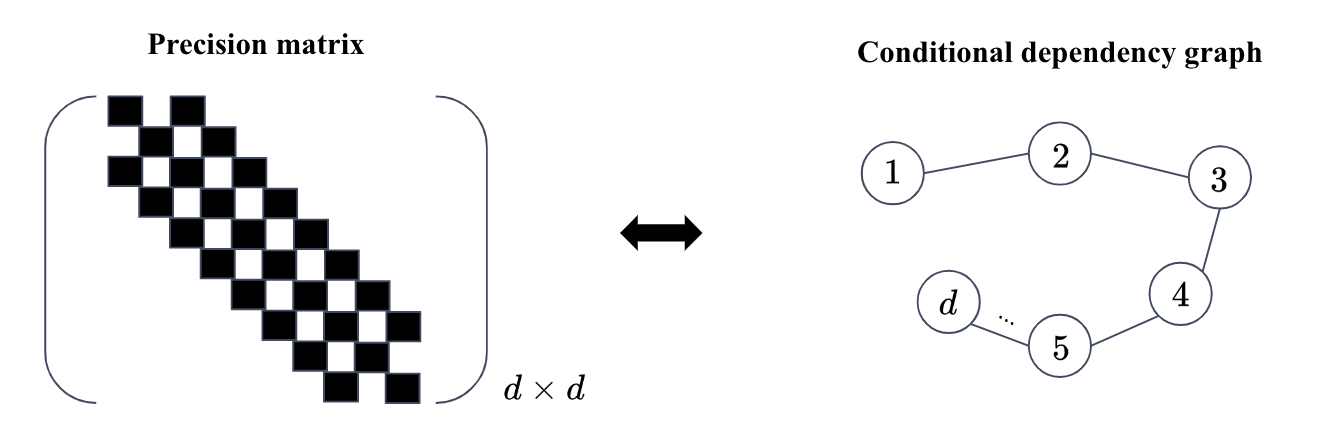}
    \caption{For multivariate Gaussian variables, the conditional dependence structure encoded in the precision matrix (left) can be represented by a simple chain graph (right).}
    \label{fig:ggm}
\end{figure}

This success has naturally led researchers to generalize these methods to multiway tensor-valued data. Such generalizations are of benefit for many applications, including for example, the estimation of brain connectivity in neuroscience, reconstruction of molecular networks, and detecting anomalies in social networks over time. The first generalizations of multivariate analysis to the tensor-variate settings were presented by \citet{dawid1981some}, where the matrix-variate (a.k.a. two-dimensional tensor) distribution was first introduced to model the dependency structures among both rows and columns. \citet{dawid1981some} extended the multivariate setting by rewriting the tensor-variate data as a vectorized (vec) representation of the tensor samples $\tensor{X} \in \bbR^{d_1 \times \cdots \times d_k}$ and analyzing the overall precision matrix $\mat{\Omega} \in \bbR^{d \times d}$, where $d = \prod_{k=1}^K d_k$. Even for a two-dimensional tensor $\tensor{X} \in \bbR^{d_1 \times d_2}$, the computation complexity and sample complexity is high since the number of parameters in the precision matrix grows quadratically as $d^2$. Therefore, in the regime of tensor-variate data, unstructured precision matrix estimation has posed challenges due to the large number of samples needed for accurate structure recovery. 

To address the sample complexity challenges, sparsity can be imposed on the precision matrix $\mat{\Omega}$ by using a sparse Kronecker product (KP) or Kronecker sum (KS) decomposition of $\mat{\Omega}$, where each decomposed factor has an underlying graphical representation like Figure~\ref{fig:ggm} that can be modeled, estimated, and interpreted separately. The earliest and most popular form of sparse structured precision matrix estimation represents $\mat{\Omega}$ as the KP of smaller precision matrices, which corresponds to a separable structure across different modes of a data tensor (see Figure~\ref{fig:multiway_patternedcov}). \citet{tsiligkaridis2013convergence} and \citet{zhou2014gemini} proposed to model the precision matrix as a sparse KP of the precision matrices along each mode of the tensor in the form $\bm \Omega =  \bm{\Psi}_1 \otimes \cdots \otimes \bm \Psi_K$. The KP structure on the precision matrix has the nice property that the corresponding covariance matrix is also a KP. \citet{zhou2014gemini} provides a theoretical framework for estimating the $\mat{\Omega}$ under KP structure and showed that the precision matrices can be estimated from a single instance under the matrix-variate normal distribution. \citet{lyu2019tensor} extended the KP structured model to tensor-valued data, and provided new theoretical insights into such models. An alternative, called the Bigraphical Lasso, was proposed by \citet{kalaitzis2013bigraphical} to model conditional dependency structures of precision matrices by using a KS representation $\bm \Omega = \bm \Psi_1 \oplus \bm \Psi_2 = (\bm \Psi_1 \otimes \mat I) + (\mat I \otimes \bm \Psi_2)$. On the other hand, \citet{rudelsonzhou17errinvardependent} and \cite{park2017non} studied the KS structure on the covariance matrix $\bm \Sigma = \mat A \oplus \mat B$ which corresponds to errors-in-variables models. More recently, \citet{greenewald2019tensor} proposed a model that generalized the KS structure to model tensor-valued data, called the TeraLasso. As shown in their paper, compared to the KP structure, KS structure on the precision matrix leads to a different type of separability on the covariance matrix that provides a more parsimonious representation.

\begin{figure}
    \centering
    \includegraphics[width=\textwidth]{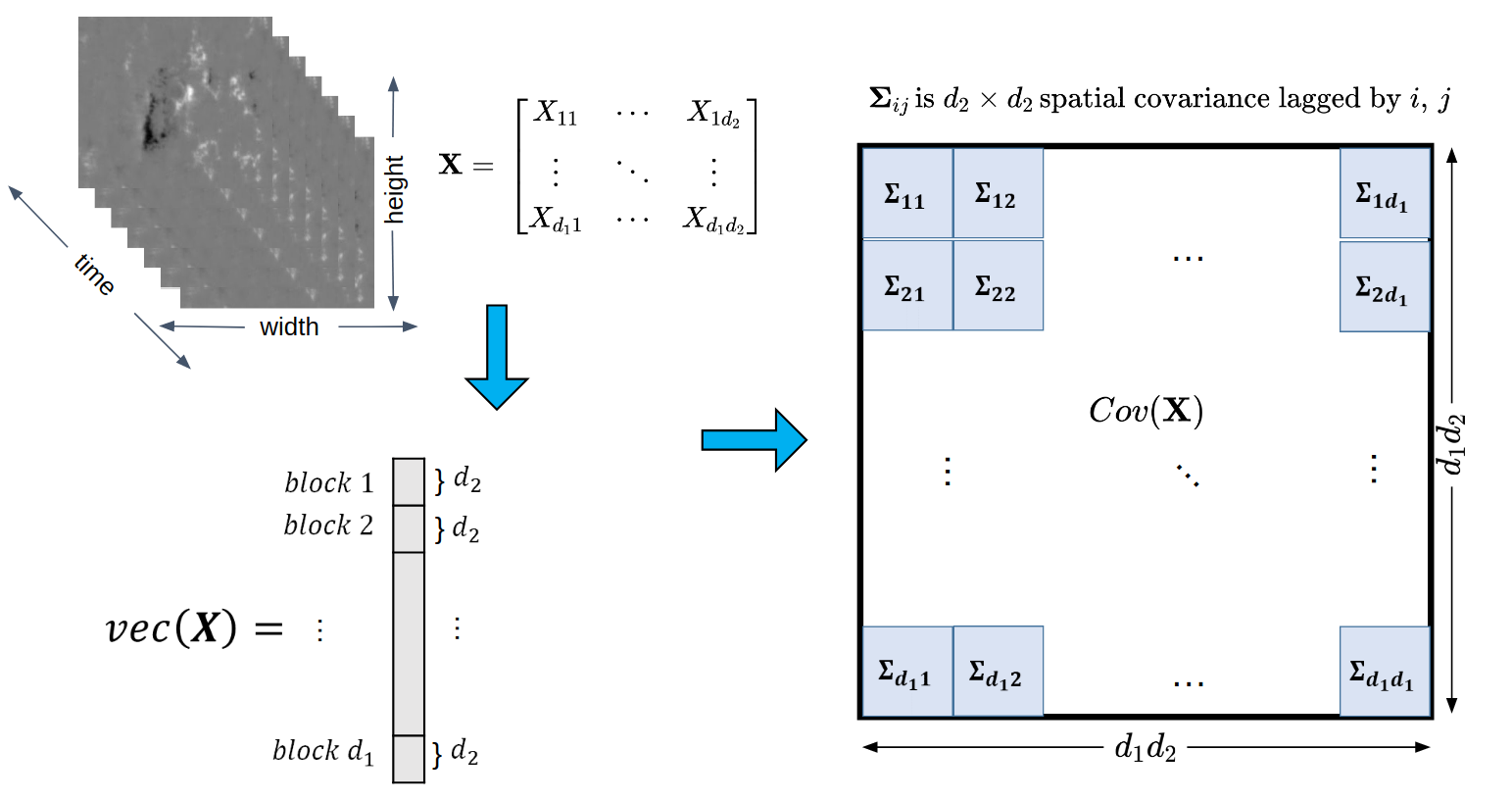}
    \caption{Multiway data results in a patterned covariance. This structure can be exploited by assuming similar patterns within each block, such as those assumed in the Kronecker product models.}
    \label{fig:multiway_patternedcov}
\end{figure}

Despite being modeling choices, the KP and KS structures admit their own pros and cons. The KP model admits a simple stochastic representation, which defines a generating process for the underlying data. Unlike the KP model, the KS model does not lead to a natural generative interpretation. From another perspective, Kronecker structures can be characterized by the product graphs of the individual components. In particular, \citet{kalaitzis2013bigraphical} first motivated the KS structure on the precision matrix by relating Kronecker sum of matrices to the associated Cartesian product graph. Thus, the overall structure of $\mat{\Omega}$ naturally leads to a parsimonious model that brings the individual components together. The KP, however, corresponds to the direct tensor product of the individual graphs and leads to a denser dependency structure in the precision matrix \citet{greenewald2019tensor}. Chapter~\ref{ch:syglasso} proposes a new Kronecker-structured graphical model that admits a natural stochastic representation for precision matrices associated with tensor data. The resulted Gaussian graphical model strikes a balance between the KP- and KS- structured models. The new model poses additional challenge in computation, Chapter~\ref{ch:sgpalm} proposes an estimation algorithm that utilizes state-of-the-art optimization technique and scales the method to modern big data applications. Chapter~\ref{ch:enkf} studies the connection between multiway Gaussian graphical models and second-order representation of spatio-temporal partial differential equations (PDE) and introduces an Kalman filtering framework for model-based physics-informed data assimilation.

\section{Dynamic Topic Models}
Probabilistic topic model is a suite of algorithms that aim to automatically discover and annotate large collections of documents that contain useful information with thematic labels. Topic modeling algorithms are statistical methods that analyze the words of the original texts to discover the themes that run through them, how those themes are connected to each other, and how they change over time. One such model that has been very successful is the Latent Dirichlet Allocation (LDA) model~\citep{blei2003latent}, which infers the topics (i.e., thematic information) in a corpus by assuming an underlying generative process whereby the documents are created, so that one may infer, or reverse engineer, it. The LDA model posits that documents are represented as random mixtures over latent topics, where each topic is characterized by a distribution over all the words. The complete probabilistic structure can be represented by a simple graphical model shown in Figure~\ref{fig:lda_plate}.

\begin{figure}[thb!]
\centering
    \begin{tikzpicture}[x=1.5cm,y=1.3cm]
      \node[obs]                   (W)      {$w_{d,n}$}; %
      \node[latent, above=of W]    (Z)      {$z_{d,n}$}; %
      \node[latent, above=of Z]     (theta) {$\theta_d$};
      \node[latent, left=of W]    (beta)   {$\beta_k$};
      
      \node[const, above=of theta] (alpha)  {$\alpha$}; %
      \node[const, left=of beta] (eta) {$\eta$};
    
      \factor[left=of W]     {W-f}     {above:Multi} {} {} ; %
      \factor[above=of Z]     {Z-f}     {left:Multi} {} {} ; %
      \factor[above=of theta]   {theta-f}   {left:Dir} {} {} ; %
      \factor[left=of beta]   {beta-f}   {above:Dir} {} {} ; %

      \factoredge {Z} {W} {};
      \factoredge {alpha}  {theta-f} {theta}; %
      \factoredge {theta}  {Z-f} {Z}; %
      \factoredge {eta} {beta-f} {beta}; %
      \factoredge {beta}  {W-f} {W}; %
    
      \plate {plate1} { %
        (W) %
        (Z) %
      } {$N$}; %
      \plate {} { %
        (plate1) %
        (theta)%
      } {$D$} ; %
      \plate {} { %
        (beta) %
      } {$K$} ; %
    \end{tikzpicture} 
\caption{Graphical representation of the standard Latent Dirichlet Allocation (LDA) model. Here nodes are random variables; edges indicate dependence through probability distributions (e.g., Dirichlet or multinomial). Shaded nodes are observed; unshaded nodes are latent. Plates indicate replicated variables.}
\label{fig:lda_plate}
\end{figure}
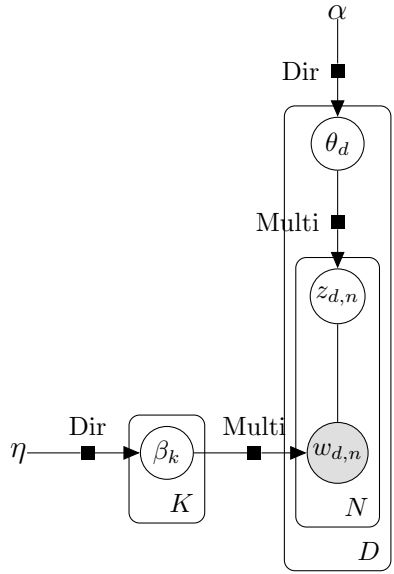

Numerous extensions of the original LDA model have been proposed to handle more complex data that exhibits serial dependencies. In particular, \citet{blei2006dynamic} proposed a Dynamic Topic Model (DTM) that models time-varying corpus (e.g., archive of articles published on the Science journal from 1990 to 2020), and the alignment among topics across time steps is captured by a Kalman filter procedure with a Markov assumption where the state (of topics) at time $t + 1$ is independent of all other history given the state at time $t$. \citet{wang2006topics} deals with similar data and introduced a non-Markov continuous-time model called the Topics-over-Time (TOT), which captures changes in the occurrence (and co-occurrence conditioned on time) of the topics themselves, not changes in the word distribution of each topic. \citet{wang2012continuous} further improved the DTM with a continuous time variant called cDTM that uses Brownian motion to model the latent topics in a sequential collection of documents, where a topic is a pattern of word use that is expected to evolve over the course of the collection. 

All the methods mentioned above try to build certain dynamical or flexible structures explicitly into the probabilistic model. Besides the fact that these methods generally rely on complex stochastic process specifications to model temporal or other dependency structures, they all suffer from the following: 1. natural interpretation comes at the cost of correct model assumption: DTM and its variants achieve interpretability under the assumption of model being correct, which is restrictive as complicated real world applications tend not to follow the models perfectly and any abrupt change in the data makes modeling results hard to interpret; 2. computational instability: as many of these methods rely on either expensive MCMC sampling schemes or variational approximations as inference algorithms, they face the issue of getting trapped into local minimum/maximum of their objectives, which makes the results hard to interpret with confidence; 3. there is inherent inflexibility to diverse dynamical structures, as most methods are developed for handling specific temporal dynamics and are not able to capture all types (e.g., abnormality, clustering, etc) of variations jointly. In Chapter~\ref{ch:gdtm}, a scalable and interpretable framework is developed that attempts to overcome those issues in traditional dynamic topic models. Additionally, the proposed temporal topic modeling approach is extended to incorporate side information via weak supervision.

\section{Outline and Contributions}
This section lists the chapters and corresponding contributions in this dissertation. Each chapter aims to be a self contained exposition on a specific topic; as a result, some introductory material for particular chapters are similar in scope.

Chapter~\ref{ch:syglasso} describes a structured Gaussian graphical model for tensor-valued data. Here, we consider the underlying generating process of the data to be governed by a Sylvester equation. We show that this leads to a Kronecker sum structural assumption on the square root factor of the precision matrix of the data. The resulted modeling approach is able to simultaneously improve robustness, richness, and interpretability of existing Kronecker-structured models. This chapter is based on \citet{wang2020sylvester} and was published in the \textit{Proceedings of the $23^{\text{rd}}$ International Conference on Artificial Intelligence and Statistics}.

Chapter~\ref{ch:sgpalm} tackles a challenging optimization problem posed by the Sylvester graphical model. An algorithm based on the proximal alternating linearized minimization is proposed to estimate generating factors of the model. State-of-the-art convergence rate is achieved and a comprehensive convergence analysis is done via recent development of optimization theory. Practically, we apply the new procedure to astrophysics-related application in solar flare prediction, where we model the solar magnetogram and atmosphere as Guassian Markov Random Field (GMRF) induced by a Sylvester-structured precision matrix. The utility of the estimated precision matrix is demonstrated via a linear prediction of evolution of the solar active regions. The chapter is based on the work of~\citet{wang2021sg} that was published in the \textit{Proceedings of the $38^{\text{th}}$ International Conference on Machine Learning}.

Chapter~\ref{ch:enkf} connects Kronecker-structured (inverse) covariance modeling and spatio temporal PDEs via the ensemble Kalman filter (EnKF) framework for data assimilation. A new EnKF algorithm is introduced and the emergence of sparsity and multiway structures in second-order statistical characterizations of dynamical processes governed by PDEs is studied. We demonstrate promises of the new approach for tracking complex spatio-temporal systems. The chapter is based on the work of \citet{wang2021multiway} presented in the \textit{Workshop on Machine Learning and the Physical Sciences at the $35^{\text{th}}$ Conference on Neural Information Processing Systems} and the joint work with Zeyu Sun, Dogyoon Song, and Alfred Hero that was under revision at \textit{Statistics Surveys}. Additionally, a \texttt{Julia} software package called \texttt{TensorGraphicalModels}~\citep{WANG2022100308} has been
developed to accompany this work.

Chapter~\ref{ch:gdtm} introduces a simple and modular approach for modeling time-varying texts that combines standard LDA, shortest path algorithms on neighborhood graphs, and geometric embedding. This approach enables interpretation and visualization of latent thematic information that are intrinsically temporally dependent. We demonstrate that the framework is able to capture perceptually natural temporal trajectories of latent topics with minimal modeling assumptions. Further, we show that the framework is able to incorporate side information (e.g., labels) via weak supervision. Two important applications are considered: analysis of Twitter data for understanding COVID-19 related public discourse; and analysis of TalkLife data for understanding mental health related issues and aiding early detection and intervention. The work is partially based on the work of \citet{wang2021Geometry} published in the \textit{Harvard Data Science Review}.


%% file: Chap2/chap2.tex
In this chapter we introduce the~\textit{Sylvester graphical lasso} (SyGlasso) that captures multiway dependencies present in tensor-valued data. The model is based on the Sylvester equation that defines a generative model. The proposed model complements the tensor graphical lasso \citep{greenewald2019tensor} which imposes a Kronecker sum model for the inverse covariance matrix, by providing an alternative Kronecker sum model that is generative and interpretable. The interpretability follows from the Sylvester generative model on which SyGlasso is based: the model is exact for any observation process that is a solution of a diffusion-based partial differential equation. A nodewise regression approach is adopted for estimating the conditional independence relationships among variables. The statistical convergence of the method is established, and empirical studies are provided to demonstrate the recovery of meaningful conditional dependency graphs. We apply the SyGlasso to an electroencephalography (EEG) study to compare the brain connectivity of alcoholic and nonalcoholic subjects. We demonstrate that our model can simultaneously estimate both the brain connectivity and its temporal dependencies.

\section{Introduction}\label{sec:syglasso-intro}
To address the sample complexity challenges that arise in modern multivariate analysis of tensor-variate data, sparsity can be imposed on the second order information - the covariance $\mat{\Sigma}$ or the inverse covariance $\mat{\Omega}$ - by using a sparse Kronecker product (KP) or Kronecker sum (KS) decomposition of $\mat{\Sigma}$ or $\mat{\Omega}$. The earliest and most popular form of sparse structured precision matrix estimation approaches represent $\mat{\Omega}$ as the KP of smaller precision matrices, which means that the resulting $\mat{\Sigma}$ also composes of KP of smaller covariance matrices due to the property of KP. \citet{tsiligkaridis2013convergence, zhou2014gemini, lyu2019tensor} have developed estimation and statistical inference procedures under the KP structure and showed that the underlying true precision matrix can be estimated efficiently with high-dimensional consistency guarantees with single matrix or tensor sample. Alternatively, \citet{kalaitzis2013bigraphical, greenewald2019tensor} propose to model conditional dependency structures of precision matrices by using a KS representation. \citet{rudelsonzhou17errinvardependent, park2017non} studied the KS structure on the covariance matrix which corresponds to errors-in-variables models.

\textbf{KP vs KS:} One of the advantages of the KP model is that it admits a simple stochastic representation as $\mat{X}=\mat{C}^{-1}\mat{Z}\mat{D}^{-1}$, where $\mat{A}=\mat{C}\mat{C}^T, \mat{B}=\mat{D}\mat{D}^T$, and $\mat{Z}$ is white Gaussian. It can be shown using properties of KP that $\mat{X} \sim \mathcal{N}(0,(\mat{A} \otimes \mat{B})^{-1})$.  Unlike the KP model, the KS model does not have a simple stochastic representation. From another perspective, the Kronecker structures can be characterized by different types of product graphs of the individual component graphs. Specifically, \citet{kalaitzis2013bigraphical} relates $(\mat{\Psi}_1 \oplus \cdots \oplus \mat{\Psi}_K)$ to the associated Cartesian product graph. As a result, the overall number of edges (active conditional dependencies) is additive in the number of edges in the individual graphs. The KP, however, corresponds to the direct tensor product of the individual graphs and leads to a denser dependency structure in the precision matrix, as the number of overall edges is multiplicative in the number of individual edges~\footnote{From ~\citet{greenewald2019tensor} KS (Cartesian product graph) edges: $\sum_{k=1}^K|E_k|\prod_{i \neq k}|V_i|$; KP (direct product graph) edges: $\frac{1}{2}\prod_{k=1}^K(2|E_k|+|V_k|)-\prod_{k=1}^K|V_k|$; where $E_k$ and $V_k$ denote the edge and vertex sets, respectively for component $k$.}. 

\textbf{The Sylvester Graphical Lasso (SyGlasso):} We propose a \textit{Sylvester structured graphical model} to estimate precision matrices associated with tensor data. Similar to the KP- and KS-structured graphical models, we simultaneously learn $K$ graphs along each mode of the tensor data. However, instead of a KS or KP model for the precision matrix, the Sylvester structured graphical model uses a KS model for the square root factor of the precision matrix. The model is estimated by joint sparse regression models that impose sparsity on the individual components $\bm\Psi_k$ for $k=1, \dots, K$. The Sylvester model reduces to a squared KS representation for the precision matrix $\mat{\Omega} = (\mat{\Psi}_1 \oplus \cdots \oplus \mat{\Psi}_K)^2$, which is motivated by a stochastic representation of multivariate data with such a precision matrix. SyGlasso is the first KS-based graphical lasso model that admits a stochastic representation (i.e., Sylvester). Thus, our proposed SyGlasso puts the KS representations on similar ground as the KP representations in terms of interpretablility.

\subsection{Notations}
We adopt the notations used by \citet{kolda2009tensor}. A $K$-th order tensor is denoted by boldface Euler script letters, e.g, $\tensor{X} \in \bbR^{m_1 \times \dots \times m_K}$. $\tensor{X}$ reduces to a vector for $K=1$ and to a matrix for $K=2$. The $(i_1,\dots, i_K)$-th element of $\tensor{X}$ is denoted by $\tensor{X}_{i_1,\dots, i_K}$, and we define the vectorization of $\tensor{X}$ to be $\vecto(\tensor{X}) := (\tensor{X}_{1,1,\dots,1},\tensor{X}_{2,1,\dots,1},\dots,\tensor{X}_{m_1,1,\dots,1},\tensor{X}_{1,2,\dots,1},$ $\dots,\tensor{X}_{m_1,m_2,\dots,m_k})^T \in \bbR^p$ with $p=\prod_{k=1}^K m_k$.

There are several tensor algebra concepts that we recall. A fiber is the higher order analogue of the row and column of matrices. It is obtained by fixing all but one of the indices of the tensor, e.g., the mode-$k$ fiber of $\tensor{X}$ is $\tensor{X}_{i_1,\dots,i_{k-1},:,i_{k+1},\dots,i_K}$. Matricization, also known as unfolding, is the process of transforming a tensor into a matrix. The mode-$k$ matricization of a tensor $\tensor{X}$, denoted by $\tensor{X}_{(k)}$, arranges the mode-$k$ fibers to be the columns of the resulting matrix. It is possible to multiply a tensor by a matrix -- the $k$-mode product of a tensor $\tensor{X} \in \bbR^{m_1 \times \dots \times m_K}$ and a matrix $\mat{A} \in \bbR^{J \times m_k}$, denoted as $\tensor{X} \times_k \mat{A}$, is of size $m_1 \times \dots \times m_{k-1} \times J \times m_{k+1} \times \dots m_k$. Its entry is defined as $(\tensor{X} \times_k \mat{A})_{i_1,\dots,i_{k-1},j,i_{k+1},\dots,i_K} := \sum_{i_k=1}^{m_k} \tensor{X}_{i_1,\dots,i_K} A_{j,i_k}$. In addition, for a list of matrices $\{\mat{A}_1,\dots,\mat{A}_K\}$ with $\mat{A}_k \in \bbR^{m_k \times m_k}$, $k=1,\dots,K$, we define $\tensor{X} \times \{\mat{A}_1,\dots,\mat{A}_K\} := \tensor{X} \times_1 \mat{A}_1 \times_2 \dots \times_K \mat{A}_K$. Lastly, we define the $K$-way Kronecker product as $\bigotimes_{k=1}^K \bm\Psi_k = \bm\Psi_1 \otimes \cdots \otimes \bm\Psi_K$, and the equivalent notation for the Kronecker sum as $\bigoplus_{k=1}^K \bm\Psi_k = \bm\Psi_1 \oplus \dots \oplus \bm\Psi_K = \sum_{k=1}^K \mat I_{[m_{k+1:K}]} \otimes \bm\Psi_k \otimes \mat I_{[m_{1:k-1}]}$, where $\mat I_{[m_{k:\ell}]} = \mat I_{m_k} \otimes \dots \otimes \mat I_{m_\ell}$.

\subsection{Outline}
We briefly outline the structure of this chapter. Section~\ref{sec:syglasso-method} introduces the SyGlasso method in details. Section~\ref{sec:syglasso-thm} studies the statistical convergence of the SyGlasso. Section~\ref{sec:syglasso-experiments} provides numerical illustrations of the method using synthetic data. Section~\ref{sec:syglasso-eeg} provides numerical illustrations of the method using real data that arises from Solar flare prediction problems. Section~\ref{sec:syglasso-conclusion} concludes the chapter.

\section{Sylvester Graphical Lasso}\label{sec:syglasso-method}
Let a random tensor $\tensor{X} \in \bbR^{m_1 \times \dots \times m_K}$ be generated by the following representation:
\begin{equation}\label{eqn:tensor_sylvester}
    \tensor{X} \times_1 \mat{\Psi}_1 + \cdots + \tensor{X} \times_K \mat{\Psi}_K = \tensor{T},
\end{equation} 
where $\mat{\Psi}_k \in \bbR^{m_k \times m_k}, k=1,\dots,K$ are sparse symmetric positive definite matrices and $\tensor{T}$ is a random tensor of the same order as $\tensor{X}$. Equation \eqref{eqn:tensor_sylvester} is known as the Sylvester tensor equation. The equation often arises in finite difference discretization of linear partial equations in high dimension \citep{bai2003hermitian} and discretization of separable PDEs \citep{kressner2010krylov,grasedyck2004existence}. When $K=2$ it reduces to the Sylvester matrix equation $\mat{\Psi_1} \mat{X} + \mat{X} \mat{\Psi_2}^T = \mat{T}$ which has wide application in control theory, signal processing and system identification (see, for example \citet{golub1979hessenberg} and references therein).

It is not difficult to verify that the Sylvester representation \eqref{eqn:tensor_sylvester} is equivalent to the following system of linear equations:
\begin{equation}\label{eqn:linear_tensor_sylvester}
  \left( \bigoplus_{k=1}^K \bm\Psi_k \right ) \vecto(\tensor{X}) = \vecto(\tensor{T}),
\end{equation}
If $\tensor{T}$ is a random tensor such that $\vecto(\tensor{T})$ has zero mean and identity covariance, it follows from \eqref{eqn:linear_tensor_sylvester} that any $\tensor{X}$ generated from the stochastic relation \eqref{eqn:tensor_sylvester} satisfies $\bbE \vecto(\tensor{X}) = \mat{0}$ and $\mat{\Sigma} = \mat{\Omega}^{-1} := \bbE \vecto(\tensor{X}) \vecto(\tensor{X})^T = \left ( \bigoplus_{k=1}^K \mat{\Psi}_k \right)^{-2}$. In particular, when $\vecto(\tensor{T}) \sim \mathcal{N}(\mat{0},\mat{I}_m)$, we have that $\vecto(\tensor{X}) \sim \mathcal{N}\left (\mat{0}, \left ( \bigoplus_{k=1}^K \mat{\Psi}_k \right)^{-2} \right)$.

This paper proposes a procedure for estimating $\mat{\Omega}$ with 
$N$ independent copies of the tensor data $\{\tensor{X}^i\}_{i=1}^N$ that are generated from \eqref{eqn:tensor_sylvester}. For the rest of the paper, we assume that the last mode of the data tensor corresponds to the observations mode. For example, when $K=2$, $\tensor{X} \in \bbR^{m_1 \times m_2 \times N}$ is the matrix-variate data with $N$ observations. Our goal is to estimate the $K$ precision matrices $\{ \mat{\Psi_k} \}_{k=1}^K$ each of which describes the conditional independence of $k$-th data dimension. The resulting precision matrix is $\mat{\Omega} = \left ( \bigoplus_{k=1}^K \mat{\Psi}_k \right)^2$. By rewriting \eqref{eqn:linear_tensor_sylvester} element-wise, we first observe that
\begin{equation}
\label{eqn:elementwise_tensor_sylvester}
\begin{aligned}
    & \left( \sum_{k=1}^K (\mat{\Psi}_k)_{i_k,i_k} \right) \tensor{X}_{i_{[1:K]}} \\
    & = -\sum_{k=1}^K \sum_{j_k \neq i_k} (\mat{\Psi}_k)_{i_k,j_k} \tensor{X}_{i_{[1:k]},j_k,i_{[k+1:K]}} + \tensor{T}_{i_{[1:K]}}.
\end{aligned}
\end{equation} 
Note that the left-hand side of \eqref{eqn:elementwise_tensor_sylvester} involves only the summation of the diagonals of the $\mat{\Psi}$'s and the right-hand side is composed of columns of $\bm\Psi$'s that exclude the diagonal terms. Equation \eqref{eqn:elementwise_tensor_sylvester} can be interpreted as an autogregressive model relating the $(i_1,\dots,i_K)$-th element of the data tensor (scaled by the sum of diagonals) to other elements in the fibers of the data tensor. The columns of $\mat{\Psi}'s$ act as regression coefficients. The formulation in \eqref{eqn:elementwise_tensor_sylvester} naturally leads us to consider a pseudolikelihood-based estimation procedure \citep{besag1977efficiency} for estimating $\bm\Omega$. It is known that inference using pseudo-likelihood is consistent and enjoys the same $\sqrt{N}$ convergence rate as the MLE in general \citep{varin2011overview}. This procedure can also be more robust to model misspecification. Specifically, we define the sparse estimate of the underlying precision matrices along each axis of the data as the solution of the following convex optimization problem:
\begin{equation}
\label{eqn:syglasso_objective}
  \begin{aligned}
    & \min_{\substack{\mat{\Psi}_k \in \bbR^{m_k \times m_k}\\k=1,\dots K}} -N \sum_{i_1,\dots,i_K} \log \tensor{W}_{i_{[1:K]}} \\ 
    & \qquad + \frac{1}{2} \sum_{i_1,\dots,i_K} \|(I) + (II)\|_2^2 + \sum_{k=1}^K P_{\lambda_k}(\mat{\Psi}_k).
  \end{aligned} \
\end{equation}
where $P_{\lambda_k}(\cdot)$ is a penalty function indexed by the tuning parameter $\lambda_k$ and 
\begin{align*}
  (I) & = \tensor{W}_{i_{[1:K]}}\tensor{X}_{i_{[1:K]}} \\
  (II) & = \sum_{k=1}^K \sum_{j_k \neq i_k} (\mat{\Psi}_k)_{i_k,j_k} \tensor{X}_{i_{[1:k]},j_k,i_{[k+1:K]}},
\end{align*}
with $\tensor{W}_{i_{[1:K]}} := \sum_{k=1}^K (\mat{\Psi}_k)_{i_k,i_k}$. Here we focus on the $\ell_1$-norm penalty, i.e., $P_{\lambda_k}(\mat{\Psi}_k) = \lambda_k \|\mat{\Psi}_k\|_{1,\text{off}}$.

The optimization problem \eqref{eqn:syglasso_objective} can be put into the following matrix form:
\begin{equation}\label{eqn:syglasso_objective_matrix}
    \begin{aligned}
    \min_{\substack{\mat{\Psi}_k \in \bbR^{m_k \times m_k}\\ k=1,\dots K}} 
    & -\frac{N}{2} \log|(\text{diag}(\mat{\Psi}_1) \oplus \dots \oplus \text{diag}(\mat{\Psi}_K))^2| \\ \nonumber
    + & \frac{N}{2} \tr(\mat{S}(\mat{\Psi}_1 \oplus \dots \oplus \mat{\Psi}_K)^2) + \sum_{k=1}^K P_{\lambda_k}(\mat{\Psi}_k) \nonumber
    \end{aligned}
\end{equation}
where $\diag(\mat{\Psi}_k) \in \bbR^{m_k \times m_k}$ is a matrix of the diagonal entries of $\mat{\Psi}_k$ and $\mat{S} \in \bbR^{m \times m}$ is the sample covariance matrix, i.e., $\mat{S}=\frac{1}{N} \vecto(\tensor{X})^T \vecto(\tensor{X})$. Note that the pseudolikelihood above approximates the $\ell_1$-penalized Gaussian negative log-likelihood in the log-determinant term by including only the Kronecker sum of the diagonal matrices instead of the Kronecker sum of the full matrices. Further discussion of pseudolikelihood- and likelihood-based approaches for (inverse) covariance estimations can be found in \citet{khare2015convex}.

We also note that when $K=1$ the objective \eqref{eqn:syglasso_objective} reduces to the objective of the CONCORD estimator \citep{khare2015convex}, and is similar to those of SPACE \citep{peng2009partial} and Symmetric lasso \citep{friedman2010applications}. Our framework is a generalization of these methods to higher order tensor-valued data, when the Sylvester representation \eqref{eqn:tensor_sylvester} holds.

\begin{remark}
In our formulation $\mat{\Omega}=(\bigoplus_{k=1}^K \mat \Psi_k)^2$ does not uniquely determine $\{\mat{\Psi}_k\}_{k=1}^K$ due to the trace ambiguity: scaled identity factors can be added to/subtracted from the $\mat{\Psi}_k's$ without changing the matrix $\bm\Omega$. To address this non-identifiability, we rewrite the overall precision matrix $\mat{\Omega}$ as
\begin{equation*}
\begin{aligned}
  \mat{\Omega}  = \left( \bigoplus_{k=1}^K \mat{\Psi}_k \right )^2
   = \left( \bigoplus_{k=1}^K \mat{\Psi}_k^{\text{off}} + \bigoplus_{k=1}^K \diag(\mat{\Psi}_k) \right )^2,
\end{aligned}
\end{equation*} where $\mat{\Psi}_k^{\text{off}}=\mat{\Psi}_k-\text{diag}(\mat{\Psi}_k)$, and estimate the diagonal and off-diagonal entries $\mat{\Psi}_k$'s separately. This allows us to reconstruct the overall precision matrix $\mat{\Omega}$ when $\bm\Psi_k^{\text{off}}$ is penalized with an $\ell_1$ penalty.
\end{remark}

\subsection{Estimation of the graphical model}
Let $Q_N(\tensor{W},\{\mat{\Psi}_k^{\text{off}}\}_{k=1}^K)$ denote the objective function defined in \eqref{eqn:syglasso_objective}. Here, $\tensor{W}=\bigoplus_{k=1}^K \text{diag}(\mat{\Psi}_k)$. We adopt a convergent alternating minimization approach \citep{khare2014convergence} that cycles between optimizing $\mat{\Psi}_k$ and $\tensor{W}$ while fixing other parameters.
In particular, for $1 \leq k \leq K$, $1 \leq i_k < j_k \leq m_k$, define
\begin{equation}
\label{eq:sylvester_tensor_update}
\begin{aligned}
    T_{i_kj_k}(\mat{\Psi}_k^{\text{off}}) & = \argmin_{\substack{(\Tilde{\mat{\Psi}}_l)_{m,n}=(\mat{\Psi}_l)_{m,n} \\ \forall (l,m,n) \neq (k,i_k,j_k)}} 
    Q_N(\Tilde{\tensor{W}},\{\Tilde{\mat{\Psi}}_k^{\text{off}}\}_{k=1}^K)\\
    T(\tensor{W}) & = \quad \; \argmin_{\substack{\Tilde{\mat{\Psi}}_k^{\text{off}}=\mat{\Psi}_k^{\text{off}} \\ \forall k}} 
    \quad \; \; Q_N(\Tilde{\tensor{W}},\{\Tilde{\mat{\Psi}}_k^{\text{off}}\}_{k=1}^K).
\end{aligned}
\end{equation}
For each $(k,i_k,j_k)$, $T_{i_kj_k}(\mat{\Psi}_k^{\text{off}})$ 
updates the $(i_k,j_k)$-th entry with the minimizer of $Q_N(\tensor{W},\{\mat{\Psi}\}_{k=1}^K)$ with respect to $(\mat{\Psi}_k)_{i_kj_k}^{\text{off}}$ holding all other variables constant. Similarly, $T(\tensor{W})$ updates $\tensor{W}_{i_{[1:K]}}$ with the solution of $\min Q_N(\tensor{W},\{\mat{\Psi}\}_{k=1}^K)$ with respect to $\tensor{W}_{i_{[1:K]}}$ holding all other variables constant. The closed form updates $T_{i_kj_k}(\mat{\Psi}_k^{\text{off}})$ and $T(\tensor{W})$ are detailed in Appendix~\ref{app:syglasso}.

\begin{algorithm}
\caption{Nodewise SyGlasso}
\label{alg:nodewise_tensor_lasso}
  \SetAlgoLined
  \KwIn{Standardized data $\tensor{X}$, penalty parameter $\lambda_k$}
  \KwOut{$\{\hat{\mat{\Psi}}_k\}_{k=1}^K$, $\hat{\mat{\Omega}}=\left( \bigoplus_{k=1}^K \hat{ \mat{\Psi}}_k \right )^2$}
  Initialize $\{\hat{\mat{\Psi}}_k^{(0)}\}_{k=1}^K$, $\hat{\mat{\Omega}}^{(0)}=\left( \bigoplus_{k=1}^K \hat{ \mat{\Psi}}_k^{(0)} \right )^2$ \\
  \While{not converged}{
  \texttt{\#} \textit{Update off-diagonal elements}\;
    \For{$k \leftarrow 1,\dots,K$}{
      \For{$i_k \leftarrow 1,\dots,m_k-1$}{
        \For{$j_k \leftarrow i_k+1,\dots,m_k$}{
          $(\hat{\mat{\Psi}}_k^{\text{(t+1)}})_{i_k,j_k} \leftarrow (T_{i_k,j_k}(\mat{\Psi}_k^{\text{(t)}}))_{i_k,j_k}$\;
          \hfill from \eqref{eqn:update_offdiag} in Appendix~\ref{app:syglasso}
        }
      }
    } 
    \texttt{\#} \textit{Update diagonal elements}\;
    $\hat{\tensor{W}}^{\text{(t+1)}} \leftarrow T(\tensor{W}^{\text{(t)}})$ from \eqref{eqn:update_diag} in Appendix~\ref{app:syglasso}
  } 
\end{algorithm}

\section{Large Sample Properties}\label{sec:syglasso-thm}
We show that under suitable conditions, the Sylvester graphical lasso (SyGlasso) estimator (Algorithm \ref{alg:nodewise_tensor_lasso}) achieves both model selection consistency and estimation consistency. As in other studies \citep{khare2015convex, peng2009partial}\footnote{When $K=1$ it is possible to relax this assumption to require only accurate estimates of the diagonals, see \citet{khare2015convex, peng2009partial} for details.}, for the convergence analysis we make standard assumptions that the diagonal of $\mat{\Omega}$ is known. We analyze the theoretical properties of the SyGlasso under the assumption that $\tensor{W}$ is given. In practice, we can estimate $\tensor{W}$ using Algorithm \ref{alg:nodewise_tensor_lasso}, and if the diagonals of each individual $\mat{\Psi}_k$ are desired, we can incorporate any available prior knowledge of the variation along each data dimension.

We estimate $\{\mat{\Psi}_k^{\text{off}}\}_{k=1}^K$ by solving the following $\ell_1$ penalized problem:
\begin{equation}
    \min_{\bm{\beta}} L_N \Big(\tensor{W},\bm{\beta},\tensor{X}\Big) + \sum_{k=1}^K \lambda_k \|\mat{\Psi}_k\|_{1,\text{off}},
\end{equation} 
where $L_N \Big(\tensor{W},\bm{\beta},\tensor{X}\Big):=\frac{1}{N}\sum_{s=1}^{N} L\Big(\tensor{W},\bm{\beta},\tensor{X}^s \Big)$, with
\begin{equation}
\begin{aligned}
    L\Big(\tensor{W},\bm{\beta},\tensor{X}^s \Big) & = - N \sum_{i_{[1:K]}} \log \tensor{W}_{i_{[1:K]}} \\ 
    & \qquad \qquad + \frac{1}{2} \sum_{i_1,\dots,i_K} ((I) + (II))^2. \\ 
\end{aligned}
\vspace{-10pt}
\end{equation} 
where
\begin{align*}
  & (I)  = \tensor{W}_{i_{[1:K]}}\tensor{X}_{i_{[1:K]}} \\
  & (II)  = \sum_{k=1}^K \sum_{j_k \neq i_k} (\mat{\Psi}_k)_{i_k,j_k} \tensor{X}_{i_{[1:k-1]},j_k,i_{[k+1:K]}}\\
  & \bm{\beta}  = ((\mat{\Psi}_1)_{1,2},(\mat{\Psi}_1)_{1,3},\dots,(\mat{\Psi}_1)_{1,m_1},\dots,(\mat{\Psi}_k)_{m_k-1,m_k})^T
\end{align*} and $\bm \beta$ denotes the off-diagonal entries of all $\mat{\Psi}_k's$. 

We first state the regularity conditions needed for establishing convergence of the SyGlasso estimator. Let $\mathcal{A}_{k}:=\{(i,j):(\mat{\Psi}_k)_{i,j} \neq 0, i \neq j\}$ and $q_{k}:=|\mathcal{A}_{k}|$ for $k=1,\dots,K$ be the true edge set and the number of edges, respectively. Let $\mathcal{A} = \cup_{k=1}^K \mathcal{A}_{k}$. We use $\bar{\bm\beta}, \bar{\bm\Omega}, \bar{\tensor{W}}$ to emphasize that they are the true values of the corresponding parameters.

\noindent \textbf{(A1 - Subgaussianity)} The data $\{\tensor{X}^s\}_{s=1}^N$ are i.i.d subgaussian random tensors, that is, $\vecto(\tensor{X}^s) \sim \mat{x}$, where $\mat{x}$ is a subgaussian random vector in $\mathbb{R}^p$, i.e., there exist a constant $c>0$, such that for every $\mat{a} \in \mathbb{R}^p$, $\mathbb{E}e^{\mat{a}^T \mat{x}} \leq e^{c\mat{a}^T \bar{\mat{\Sigma}} \mat{a}}$, and there exist $\rho_j > 0$ such that $\mathbb{E}e^{tx_j^2} \leq K$ whenever $|t| < \rho_j$, for $1 \leq j \leq p$.

\noindent \textbf{(A2 - Bounded eigenvalues)} There exist constants $0 < \Lambda_{\min} \leq \Lambda_{\max} < \infty$, such that the minimum and maximum eigenvalues of $\mat{\Omega}$ are bounded with $\lambda_{\min}(\bar{\mat{\Omega}}) = (\sum_{k=1}^K \lambda_{\max}(\mat{\Psi}_k))^{-2} \geq \Lambda_{\min}$ and $\lambda_{\max}(\bar{\mat{\Omega}}) = (\sum_{k=1}^K \lambda_{\min}(\mat{\Psi}_k))^{-2} \leq \Lambda_{\max}$.

\noindent \textbf{(A3 - Incoherence condition)} There exists a constant $\delta < 1$ such that for $k=1,\dots,K$ and all $(i,j) \in \mathcal{A}_{k}$
\begin{equation*}
    |\bar{L}_{ij,\mathcal{A}_{k}}^{''}(\bar{\tensor{W}},\bar{\bm{\beta}})[\bar{L}_{\mathcal{A}_{k},\mathcal{A}_{k}}^{''}(\bar{\tensor{W}},\bar{\bm{\beta}})]^{-1} \text{sign}(\bar{\bm{\beta}}_{\mathcal{A}_{k}})| \leq \delta,
\end{equation*} where for each $k$ and $1 \leq i < j \leq m_k$, $1 \leq k < l \leq m_k$,
\begin{equation*}
    \bar{L}_{ij,kl}^{''}(\bar{\tensor{W}},\bar{\bm{\beta}}) := E_{\bar{\tensor{W}},\bar{\bm{\beta}}} \Bigg(\frac{\partial^2 L(\tensor{W},\bm{\beta},\tensor{X})}{\partial(\mat{\Psi}_k)_{i,j} \partial(\mat{\Psi}_k)_{k,l}}|_{\tensor{W}=\bar{\tensor{W}},\bm{\beta}=\bar{\bm{\beta}}} \Bigg).
\end{equation*}

Note that conditions analogous to (A3) have been used in \citet{meinshausen2006high} and \citet{peng2009partial} to establish high-dimensional model selection consistency of the nodewise graphical lasso in the case of $K=1$. \citet{zhao2006model} show that such a condition is almost necessary and sufficient for model selection consistency in lasso regression, and they provide some examples when this condition is satisfied.

Inspired by \citet{meinshausen2006high} and \citet{peng2009partial} we prove the following properties:
\begin{enumerate} 
    \vspace{-8pt}
    \item Theorem~\ref{thm:restricted_problem} establishes estimation consistency and sign consistency for the nodewise SyGlasso restricted to the true support, i.e., $\bm{\beta}_{\mathcal{A}^c}=0$,
    \item Theorem~\ref{thm:edge_selection} shows that no wrong edge is selected with probability tending to one,
    \item Theorem~\ref{thm:consistency} establishes consistency result of the nodewise SyGlasso.
\end{enumerate}

\begin{theorem}\label{thm:restricted_problem}
  Suppose that conditions (A1-A2) are satisfied. Suppose further that $\lambda_{N,k}=O(\sqrt{\frac{m_k\log p}{N}})$ for all $k$ and $N > O(\max_k q_k m_k \log p)$ as $N \rightarrow \infty$. Then there exists a constant $C(\bar{\bm{\beta}})$, such that for any $\eta>0$, the following hold with probability at least $1-O(\exp(-\eta \log p))$:
  \begin{itemize}
      \item There exists a global minimizer $\hat{\bm{\beta}}_{\mathcal{A}}$ of the restricted SyGlasso problem:
      \begin{equation}\label{eqn:restricted_problem}
          \min_{\bm{\beta}:\bm{\beta}_{\mathcal{A}^c}=0} L_N(\bar{\tensor{W}},\bm{\beta},\tensor{X}) + \sum_{k=1}^K \lambda_k \|\mat{\Psi}_k\|_{1,\text{off}}.
      \end{equation}
      \item (Estimation consistency) Any solution $\hat{\bm{\beta}}_{\mathcal{A}}$ of \eqref{eqn:restricted_problem} satisfies:
      \begin{equation*}
          \|\hat{\bm{\beta}}_{\mathcal{A}} - \bm{\beta}_{\mathcal{A}}\|_2 \leq C(\bar{\bm{\beta}})\sqrt{K}\max_{k}\sqrt{q_{k}}\lambda_{N,k}.
      \end{equation*}
      \item (Sign consistency) If further the minimal signal strength: $\min_{(i,j) \in \mathcal{A}_{k}}|(\mat{\Psi}_k)_{i,j}| \geq 2C(\bar{\bm{\beta}})\sqrt{K}\max_{k}\sqrt{q_{k}}\lambda_{N,k}$ for each $k$, then sign($\hat{\bm{\beta}}_{\mathcal{A}_{k}}$)=sign($\bar{\bm{\beta}}_{\mathcal{A}_{k}}$). 
  \end{itemize}
  
\end{theorem}

\begin{theorem}\label{thm:edge_selection}
Suppose that the conditions of Theorem~\ref{thm:restricted_problem} and (A3) are satisfied. Suppose further that $p=O(N^{\kappa})$ for some $\kappa \geq 0$. Then for $\eta>0$, for $N$ sufficiently large, the solution of \eqref{eqn:restricted_problem} satisfies:
  \begin{align*}
      P_{\bar{\tensor{W}},\bar{\bm{\beta}}} & \Big(\max_{(i,j) \in \mathcal{A}_{k}^c} |L_{N,ij}^{\prime}(\bar{\tensor{W}},\hat{\bm{\beta}}_{\mathcal{A}_{k}},\tensor{X})|<\lambda_{N,k}\Big) \\
      & \geq 1 - O(\exp(-\eta \log p))
  \end{align*} for each $k$, where $L_{N,ij}^{\prime} := \partial L_N / \partial (\mat{\Psi}_k)_{ij}$.
\end{theorem}

\begin{theorem}\label{thm:consistency}
  Assume the conditions of Theorem~\ref{thm:edge_selection}. Then there exists a constant $C(\bar{\bm{\beta}})>0$ such that for any $\eta>0$ the following events hold with probability at least $1 - O(\exp(-\eta \log p))$:
  \begin{itemize}
      \item There exists a global minimizer $\hat{\bm{\beta}}$ to problem \eqref{eqn:syglasso_objective}.
      \item (Estimation consistency) Any minimizer $\hat{\bm{\beta}}$ of \eqref{eqn:syglasso_objective} satisfies:
      \begin{equation*}
          \|\hat{\bm{\beta}} - \bm{\beta}\|_2 \leq C(\bar{\bm{\beta}})\sqrt{K}\max_{k}\sqrt{q_{k}}\lambda_{N,k}.          
      \end{equation*}
      \item (Sign consistency) If further the minimal signal strength: $\min_{(i,j) \in \mathcal{A}_{k}}|(\mat{\Psi}_k)_{i,j}| \geq 2C(\bar{\bm{\beta}})\max_{k}\sqrt{q_{k}}\lambda_{N,k}$ for each $k$, then sign($\hat{\bm{\beta}}$)=sign($\bar{\bm{\beta}}$).
  \end{itemize}
\end{theorem}
Proofs of the above theorems are given in Appendix~\ref{app:syglasso}.

\section{Numerical Illustrations}\label{sec:syglasso-experiments}
We evaluate the proposed SyGlasso estimator in terms of optimization and graph recovery accuracy. We also compare the graph recovery performance with other models recently proposed for matrix- and tensor-variate precision matrices. We first illustrate the differences among these models by investigating the sparsity pattern of $\bm\Omega$ with $K=3$ and $m_k = 4, \forall k$. For simplicity, we generate $\bm\Psi_k$ for $k=1, 2, 3$ as identical $4 \times 4$ precision matrices that follow a one dimensional autoregressive-1 (AR1) process. We recall the KP and KS models:

\begin{figure}[h!] \centering
\begin{subfigure}[t]{0.45\linewidth} \centering
\includegraphics[width=\linewidth]{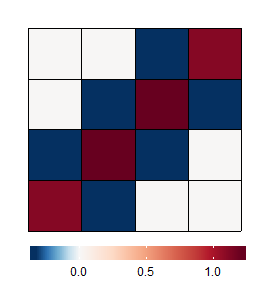}
\caption{$\bm\Psi_k$}
\end{subfigure}
\hspace{7pt}
\begin{subfigure}[t]{0.45\linewidth} \centering
\includegraphics[width=\linewidth]{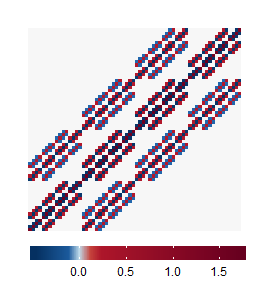}
\caption{KP $\mat{\Omega}$}
\end{subfigure}

\begin{subfigure}[t]{0.45\linewidth} \centering
\includegraphics[width=\linewidth]{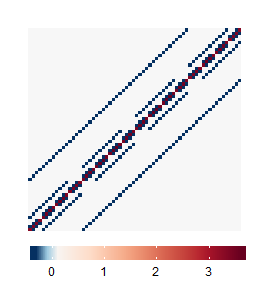}
\caption{KS $\mat{\Omega}$}
\end{subfigure}
\hspace{7pt}
\begin{subfigure}[t]{0.45\linewidth} \centering
\includegraphics[width=\linewidth]{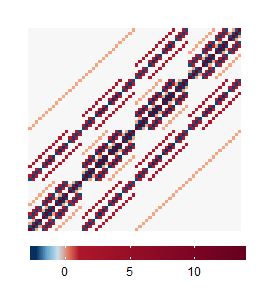}
\caption{SyGlasso $\mat{\Omega}$}
\end{subfigure}
\caption{Comparison of SyGlasso to Kronecker sum (KS) and product (KP) structures. All models are composed of the same components $\mat{\Psi}_k$ for $k=1, 2, 3$ generated as an AR(1) model with $m_k=4$ as shown in (a). The AR(1) components are brought together to create the final $64 \times 64$ precision matrix $\mat{\Omega}$ following (b) the KP structure with $\mat{\Omega} = \bigotimes_{k=1}^3 \mat{\Psi}_k$, (c) the KS structure with $\mat{\Omega}= \bigoplus_{k=1}^3 \mat{\Psi}_k$, and (d) the proposed Sylvester model with $\mat{\Omega}=\left(\bigoplus_{k=1}^3 \mat{\Psi}_k\right)^2$. The KP does not capture nested structures as it simply replicates the individual component with different multiplicative scales. The SyGlasso model admits a precision matrix structure that strikes a balance between KS and KP.}
\label{fig:AR_comparison}
\end{figure}

\textbf{Kronecker Product (KP):} 
The KP model restricts the precision matrix and the covariance matrix to be separable across the $K$ data dimensions and suffers from a multiplicative explosion in the number of edges. As they are separable models and the constructed $\bm\Omega$ corresponds to the direct product of the $K$ graphs, KP is unable to capture more complex nested patterns captured by the KS and SyGlasso models as shown in Figure \ref{fig:AR_comparison} (c) and (d). 

\textbf{Kronecker Sum (KS):} 
The covariance matrix under the KS precision matrix assumption is nonseparable across $K$ data dimensions, and the KS-structured models can be motivated from a maximum entropy point of view. Contrary to the KP structure, the number of edges in the KS structure grows as the sum of the edges of the individual graphs (as a result of Cartesian product of the associated graphs), which leads to a more controllable number of edges in $\bm\Omega$. 

We compare these methods under different model assumptions to explore the flexibility of the proposed SyGlasso model under model mismatch. To empirically assess the efficiency of the proposed model, we generate tensor-valued data based on three different precision matrices. The $\mat{\Psi}_k$'s are generated from one of 1) AR1($\rho$), 2) Star-Block (SB), or 3) Erdos-Renyi (ER) random graph models described in Appendix \ref{app:syglasso}. 

\begin{figure}[htb] \centering
\begin{subfigure}[t]{0.45\linewidth} \centering
\includegraphics[width=\linewidth]{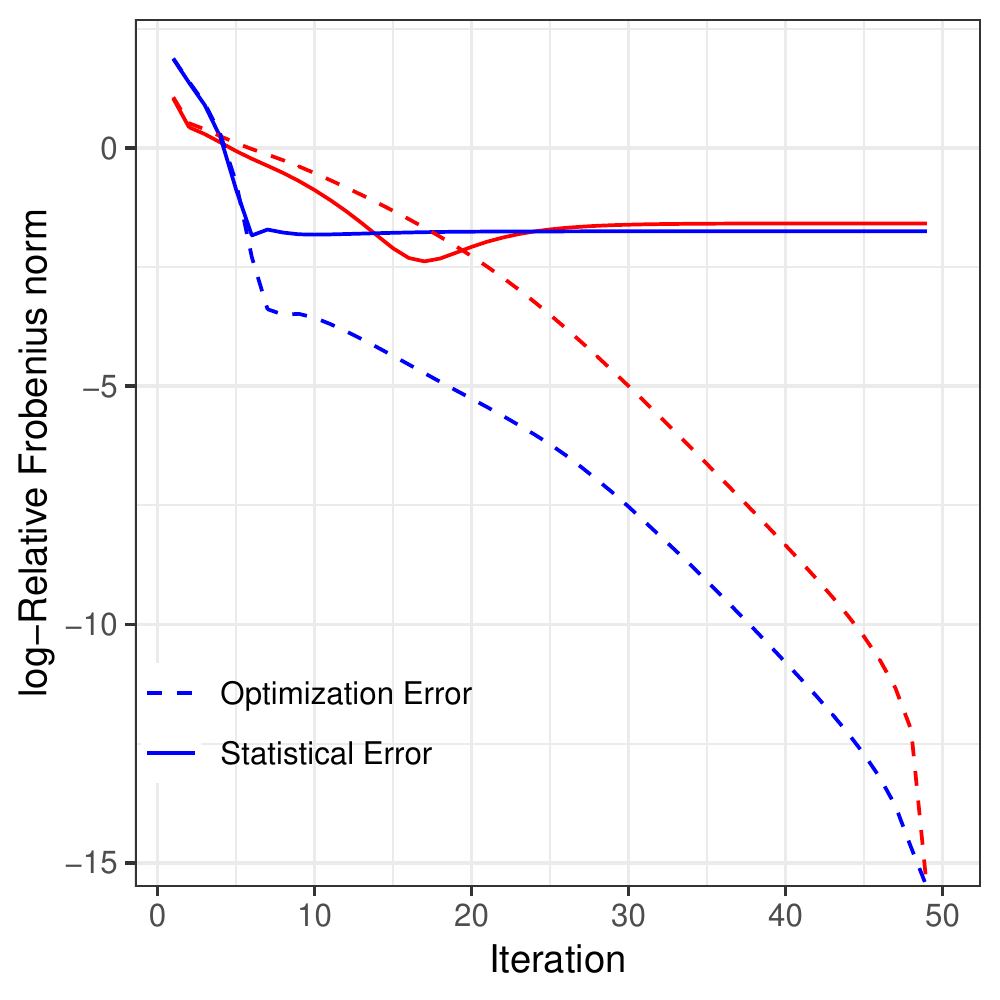}
\caption{SB and AR}
\end{subfigure}
~
\begin{subfigure}[t]{0.45\linewidth} \centering
\includegraphics[width=\linewidth]{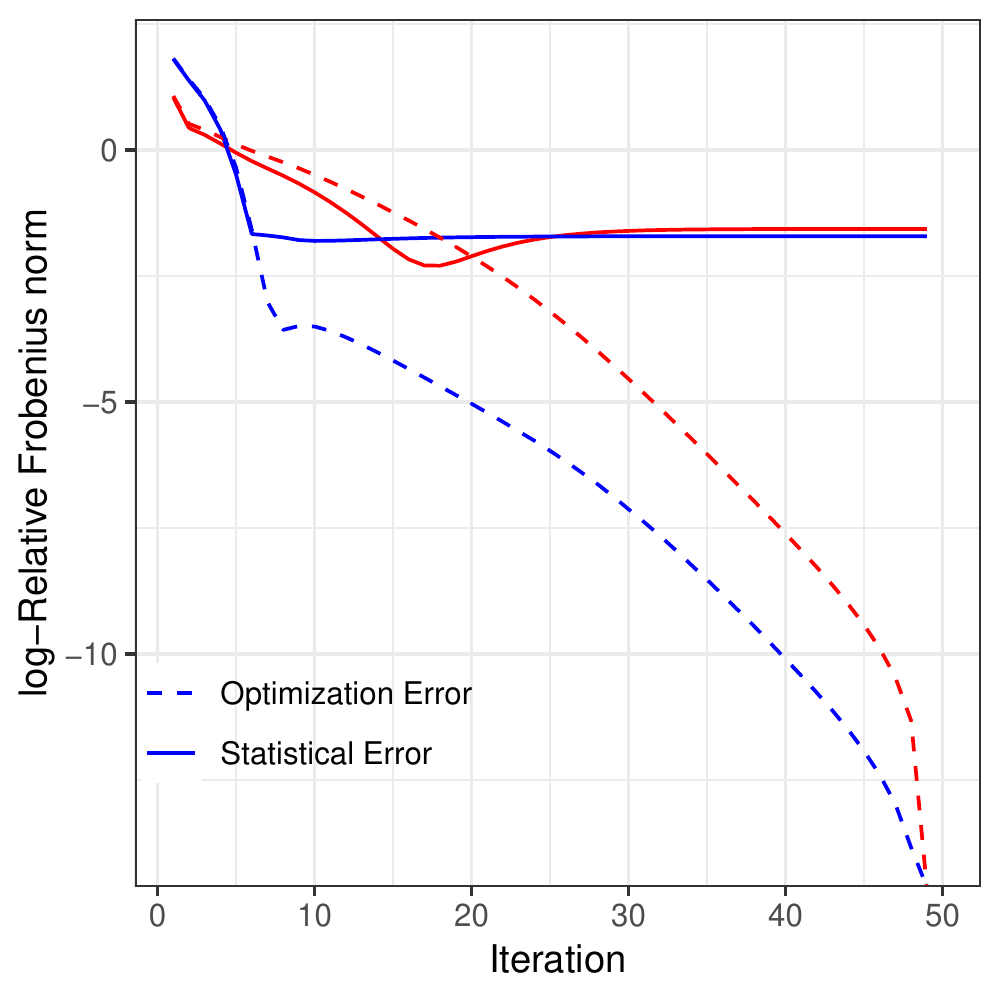}
\caption{SB and ER}
\end{subfigure}
\caption{Performance of the SyGlasso estimator against the number of iterations under different topologies of $\mat{\Psi}_k$'s. The solid line shows the statistical error $\log(\|\hat{\mat{\Psi}}_k^{(t)} - \mat{\Psi}_k\|_F \text{\textbackslash} \|\mat{\Psi}_k\|_F)$, and the dotted line shows the optimization error $\log(\|\hat{\mat{\Psi}}_k^{(t)} - \hat{\mat{\Psi}}_k\|_F \text{\textbackslash} \|\hat{\mat{\Psi}}_k\|_F)$, where $\hat{\mat{\Psi}}_k$ is the final SyGlasso estimator. The performances of $\mat{\Psi}_1$ and $\mat{\Psi}_2$ are represented by red and blue lines, respectively.}


\label{fig:sim_fnorm}
\end{figure}

We test SyGlasso with $K=2$ under: 1) SB with $\rho=0.6$ and sub-blocks of size $16$ and AR1($\rho=0.6$); 2) SB with $\rho=0.6$ and sub-blocks of size $16$ and ER with $256$ randomly selected edges. In both scenarios we set $m_1=128$ and $m_2=256$ with $10$ samples. Figure \ref{fig:sim_fnorm} shows the iterative optimization performance of Algorithm \ref{alg:nodewise_tensor_lasso}. All the plots for the various scenarios exhibit iterative optimization approximation errors that quickly converge to values below the statistical errors. Note that these plots also suggest that our algorithm can attain linear convergence rates. We also test our method for model selection accuracy over a range of penalty parameters (we set $\lambda_k=\lambda,\forall k$). Figure \ref{fig:sim_roc} displays the sum of false positive rate and false negative rate (FPR+FNR), it suggests that the nodewise SyGlasso estimator is able to fully recover the graph structures for each mode of the tensor data. 

\begin{figure}[htb] \centering
\begin{subfigure}[t]{0.45\linewidth} \centering
\includegraphics[width=\linewidth]{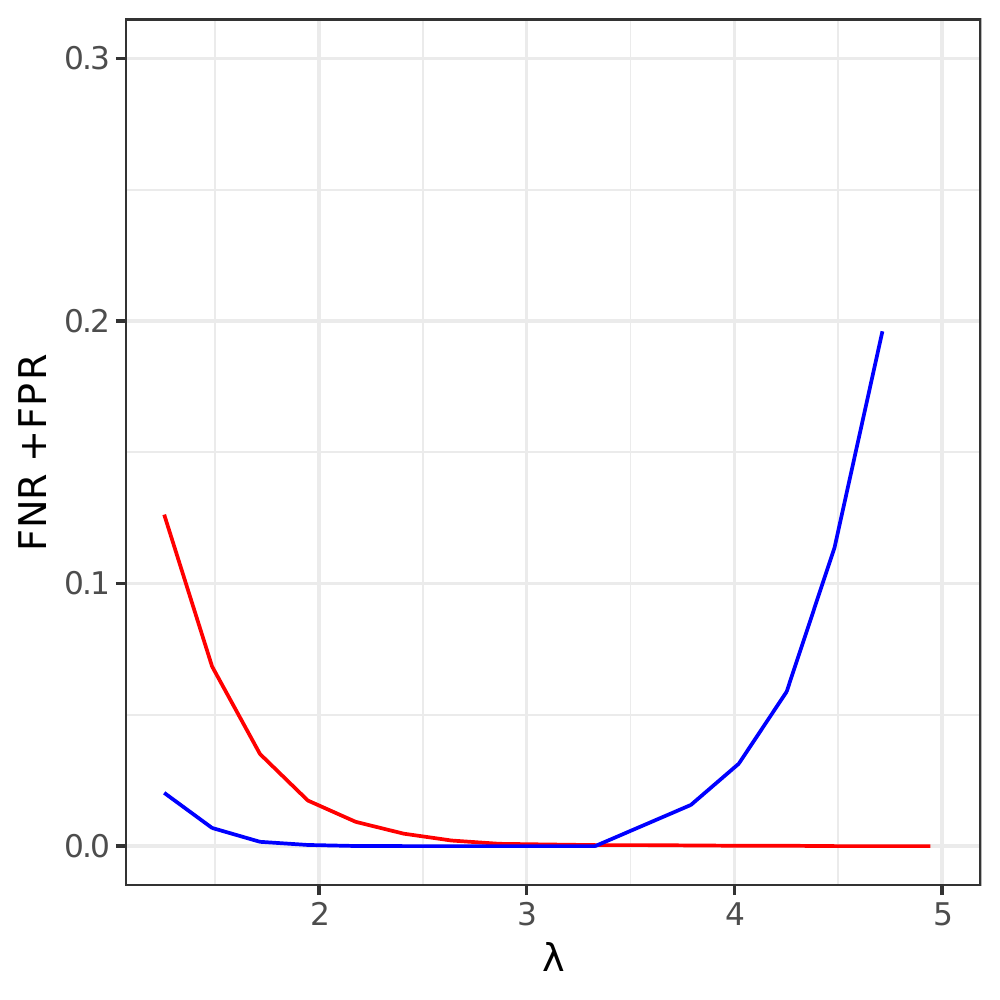}
\caption{SB and AR}
\end{subfigure}
~
\begin{subfigure}[t]{0.45\linewidth} \centering
\includegraphics[width=\linewidth]{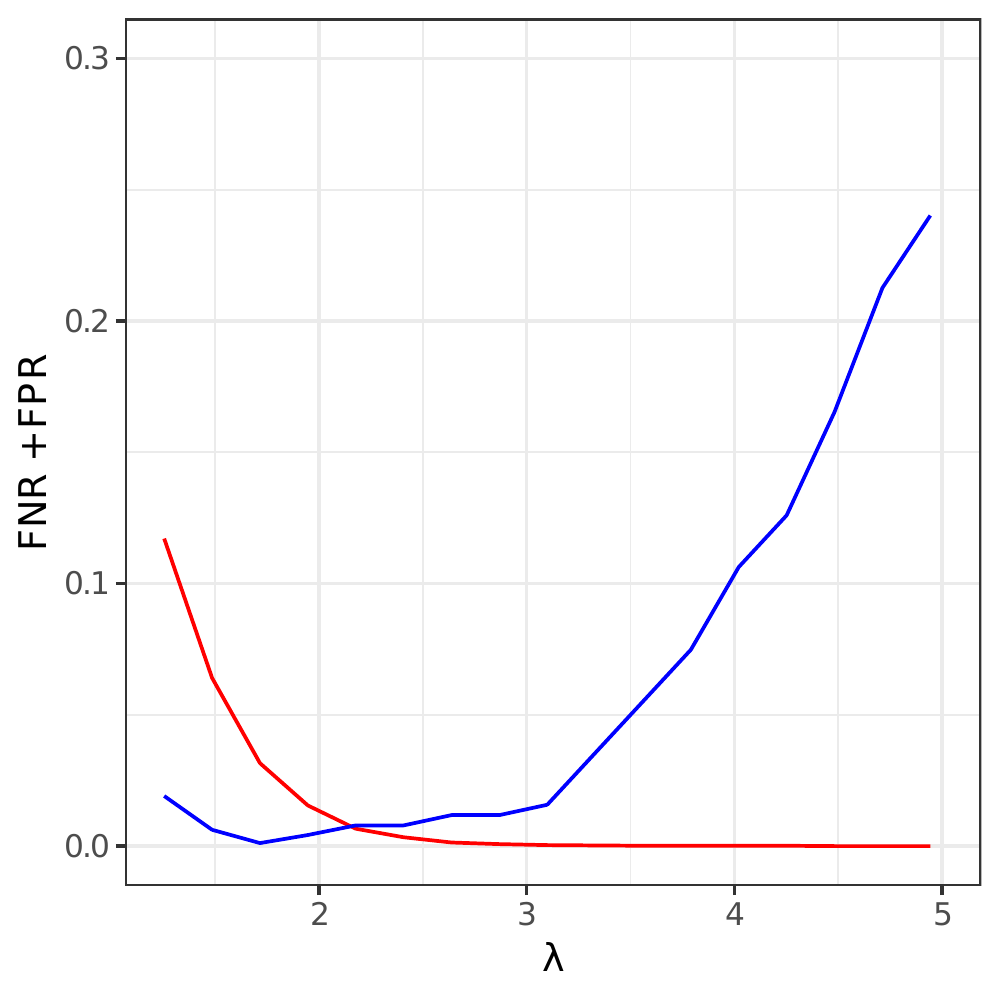}
\caption{SB and ER}
\end{subfigure}
\caption{The performance of model selection measured by FPR + FNR. The performances of $\mat{\Psi}_1$ and $\mat{\Psi}_2$ are represented by red and blue lines, respectively. With an appropriate choice of $\lambda$, the SyGlasso recovers the dependency structures encoded in each $\mat{\Psi}_k$.}
\label{fig:sim_roc}
\end{figure}

We compare the proposed SyGlasso to the TeraLasso estimator \citep{greenewald2019tensor}, and to the Tlasso estimator proposed by \citet{lyu2019tensor} for KP, on data generated using precision matrices $(\mat{\Psi}_1 \oplus \mat{\Psi}_2 \oplus \mat{\Psi}_3)^2$, $\mat{\Psi}_1 \oplus \mat{\Psi}_2 \oplus \mat{\Psi}_3$, and $\mat{\Psi}_1 \otimes \mat{\Psi}_2 \otimes \mat{\Psi}_3$, where $\mat{\Psi}$'s are each $16 \times 16$ ER graphs with $16$ nonzero edges. We use the Matthews correlation coefficient (MCC) to compare model selection performances. The MCC is defined as \citep{matthews1975comparison}
\begin{equation*}
    \text{MCC} = \frac{\text{TP}\times\text{TN}-\text{FP}\times\text{FN}}{\sqrt{(\text{}TP+\text{FP})(\text{TP}+\text{FN})(\text{TN}+\text{FP})(\text{TN}+\text{FN})}},
\end{equation*} where we follow \citet{greenewald2019tensor} to consider each nonzero off-diagonal element of $\mat{\Psi}_k$ as a single edge. 

The results shown in Figure \ref{fig:modelmismatch} indicate that all three estimators perform well when $N=5$, even under model misspecification. In the single sample scenario, the graph recovery performance of each estimator does well under each true underlying data generating process. Note that for data generated using KP, the SyGlasso performs surprisingly well and is comparable to Tlasso. These results seem to indicate that SyGlasso is very robust under model misspecification. The superior performance of SyGlasso under KP model, even with one sample, suggests again that SyGlasso structure has a flavor of both KS and KP structures, as seen in Figure \ref{fig:AR_comparison}. This follows from the observation that  $(\mat{\Psi}_1 \oplus \mat{\Psi}_2)^2 = \mat I_{m_1} \otimes \mat{\Psi}_1^2 + \mat{\Psi}_2^2 \otimes \mat I_{m_2} + 2\mat{\Psi}_1 \otimes \mat{\Psi}_2 = \mat{\Psi}_1^2 \oplus \mat{\Psi}_2^2 + 2\mat{\Psi}_1 \otimes \mat{\Psi}_2$.

\begin{figure}[!htb] \centering
\begin{tabular}{@{}cc@{}}
\qquad \qquad  $N = 1$ & \quad  $N = 5$\\
\rotatebox{90}{\qquad \qquad SyGlasso} \qquad 
\includegraphics[width=0.3\linewidth]{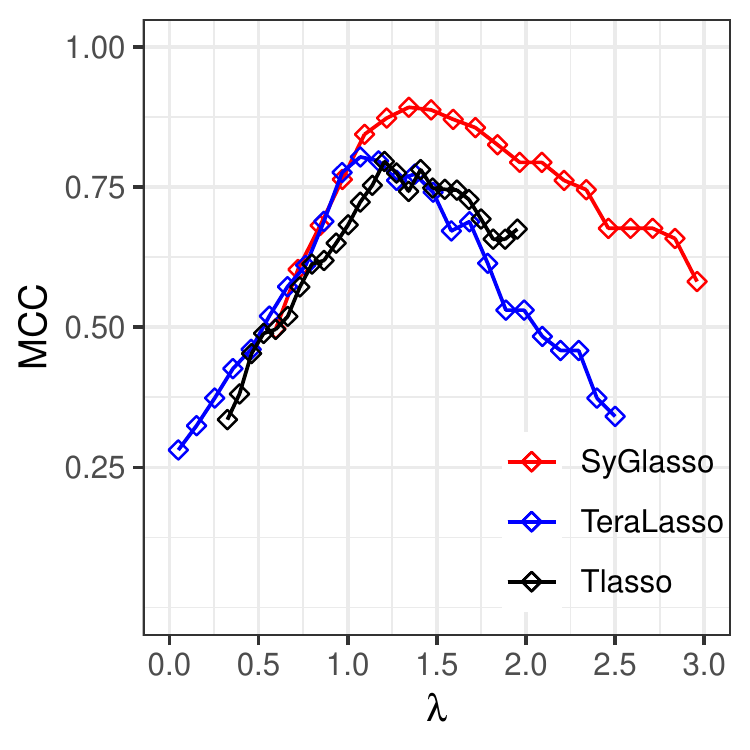}
&
\includegraphics[width=0.3\linewidth]{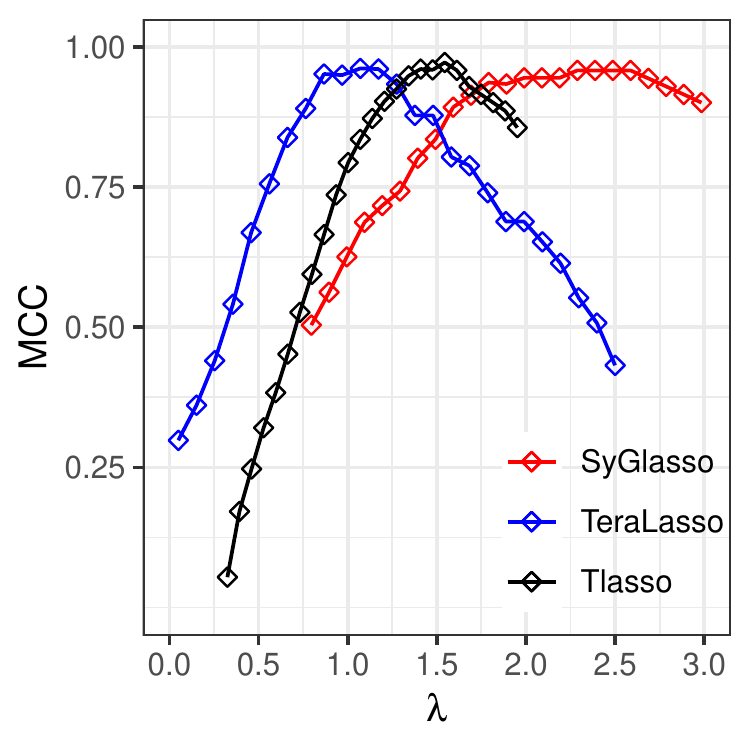}\\
\rotatebox{90}{\qquad \qquad KS} \qquad  
\includegraphics[width=0.3\linewidth]{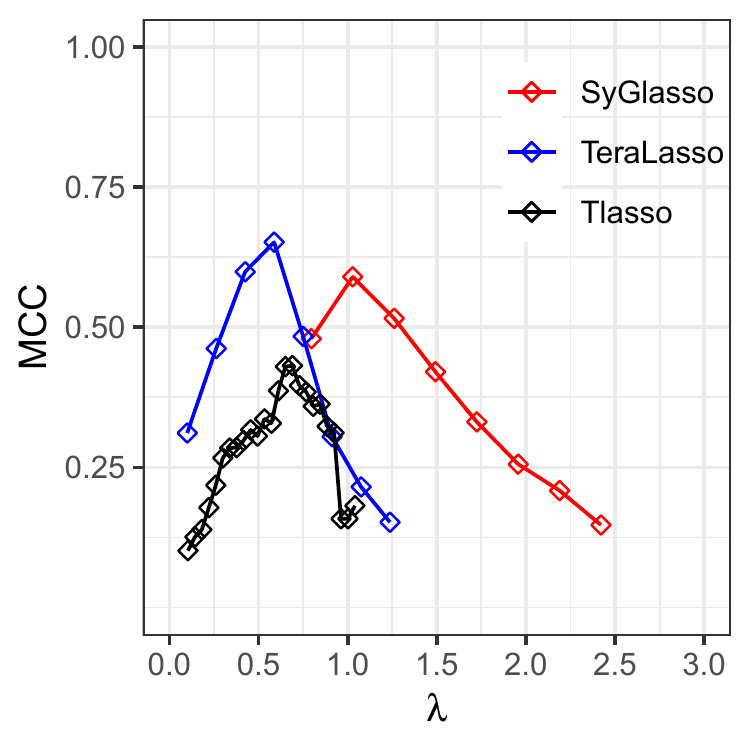}
& 
\includegraphics[width=0.3\linewidth]{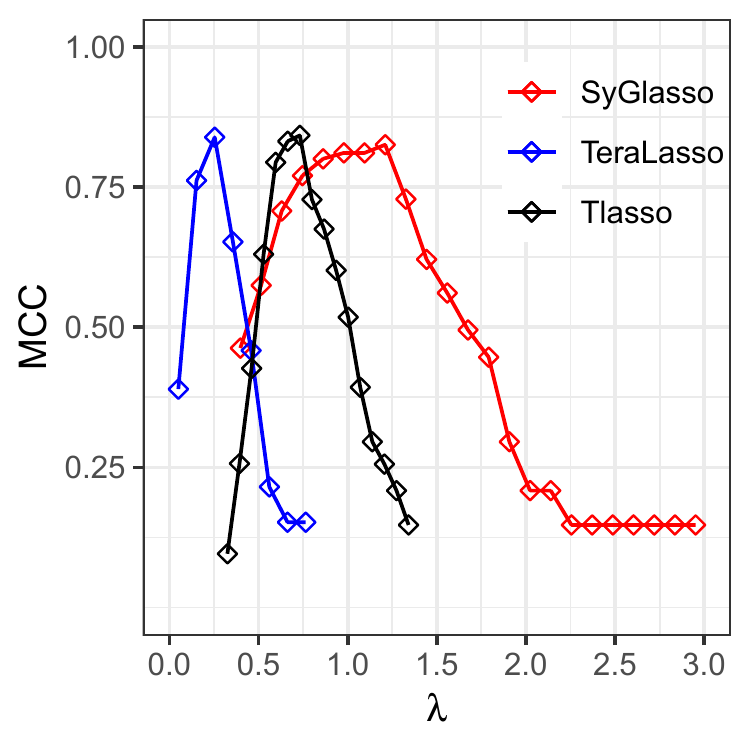}\\
\rotatebox{90}{\qquad \qquad KP} \qquad 
\includegraphics[width=0.3\linewidth]{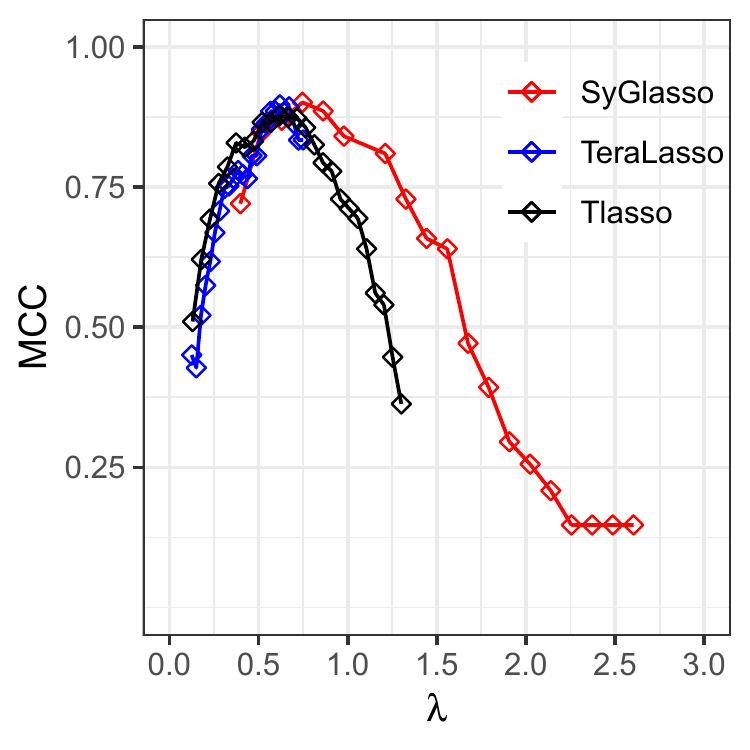}
& 
\includegraphics[width=0.3\linewidth]{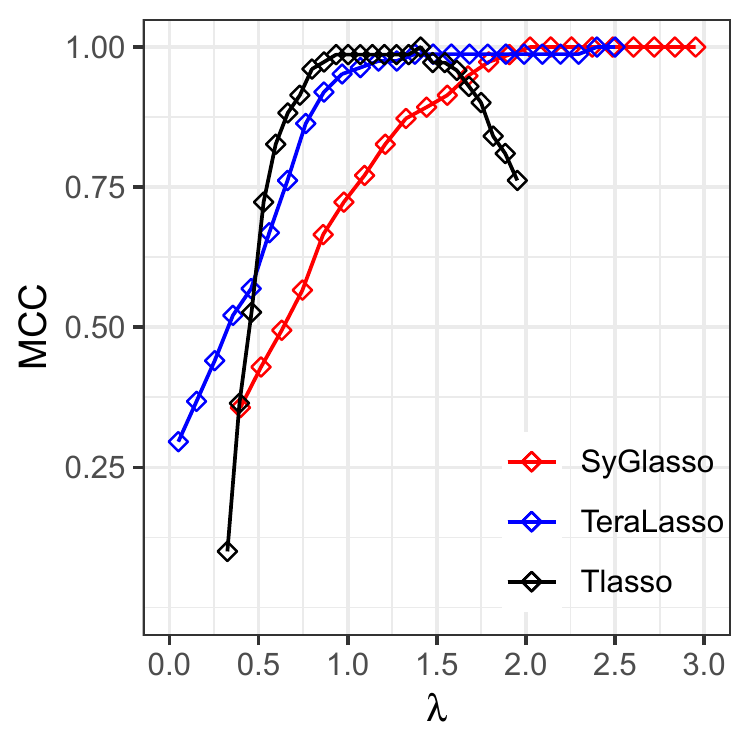}
\end{tabular}
\caption{Performance of SyGlasso, TeraLasso (KS), and Tlasso (KP) measured by MCC under model misspecification. MCC of $1$ represents a perfect recovery of the sparsity pattern in $\mat{\Omega}$, and MCC of $0$ corresponds to random guess. From top to bottom, the synthetic data were generated with the precision matrices from SyGlasso, KS, and KP models. The left column shows the results for a single sample ($N=1$), and the right column shows the results for $N=5$ observations. Note that the SyGlasso has better performance for a single sample (left column) when data is generated from the matched Kronecker model and as good performance for the mismatched Kronecker models.}
\label{fig:modelmismatch}
\end{figure}

\section{EEG Analysis}\label{sec:syglasso-eeg}
We revisit the alcoholism study conducted by \citet{zhang1995event} to explore multiway relationships in EEG measurements of alcoholic and control subjects. Each of 77 alcoholic subjects and 45 control subjects was visually stimulated by either a single picture or a pair of pictures on a computer monitor. Following the analyses of \citet{zhu2016bayesian} and \citet{qiao2019functional}, we focus on the $\alpha$ frequency band (8 - 13 Hz) that is known to be responsible for the inhibitory control of the subjects (see \citet{knyazev2007motivation} for more details). The EEG signals were bandpass filtered with the cosine-tapered window to extract $\alpha$-band signals. Previous Gaussian graphical models applied to such $\alpha$ frequency band filtered EEG data could only estimate the connectivity of the electrodes as they cannot be generalized to tensor valued data. The SyGlasso reveals similar dependency structure as reported in \citet{zhu2016bayesian} and \citet{qiao2019functional} while recovering the chain structure of the temporal relationship.

Specifically, after the band-pass filter was applied, we work with the tensor data $\tensor{X}_{alcoholic}, \tensor{X}_{control} \in \mathbb{R}^{m_{nodes} \times m_{time} \times m_{trial}}$ corresponding to an alcoholic subject and a control subject. We simultaneously estimate $\mat{\Psi}_{node} \in \mathbb{R}^{m_{node} \times m_{node}}$ that encodes the dependency structure among electrodes and $\mat{\Psi}_{time} \in \mathbb{R}^{m_{time}\times m_{time}}$ that shows the relationship among time points that span the duration of each trial. Previous studies consider the average of all trials, for each subject and use the number of subjects as observations to estimate the dependency structures among $64$ electrodes. Instead, we look at one subject at a time and consider different experimental trials as observations. Our analysis focuses on recovering the precision matrices of electrodes and time points, but it can be easily generalized to estimate the dependency structure among trials as well.
\begin{figure}[!thb] \centering
\begin{subfigure}[t]{0.5\linewidth} \centering
\includegraphics[width=\linewidth]{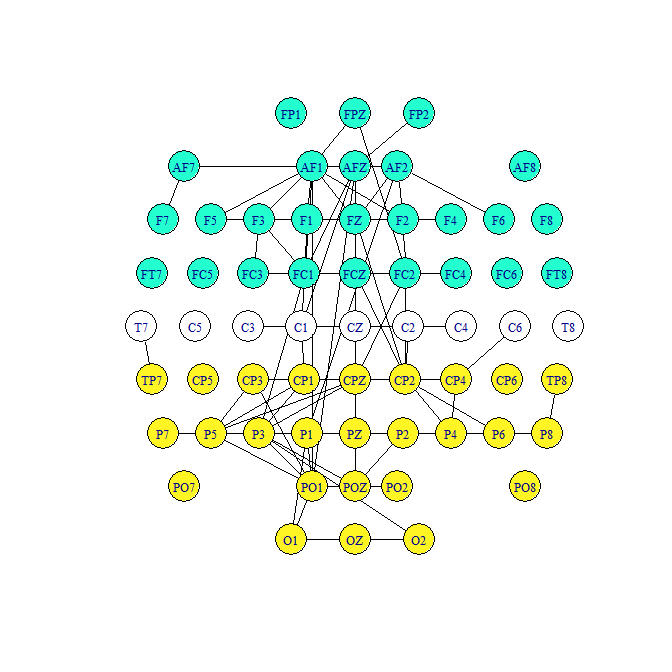}
\caption{Alcoholic subject}
\end{subfigure}
\hspace{-20pt}
\begin{subfigure}[t]{0.5\linewidth} \centering
\includegraphics[width=\linewidth]{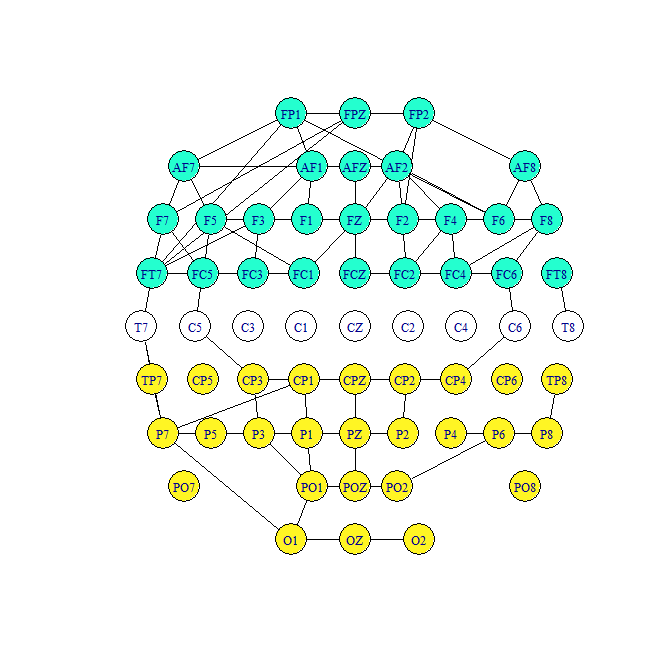}
\caption{Control subject}
\end{subfigure}
\caption{Estimated brain connectivity results from SyGlasso for (a) the alcoholic subject and (b) the control subject. The blue nodes correspond to the frontal region, and the yellow nodes correspond to the parietal and occipital regions. The alcoholic subject has asymmetric brain connections in the frontal region compared to the control subject.}
\label{fig:alcoholism_node_result}
\end{figure}

Figure \ref{fig:alcoholism_node_result} shows the result of the SyGlasso estimated network of electrodes. For comparison, both graphs were thresholded to match 5\% sparsity level. Similar to the findings of \citet{qiao2019functional}, our estimated graph $\bm\Psi_{node}$ for the alcoholic group shows the asymmetry between the left and the right side of the brain compared to the more balanced control group. Our finding is also consistent with the result in \citet{hayden2006patterns} and \citet{zhu2016bayesian} that showed frontal asymmetry of the alcoholic subjects. 

\begin{figure}[!htb] \centering
\begin{subfigure}[t]{0.45\linewidth} \centering
\includegraphics[width=\linewidth]{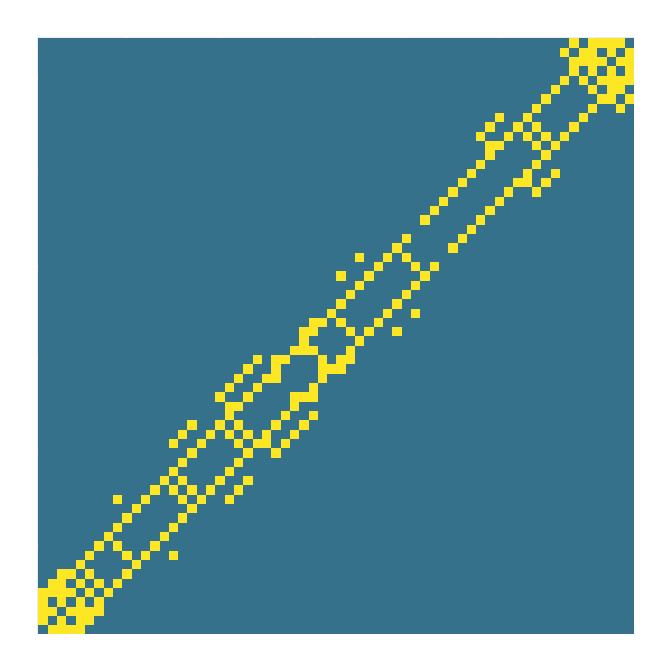}
\caption{Alcoholic subject}
\end{subfigure}
\begin{subfigure}[t]{0.45\linewidth} \centering
\includegraphics[width=\linewidth]{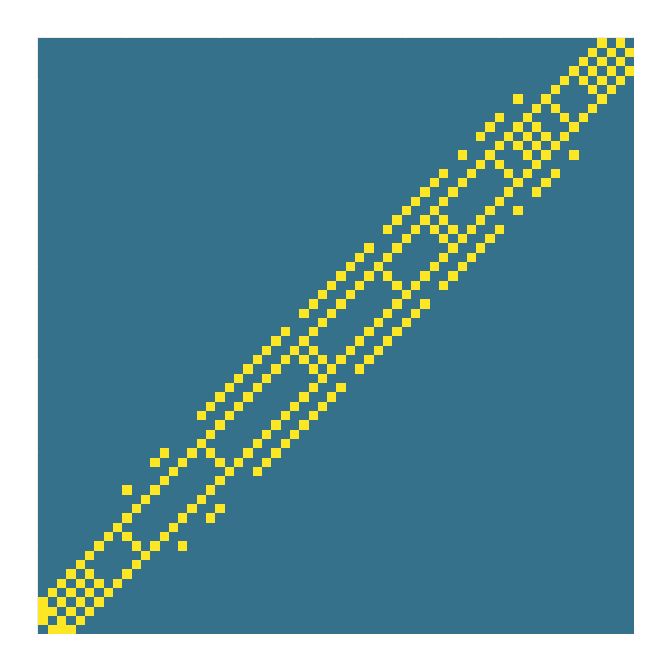}
\caption{Control subject}
\end{subfigure}
\caption{Support (off-diagonals) of SyGlasso-estimated temporal Sylvester factor $\hat{\mat{\Psi}}_{time}$ of the precision matrix for (a) the alcoholic subject and (b) the control subject. Both subjects exhibit banded conditional dependency structures over time.}
\label{fig:alcoholism_time_result}
\end{figure}
While previous analyses on this EEG data using graphical models only focused on the precision matrix of the electrodes, here we exhibit in Figure \ref{fig:alcoholism_time_result} the second precision matrix that encodes temporal dependency. Overall both subjects exhibit banded dependency structures over time, since adjacent timepoints are conditionally dependent. However, note that the conditional dependency structure of the timepoints for the alcoholic subject appears to be more chaotic.

\section{Conclusion}\label{sec:syglasso-conclusion}
This chapter proposed a Sylvester-structured graphical model and an inference algorithm, the SyGlasso, that can be applied to tensor-valued data. The current frameworks available for researchers are limited to Kronecker product and Kronecker sum models on either the covariance or the precision matrix. Our model is motivated by a generative stochastic representation based on the Sylvester equation. We showed that the resulting precision matrix corresponds to the squared Kronecker sum of the precision matrices $\mat{\Psi}_k$ along each mode. The individual components $\mat{\Psi}_k$'s are estimated by the nodewise regression based approach. 

There are several promising future directions. First is to relax the assumption that the diagonals of the factors are fixed - an assumption that is standard among the Kronecker structured models for theoretical analysis. Practically, SyGlasso is able to recover the off-diagonals of the individual $\bm\Psi_k$ and the diagonal of $\bm\Omega$, which only requires to estimating $\bigoplus_{k=1}^K \text{diag}(\bm{\Psi}_k)$ instead of all diagonal entries $\text{diag}(\bm{\Psi_k})$ for all $k$. Secondly, in terms of the statistical properties, our theoretical results guarantee sparsistency of the individual graphs with a slower convergence rate than that is proposed in \citet{greenewald2019tensor}, while empirical evidence suggests that a faster rate can be achieved. Improvement of this statistical convergence rate analysis will be worthwhile. Also, our results do not guarantee statistical convergence of individual $\mat\Psi_k$'s nor $\bm\Omega$ with respect to the operator norm. Similar to the solution proposed in \citet{zhou2011high}, we plan to adopt a two-step procedure using SyGlasso for variable selection followed by refitting the precision matrix $\mat{\Omega}$ using maximum likelihood estimation with edge constraint.

%% file: Chap3/chap3.tex
In this chapter, we extend the Sylvester graphical model introduced in Chapter~\ref{ch:syglasso} to incorporate a new inference procedure, called SG-PALM, for learning conditional dependency structure of high-dimensional tensor-variate data. Unlike the SyGlasso, the new method is computationally scalable to ultra-high dimension. Scalability of SG-PALM follows from the fast proximal alternating linearized minimization (PALM) procedure that SG-PALM uses during training. We establish that SG-PALM converges linearly (i.e., geometric convergence rate) to a global optimum of its objective function. We demonstrate the scalability and accuracy of SG-PALM for an important but challenging climate prediction problem: spatio-temporal forecasting of solar flares from multimodal imaging data.

\section{Introduction}\label{sec:sgpalm-intro}

A common challenge for structured tensor graphical models is the efficient estimation of the underlying (conditional) dependency structures. KP-structured models are generally estimated via extension of GLasso~\citep{friedman2008sparse} that iteratively minimize the $\ell_1$-penalized negative likelihood function for the matrix-normal data with KP covariance. This procedure was shown to converge to some local optimum of the penalized likelihood function~\citep{yin2012model,tsiligkaridis2013convergence}. Similarly, \citet{kalaitzis2013bigraphical} further extended GLasso to the KS-structured case for $2$-way tensor data. \citet{greenewald2019tensor} extended this to multiway tensors, exploiting the linearity of the space of KS-structured matrices and developing a projected proximal gradient algorithm for KS-structured inverse covariance matrix estimation, which achieves linear convergence (i.e., geometric convergence rate) to the global optimum. In Chapter~\ref{ch:syglasso}, the Sylvester-structured graphical model is estimated via a nodewise regression approach inspired by algorithms for estimating a class of vector-variate graphical models~\citep{meinshausen2006high,khare2015convex}. However, no theoretical convergence result for the algorithm was established nor did they study the computational efficiency of the algorithm.

In the modern era of big data, both computational and statistical learning accuracy are required of algorithms. Furthermore, when the objective is to learn representations for physical processes, interpretablility is crucial. In this chapter, we bridge this ``Statistical-to-Computational-to-Interpretable gap'' for Sylvester graphical models. We develop a simple yet powerful first-order optimization method, based on the Proximal Alternating Linearized Minimization (PALM) algorithm, for recovering the conditional dependency structure of such models. Moreover, we provide the link between the Sylvester graphical models and physical processes obeying differential equations and illustrate the link with a real-data example. The following are our principal contributions:
\begin{enumerate}
    \item A fast algorithm that efficiently recovers the generating factors of a representation for high-dimensional multiway data, significantly improving on the SyGlasso algorithm described in Chapter~\ref{ch:syglasso}.
    \item A comprehensive convergence analysis showing linear convergence of the objective function to its global optimum and providing insights for choices of hyperparameters.
    \item A novel application of the algorithm to an important multi-modal solar flare prediction problem from solar magnetic field sequences. For such problems, SG-PALM is physically interpretable in terms of the Poisson differential equation for solar magnetic induction fields proposed by heliophysicists.
\end{enumerate}

\section{Background and Notation}\label{sec:sgpalm-background}



\subsection{Notations}
In this chapter, scalar, vector and matrix quantities are denoted by lowercase letters, boldface lowercase letters and boldface capital letters, respectively. For a matrix $\mat{A} = (\mat{A}_{i,j}) \in \mathbb{R}^{d \times d}$, we denote $\|\mat{A}\|_2, \|\mat{A}\|_F$ as its spectral and Frobenius norm, respectively. We define $\|\mat{A}\|_{1,\text{off}} := \sum_{i \neq j} |\mat{A}_{i,j}|$ as its off-diagonal $\ell_1$ norm. For tensor algebra, we adopt the notations used by \citet{kolda2009tensor}. A $K$-th order tensor is denoted by boldface Euler script letters, e.g, $\tensor{X} \in \bbR^{d_1 \times \dots \times d_K}$. The $(i_1,\dots, i_K)$-th element of $\tensor{X}$ is denoted by $\tensor{X}_{i_1,\dots, i_K}$, and the vectorization of $\tensor{X}$ is the $d$-dimensional vector $\vecto(\tensor{X}) := (\tensor{X}_{1,1,\dots,1},\tensor{X}_{2,1,\dots,1},\dots,\tensor{X}_{d_1,1,\dots,1},\dots,\tensor{X}_{d_1,d_2,\dots,d_k})^T$ with $d=\prod_{k=1}^K d_k$. A fiber is the higher order analogue of the row and column of matrices. It is obtained by fixing all but one of the indices of the tensor. Matricization, also known as unfolding, is the process of transforming a tensor into a matrix. The mode-$k$ matricization of a tensor $\tensor{X}$, denoted by $\tensor{X}_{(k)}$, arranges the mode-$k$ fibers to be the columns of the resulting matrix. The $k$-mode product of a tensor $\tensor{X} \in \bbR^{d_1 \times \dots \times d_K}$ and a matrix $\mat{A} \in \bbR^{J \times d_k}$, denoted as $\tensor{X} \times_k \mat{A}$, is of size $d_1 \times \dots \times d_{k-1} \times J \times d_{k+1} \times \dots d_K$. Its entry is defined as $(\tensor{X} \times_k \mat{A})_{i_1,\dots,i_{k-1},j,i_{k+1},\dots,i_K} := \sum_{i_k=1}^{d_k} \tensor{X}_{i_1,\dots,i_K} A_{j,i_k}$. For a list of matrices $\{\mat{A}_k\}_{k=1}^K$ with $\mat{A}_k \in \bbR^{d_k \times d_k}$, we define $\tensor{X} \times \{\mat{A}_1,\dots,\mat{A}_K\} := \tensor{X} \times_1 \mat{A}_1 \times_2 \dots \times_K \mat{A}_K$. Lastly, we define the $K$-way Kronecker product as $\bigotimes_{k=1}^K \mat{A}_k = \mat{A}_1 \otimes \cdots \otimes \mat{A}_K$, and the equivalent notation for the Kronecker sum as $\bigoplus_{k=1}^K \mat{A}_k = \mat{A}_1 \oplus \dots \oplus \mat{A}_K = \sum_{k=1}^K \mat I_{[d_{k+1:K}]} \otimes \mat{A}_k \otimes \mat I_{[d_{1:k-1}]}$, where $\mat I_{[d_{k:\ell}]} = \mat I_{d_k} \otimes \dots \otimes \mat I_{d_\ell}$. For the case of $K=2$, $\mat{A}_1 \oplus \mat{A}_2 = \mat{I}_{d_2} \otimes \mat{A}_1 + \mat{A}_2 \otimes \mat{I}_{d_1}$.

\subsection{Tensor Gaussian graphical models}
A random tensor $\tensor{X} \in \bbR^{d_1 \times \dots \times d_K}$ follows the tensor normal distribution with zero mean when $\vecto(\tensor{X})$ follows a normal distribution with mean $\mat{0} \in \bbR^d$ and precision matrix $\bm\Omega := \bm\Omega(\bm\Psi_1,\dots,\bm\Psi_K)$, where $d=\prod_{k=1}^K d_k$. Here, $\bm\Omega(\bm\Psi_1,\dots,\bm\Psi_K)$ is parameterized by $\bm\Psi_k \in \bbR^{d_k \times d_k}$ via either Kronecker product, Kronecker sum, or the Sylvester structure, and the corresponding negative log-likelihood function (assuming $N$ independent observations $\tensor{X}^i, i=1,\dots,N$)
\begin{equation}\label{eqn:gaussiannegloglik}
     -\frac{N}{2} \log|\bm\Omega| + \frac{N}{2}\tr(\mat{S}\bm\Omega),
\end{equation}
where $\bm\Omega = \bigotimes_{k=1}^K \bm\Psi_k$, $\bigoplus_{k=1}^K \bm\Psi_k$, or $\Big(\bigoplus_{k=1}^K \bm\Psi_k\Big)^2$ for KP, KS, and Sylvester models, respectively; and $\mat{S} = \frac{1}{N}\sum_{i=1}^N \vecto(\tensor{X}^i) \vecto(\tensor{X}^i)^T$. For $K=1$, this formulation reduces to the vector normal distribution with zero mean and precision matrix $\bm\Psi_1$.

To encourage sparsity in the high-dimensional scenario, penalized negative log-likelihood function is proposed
\begin{equation*}
     -\frac{N}{2} \log|\bm\Omega| + \frac{N}{2}\tr(\mat{S}\bm\Omega) + \sum_{k=1}^K P_{\lambda_k}(\bm\Psi_k),
\end{equation*}
where $P_{\lambda_k}(\cdot)$ is a penalty function indexed by the tuning parameter $\lambda_k$ and is applied elementwise to the off-diagonal elements of $\bm\Psi_k$. Popular choices for $P_{\lambda_k}(\cdot)$ include the lasso penalty~\citep{tibshirani1996regression}, the adaptive lasso penalty~\citep{zou2006adaptive}, the SCAD penalty~\citep{fan2001variable}, and the MCP penalty~\citep{zhang2010nearly}. 

\subsection{The Sylvester generating equation}
The Sylvester graphical model uses the Sylvester tensor equation to define a generative process for the underlying multivariate tensor data. The Sylvester tensor equation has been studied in the context of finite-difference discretization of high-dimensional elliptical partial differential equations~\citep{grasedyck2004existence,kressner2010krylov}. Any solution $\tensor{X}$ to such a PDE  must have the (discretized) form:
\begin{equation}\label{eqn:sylvester}
    \begin{aligned}
        \sum_{k=1}^K \tensor{X} \times_k \bm\Psi_k = \tensor{T} &\Longleftrightarrow  \Big(\bigoplus_{k=1}^K \bm\Psi_k \Big) \vecto(\tensor{X}) = \vecto(\tensor{T}).
    \end{aligned}
\end{equation} 
where $\tensor{T}$ is the driving source on the domain, and $\bigoplus_{k=1}^K \bm\Psi_k$ is a Kronecker sum of $\bm\Psi_k$'s representing the discretized differential operators for the PDE, e.g., Laplacian, Euler-Lagrange operators, and associated coefficients. These operators are often sparse and structured. 

For example, consider a physical process characterized as a function $u$ that satisfies:
\begin{equation*}
    \mathcal{D}u = f \quad \text{in} \quad \Omega, \quad u(\Gamma)=0, \quad \Gamma = \partial \Omega.
\end{equation*}
where $f$ is a driving process, e.g., a Wiener process (white Gaussian noise); $\mathcal{D}$ is a differential operator, e.g, Laplacian, Euler-Lagrange; $\Omega$ is the domain; and $\Gamma$ is the boundary of $\Omega$. After discretization, this is equivalent to (ignoring discretization error) the matrix equation
\begin{equation*}
    \mat{D}\mat{u} = \mat{f}.
\end{equation*}
Here, $\mat{D}$ is a sparse matrix since $\mathcal{D}$ is an infinitesimal operator. Additionally, $\mat{D}$ admits Kronecker structure as a mixture of Kronecker sums and Kronecker products.

The matrix $\mat{D}$ reduces to a Kronecker sum when $\mathcal{D}$ involves no mixed derivatives. For instance, consider the Poisson's equation in 2D, where $u(x,y)$ on $[0,1]^2$ satisfies the elliptical PDE
\begin{equation*}
    \mathcal{D}u = (\partial^2_x + \partial^2_y)u = f.
\end{equation*}
The Poisson equation governs many physical processes, e.g., electromagnetic induction, heat transfer, convection, etc. A simple Euler discretization yields $\mat{U} = (u(i,j))_{i,j}$, where $u(i,j)$ satisfies the local equation (up to a constant discretization scale factor)
\begin{equation*}
\begin{aligned}
    2 u(i,j) &= u(i+1,j) + u(i-1,j) + u(i,j+1)  \\
    & \quad + u(i,j-1) - 4f(i,j).  
\end{aligned}
\end{equation*}
Defining $\mat{u}=\text{vec}(\mat{U})$ and $\mat{A}$ (a tridiagonal matrix)
\begin{equation*}
\mat{A} = 
    \begin{bmatrix}
    -1 & 2 & -1 & & & \\
      & \ddots & \ddots & \ddots \\
      &  & \ddots & \ddots & \ddots \\
      &  &  & -1 & 2 & -1
    \end{bmatrix},
\end{equation*}
then $(\mat{A} \oplus \mat{A})\mat{u} = \mat{f}$, which is the Sylvester equation ($K=2$).

For the Poisson example, if the source $\mat{f}$ is a white noise random variable, i.e., its covariance matrix is proportional to the identity matrix, then the inverse covariance matrix of $\mat{u}$ has sparse square-root factors, since $\text{Cov}^{-1}(\mat{u})=(\mat{A} \oplus \mat{A})(\mat{A} \oplus \mat{A})^T$. Other physical processes that are generated from differential equations will also have sparse inverse covariance matrices, as a result of the sparsity of general discretized differential operators. Note that similar connections between continuous state physical processes and sparse ``discretized'' statistical models have been established by \citet{lindgren2011explicit}, who elucidated a link between Gaussian fields and Gaussian Markov Random Fields via stochastic partial differential equations. 

The Sylvester generative (SG) model~\eqref{eqn:sylvester} leads to a tensor-valued random variable $\tensor{X}$ with a precision matrix $\bm\Omega=\Big(\bigoplus_{k=1}^K \bm\Psi_k\Big)^2$, given that $\tensor{T}$ is white Gaussian. The Sylvester generating factors $\bm\Psi_k$'s can be obtained via minimization of the penalized negative log-pseudolikelihood
\begin{equation}
  \label{eqn:objective}
  \begin{aligned}
    \mathcal{L}_{\bm\lambda}(\bm\Psi)
    = & -\frac{N}{2} \log | (\bigoplus_{k=1}^K \diag(\bm\Psi_k))^2| \\
    & + \frac{N}{2} \tr(\mat{S} \cdot (\bigoplus_{k=1}^K \bm\Psi_k)^2) + \sum_{k=1}^K \lambda_k \|\bm\Psi_k\|_{1, \text{off}}.
  \end{aligned}
\end{equation}
This differs from the penalized Gaussian negative log-likelihood in the exclusion of off-diagonals of $\bm\Psi_k$'s in the log-determinant term. \eqref{eqn:objective} is motivated and derived directly using the Sylvester equation defined in~\eqref{eqn:sylvester}, from the perspective of solving a sparse linear system. This maximum pseudolikelihood estimation procedure has been applied to vector-variate Gaussian graphical models (see \citet{khare2015convex} and references therein for discussions). It is known that inference using pseudo-likelihood is consistent and enjoys the same $\sqrt{N}$ convergence rate as the MLE in general~\citep{varin2011overview}. This procedure can also be more robust to model misspecification. Detailed derivations are provided in Appendix~\ref{supp:pseudolik}.

\section{The SG-PALM Method}\label{sec:sgpalm-method}

Estimation of the generating parameters $\bm\Psi_k$'s of the SG model is challenging since the sparsity penalties are applied to the square root factors of the precision matrix and the likelihood function involves a mix of Kronecker sums and Kronecker products of matrix-valued parameters. The previously proposed estimation procedure called SyGlasso (see Chapter~\ref{ch:syglasso}), recovers only the off-diagonal elements of each Sylvester factor. This is a deficiency in many applications where the factor-wise variances are desired. Moreover, the convergence rate of the cyclic coordinate-wise algorithm used in SyGlasso is unknown and the computational complexity of the algorithm is higher than other sparse Glasso-type procedures. To overcome these deficiencies, we propose a proximal alternating linearized minimization method, called SG-PALM, for finding the minimizer of \eqref{eqn:objective}. SG-PALM is designed to exploit structures of the coupled objective function and yields simultaneous estimates for both off-diagonal and diagonal entries.

The PALM algorithm was originally proposed to solve nonconvex optimization problems with separable structures, such as those arising in nonnegative matrix factorization~\citep{xu2013block,bolte2014proximal}. Its efficacy in solving convex problems has also been established, for example, in regularized linear regression problems~\citep{shefi2016rate}, it was proposed as an attractive alternative to iterative soft-thresholding algorithms (ISTA). For simplicity, we consider the $\ell_1$-regularized case~\eqref{eqn:objective}, and the general, possibly non-convex, case is described in the supplement. The SG-PALM procedure is summarized in Algorithm~\ref{alg:sg-palm}.

For clarity of notation we write
\begin{equation}\label{eqn:decomp_obj}
    \mathcal{L}_{\bm\lambda}(\bm\Psi_1,\dots,\bm\Psi_K) = H(\bm\Psi_1,\dots,\bm\Psi_K) + \sum_{k=1}^K G_k(\bm\Psi_k),
\end{equation}
where $H: \mathbb{R}^{d_1 \times d_1} \times \cdots \times \mathbb{R}^{d_K \times d_K} \rightarrow \mathbb{R}$ represents the log-determinant plus trace terms in \eqref{eqn:objective} and $G_k: \mathbb{R}^{d_k \times d_k} \rightarrow (-\infty,+\infty]$ represents the penalty term in \eqref{eqn:objective} for each axis $k=1,\dots,K$. For notational simplicity we use $\bm\Psi$ (i.e., omitting the subscript) to denote the set $\{\bm\Psi_k\}_{k=1}^K$ or the $K$-tuple $(\bm\Psi_1,\dots,\bm\Psi_K)$ whenever there is no risk of confusion. The gradient of the smooth function $H$ with respect to $\bm\Psi_k$, $\nabla_k H(\bm\Psi)$, is given by
\begin{equation}\label{eqn:block-grad}
\begin{aligned}
    & \diag\Big(\Big\{\tr[(\diag((\bm\Psi_k)_{ii}) + \bigoplus_{j \neq k}\diag(\bm\Psi_j))^{-1}] \Big\}_{i=1}^{d_k} \Big) \\
    & \quad + \mat{S}_k\bm\Psi_k + \bm\Psi_k\mat{S}_k + 2\sum_{j \neq k}\mat{S}_{j,k}.
\end{aligned}
\end{equation}
Here, the first ``$\diag$'' maps a $d_k$-vector to a $d_k \times d_k$ diagonal matrix, the second one maps a scalar (i.e., $(\bm\Psi_k)_{ii}$) to a $(\prod_{j \neq k}d_j) \times (\prod_{j \neq k}d_j)$ diagonal matrix with the same elements, and the third operator maps a symmetric matrix to a matrix containing only its diagonal elements. In addition, we define:
\begin{equation}
    \begin{aligned}
        & \mat{S}_k = \frac{1}{N}\sum_{i=1}^N \tensor{X}_{(k)}^i(\tensor{X}_{(k)}^i)^T, \\
        & \mat{S}_{j,k} = \frac{1}{N}\sum_{i=1}^N \mat{V}_{j,k}^i(\mat{V}_{j,k}^i)^T, \\
        & \mat{V}_{j,k}^i = \tensor{X}_{(k)}^i\Big(\mat{I}_{d_{1:j-1}} \otimes \bm\Psi_j \otimes \mat{I}_{d_{j:K}}\Big)^T, \quad j \neq k.
    \end{aligned}
\end{equation}
A key ingredient of the PALM algorithm is a proximal operator associated with the non-smooth part of the objective, i.e., $G_k$'s. In general, the proximal operator of a proper, lower semi-continuous convex function $f$ from a Hilbert space $\mathcal{H}$ to the extended reals $(-\infty,+\infty]$ is defined by~\citep{parikh2014proximal}
\begin{equation*}
    \text{prox}_f(v) = \argmin_{x \in \mathcal{H}} f(x) + \frac{1}{2}\|x-v\|^2_2
\end{equation*}
for any $v \in \mathcal{H}$. The proximal operator well-defined as the expression on the right-hand side above has a unique minimizer for any function in this class. For $\ell_1$-regularized case, the proximal operator for the function $G_k$ is given by
\begin{equation}
     \text{prox}_{G_k}^{\lambda_k}(\bm\Psi_k) = \diag(\bm\Psi_k) + \text{soft}(\bm\Psi_k-\diag(\bm\Psi_k), \lambda_k),
\end{equation}
where the soft-thresholding operator $\text{soft}_{\lambda}(x) = \text{sign}(x)\max(|x|-\lambda,0)$ has been applied element-wise.

\begin{algorithm}[!tbh]
\begin{minipage}{0.9\linewidth}
\begin{algorithmic}
\caption{SG-PALM}\label{alg:sg-palm}
\REQUIRE Data tensor $\tensor{X}$, mode-$k$ Gram matrix $\mat{S}_k$, regularizing parameter $\lambda_k$, backtracking constant $c \in (0,1)$, initial step size $\eta_0$, initial iterate $\bm\Psi_k$ for each $k=1,\dots,K$.
\WHILE{not converged}
    \FOR{$k=1,\dots,K$}
        \STATE \textit{Line search:} 
        
        Let $\eta^t_k$ be the largest element of $\{c^j \eta_{k,0}^t\}_{j=1,\dots}$ such that condition~\eqref{eqn:linesearch-cond} is satisfied.
        \STATE \textit{Update:} 
        
        $\bm\Psi_k^{t+1} \leftarrow \text{prox}^{\eta^t_k\lambda_k}_{G_k}\Big(\bm\Psi_k^t - \eta^t_k \nabla_k H(\bm\Psi_{i < k}^{t+1},\bm\Psi_{i \geq k}^t)\Big)$.
    \ENDFOR
    \STATE \textit{Update initial step size:} Compute Barzilai-Borwein step size $\eta_0^{t+1}=\min_k \eta^{t+1}_{k,0}$, where $\eta^{t+1}_{k,0}$ is computed via~\eqref{eqn:bb-step}.
\ENDWHILE
\ENSURE Final iterates $\{\bm\Psi_k\}_{k=1}^K$.
\end{algorithmic}
\end{minipage}
\end{algorithm}

\subsection{Choice of step size}
In the absence of a good estimate of the blockwise Lipchitz constant, the step size of each iteration of SG-PALM is chosen using backtracking line search, which, at iteration $t$, starts with an initial step size $\eta_0^t$ and reduces the size with a constant factor $c \in (0,1)$ until the new iterate satisfies the sufficient descent condition:
\begin{equation}\label{eqn:linesearch-cond}
    H(\bm\Psi_{i \leq k}^{t+1},\bm\Psi_{i > k}^t) \leq Q_{\eta^t}(\bm\Psi_{i \leq k}^{t+1},\bm\Psi_{i > k}^t;\bm\Psi_{i < k}^{t+1},\bm\Psi_{i \geq k}^t).
\end{equation}
Here, 
\begin{equation*}
\begin{aligned}
    & Q_{\eta}(\bm\Psi_{i < k},\bm\Psi_k,\bm\Psi_{i > k};\bm\Psi_{i < k},\bm\Psi_k',\bm\Psi_{i > k}) \\
    &= H(\bm\Psi_{i < k},\bm\Psi_k,\bm\Psi_{i > k}) \\ 
    &+ \tr\Big((\bm\Psi_k'-\bm\Psi_k)^T \nabla_k H(\bm\Psi_{i < k},\bm\Psi_k,\bm\Psi_{i > k})\Big) \\
    &+ \frac{1}{2\eta}\|\bm\Psi_k'-\bm\Psi_k\|_F^2.
\end{aligned}
\end{equation*}
The sufficient descent condition is satisfied with any $\frac{1}{\eta}=M_k$ and $M_k \geq L_k$, for any function that has a block-wise Lipschitz gradient with constant $L_k$ for $k=1,\dots,K$. In other words, so long as the function $H$ has block-wise gradient that is Lipschitz continuous with some block Lipschitz constant $L_k>0$ for each $k$, then at each iteration $t$, we can always find an $\eta^t$ such that the inequality in \eqref{eqn:linesearch-cond} is satisfied. Indeed, we proved in Lemma~\ref{lemma:lip} in the Appendix that $H$ has the desired properties. Additionally, in the proof of Theorem~\ref{thm:sg-palm-main} we also showed that the step size found at each iteration $t$ satisfies $\frac{1}{\eta_k^{0}} \leq L_k \leq \frac{1}{\eta_k^{t}} \leq c L_k$.

In terms of the initialization, a safe step size (i.e., very small $\eta_0^t$) often leads to slower convergence. Thus, we use the more aggressive Barzilai-Borwein (BB) step~\citep{barzilai1988two} to set a starting $\eta_0^t$ at each iteration (see Appendix~\ref{supp:bb-step-size} for justifications of the BB method). In our case, for each $k$, the step size is given by
\begin{equation}\label{eqn:bb-step}
    \eta_{k,0}^t = \frac{\|\bm\Psi_k^{t+1}-\bm\Psi_k^{t}\|_F^2}{\tr(\bm{A})},
\end{equation}
where
\begin{equation*}
\begin{aligned}
    \bm{A} &= (\bm\Psi_k^{t+1}-\bm\Psi_k^{t})^T \times \\
    &(\nabla_k H(\bm\Psi_{i \leq k}^{t+1},\bm\Psi_{i > k}^t)
    - \nabla_k H(\bm\Psi_{i < k}^{t+1},\bm\Psi_{i \geq k}^t)).
\end{aligned}
\end{equation*}

\subsection{Computational complexity}
After pre-computing $\mat{S}_k$, the most significant computation for each iteration in the SG-PALM algorithm is the sparse matrix-matrix multiplications $\mat{S}_k \bm\Psi_k$ and $\mat{S}_{j,k}$ in the gradient calculation. In terms of computational complexity, the former and latter can be computed using $O(d_k^3)$ and $O(N \sum_{j \neq k} d_jm_j^2)$ operations, respectively, there $m_j = \prod_{i \neq j} d_i$. Thus, each iteration of SG-PALM can be computed using $O\Big(\sum_{k=1}^K (d_k^3 + N \sum_{j \neq k} d_jm_j^2) \Big)$
floating point operations, which is significantly lower than competing methods.

\begin{remark}
All the structured precision estimation algorithms are variants of Glasso, implemented with techniques tailored to the model assumptions for speedup. Generally speaking, the resulting complexity consists of the mode-wise complexity ($d_k^3$) and the cost of updating the objective: $dK$ for TeraLasso~\citep{greenewald2019tensor}, $N\sum_k d_k m_k^2$ for Tlasso~\citep{lyu2019tensor}, and $N \sum_k \sum_{j \neq k} d_j m_j^2$ for SG-PALM. The mode-wise complexity of TeraLasso is dominated by matrix inversion, which is hard to scale for general problem instances. For Tlasso/KGlasso, the mode-wise complexity is the same as that of running a Glasso-type algorithm for each mode, which could be improved by applying state-of-the-art optimization techniques developed for vector-variate Gaussian graphical models. For SG-PALM, the mode-wise operations involve only sparse-dense matrix multiplications, which could be improved to $O(d_k \cdot \textsf{nnz})$, where $\textsf{nnz}$ counts the number of non-zero elements of the sparse matrix (i.e., the estimated $\bm\Psi_k$ at each iteration). This could greatly reduce the computational cost for extremely sparse $\bm\Psi_k$, e.g., with only $O(d_k)$ non-zero elements. Further, Tlasso and SG-PALM both incur a cost of $O(N d_k m_k^2)$ for each mode-wise update. This can also be reduced to be $\approx d$ for sparse estimated $\bm\Psi_k$'s at each iteration. Overall, for sample-starved setting where we only have access to a handful of data samples, structured KP and KS models run similarly fast, while the Sylvester GM runs slower theoretically due to the extra and richer structures that it takes into account.

Additionally, TG-ISTA and the Tlasso proposed both require inversion of $d_k \times d_k$ matrices, which is not easily parallelizable and cannot easily exploit the sparsity of $\bm\Psi_k$'s. The cyclic coordinate-wise method used in SyGlasso does not allow for parallelization since it requires cycling through entries of each $\bm\Psi_k$ in specified order. In contrast, SG-PALM can be implemented in parallel to distribute the sparse matrix-matrix multiplications because at no step do the algorithms require storing all dense matrices on a single machine. Therefore, with the adaptation of communication-efficient algorithms (such as that proposed in \citet{koanantakool2018communication} for vector-variate Gaussian graphical models), the scalability of the distributed SG-PALM is restricted only by the number of machines available. 
\end{remark}

\section{Convergence Analysis}\label{sec:sgpalm-convergence}

In this section, we present the main convergence theorems. Detailed proofs are included in the supplement. Here, we study the convergence behavior in the convex cases, but similar convergence rate can be established for non-convex penalties (see supplement).

We first establish statistical convergence of a global minimizer $\hat{\bm\Psi}$ of \eqref{eqn:objective} to its true value, denoted as $\bar{\bm\Psi}$, under the correct statistical model.
\begin{theorem}\label{thm:statistical}
Let $\mathcal{A}_{k}:=\{(i,j):(\bar{\bm\Psi}_k)_{i,j} \neq 0, i \neq j\}$ and $q_{k}:=|\mathcal{A}_{k}|$ for $k=1,\dots,K$. If $N > O(\max_k q_k d_k \log d)$ and $d:=d_N=O(N^{\kappa})$ for some $\kappa \geq 0$, and further, if the penalty parameter satisfies $\lambda_k:=\lambda_{N,k}=O(\sqrt{\frac{d_k\log d}{N}})$ for all $k=1,\dots,K$, then under conditions (A1-A3) in Appendix~\ref{supp:thm_statistical}, there exists a constant $C>0$ such that for any $\eta>0$ the following events hold with probability at least $1 - O(\exp(-\eta \log d))$:
  \begin{equation*}
    \begin{aligned}
        & \sum_{k=1}^K\|\text{offdiag}(\hat{\bm\Psi}_k) - \text{offdiag}(\bar{\bm\Psi}_k)\|_F \\
        & \leq C\sqrt{K}\max_{k}\sqrt{q_{k}}\lambda_{k}.  
    \end{aligned}
  \end{equation*}
Here $\text{offdiag}(\bm\Psi_k)$ is the the off-diagonal part of $\bm\Psi_k$. If further $\min_{(i,j) \in \mathcal{A}_{k}}|(\bar{\bm{\Psi}}_k)_{i,j}| \geq 2C\max_{k}\sqrt{q_{k}}\lambda_{k}$ for each $k$, then sign($\hat{\bm{\Psi}}_k$)=sign($\bar{\bm{\Psi}}_k$).
\end{theorem}

Theorem~\ref{thm:statistical} means that under regularity conditions on the true generative model, and with appropriately chosen penalty parameters $\lambda_k$'s guided by the theorem, one is guaranteed to recover the true structures of the underlying Sylvester generating parameters $\bm\Psi_k$ for $k=1,\dots,K$ with probability one, as the sample size and dimension grow.

We next turn to convergence of the iterates $\{\bm\Psi^t\}$ from SG-PALM to a global optimum of \eqref{eqn:objective}. 

\begin{theorem}\label{thm:sg-palm-main}
Let $\{\bm\Psi^{(t)}\}_{t \geq 0}$ be generated by SG-PALM.
Then, SG-PALM converges in the sense that
\begin{equation*}
\begin{aligned}
    & \frac{\mathcal{L}_{\bm\lambda}(\bm\Psi^{(t+1)}) - \min \mathcal{L}_{\bm\lambda}}{\mathcal{L}_{\bm\lambda}(\bm\Psi^{(t)}) - \min \mathcal{L}_{\bm\lambda}} \\
    & \leq \Bigg(\frac{\alpha^2L_{\min}}{4Kc^2(\sum_{j=1}^K L_j)^2 + 4c^2L_{\max}} + 1\Bigg)^{-1},
\end{aligned}
\end{equation*}
where $\alpha$, $L_k,k=1,\dots,K$ are positive constants, $L_{\min}=\min_jL_j$, $L_{\max}=\max_jL_j$, and $c \in (0,1)$ is the backtracking constant defined in Algorithm~\ref{alg:sg-palm}.
\end{theorem}
Note that the term on  the right hand side of the inequality above is strictly less than $1$. This means that the SG-PALM algorithm converges linearly, which is a strong results for a non-strongly convex objective (i.e., $\mathcal{L}_{\bm\lambda}$). To the best of our knowledge, for first-order optimization methods, this rate is faster than any other Gaussian graphical models having non-strongly convex objectives (see \citet{khare2015convex,oh2014optimization} and references therein) and comparable with those having strongly-convex objectives (see, for example, \citet{guillot2012iterative,dalal2017sparse,greenewald2019tensor}). In practical large-scale applications, a fast rate is vital as it would be desired to have the iterative optimization approximation errors quickly converge to values below the statistical errors.

\section{Experiments}\label{sec:sgpalm-experiments}
Experiments in this section were performed in a system with \texttt{8-core Intel Xeon CPU E5-2687W v2 3.40GHz} equipped with \texttt{64GB RAM}. SG-PALM was implemented in \texttt{Julia v1.5}. For synthetic data analyses, we used the SyGlasso implementation in \texttt{R} with \texttt{C++} speed-up (\url{https://github.com/ywa136/syglasso}). For real data analyses, we used the \texttt{Tlasso} package implementation in \texttt{R}~\citep{r-Tlasso} and the TeraLasso implementation in \texttt{MATLAB} (\url{https://github.com/kgreenewald/teralasso}).

\subsection{Synthetic data}

We first validate the convergence theorems discussed in the previous section via simulation studies. Synthetic datasets were generated from true sparse Sylvester factors  
 $\{\bm\Psi_k\}_{k=1}^K$ where $K=3$ and $d_k=\{16,32,64\}$ for all $k$. Instances of the random matrices used here have uniformly random sparsity patterns with edge densities (i.e., the proportion of non-zero entries) ranging from $0.1\% -30\%$ on average over all $\bm\Psi_k$'s. For each $d$ and edge density combination, random samples of size $N=\{10,100,1000\}$ were tested. For comparison, the initial iterates, convergence criteria were matched between SyGlasso and SG-PALM. Highlights of the results in run times are summarized in Table~\ref{tab:synthetic_run_time}.

\begin{table}[tbh!]
\centering
\caption{Run time comparisons (in seconds with N/As indicating those exceeding $24$ hour) between SyGlasso and SG-PALM on synthetic datasets with different dimensions, sample sizes, and densities of the generating Sylvester factors. Note that the proposed SG-PALM has average speed-up ratios ranging from $1.5$ to $10$ over SyGlasso.}
\label{tab:synthetic_run_time}
\begin{tabular}{|c|p{0.5cm}|p{0.8cm}||r|r|}
 \multicolumn{5}{c}{} \\
 \hline
 \multirow{2}{*}{$d$} & \multirow{2}{*}{$N$} & \multirow{2}{*}{NZ\%} & \textbf{SyGlasso} & \textbf{SG-PALM} \\
 \cline{4-5} 
 &&& \textbf{iter} \quad \textbf{sec} & \textbf{iter} \quad \textbf{sec} \\
 \hline
 \multirow{6}{*}{$16^3$} & \multirow{2}{*}{$10^1$} & $0.11$ &
 $9$ \quad $4.6$ & $11$ \quad $4.5$ \\
 && $4.10$ &
 $9$ \quad $5.1$ & $32$ \quad $5.1$ \\
 \cline{2-5}
 & \multirow{2}{*}{$10^2$} & $0.21$ & 
 $8$ \quad $8.8$ & $11$ \quad $5.4$ \\
 && $2.60$ &
 $8$ \quad $10.8$ & $35$ \quad $7.2$ \\
 \cline{2-5}
 & \multirow{2}{*}{$10^3$} & $0.26$ & 
 $8$ \quad $82.4$ & $12$ \quad $14.3$ \\
 && $3.40$ &
 $10$ \quad $99.2$ & $37$ \quad $33.5$ \\
 \cline{1-5}
 \multirow{6}{*}{$32^3$} & \multirow{2}{*}{$10^1$} & $0.13$ &
 $10$ \quad $191.2$ & $19$ \quad $7.3$ \\
 && $7.50$ &
 $17$ \quad $304.8$ & $42$ \quad $10.2$ \\
 \cline{2-5}
 & \multirow{2}{*}{$10^2$} & $0.46$ & 
 $9$ \quad $222.4$ & $24$ \quad $28.9$ \\
 && $7.00$ &
 $17$ \quad $395.2$ & $41$ \quad $48.5$ \\
 \cline{2-5}
 & \multirow{2}{*}{$10^3$} & $0.10$ & 
 $9$ \quad $1764.8$ & $22$ \quad $226.4$ \\
 && $6.90$ &
 $19$ \quad $3789.4$ & $41$ \quad $473.9$ \\
 \cline{1-5}
 \multirow{6}{*}{$64^3$} & \multirow{2}{*}{$10^1$} & $0.65$ &
 $10$ \quad $583.7$ & $42$ \quad $91.3$ \\
 && $14.5$ &
 $22$ \quad $952.2$ & $47$ \quad $119.0$ \\
 \cline{2-5}
 & \multirow{2}{*}{$10^2$} & $0.62$ & 
 $9$ \quad $6683.7$ & $41$ \quad $713.9$ \\
 && $14.4$ &
 $21$ \quad $15607.2$ & $48$ \quad $1450.9$ \\
 \cline{2-5}
 & \multirow{2}{*}{$10^3$} & $0.85$ & 
 N/A  & $39$ \quad $6984.4$ \\
 && $14.0$ &
 N/A  & $48$ \quad $12968.7$ \\
 \hline
\end{tabular}
\end{table}

Convergence behavior of SG-PALM is shown in Figure~\ref{fig:convergence_sg-palm} (a) for the datasets with $d_k=32$, $N=\{10,100\}$, and edge densities roughly around $5\%$ and $20\%$, respectively. Geometric convergence rate of the function value gaps under Theorem~\ref{thm:sg-palm-main} can be verified from the plot. Note an acceleration in the convergence rate (i.e., a steeper slope) near the optimum, which is suggested by the ``localness'' of the Kurdyka - \L ojasiewicz (KL) property (defined in Section B.2 of the Appendix) of the objective function close to its global optimum. 
Further for the same datasets, in Figure~\ref{fig:convergence_sg-palm} (b), SG-PALM graph recovery performances is illustrated, where the Matthew's Correlation Coefficients (MCC) is plotted against run time. Here, MCC is defined by 
\begin{equation*}
    \text{MCC} = \frac{\text{TP}\times\text{TN}-\text{FP}\times\text{FN}}{\sqrt{(\text{TP}+\text{FP})(\text{TP}+\text{FN})(\text{TN}+\text{FP})(\text{TN}+\text{FN})}},
\end{equation*}
where TP is the number of true positives, TN the number of true negatives, FP the number of false positives, and FN the number of false negatives of the estimated edges (i.e., non-zero elements of $\bm\Psi_k$'s). An MCC of $1$ represents a perfect prediction, $0$ no better than random prediction and $-1$ indicates total disagreement between prediction and observation. The results validate the statistical accuracy under Theorem~\ref{thm:statistical}. It also shows that SG-PALM outperforms SyGlasso (indicated by blue/red solid dots) within the same time budget.

\begin{figure*}[tbh!]
    \centering
    \begin{subfigure}{0.45\linewidth}
    \centering
    \includegraphics[width=\textwidth]{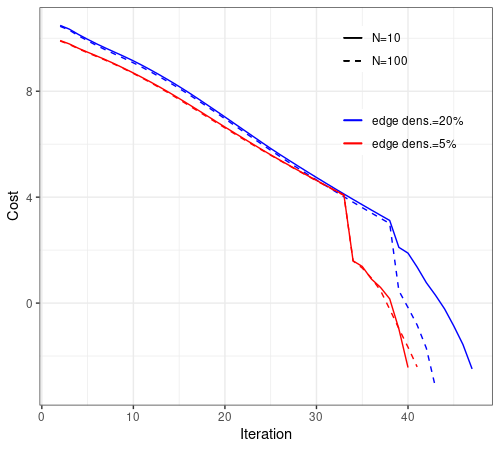}
    \caption{Cost gap vs. Iteration}
    \end{subfigure}
    \hspace{0.1pt}
    \begin{subfigure}{0.45\linewidth}
    \centering
    \includegraphics[width=\textwidth]{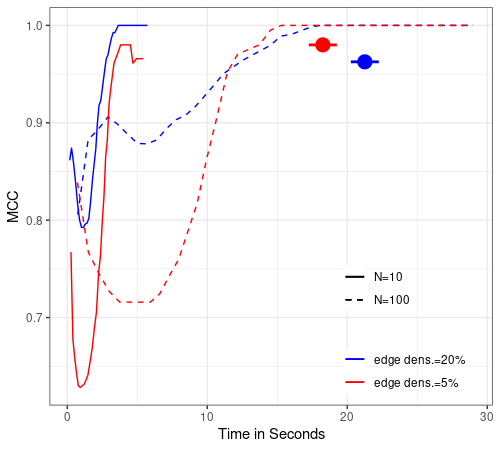}
    \caption{MCC vs. Run time}
    \end{subfigure}
    \caption{Convergence of SG-PALM algorithm under datasets with varying sample sizes (solid and dashed) generated via matrices with different sparsity (red and blue). The function value gaps on log-scale (left) verifies the geometric convergence rate in all cases and the MCC over time (right) demonstrates the algorithm's accuracy and efficiency. Note that the SG-PALM reached almost perfect recoveries (i.e., MCC of $1$) within $20$ seconds in all cases. In comparison, SyGlasso (big solid dots with line-range) was only able to achieve at lower MCCs for lower sample-size cases within $30$ seconds.}
    \label{fig:convergence_sg-palm}
\end{figure*}

\subsection{Solar imaging data}

Solar active regions are temporary centers of strong and complex magnetic field on the sun, the principal source of violent eruptions such as solar flares \citep{van2015evolution}. While weak flares of, for example, B-class, have only limited terrestrial effect, strong flares of M- and X-class can produce tremendous amount of electromagnetic radiation, causing disturbance or damage to satellites, power grids, and communication systems. Therefore, it would be great value to be able to predict how active regions evolve before the onset of solar flares. 

Although there are numerous studies that use active region images or physical parameters to predict flare activities \citep{leka2003photospheric,chen2019identifying, jiao2020solar,wang2020predicting,sun2021predicting}, fewer studies have attempted to predict the complicated preflare evolution of active regions without physical modeling \citep{bai2021predicting}. Furthermore, existing work tends to focus on predictions using images collected from a single space instrument. In this section, to illustrate the viability of the proposed tensor graphical models, we use multiwavelength active region observations acquired by multiple instruments: the Solar Dynamics Observatory (SDO)/Helioseismic and Magnetic Imager (HMI) and SDO/Atmospheric Imaging Assembly (AIA), to predict the evolution of two types of active regions that lead to either a weak (B-class) flare or a strong (M- or X-class) flare.

We construct a multiwavelength active region video dataset from the curated dataset generated by \citet{galvez2019machine}. The video data are taken in four wavelengths (94\r{A}, 131\r{A}, 171\r{A}, and 193\r{A}) by the Atmospheric Imaging Assembly \citep[AIA,][]{lemen2011atmospheric} plus the three prime HMI vector magnetic field components Bx, By, and Bz, both aboard the Solar Dynamics Observatory (SDO) satellite. Each video is a 24-hour image sequence of an active region at 1-hour cadence before a strong (M- or X-class) or a weak (B-class) flare occurs in the region. We spatially interpolate the videos so that each video is represented as a $d_1\times d_2\times d_3 \times d_4$ tensor, where $d_1 = 13$ denotes the number of frames in the video, $d_2 = 50$ denotes the height of the frames after interpolation, $d_3 = 100$ denotes the width of the frames after interpolation, and $d_4=7$ represents the number of different channels/wavelength/components at which the images are recorded. To prevent information leakage, we chronologically split the active region videos into a training set (year 2011 to 2014) and a test set (year 2011 to 2014). In the training set, there are 186 active region videos that lead to a B-class flare and 48 active region videos that lead to a M/X-class flare. In the test set, the sample sizes are 93 and 24 for the B-class and the M/X-class, respectively.

To perform active region prediction, we first fit the tensor graphical models on the training set to estimate the covariance or prediction matrices for each of the two types of active region videos, and then we use the best linear predictor to predict the last frame from all previous frames for videos in the test set.
The forward linear predictor is constructed in a multi-output least squares regression setting as
\begin{equation}
    \hat{\mat{y}}_t = -\mat\Omega_{2,2}^{-1}\mat\Omega_{2,1}\mat{y}_{t-1:t-(p-1)}
\end{equation}
when the precision estimate is available. Here, $t = d_1$ for predicting the last frame of a video. For notational convenience, let $p=d_1$ and $q=d_2d_3d_4$, then $\mat{y}_{t-1:t-(p-1)} = \mat{y}_{p-1:1} \in \mathbb{R}^{(p-1)q}$ is the stacked set of pixel values from the previous $p-1$ time instances and $\mat\Omega_{2,1} \in \mathbb{R}^{q \times (p-1)q}$ and $\mat\Omega_{2,2} \in \mathbb{R}^{q \times q}$ are submatrices of the $pq \times pq$ estimated precision matrix:
\begin{equation*}
    \hat{\mat\Omega} =
    \begin{pmatrix}
    \mat\Omega_{1,1} & \mat\Omega_{1,2} \\
    \mat\Omega_{2,1} & \mat\Omega_{2,2}
    \end{pmatrix}.
\end{equation*}

The predictors were tested on the data containing flares observed from different active regions than those in training set, so that the predictor has never ``seen'' the frames that it attempts to predict, corresponding to $117$ observations of which $93$ are B-class flares and $24$ are MX-class flares. Figure~\ref{fig:nrmse_comparison} shows the root mean squared error normalized by the difference between maximum and minimum pixels (NRMSE) over the testing samples, for the forcasts based on the SG-PALM estimator, TeraLasso estimator~\citep{greenewald2019tensor}, Tlasso estimator~\citep{lyu2019tensor}, and IndLasso estimator. Here, the TeraLasso and the Tlasso are estimation algorithms for a KS and a KP tensor precision matrix model, respectively; the IndLasso denotes an estimator obtained by applying independent and separate $\ell_1$-penalized regressions to each pixel in $\bm{y}_t$. The SG-PALM estimator was implemented using a regularization parameter $\lambda_{N}=C_1\sqrt{\frac{\min(d_k)\log(d)}{N}}$ for all $k$ with the constant $C_1$ chosen by optimizing the prediction NRMSE on the training set over a range of $\lambda$ values parameterized by $C_1$. The TeraLasso estimator and the Tlasso estimator were implemented using $\lambda_{N,k}=C_2\sqrt{\frac{\log(d)}{N\prod_{i \neq k}d_i}}$ and $\lambda_{N,k}=C_3\sqrt{\frac{\log(d_k)}{Nd}}$ for $k=1,2,3$, respectively, with $C_2, C_3$ optimized in a similar manner. Each sparse regression in the IndLasso estimator was implemented and tuned independently with regularization parameters chosen from a grid via cross-validation.

We observe that SG-PALM outperforms all three other methods, indicated by NRMSEs across pixels. Figure~\ref{fig:predicted_vs_real_img} depicts examples of predicted images, comparing with the ground truth. The SG-PALM estimates produced most realistic image predictions that capture the spatially varying structures and closely approximate the pixel values (i.e., maintaining contrast ratios). The latter is important as the flares are being classified into weak (B-class) and strong (MX-class) categories based on the brightness of the images, and stronger flares are more likely to lead to catastrophic events, such as those damaging spacecrafts. Lastly, we compare run times of the SG-PALM algorithm for estimating the precision matrix from the solar flare data with SyGlasso. Table~\ref{tab:solar_flare_run_time} in Appendix~\ref{supp:additional_experiments} illustrates that the SG-PALM algorithm converges faster in wallclock time. Note that in this real dataset, which is potentially non-Gaussian, the convergence behavior of the algorithms is different compare to synthetic examples. Nonetheless, SG-PALM enjoys an order of magnitude speed-up over SyGlasso.

\begin{figure*}[!tbh]
\centering
\begin{tabular}{@{}c@{}}
    \quad Avg. NRMSE = $0.0379$, $0.0386$, $0.0579$, $0.1628$ (from left to right) \\
    \rotatebox{90}{\qquad AR B} 
    \includegraphics[width=0.9\textwidth]{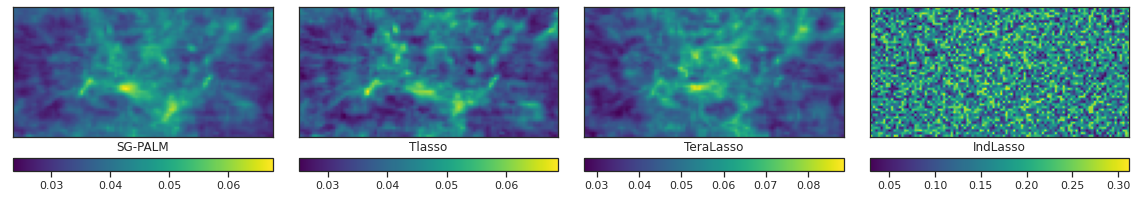}  \\
    \quad Avg. NRMSE = $0.0620$, $0.0790$, $0.0913$, $0.1172$ (from left to right) \\
    \rotatebox{90}{\qquad AR M/X} 
    \includegraphics[width=0.9\textwidth]{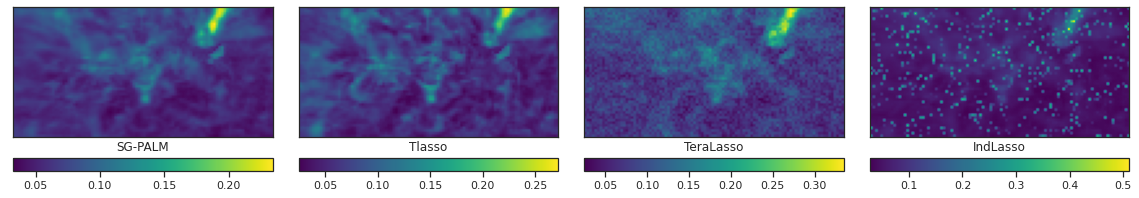} 
\end{tabular}
\caption{Comparison of the SG-PALM, Tlasso, TeraLasso, IndLasso performances measured by NRMSE in predicting the last frame of $13$-frame video sequences leading to B- and MX-class solar flares. The NRMSEs are computed by averaging across testing samples and AIA channels for each pixel. 2D images of NRMSEs are shown to indicate that certain areas on the images (usually associated with the most abrupt changes of the magnetic field/solar atmosphere) are harder to predict than the rest. SG-PALM achieves the best overall NRMSEs across pixels. B flares are generally easier to predict due to both a larger number of samples in the training set and smoother transitions from frame to frame within a video (see the supplemental material for details).}
\label{fig:nrmse_comparison}
\end{figure*}

\begin{figure*}[!tbh]
\centering
\begin{tabular}{@{}c@{}}
    Predicted examples - B vs. M/X \\
    \rotatebox{90}{\quad AR B}
    \includegraphics[width=0.9\textwidth]{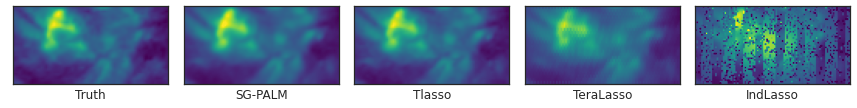}  \\
    \rotatebox{90}{\quad AR B}
    \includegraphics[width=0.9\textwidth]{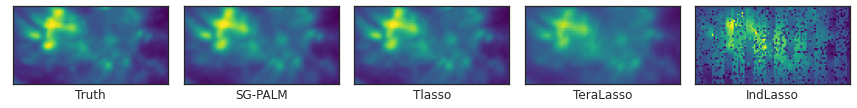} \\
    \rotatebox{90}{\quad AR M/X}
    \includegraphics[width=0.9\textwidth]{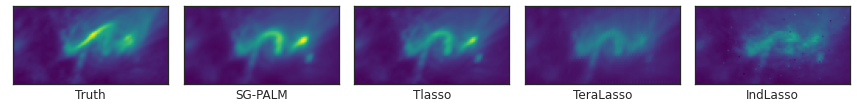} \\
    \rotatebox{90}{\quad AR M/X}
    \includegraphics[width=0.9\textwidth]{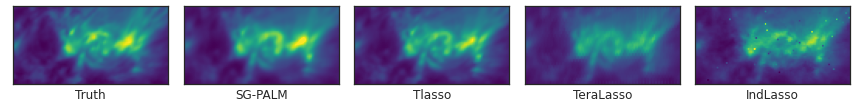}
\end{tabular}
\caption{Examples of one-hour ahead prediction of the first two AIA channels of last frames of $13$-frame videos, leading to B- (first two rows) and MX-class (last two rows) flares, produced by the SG-PALM, Tlasso, TeraLasso, IndLasso algorithms, comparing to the real image (far left column). Note that in general linear forward predictors tend to underestimate the contrast ratio of the images. The proposed SG-PALM produced the best-quality images in terms of both the spatial structures and contrast ratios. See the supplemental material for examples of predicted images from the HMI instrument.}
\label{fig:predicted_vs_real_img}
\end{figure*}

\section{Conclusion}\label{sec:sgpalm-conclusion}

We proposed SG-PALM, a proximal alternating linearized minimization method for solving a pseudo-likelihood based sparse tensor-variate Gaussian precision matrix estimation problem. Geometric rate of convergence of the proposed algorithm is established building upon recent advances in the theory of PALM-type algorithms. We demonstrated that SG-PALM outperforms the coordinate-wise minimization method in general, and in ultra-high dimensional settings SG-PALM can be faster by at least an order of magnitude. A link between the Sylvester generating equation underlying the graphical model and certain physical processes was established. This connection was illustrated on a novel astrophysics application, where multi-instrument imaging datasets characterizing solar flare events were used. The proposed methodology was able to robustly forward predict both the patterns and intensities of the solar atmosphere, yielding potential insights to the underlying physical processes that govern the flaring events. 

Future directions include additional downstream tasks involving solar flare predictions using the estimated precision matrix, such as classification of strong/weak flares. Furthermore, the statistical convergence rate outlined in Theorem~\ref{thm:statistical} might not be optimal. We have observed that, for example, from the simulation study in Section~\ref{sec:sgpalm-experiments}, where we see in Figure~\ref{fig:convergence_sg-palm}(b) that the estimator achieves perfect graph recovery accuracy even when $N=10$, which is better than the sample complexity implied by the theorem. We are actively working towards obtaining a tighter upper bound on the statistical error.

%% file: Chap4/chap4.tex
In this chapter, we develop methods of forecasting multiway times-series generated by dynamical systems. These methods can be used to study the emergence of sparsity and multiway structures in second-order statistical characterizations of dynamical processes governed by partial differential equations (PDEs). We consider several state-of-the-art multiway covariance and inverse covariance (precision) matrix estimators and examine their pros and cons in terms of accuracy and interpretability in the context of physics-driven forecasting when incorporated into the ensemble Kalman filter (EnKF). In particular, we show that multiway data generated from the Poisson, the convection-diffusion, and the Kuramoto–Sivashinsky types of PDEs can be accurately tracked via EnKF when integrated with appropriate covariance and precision matrix estimators.

\section{Introduction}\label{sec:enkf-intro}
There has recently been a resurgence of interest in integrating machine learning with physics-based modeling. Much of the recent work has focused on black-box models such as deep neural networks~\citep{takeishi2017learning, long2018pde, NEURIPS2018_e2ad76f2, vlachas2018data, reichstein2019deep, wang2020towards}. However, seeking shallower models that capture mechanism in a physically  interpretable manner has been a recurring theme in both machine learning and physics~\citep{weinan2020integrating}. The Kalman filter is a well-known technique to track a linear dynamical system over time by assimilating real-world observations into physical knowledge. Many variants based on the extended and ensemble Kalman filters have been proposed to deal with non-linear systems. However, these systems are often high dimensional and
forecasting each ensemble member forward through the system is computationally expensive. Moreover, in the high dimensional and low sample regime ($N \ll d$), the sample covariance matrix of the forecast ensemble is extremely noisy. Previous methods for dealing with these sampling errors can be dived into the ``stochastic filters'' and the ``deterministic filters''. The former often involve manually ``tuning'' of the sample covariance with variance inflation and localization~\citep{hamill2001distance,houtekamer2001sequential,ott2004local,wang2007comparison,anderson2007adaptive,anderson2009spatially,li2009simultaneous,bishop2009ensemble1,bishop2009ensemble2,campbell2010vertical,greybush2011balance,miyoshi2011gaussian}. However, these schemes require carefully choosing the inflation factor and using expert knowledge to determine local
areas of interest that are used in assimilation. Additionally, they work with perturbed observations that introduce further sampling errors due to the lack of orthogonality between the perturbation noise and the ensembles. This has led to the development of deterministic versions of the EnKF such as the square root and transform filters~\citep{bishop2001adaptive,evensen2004sampling,whitaker2002ensemble,tippett2003ensemble,hunt2007efficient,godinez2012efficient,nerger2012unification,todter2015second}, which do not perturb the observations and are designed to avoid these additional sampling errors. \citet{lawson2004implications} studies
the differences between different approaches (stochastic vs. deterministic) of EnKF and the implications of those differences in various regimes, and claims that the stochastic filters can better withstand regimes with nonlinear error growth.

Most similar to our proposed work is \citet{hou2021penalized}, which suggests to implement EnKF with a sparse inverse covariance estimator to handle the high-dimensional regime. However, we note that many real-world processes are complex and generating heterogeneous multiway/tensor-variate data. For example, weather satellites measure spatio-temporal climate variables such as temperature, wind velocity, sea level, pressure, etc. Due to the non-homogeneous nature of these data, estimation of the second-order information that encodes (conditional) dependency structure within the data is of great importance. Assuming the data are drawn from a tensor normal distribution, a straightforward way to estimate this structure is to vectorize the tensor and estimate the underlying Gaussian graphical model associated with the vector, as suggested by~\citet{hou2021penalized}. Such an approach ignores the tensor structure and requires estimating a rather high dimensional precision matrix, often with insufficient sample size. In many scientific applications the sample size can be as small as one when only a single tensor-valued measurement is available. In this chapter, we introduce a high-dimensional statistical approach that naturally integrates physics and machine learning through Kronecker-structured Gaussian graphical models. The learned representation can then be incorporated into a high dimensional predictive model using the ensemble Kalman filtering framework.

\section{Background}\label{sec:enkf-background}

We consider a noisy, non-linear dynamical model $f(\cdot)$ that evolves some unobserved states $\mat{x}_t \in \bbR^d$ through time. A noisy version of the states, $\mat{x}_t \in \bbR^r_t$, is observed via a transformation of $\mat{x}_t$ by a function $h(\cdot)$. Both the state/process noise $\mat{v}_t$ and the observation noise $\mat{w}_t$ are assumed to be independent of the states. Further, we assume both noises are zero-mean Gaussians with known diagonal covariance matrices $\mat{Q}_t$ and $\mat{R}_t$. Specifically,
\begin{equation}
    \begin{aligned}
        \bx_t &= f(\bx_{t-1}) + \bfv_t, \\
        \by_t &= h(\bx_t) + \bfw_t.
    \end{aligned}
\end{equation}
In this work, we further restrict both noise variables $\bfv$ \& $\bfw$ and the observational process are time-invariant, i.e., $\bfv_t = \bfv$, $\bfw_t = \bfw$, and $r_t = r$, although the methods developed here work in time-variant scenarios. 
In geophysical problems such as weather prediction, the state and observation dimensions are often enormous (i.e., $d \geq 10^7$ and $r_t \geq 10^5$). Therefore, as with localization methods, we make an assumption about the correlation structure of the state vector in order to handle the high dimensionality of the state. Specifically, only a small number of pairs of state variables are assumed to have non-zero conditional correlation, i.e., $\cov(x_i, x_j | \bx_{-(i, j)}) \neq 0$ where $\bx_{-(i, j)}$ represents all state variables except $x_i$ and $x_j$. For an illustrating example, consider a one-dimensional spatial field with three locations $x_1$ , $x_2$, and $x_3$ where $x_1$
and $x_3$ are both connected to $x_2$, but not each other. In this case, it is natural to
model $x_1$ and $x_3$ as uncorrelated conditioned on $x_2$ although they are not necessarily marginally uncorrelated, that is, $\cov(x_1, x_3 | x_2) = 0$ but $\cov(x_1, x_3) \neq 0$. Similar conditional independence assumptions have been used in the study of Markov random fields (MRFs), which find applications in, for example, image processing to generate textures as they can be used to generate flexible and stochastic image models~\citep{kindermann1980markov}. A spacial case is the Gaussian MRFs, which is most widely used in spatial statistics~\citep{rue2005gaussian}. For Gaussian states, the assumption that the set of non-zero conditional correlations is sparse is equivalent to the assumption that the inverse correlation matrix of the model state is sparse with few non-
zero off-diagonal entries~\citep{lauritzen1996graphical}.

\subsection{Ensemble Kalman filter}
The ensemble Kalman filter (EnKF) is particularly effective when the dynamical system is complicated and non-linear, which is often the case in physical systems~\citep{evensen1994sequential,burgers1998analysis}. In these cases, analytic propagation of the entire Gaussian systems as in the classic Kalman filter (KF) algorithm fails~\citep{evensen2003ensemble}. The EnKF can be viewed as an approximate version of the KF, in which the state distribution is represented by a sample or ``ensemble'' from the distribution. This ensemble is then propagated forward through time and updated when new data become available. 

The forecast covariance matrix is replaced by its sample estimate obtained from the forecast ensemble. However, such systems are often high-dimensional and the EnKF operates in the regime where the number of ensemble members, $N$, is much less than the size of the state, $d$, suggesting that the sample covariance matrix is singular and may introduce spurious correlations~\citep{greybush2011balance}. In this case, regularized inverse covariance models will be especially attractive. \citet{hou2021penalized} introduced a sparsity-penalized EnKF, which replace the sample covariance with an estimator of the forecasting covariance whose inverse is sparsity regularized. Here we propose incorporating the multiway covariance / inverse covariance models into the penalized EnKF framework of \citet{hou2021penalized}.


\subsection{Multiway representations for diffusion processes}\label{sec:enkf-diffusion}
Multiway representations are particularly useful when modeling data generated from physical processes as many of these processes obey partial differential equations of the form 
\begin{equation}
    \label{eq:pde}
    \begin{aligned}
        \mathcal{D}u &= f \quad \text{in } \Omega, \\
        u & = g \quad \text{on } \partial\Omega,
    \end{aligned}
\end{equation}
where $u$ is the unknown physical process, $f$ is the driving process (e.g., white Gaussian noise), $g$ is the function value of $u$ on the boundary, $\mathcal{D}$ is some differential operator (e.g, a Laplacian or an Euler-Lagrange operator), and $\Omega$ is the domain. After finite difference discretization over the domain $\Omega$, the model is equivalent to (ignoring discretization error) the matrix equation
\begin{equation*}
    \bD\bu = \mat{f}.
\end{equation*}
Here, $\bD$ is a sparse matrix since $\mathcal{D}$ is a differential operator. Additionally, as shown below, $\bD$ admits the Kronecker structure as a mixture of Kronecker sums and Kronecker products.

The matrix $\bD$ reduces to a Kronecker sum when $\mathcal{D}$ involves no mixed derivatives. As an example, we consider the Poisson equation, an elliptical PDE that governs many physical processes including electromagnetic induction, heat transfer, and convection. On a rectangular region $\Omega=(0,d_1)\times(0,d_2)$ in the 2D Cartesian plane, the Poisson equation with homogeneous Dirichlet boundary condition is expressed as
\begin{equation}
    \begin{aligned}
        \mathcal{D}u = (\partial^2_x + \partial^2_y)u &= f \quad \text{in } \Omega, \\
        u &= 0 \quad \text{on } \partial\Omega
        \label{eq:Poisson}
    \end{aligned}
\end{equation}
where $f: \Omega \to \bbR$ is the given source function and $u: \Omega \to \bbR$ is the unknown process of interest. Using the finite difference method with a square mesh grid with unit spacing, the unknown and the source can be expressed as $d_1$-by-$d_2$ matrices, $\bU$ and $\bF$, respectively, that are related to each other via
\begin{align}
    U_{i+1,j} + U_{i-1,j} + U_{i,j+1} + U_{i,j-1} - 4 U_{i,j} = F_{i,j}
\end{align}
for any interior grid point $(i,j)$. Defining $n$-by-$n$ square matrix
\begin{equation*}
\mat{A}_n = 
    \begin{bmatrix}
    2   &   -1  &       &   \\
    -1  &   2   & \ddots&   \\
        & \ddots& \ddots& -1\\
        &       &   -1  & 2
    \end{bmatrix},
\end{equation*}
the above relation can be expressed as the (vectorized) Sylvester equation with $K=2$:
\begin{equation}
    (\bA_{d_1} \oplus \bA_{d_2})\bu = \bff,
    \label{eq:poisson_discrete}
\end{equation}
where $\bu = \vecto(\bU)$, $\bff = \vecto(\bF)$. Note that $\mat{A}$ is tridiagonal. In the case where $\mat{f}$ is white noise with variance $\sigma^2$, the inverse covariance matrix of $\bu$ has the form $\cov^{-1}(\bu)=\sigma^{-2}(\bA_{d_1} \oplus \bA_{d_2})^T(\bA_{d_1} \oplus \bA_{d_2})$ and hence sparse.

More generally, any physical process generated from Equation (\ref{eq:pde}) also has sparse inverse covariance matrices due to the sparsity of general discretized differential operators. Note that similar connections between continuous state physical processes and sparse ``discretized'' statistical models have been established by \citet{lindgren2011explicit}, who elucidated a link between Gaussian fields and Gauss Markov Random Fields via stochastic partial differential equations. 

\subsection{Kronecker-structured covariance models}
Classic regularized estimators such as the graphical lasso~\citep[Glasso,][]{friedman2008sparse} for the (inverse) covariance induced by Equation~\eqref{eq:poisson_discrete} may fail because: 1) both $d_1$ and $d_2$ may be large (and as a result, $d=d_1d_2$ is large) for large spatial fields/domains; 2) ignoring the Kronecker structure may lead to (statistical) inefficiency of the method, i.e., the estimator not converging to the estimand; 3) ignoring the generative process in Equation~\eqref{eq:poisson_discrete} will result in learned structures that are not easily (physically) interpretable.

To address these issues in learning second-order representations for multiway (tensor) data, (sparse) Kronecker product (KP) or Kronecker sum (KS) decomposition of $\mat{\Sigma}$ or $\mat{\Omega}$ are often employed. Statistical models and corresponding learning algorithms can be derived using generative models or matrix approximations. The former include: KGlasso/Tlasso~\citep{tsiligkaridis2013convergence,lyu2019tensor} for estimating $\mat{\Omega}=\mat{A} \otimes \mat{B}$, using a autoregressive representation $\mat{A}\mat{X}\mat{B} = \mat{Z}$ for data $\mat{X}$ when $\mat{Z}$ is white noise. Another generative model is SyGlasso/SG-PALM (Chapter~\ref{ch:syglasso},\ref{ch:sgpalm}) that models the precision matrix as $\mat{\Omega}=(\mat{A} \oplus \mat{B})^2$, which corresponds to assuming the data $\mat{X}$ obeys a Sylvester equation $\mat{X} \mat{A} + \mat{B} \mat{X} = \mat{Z}$. Matrix approximation methods include: KPCA~\citep{tsiligkaridis2013covariance,greenewald2015robust} that approximates the covariance matrix as $\mat{\Sigma}=\sum_{i=1}^l \mat{A}_i \otimes \mat{B}_i$, i.e., low separation rank $l$. Another matrix approximation method is the TeraLasso~\citep{greenewald2019tensor} that models the precision matrix as $\mat{\Omega}=\mat{A} \oplus \mat{B}$. TeraLasso is equivalent to approximation of the conditional dependency graph (encoded by the precision matrix) with a Cartesian product of smaller graphs~\footnote{Note that Tlasso, TeraLasso, Syglasso/SG-PALM are generalizable to precision matrices of the form $\bigotimes_{k=1}^K \bm\Psi_k$, $\bigoplus_{k=1}^K \bm\Psi_k$, and $(\bigoplus_{k=1}^K \bm\Psi_k)^2$, respectively, for $K \geq 2$.}. 

All of KGlasso/Tlasso, TeraLasso, and SyGlasso/SG-PALM can be formulated using a penalized Gaussian likelihood approach. Here, we give a brief review of penalized Gaussian graphical models for multiway, tensor-valued data. A random tensor $\tensor{X} \in \bbR^{d_1 \times \dots \times d_K}$ follows the tensor normal distribution with zero mean when $\vecto(\tensor{X})$ follows a normal distribution with mean $\mat{0} \in \bbR^d$ and precision matrix $\bm\Omega := \bm\Omega(\bm\Psi_1,\dots,\bm\Psi_K)$, where $d=\prod_{k=1}^K d_k$. Here, $\bm\Omega(\bm\Psi_1,\dots,\bm\Psi_K)$ is parameterized by $\bm\Psi_k \in \bbR^{d_k \times d_k}$ via either Kronecker product, Kronecker sum, or the Sylvester structure, and the corresponding negative log-likelihood function (assuming $N$ independent observations $\tensor{X}^i, i=1,\dots,N$)
\begin{equation}
     -\frac{N}{2} \log|\bm\Omega| + \frac{N}{2}\tr(\mat{S}\bm\Omega),
\end{equation}
where $\bm\Omega = \bigotimes_{k=1}^K \bm\Psi_k$, $\bigoplus_{k=1}^K \bm\Psi_k$, or $\Big(\bigoplus_{k=1}^K \bm\Psi_k\Big)^2$ for KP, KS, and Sylvester models, respectively; and $\mat{S} = \frac{1}{N}\sum_{i=1}^N \vecto(\tensor{X}^i) \vecto(\tensor{X}^i)^T$. For $K=1$, this formulation reduces to the vector normal distribution with zero mean and precision matrix $\bm\Psi_1$.

To encourage sparsity in the high-dimensional scenario, penalized negative log-likelihood function is proposed
\begin{equation*}
     -\frac{N}{2} \log|\bm\Omega| + \frac{N}{2}\tr(\mat{S}\bm\Omega) + \sum_{k=1}^K P_{\lambda_k}(\bm\Psi_k),
\end{equation*}
where $P_{\lambda_k}(\cdot)$ is a penalty function indexed by the tuning parameter $\lambda_k$ and is applied elementwise to the off-diagonal elements of $\bm\Psi_k$. Popular choices for $P_{\lambda_k}(\cdot)$ include the lasso penalty~\citep{tibshirani1996regression}, the adaptive lasso penalty~\citep{zou2006adaptive}, the SCAD penalty~\citep{fan2001variable}, and the MCP penalty~\citep{zhang2010nearly}. 

To further reduce computational complexity and improve robustness, the Sylvester models in SyGlassoi/SG-PALM consider the penalized negative log-pseudolikelihood
\begin{equation}
  \begin{aligned}
    \mathcal{L}_{\mat\lambda}(\mat\Psi)
    = & -\frac{N}{2} \log | (\bigoplus_{k=1}^K \diag(\mat\Psi_k))^2| \\
    & + \frac{N}{2}\tr(\mat{S}\bm\Omega) + \sum_{k=1}^K P_{\lambda_k}(\bm\Psi_k).
  \end{aligned}
\end{equation}
This differs from the true penalized Gaussian negative log-likelihood in the exclusion of off-diagonals of $\mat\Psi_k$'s in the log-determinant term. It is motivated and derived directly using the Sylvester equation, from the perspective of solving a sparse linear system (see Chapters~\ref{ch:sgpalm} and \ref{ch:sgpalm} for details). This maximum pseudolikelihood estimation procedure has been applied to vector-variate Gaussian graphical models (see \citet{khare2015convex} and references therein). 

Lastly, the matrix approximation approach of \cite{tsiligkaridis2013covariance} to multiway covariance estimation is based on the representation $\mat{\Sigma}=\sum_{i=1}^l \mat{A}_i \otimes \mat{B}_i$. The representation is universal: any square matrix can be represented as a sum of $l$ Kronecker products for sufficiently large $l\leq \min\{d_1^2,d_2^2\}$, as shown in \citet{van1993approximation}. The Kronecker components can be obtained via a penalized optimization approach for estimating a rank $l$ Kronecker product decomposition of the sample covariance $\mat{S}$, i.e., 
$$
\min_{\{\bA_i,\bB_i\}} \left\| \bS-\sum_{i=1}^l \bA_i \otimes \bB_i \right\|^2_F+\lambda\left\|\sum_{i=1}^l \bA_i \otimes \bB_i \right\|_*
$$
for a user-supplied regularization parameter $\lambda>0$. The solution to this penalized optimization is specified by the first $l$ principal components of the  singular value decomposition (SVD) of $\calR(\bS)$ where $l$ is determined by $\lambda$ through a soft-thresholding of the SVD spectrum. In analogy to the ordinary PCA algorithm, the soft-thresholding SVD solution to this optimization problem was called Kronecker PCA (KPCA) in \cite{greenewald2014kronecker}.

\section{Penalized Multiway Ensemble Kalman Filter}\label{sec:enkf-method}
The proposed multiway ensemble Kalman filter, whose pseudo code is shown in Algorithm~\ref{alg:enkf}, modifies the EnKF by using a forecast (inverse) covariance estimator $\widehat{\bSigma}^f_t = (\widehat{\bOmega}^f_t)^{-1}$ obtained from one of the Kronecker-structured methods. From this, the correspondingly modified Kalman gain matrix is given by
\begin{equation}\label{eq:kalman_gain}
    \widehat{\bK}_t = \widehat{\bSigma}^f_t \bH^T (\bH \widehat{\bSigma}^f_t \bH^T + \bR)^{-1} = ((\widehat{\bOmega}^f_t)^{-1} + \bH^T \bR^{-1} \bH)^{-1} \bH^T \bR^{-1}.
\end{equation}

As the observations are assimilated into the EnKF through the ensemble update, which depends linearly on the Kalman gain matrix (as outlined in Algorithm~\ref{alg:enkf}), an accurate estimate of the true $\bK_t$ ensures that data is properly incorporated into the forecast ensemble. \citet[Theorem 1,][]{hou2021penalized} argues that the estimator $\widehat{\bK}_t$ via a Glasso covariance estimator is asymptotically consistent with the true Kalman gain matrix, given conditions on the state ensembles and the regularization parameters in the covariance estimation procedure. This nice property is a result of convergence of the Glasso-type sparse covariance estimator. Here, the Kronecker-structured estimators outlined in the previous section all enjoy faster rates of convergence, under appropriate regularity conditions. In Table~\ref{tab:enkf_guarantees}, we summarize the theoretical guarantees on the (inverse) covariance estimators, and hence the estimator $\widehat{\bK}_t$ under different modeling assumptions, i.e., KP, KS, Sylvester. All estimators have a similar $\sqrt{\log d / N}$ factor. However, comparing to Glasso that has an additional $\sqrt{d+s}$ factor, the Kronecker-structured estimators have additional factors that depend on the smaller $d_k$'s (assuming the number of tensor modes remain constant), that is, $\sqrt{l \sum_k d_k^2}$ for matrix approximation based estimators~\footnote{Here, $l$ indicates the separation rank.}, $\sqrt{\sum_k d_k}$ for KP based inverse covariance estimators, $\sqrt{(d+s)/\min_k m_k}$ for KS based inverse covariance estimators, and $\max_k \sqrt{s_k d_k}$ for the Sylvester estimators. These indicate improved theoretical accuracy on estimating the state (inverse) covariance and the Kalman gain matrix.

\begin{algorithm}[!tbh]
\begin{minipage}{0.9\linewidth}
\begin{algorithmic}
\caption{Multiway Ensemble Kalman Filter}\label{alg:enkf}
\REQUIRE Initial ensemble $\widehat{\bx}^{(1)}_{0},\dots,\widehat{\bx}^{(N)}_{0}$, observations at each time $\by_t$, measurement operator $\mat{H}$,  state and observation noise covariance matrices $\mat{Q}$ and $\mat{R}$
\FOR{$t = 1,\dots,T$}
    \FOR{$i = 1,\dots,N$}
        \STATE \textit{Forecast Step:} 
        Evolve each ensemble member forward in time via $\widetilde{\bx}^{(i)}_t = f(\widehat{\bx}^{(i)}_{t-1}) + \bfw^{(i)}$, with $\bfw^{(i)} \sim \mathcal{N}_d(\mat{0}, \mat{Q})$
        
        \STATE \textit{Multiway Covariance Estimation:} 
        Estimate the (inverse) covariance via and compute the Kalman gain matrix $\widehat{\bK}_t$ via \eqref{eq:kalman_gain} 
        
        \STATE \textit{Update Step:} 
        Update the ensemble with the observations by computing $\widehat{\bx}^{(i)}_t = \widetilde{\bx}^{(i)}_t + \widehat{\bK}_t(\by_t + \bfv_t^{(i)} - \bH \Tilde{\bx}^{(i)}_t)$, where $\bfv_t^{(i)} \sim \mathcal{N}_r(\mat{0}, \mat{R})$
    \ENDFOR
\ENDFOR
\ENSURE Final ensemble $\widehat{\bx}^{(1)}_{T},\dots,\widehat{\bx}^{(N)}_{T}$.
\end{algorithmic}
\end{minipage}
\end{algorithm}

In terms of computational complexity, all the structured precision estimation algorithms are variants of Glasso, implemented with techniques tailored to the model assumptions for speedup. Generally speaking, the resulting complexity consists of the mode-wise complexity ($d_k^3$) and the cost of updating the objective: $dK$ for TeraLasso, $Nd$ for KGlasso, and $N \sum_k \sum_{j \neq k} d_j m_j^2$ for SG-PALM. The mode-wise complexity of TeraLasso is dominated by matrix inversion, which is hard to scale for general problem instances. For KGlasso, the mode-wise complexity is the same as that of running a Glasso-type algorithm for each mode, which could be improved by applying state-of-the-art optimization techniques developed for vector-variate Gaussian graphical models such as \citet{hsieh2013big}. For SG-PALM, the mode-wise operations involve only sparse-dense matrix multiplications, which could be improved to $O(d_k \cdot \textsf{nnz})$, where $\textsf{nnz}$ counts the number of non-zero elements of the sparse matrix (i.e., the estimated $\bm\Psi_k$ at each iteration). This could greatly reduce the computational cost for extremely sparse $\bm\Psi_k$, e.g., with only $O(d_k)$ non-zero elements. Further, KGlasso and SG-PALM both incur a cost of the type $O(N d_k m_k^2)$ for each mode-wise update. This can also be reduced to be $\approx d$ for sparse estimated $\bm\Psi_k$'s at each iteration. Overall, for sample-starved setting where we only have access to a handful of data samples, structured KP and KS models run similarly fast, while the Sylvester GM runs slower theoretically due to the extra and richer structures that it takes into account. The matrix approximation based estimation procedure, KPCA, is in general computationally more expensive than other KP, KS, and Sylvester based methods. This is mostly due to the absence of the sparsity structure in the latter model, as well as as SVD step involved during the its estimation algorithm. There exist faster randomized methods for truncated SVD~\citep{halko2011finding}. Thus, it still scales well for moderately high-dimensional applications. In Table~\ref{tab:enkf_runtime} wall-clock runtimes of EnKF integrated with theses (inverse) covariance estimation algorithms are compared under various settings. The table confirms the aforementioned theoretical computational complexities. 

\begin{table}
\centering
\caption{Comparison of theoretical guarantees on sample complexity (statistical error) and computational complexity of various precision / covariance estimators. Here, $M = \max\{d_1, d_2, N\}$, $m_k = \prod_{i \neq k} d_i$ is the co-dimension of the $k$-th mode, $d = \prod_{k=1}^K d_i$, and $s_k$ characterizes the sparsity of each of the inverse covariance Kronecker factors $s_k = |\{(i,j): i \neq j, [\bm\Psi_k]_{i,j} \neq 0\}|$, $s$ is the sparsity of the full inverse covariance $s = |\{(i,j): i \neq j, \bm\Omega_{i,j} \neq 0\}|$ and $s = \sum_{k=1}^K m_k s_k$ if $\bOmega$ satisfies the Kronecker sum model.
}
\label{tab:enkf_guarantees}
\resizebox{\columnwidth}{!}{
\begin{tabular}{ l c c c}
\toprule
  Model & Algorithm & Statistical Error & Computational Complexity \\
\midrule
  Sparse-Precision
    &   Glasso~\citep{friedman2008sparse} & $O_P\Big( \sqrt{\frac{(d+s) \log d}{N}} \Big)$ & $O(d^3)$\\
\midrule
  KP-Covariance
    &   Robust KPCA~\citep{greenewald2015robust} & $O_P\Big( \sqrt{\frac{l(d_1^2+d_2^2+\log M)}{N}} \Big)$ & $O(l d^2)$\\
\midrule
  KP-Precision
    &   KGlasso~\citep{tsiligkaridis2013convergence} & $O_P\Big(\sqrt{\frac{(d_1+d_2)\log M}{N}}\Big)$ & $O(d_1^3 + d_2^3 +Nd)$\\
\midrule
  KS-Precision
    &   TeraLasso~\citep{greenewald2019tensor} & $O_P\Big(\sqrt{K+1}\cdot \sqrt{\frac{(d+s) \log d}{N \min_k m_k}}\Big)$ & $O(dK+\sum_{k=1}^K d_k^3)$ \\
\midrule
  Sylvester GM
    &   SG-PALM~\citep{wang2021sg} & $O_P\Big(\sqrt{K}\cdot \max_k \sqrt{\frac{s_k d_k \log d}{N}}\Big)$ & $O\Big(\sum_{k=1}^K (d_k^3 + N \sum_{j \neq k} d_jm_j^2) \Big)$ \\
\bottomrule
\end{tabular}
}
\end{table}

\begin{landscape}
\begin{table}[tbh!]
\centering
\caption{Runtime (in seconds) of $20$ time steps of EnKF tracking using various (inverse) covariance estimation algorithms. Comparisons under various problem sizes (i.e., different $d$ and $N$) and two observation types (i.e., fully observed or partially observed) are shown. Note the sparse multiway precision models (SG-PALM, KGlasso, TeraLasso) are comparably fast and are all faster than Glasso (for large problems) and KPCA.}
\label{tab:enkf_runtime}
\begin{tabular}{|p{0.5cm}|p{0.5cm}|p{0.6cm}||r|r|r|r|r|}
 \multicolumn{8}{c}{} \\
 \hline
 \multirow{2}{*}{$d$} & \multirow{2}{*}{$N$} & \multirow{2}{*}{obs.} & \textbf{Glasso} & \textbf{SG-PALM} & \textbf{TeraLasso} & \textbf{KGlasso} & \textbf{KronPCA}\\
 \cline{4-8} 
 &&& \textbf{sec} & \textbf{sec} & \textbf{sec} & \textbf{sec} & \textbf{sec} \\
 \hline
 \multirow{6}{*}{$32^2$} & \multirow{2}{*}{$25$} & full &
$14.52 (0.20)$ & $18.03 (0.15)$ & $17.06 (0.35)$ & $18.60 (0.11)$ & $80.88 (0.20)$ \\
 && part &
$10.95 (0.01)$ & $13.94 (0.05)$ & $11.02 (0.03)$ & $11.88 (0.10)$ & $53.89 (0.88)$ \\
 \cline{2-8}
 & \multirow{2}{*}{$50$} & full & 
$24.88 (0.21)$ & $31.02 (0.08)$ & $27.53 (0.50)$ & $33.25 (0.10)$ & $92.76 (0.55)$ \\
 && part &
$18.95 (0.09)$ & $25.37 (0.53)$ & $18.63 (0.02)$ & $28.63 (0.11)$ & $66.12 (0.56)$ \\
 \cline{2-8}
 & \multirow{2}{*}{$100$} & full & 
$48.23 (0.05)$ & $58.63 (0.13)$ & $49.62 (0.28)$ & $67.53 (0.30)$ & $131.43 (1.05)$\\
 && part &
$36.41 (1.05)$ & $41.80 (0.25)$ & $36.80 (0.55)$ & $53.77 (0.14)$ & $78.74 (1.13)$ \\
 \cline{1-8}
 \multirow{6}{*}{$64^2$} & \multirow{2}{*}{$25$} & full &
$288.13 (2.45)$ & $207.71 (1.09)$ & $195.92 (0.58)$ & $217.71 (1.99)$ & $2562.19 (2.69)$ \\
 && part &
$281.96 (2.04)$ & $191.56 (1.21)$ & $185.19 (0.88)$ & $193.67 (3.09)$ & $2593.48 (3.89)$ \\
 \cline{2-8}
 & \multirow{2}{*}{$50$} & full & 
$489.82 (0.98)$ & $353.10 (1.81)$ & $222.09 (1.98)$ & $370.11 (2.00)$ & $4535.72 (2.19)$ \\
 && part &
 $422.94 (0.72)$ & $287.34 (1.90)$ & $277.78 (2.09)$ & $290.05 (0.56)$ & $3890.22 (1.96)$ \\
 \cline{2-8}
 & \multirow{2}{*}{$100$} & full & 
 $734.72 (2.01)$ & $529.65 (1.10)$ & $499.60 (0.72)$ & $555.17 (0.57)$ & $6522.67 (4.01)$ \\
 && part &
$507.53 (0.91)$  & $344.81 (1.90)$ & $333.24 (2.89)$ & $348.61 (3.90)$ & $4668.26 (2.67)$ \\
 \hline
\end{tabular}
\end{table}
\end{landscape}

\section{Numerical Experiments}\label{sec:enkf-experiments}
We describe three dynamic models that extend the spatial Poisson equation described in Section~\ref{sec:enkf-diffusion} to incorporate temporal dynamics, and the resulting multiway (inverse) covariance structure. These models will be used in our numerical experiments to generate data to demonstrate the performance of the proposed multiway ensemble Kalman filter algorithm.

\paragraph{Poisson-AR(1) Process.}
The first extension, which we call the Poisson-AR(1) process, imposes an autoregressive temporal model of order 1 on the source function $f$ in the Poisson equation (\ref{eq:Poisson}). Specifically, we say a sequence of discretized spatial observations $\{\bU^k \in \bbR^{d_1\times d_2}\}_k$ indexed by time step $k=1,\cdots,T$ is from a Poisson-AR(1) process if
\begin{equation*}\label{eq:ar1-discrete}
    \begin{aligned}
    &(\mat{A}_{d_1} \oplus \mat{A}_{d_2}) \vecto(\mat{U}^k) = \vecto(\mat{Z}^k), \\
    &\vecto(\mat{Z}^k) = a \vecto(\mat{Z}^{k-1}) + \vecto(\mat{W}^k),\quad |a|<1,   
    \end{aligned}
\end{equation*}
where $\{\mat{W}^k \in \bbR^{d_1\times d_2}\}_k$ is spatial white noise, i.e., $W_{i,j}^k \sim \mathcal{N}(0, \sigma^2_w)$, i.i.d.

\paragraph{Convection-diffusion Process.} 
The second time-varying extension of the Poisson PDE model (\ref{eq:Poisson}) is based on the convection-diffusion (C-D) process~\citep{chandrasekhar1943stochastic}
\begin{equation}\label{eqn:convec-diff}
    \pdv{u}{t} = \theta \sum_{i=1}^2 \pdv[2]{u}{x_i} - \epsilon \sum_{i=1}^2 \pdv{u}{x_i}.
\end{equation}
Here, $\theta > 0$ is the diffusivity; and $\epsilon \in \Reals$ is the convection velocity of the quantity along each coordinate. Note that for simplicity of discussion here, we assume these coefficients do not change with space and time (see, \citet{stocker2011introduction}, for example, for a detailed discussion). These equations are closely related to the Navier-Stokes equation commonly used in stochastic modeling for weather and climate prediction~\citep{chandrasekhar1943stochastic,stocker2011introduction}. Coupled with Maxwell's equations, these equations can be used to model magneto-hydrodynamics~\citep{roberts2006slow}, which characterize solar activities including flares.


A solution of Equation~\eqref{eqn:convec-diff} can be approximated similarly as in the Poisson equation case, through a finite difference approach. Denote the discrete spatial samples of $u(\mat{x},t)$ at time $t_k$ as a matrix $\mat{U}^k\in\bbR^{d_1 \times d_2}$. We obtain a discretized update propagating $u(\mat{x},t)$ in space and time, which locally satisfies
\begin{equation}\label{eqn:convec-diff-discrete}
\begin{aligned}
    \frac{U_{i,j}^k - U_{i,j}^{k-1}}{\Delta t} = &\ \theta \left(\frac{U_{i+1,j}^k + U_{i-1,j}^k + U_{i,j+1}^k + U_{i,j-1}^k - 4U_{i,j}^k}{h^2}\right) \\
    &- \epsilon \left(\frac{U_{i+1,j}^k - U_{i-1,j}^k + U_{i,j+1}^k - U_{i,j-1}^k}{2h}\right),
\end{aligned}
\end{equation}
where $\Delta t = t_{k+1} - t_{k}$ is the time step and $h$ is the mesh step (spatial grid spacing). Then, the temporal update of $\mat{U}^k$ can be shown to obey the Sylvester matrix update equation~\citep{thomas2013numerical} $\mat{A}_{d_1}\mat{U}^k + \mat{U}^k\mat{A}_{d_2}^T = \mat{U}^{k-1}$,
or equivalently,
\begin{align}
    (\mat{A}_{d_2} \oplus \mat{A}_{d_1})\vecto({\mat{U}}^k) =\vecto({\mat{U}}^{k-1}),
    \label{eq:CD_discrete_update}
\end{align}
where $\mat{A}_{d_1} = \mat{A}_{d_1}(\theta,\epsilon,h,\Delta t)$ and $\mat{A}_{d_2} = \mat{A}_{d_2}(\theta,\epsilon,h,\Delta t)$ are symmetric tridiagonal matrices whose entries depend on $\theta, \epsilon$, $\Delta t$ and $h$ \citep{grasedyck2004existence}. 



\paragraph{Kuramoto-Sivashinsky Process.} The third extension of a spatial diffusion process is the Kuramoto-Sivashinsky (K-S) equation, which is a class of non-linear fourth-order PDEs known to exhibit chaotic behaviors~\citep{hyman1986kuramoto}. Specifically, the K-S equation in a 2D spatial domain can be written as
\begin{equation*}
    u_t + \Delta u + \Delta^2 u + \frac{1}{2}|\nabla u|^2 = 0,
\end{equation*}
or equivalently,
\begin{equation}
    \pdv{u}{t} + \sum_{i=1}^2 \pdv[2]{u}{x_i} + \sum_{i=1}^2 \pdv[4]{u}{x_i} + \frac{\partial^4 u}{\partial x_1^2 \partial x_2^2} + \sum_{i=1}^2 \Big(\pdv{u}{x_i} \Big)^2.
\end{equation}
Here, although we can similarly apply finite difference approximation to the differential operators, the equation is non-linear and simple linear algebraic update like in the Poisson-AR and convection-diffusion cases is not available.

For numerical illustrations, we consider a 2D spatio-temporal process of dimension $64 \times 64$ where only half of the entries are observed, which leads to a measurement matrix $\mat{H} \in \{0,1\}^{2048 \times 4096}$. We generated the true states and the corresponding observations according to Poisson-AR(1), convection-diffusion, and Kuramoto-Sivashinsky dynamics for $T = 50$ time steps. Several realizations of the true state variables are shown in Figure~\ref{fig:enkf_states} of Appendix~\ref{app:enkf}. At each time step, we generated an ensemble of size $N = 15$ and estimated the state covariance / inverse covariance using several sparse (multiway) inverse covariance estimation methods, including Glasso~\citep{friedman2008sparse}, KPCA~\citep{greenewald2015robust}, KGlasso~\citep{tsiligkaridis2013convergence}, TeraLasso~\citep{greenewald2019tensor}, SG-PALM~\citep{wang2021sg}. 

Figure~\ref{fig:rmse_linear_enkf} shows evolution of the computed root mean squared errors (RMSEs) for the estimated states under the Poisson-AR (left panel) and the convection-diffusion (right panel) processes across all ensemble members. It is noted that SG-PALM, which corresponds to the statistical method that models the inverse covariance as a squared Kronecker sum, performs the best under the Poisson-AR generating process. In Figure~\ref{fig:inv_cov_struct} (a) we show the true and estimated (inverse) covariance matrices obtained at the last time step -- at each time step the multiway EnKF involves estimation of a sparse Kronecker sum squared inverse covariance matrix induced by the Poisson-AR process. Hence, the SG-PALM method operates under the correct model assumption in this situation. On the other hand, the KPCA method outperforms other methods as time progresses. This is due to the fact that the inverse covariance structure under the convection-diffusion dynamics model is dense due to the smoothing nature of the Kalman filtering algorithm. But, its steady-state covariance has low-dimensional structures as shown in Figure~\ref{fig:inv_cov_struct} (b). The KPCA in this case was able to approximate this structure reasonably well as it does not impose any sparsity on the precision matrix. Remarks~\ref{rmk:dense-inv-cov} and \ref{rmk:blocked-inv-cov} below further discuss this emergence of dense precision matrix for the marginal spatial process. Appendix~\ref{app:enkf} illustrates situations where the joint spatio-temporal precision matrix is sparse. Note in this case, the EnKF with Glasso converges slower than the multiway methods.

\begin{remark}\label{rmk:dense-inv-cov}
Although the state variable following the convection-diffusion dynamics evolves via a Sylvester equation, similar to the Poisson-AR case, the state (inverse) covariance matrix at time step $t_k$ admits different structures. Specifically, the state precision matrix $\mat\Omega^k=\cov^{-1}(\vecto({\mat{U}^k})) \in \Reals^{d_1d_2 \times d_1d_2}$ evolves as $\mat\Omega^k = (\mat{A}_{d_1} \oplus \mat{A}_{d_2}) \mat\Omega^{k-1} (\mat{A}_{d_1} \oplus \mat{A}_{d_2}) + \sigma^{-2}_w\mat{I}$ (see \citet{katzfuss2016understanding}, for example). This matrix is not necessarily sparse for finite $k$ but, assuming that the eigenvalues of the matrix $\mat{A}_{d_1} \oplus \mat{A}_{d_2}$ are in $(-1,1)$, the limiting precision matrix  $\mat\Omega^{\infty}=\lim_{k\rightarrow\infty} \mat\Omega^k$ is $\mat\Omega^{\infty} = (\mat{A}_{d_1} \oplus \mat{A}_{d_2}) \mat\Omega^{\infty} (\mat{A}_{d_1} \oplus \mat{A}_{d_2}) + \sigma^{-2}\mat{I}$. The $\mat\Omega^{\infty}$ matrix is sparse because $\mat{A}_{d_1}$ and $\mat{A}_{d_2}$ are both tridiagonal.
\end{remark}

\begin{figure}[!tbh]
    \centering
    \includegraphics[width=0.45\textwidth]{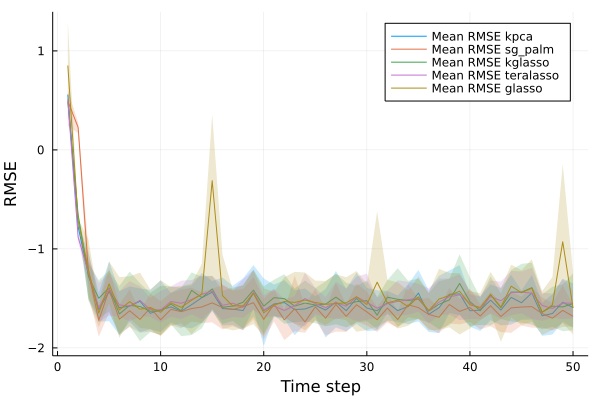}
    \includegraphics[width=0.45\textwidth]{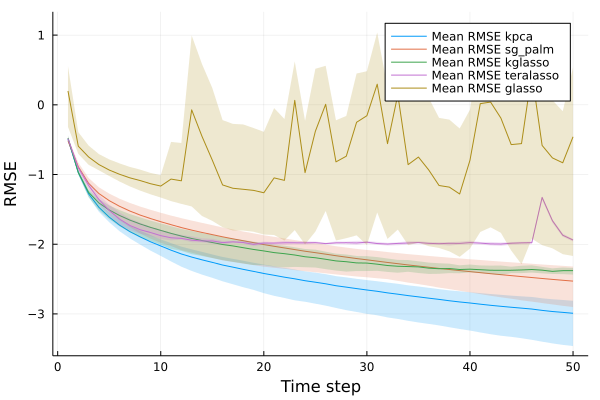}
    \caption{RMSEs of the estimated states via EnKF over $50$ time steps using different (inverse) covariance estimators. The $95\%$ posterior interval for RMSEs over all ensemble members are shown here with the posterior mean highlighted using solid lines. Here, each state is of dimension $64 \times 64$ and is generated via either a convection-diffusion (right) or Poisson-AR(1) equation (left). The best performers in terms of mean RMSE over all ensemble members are KPCA for convection-diffusion and SG-PALM for Poisson-AR(1).}
    \label{fig:rmse_linear_enkf}
\end{figure}

\begin{figure}[!tbh]
    \centering
    \begin{subfigure}{\textwidth}
    \centering
    \includegraphics[width=0.45\textwidth]{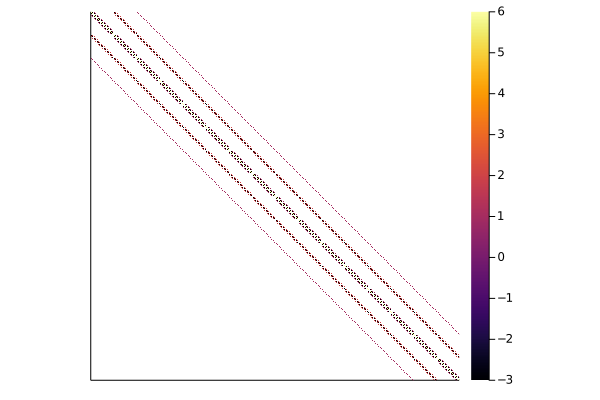}
    \includegraphics[width=0.45\textwidth]{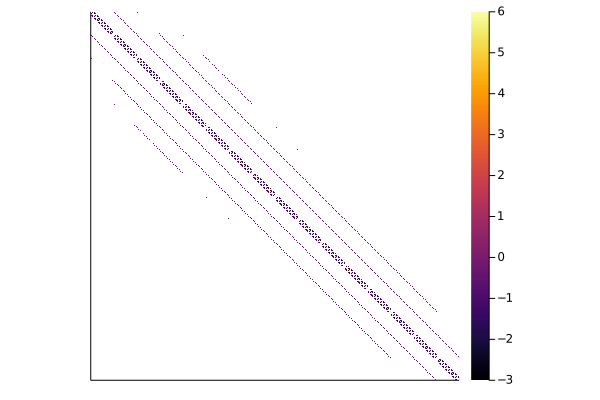}
    \caption{Poisson-AR inverse covariance structure (left) and the estimate obtained by SG-PALM (right) at the last time step.}
    \end{subfigure}
    \begin{subfigure}{\textwidth}
    \centering
    \includegraphics[width=0.45\textwidth]{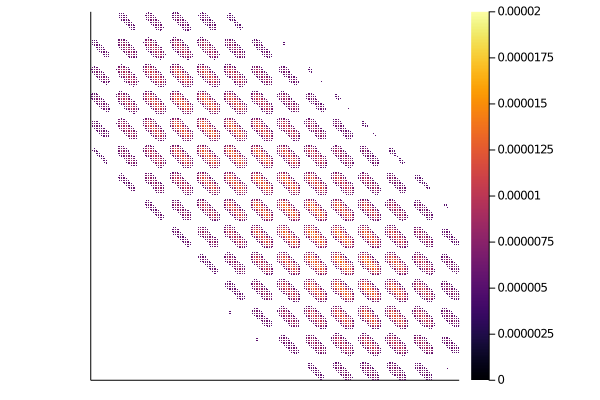}
    \includegraphics[width=0.45\textwidth]{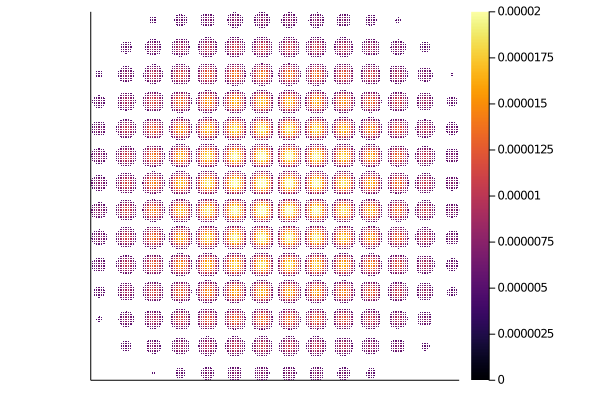}
    \caption{Convection-diffusion covariance structure (left) and the estimate obtained by KPCA at the last time step.}
    \end{subfigure}
    \caption{Covariance/precision structures for Poisson-AR and convection-diffusion dynamics and their estimates. Here, white/blank entries indicate zeros in the (inverse) covariance matrix. For Poisson-AR dynamics the Sylvester graphical model approximately matches the true structure of the precision matrix. For convection-diffusion dynamics the covariance instead of the precision matrix is structured and sparse.}
    \label{fig:inv_cov_struct}
\end{figure}

Tacking the Poisson-AR and convection-diffusion dynamics with EnKF both involve sparse (on either the covariance or its inverse) and tractable linear updates. The Kuramoto-Sivashinsky dynamical model will similarly involve sparse updates if finite difference approximations are employed for solving the PDE because the discretized differential operators will always be sparse. But, the non-linear nature of the problem makes the update intractable. Moreover, the KS equation is known to generate chaotic behaviors, making it a more realistic benchmark model for real-world systems. Here, two of the best performers under the Possion-AR and convection-diffusion dynamics (SG-PALM and KPCA) are compared against the ensemble transform Kalman filter (ETKF) and its localized version, a method known to work well for tracking high-dimensional highly non-linear systems with limited ensemble size. It has been successfully applied for data assimilation of, for example, the solar photospheric magnetic flux, which are fundamental drivers for simulations of the corona and solar wind~\citep{hickmann2015data}. Figure~\ref{fig:sgpalm_states} (a) shows that the proposed multiway EnKF outperforms the (local) ETKF. The KPCA based estimator outperforms the SG-PALM based estimator as time progresses, likely due to a similar reason discussed previously -- the inverse covariance structure becomes denser and denser, making the sparse models less appealing. The local ETKF performs similarly well as the non-local version but facilitates parallel estimation schemes where the ``local patches'' of the state variable can be updated and evolved simultaneously. Figure~\ref{fig:sgpalm_states} (b) visualizes the true and estimated KS states by SG-PALM multiway EnKF at several timestamps. It shows that the proposed method can correct the noisy observations (with missing values) and recovers the true states reasonably well.

\begin{figure}[!tbh]
\centering
    \begin{subfigure}{0.7\textwidth}
    \centering
    \includegraphics[width=\textwidth]{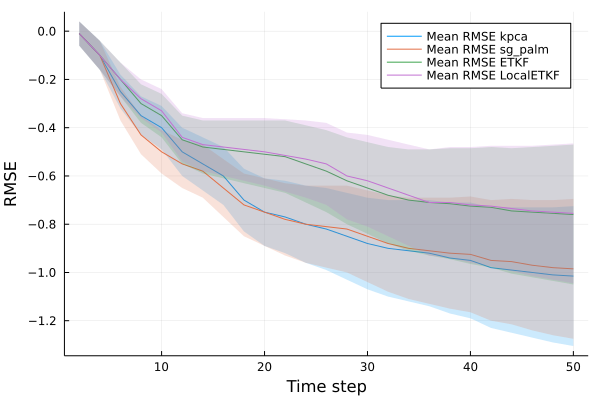}
    \caption{Root mean squared errors in $\log10$ scale for state estimated by EnKF variants. The solid lines and the shaded areas indicate the posterior mean and the $95\%$ posterior interval over all ensemble estimates.}
    \end{subfigure}
    \\
    \begin{subfigure}{\textwidth}
    \centering
    \includegraphics[width=\textwidth]{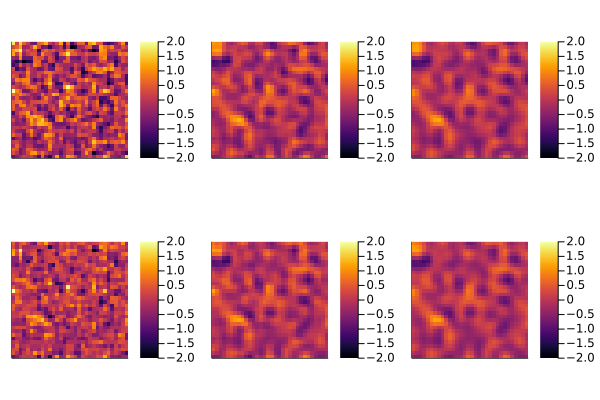}
    \caption{Comparison between true states (top) and estimated states by EnKF with the SG-PALM inverse covariance estimator (bottom) at several time stamps.}
    \end{subfigure}
    \caption{Visualizations of the performances by various EnKF methods for tracking the Kuramoto-Sivashinsky system. The proposed multiway EnKF outperforms the ETKF and its localized version.}
    \label{fig:sgpalm_states}
\end{figure}

\begin{remark}
The Sylvester matrix equations (and hence the sparse Kronecker structures) arise when the finite-difference discretization is performed on a rectangular grid. The relations~\eqref{eq:ar1-discrete} and \eqref{eq:CD_discrete_update} might not hold for finite-difference on, for example, spherical coordinates, as well as approximations to the equations using other types of methods, such as finite volume, finite element, and spectral methods.
\end{remark}

\begin{remark}\label{rmk:blocked-inv-cov}
Although the precision matrix of the state ensemble becomes dense as the temporal update progresses, making sparse Kronecker-structured methods less appealing as illustrated in Figure~\ref{fig:rmse_linear_enkf} and \ref{fig:inv_cov_struct}, if we consider ``temporal blocks'' of states then the precision matrix remains sparse. Appendix~\ref{app:enkf} includes detailed derivations of the blocked versions of the Poisson-AR and convection-diffusion dynamics, and illustrates the performances of the Kronecker-structured models under these scenarios.
\end{remark}

\section{Conclusions}\label{sec:enkf-conclude}
Spatiotemporal PDEs are prominent techniques for modeling real-world physical systems. In this chapter, we introduced a multiway ensemble Kalman filtering framework that integrates the powerful ensemble Kalman filters with state-of-the-art Kronecker-structured covariance/precision models. The resulting framework allows one to track simulated complex, potentially chaotic systems. One such system in the real world arises in space physics, where solar flares and coronal mass ejections are associated with rapid changes in field connectivity and are powered by partial dissipation of electrical currents in the solar atmosphere~\citep{schrijver2008nonlinear}. The nonlinear force-free field model is often used to describe the solar coronal magnetic field~\citep{derosa2015influence,wheatland2013state} and can be derived from the convection-diffusion process described in this work. Additionally, global maps of the solar photospheric magnetic flux are fundamental drivers for simulations of the corona and solar wind. However, observations of the solar photosphere are only made intermittently over approximately half of the solar surface. \citet{hickmann2015data} introduced the Air Force Data Assimilative Photospheric Flux Transport model that uses localized ensemble transform Kalman filtering to adjust a set of photospheric simulations to agree with the available observations. In future work, we plan to incorporate our proposed multiway EnKF framework for tracking these solar physical systems. 


%% file: Chap5/chap5.tex
A simple and scalable framework for longitudinal analysis of text data is developed that combines latent topic models with computational geometric methods. Dimensionality reduction tools from computational geometry are applied to learn the intrinsic manifold on which the latent, temporal topics reside. Then shortest path distances on the manifold are used to link together these topics. The proposed framework permits visualization of the low-dimensional embedding, which provides clear interpretation of the complex, high-dimensional trajectories that may exist among latent topics. Practical application of the proposed framework is demonstrated through its ability to 1) capture and effectively visualize natural progression of latent COVID-19--related topics learned from Twitter data; 2) learn latent topics correspond to human-labeled data and ``generate'' novel latent topics from TalkLife -- a peer support network focused on mental health. Interpretability of the trajectories and the learned topics is achieved by comparing to real-world events and expert knowledge (e.g., labeled data). The analysis demonstrates that the proposed framework is able to 1) capture granular-level impact of COVID-19 on public discussions; and 2) learn mental health focused topic clusters that resemble human-level expert knowledge.

\section{Introduction}\label{sec:gdtm-intro}
The continued digitization of public discourse in news feeds, books, scientific reports, social media, blogs, microblogs, and web pages creates opportunities to discover meaningful patterns and trends of public opinion. Methods of probabilistic topic modeling have been used to extract  such patterns using a suite of algorithms that aim to automatically discover and annotate large collections of documents with thematic labels 
\citep{blei2012probabilistic}.
Topic modeling algorithms are computational methods that manipulate word frequencies in document corpora to discover the themes that run through them, quantify how those themes are connected to each other, and how they change over time.

\subsection{Probabilistic topic models and computational geometry}
A probabilistic topic model that has seen success in many applications is the latent Dirichlet allocation (LDA) model~\citep{blei2003latent}, which uses a latent topic model to extract thematic information from document corpora to infer an underlying generative process that explains hidden relationships among documents. Many real-world document corpora, however, have complex structure and include temporal information that is ignored by traditional LDA models. For example, discussions of COVID-19 on Twitter between February and May 2020 involve the emergence, evolution, and extinction of multiple topics over time. Moreover, data generated from emerging social media platforms, such as Twitter, Reddit, TalkLife, etc. are short bursts composed in micro-text (\cite{ellen2011all}), which traditional LDA models struggle to model effectively. Additionally, side information is commonly available such as document-level labels/tags or word-level features. For example, a significant proportion of the news articles on Reuters is labeled with multiple human-provided tags~\citep{ramage2009labeled}. Effectively incorporating these additional information is key to reliability and interpretability of many machine learning algorithms, including LDA and topic models.

Extensions of the standard LDA have been proposed to learn latent topics in the context of complex structure and temporal information. An early modeling strategy is to assume a temporally Markovian relationship where the state of the process at time $t + 1$ is independent of past history given the state at time $t$. \citet{blei2006dynamic} proposed the dynamic topic model (DTM) for modeling time-varying topics, where the topical-alignment over time is captured by a Kalman filter procedure. Further improvements have been in various directions, including: (1) relaxation of the Markov assumption, as discussed by \citet{wang2006topics}, who introduced a non-Markov continuous-time model called the topics-over-time (TOT) model, capturing temporal changes in the occurrence of the topics themselves, and (2) circumvent of time discretization, as proposed by \citet{wang2012continuous} that improved the DTM using a continuous time variant, called cDTM, formulated on Brownian motion to model the latent topics in a longitudinal collection of documents. These approaches rely on spatiotemporally coupled stochastic processes for modeling the evolution of topics over time. Such integrated models employ a global joint parameterization of time evolution and word co-occurrence, producing a unified generative probabilistic model for both temporal and topical dimensions. 

However, global parameterized DTMs have several deficiencies that motivate the model proposed in this article. The main issue is that global parameterization can increase the computational complexity of parametric inference. \citet{wang2012continuous} and \citet{blei2006dynamic} argued that applying Gibbs sampling to perform inference on DTMs is more difficult than on static models, principally due to the nonconjugacy of the Gaussian and multinomial distributions. As an alternative, they proposed the use of inexact variational methods, in particular, variational Kalman filtering and variational wavelet regression, for inference. These approximate inference procedures face two issues: 1) they usually involve assumptions on the correlation structures among latent variables, for example, mean-field, which undermines uncertainty quantification; 2) the resulting optimization problems are usually nonconvex, which means that the approximate posterior distribution found might only be locally optimal--trapping the topic parameters in a neighborhood of a local optima. An additional issue is that posterior inference via variational approximation usually relies on batch algorithms that need to scan the full data set before each update of the model. This increases the computational burden, especially for long time sequences, and parallel computing cannot be easily exploited~\citep{bhadury2016scaling}. Such issues can lead to numerical instability and lack of interpretability of the model predictions. Furthermore, incorporating side information, such as document-level labels, word-level features (e.g., word volumes), imposes additional challenges to dynamic topic modeling. \citet{hong2011tracking} proposed a variant of DTM for tracking topic trends, by incorporating word volumes and assuming these volumes are generated by the latent topics through a linear model. Similarly, \citet{park2015supervised} introduced a supervised DTM (sDTM), where a time-series of numerical values are assumed to be generated by the topic assignment distributions via a normal linear model. Both of these variants require the aforementioned Kalman filtering  and variational approximation procedures for inference, in addition to the extra modeling assumption on the side information.

Rather than jointly modeling word co-occurrence and the temporal dynamics, there exist alternatives that adopt simpler analysis strategies that motivate our proposed approach. Most of these approaches to nonglobal modeling involve fitting a local time-unaware topic model to predivided discrete time slices of data, and then examining the topic distributions in each time-slice in order to assemble topic trends that connect related topics~\citep{griffiths2004finding,wang2005group,malik2013topicflow,cui2011textflow}. A difficulty with these approaches is that aligning the topics from each time slice can be challenging, even though several strategies have been proposed. \citet{malik2013topicflow} proposed a framework to connect every pair of topics from adjacent time slices whose similarity, measured by the cosine metric, exceeds a certain threshold. \citet{cui2011textflow} used a semiparametric clustering algorithm to identify similar topics at adjacent time slices. However, these approaches suffer from an inherent inflexibility in modeling diverse dynamical structures that exist in a potentially large collection of temporal topic sequences. Such methods are developed to model and visualize specific, and relatively rare, types of temporal dynamics and are often not able to capture all types of variations, for example, anomalies, bifurcations, emergence, convergence, and divergence.

We propose a flexible and scalable computational geometry framework that remedies the above mentioned issues and complements the existing methods in the dynamic topic modeling toolbox. Specifically, in this article a time-evolving topic model is introduced that uses a local LDA-type model for discrete time slices of collections of documents, and a geometric proximity model to align the topics from time to time. In contrast to global parametric dynamic latent variable approaches to summarizing time-evolving unstructured texts, our framework offers a wrapper for a suite of tools. The proposed wrapper framework has the flexibility to allow any particular topic model to be applied locally to each time slice of documents. This allows any side information to be included via supervised/semi-supervised variants of LDAs (e.g., \citet{ramage2009labeled, mcauliffe2007supervised,petterson2010word,zhu2012medlda,lu2011multi}). It then implements a fast and scalable shortest path algorithm to stitch together the locally learned LDA topics into an integrated collection of temporal topic trends. 

To facilitate visualization and interpretation of the learned topic trends, an emphasis of this article, the proposed framework also implements a recent geometric embedding method called PHATE (Potential of Heat-diffusion for Affinity-based Trajectory Embedding) that projects the high-dimensional word distributions representing latent topics to lower dimensional coordinates. The PHATE embedding has been shown to preserve the intrinsic geometry of high-dimensional time-varying data~\citep{moon2019visualizing}, which provides a clear and intuitive visualization of any progressive structure that exists among the topics. We note that similar computational geometric representations of data have been used in unsupervised, semisupervised, and supervised learning, both as principal learning models and as supplementary regularizers of other models. In manifold learning, geometric affinity (or distance) between data points drives dimensionality reduction ~\citep{tenenbaum2000global,donoho2003hessian} and dimensionality estimation methods~\citep{costa2006determining}. Several deep learning architectures, like the deep k-nearest neighbors~\citep[DkNN]{papernot2018deep}, use interpoint distances and the kNN classifier to induce interlayer representational continuity and robustness against adversarial attacks. Semisupervised classification approaches adopt geometric measures over reproducing kernel Hilbert space (RKHS) to associate unlabeled data with labeled data in geometry-regularized empirical loss frameworks~\citep{belkin2004semi}. Geometry is the driver for many missing data models, for example, synthetic minority oversampling technique~\citep[SMOTE]{chawla2002smote} and more generally, nearest neighbor interpolations.

We point out that dimensionality reduction is the basis for latent semantic analysis (LSA) in computational linguistics. In particular,  \citet{doxas2010dimensionality} had a similar objective to ours, to explore temporal evolution of discourse, but in long text with labeled corpora. 
The authors constructed semantic spaces for various corpora, and then calculated the intrinsic dimensionality of the paragraph trajectories through these corpora. The work focuses on investigating the intrinsic dimension of the trajectories and they used LSA to construct representations of the texts. However, they did not address the topic alignment or trajectory clustering problems for which our PHATE and Hellinger shortest path framework is designed.

\subsection{Application to Twitter data}
Enabled by our proposed longitudinal dynamic topic model, we leverage recent activity on social media to understand the impact of the COVID-19 pandemic and, in particular, its impact on social discourse. The utilization of novel data sources is vital, as the current data landscape for understanding the pandemic remains imperfect. For example, public databases maintained by Johns Hopkins University (\url{https://bit.ly/2UqFSuA}) and \textit{The New York Times} (\url{https://bit.ly/2vUHfrK}) provide incoming county-level information of confirmed cases and deaths. Unfortunately, these data streams are of limited utility due to limited testing capacity and selection bias~\citep{dempsey2020hypothesis}. The public health community requires auxiliary sources of information to improve national and local health policy decisions. A critical question is whether there are complementary data streams that may be leveraged to better understand the COVID-19 pandemic in the United States. Social media platforms, such as Twitter, Reddit, Facebook, and so on, are examples of such data streams. These platforms generate high resolution spatiotemporal data sets that concern public opinions on various societal issues, including health care, government decisions, and politics, all of which could be highly relevant to understanding the impact of COVID-19. 

Although use of these novel data streams create new challenges due to limitations such as high noise level, high volume, and selection bias, many recent efforts have explored social media data as a complementary source to traditional health care data and applied topic models to understand public concerns toward COVID-19~\citep{doogan2020public,stokes2020public,boon2020public,xue2020public,jang2020exploratory,du2017cbinderdb,liu2019regional,tan2022identifying,tan2022doubly}, as well as related socioeconomic issues~\citep{su2020covid,sha2020dynamic,liu2020health,tan2018changepoint}. Here, we extract information from Twitter, a particularly popular social media platform, and focus on studying its spatiotemporal behaviors that are believed to be affected by COVID-19. We use subsamples of tweets generated from February 15, 2020, to May 15, 2020, a period over which a large volume of COVID-19- related tweets occurred. An extended analysis of data collected from May to August 2020 can be found on~\url{https://github.com/ywa136/twitter-covid-topics}. We apply our temporal topic modeling framework to discover sets of COVID-19-related latent topics that impact public discourse.

\subsection{Application to TalkLife data}
In addition to its lasting physical health effects, a side effect of the COVID-19 is a noticeable and disproportionate increase in the global burden of depressive and anxiety disorders worldwide. \citet{santomauro2021global} showed that the pandemic led to a dramatic rose in the overall number of cases of mental disorders, with an additional $53.2$ million and $76.2$ million cases of anxiety and major depressive disorders (MDD), respectively. Even before the COVID-19 pandemic, mental health disorders posed a significant burden worldwide. However, access to mental healthcare resources remain poor worldwide. Online social media platform, especially peer-to-peer support platforms attempt to alleviate this fundamental gap by enabling those who struggle with mental illness to provide and receive social support from their peers.

Recent work found that social media big data combined with NLP and machine learning techniques can help address public health, especially mental health, research questions~\citep{conway2016social,de2013role,gkotsis2016language, de2016discovering, kim2021machine,amir2019mental}. As another practical application of this work, we use data from TalkLife (\url{https://www.talklife.com}), the largest online peer-to-peer support platform for mental health support. Several work~\citep{sharma2020computational,sharma2020engagement,sharma2021towards} have demonstrated the usefulness of the TalkLife data as a machine learning dataset for training models that help improve understanding of the mental health issues. Here, we obtain posts from the platform in the year of 2019 and apply our temporal topic modeling framework to extract mental health related discussions. A distinguishing feature of the TalkLife data is that labels for the posts are created by human experts for further investigation of the contents and to aid potential early prevention and intervention of mental health issues. Extending the Twitter analysis, we develop a weakly-supervised temporal topic model and show that our framework is able to capture latent topics that correlate well with a set of labels created by human experts.

\subsection{Key contributions and outline of the chapter}
We highlight key contributions of this article:
\begin{itemize}
    \item A modular framework that provides a wrapper for a suite of tools for interpretation and visualization of temporal topic models. 
    \item A new approach for aligning independently learned topic models over time based on computational geometry.
    \item A scheme for visualizing and understanding temporal structures of the aligned topics via manifold learning.
\end{itemize}

The remainder of the article is organized as follows: Section~\ref{sec:gdtm-method} introduces the methods and tools that have been applied in our analysis framework. Section~\ref{sec:gdtm-results-twitter} and Section~\ref{sec:gdtm-results-talklife} present numerical results and visualizations with several case studies. Section~\ref{sec:gdtm-conclusion} gives some concluding remarks.

\section{Methods}\label{sec:gdtm-method}
In this section, we discuss the building blocks for the proposed framework: Section~\ref{sec:lda} briefly describes the LDA model and its variants for dealing with micro-text; Section~\ref{sec:temporal} introduces two key components of the framework for propagating and associating topics over time; and Section~\ref{sec:phate} reviews and applies a dimension reduction technique to visualize the temporal trajectories of the evolving topics. 

\subsection{LDA for micro-text documents}\label{sec:lda}
 Since the literature in probabilistic topic models and their dynamic variants is enormous (see \citet{blei2012probabilistic} for a survey), we focus our discussion on the LDA~\citep{blei2003latent}, which is the building block for all other algorithms targeting similar applications. A graphical model representing its generating process is presented in Appendix~\ref{supp:tlda}. The idea of LDA is: from a collection of documents (each composed of set of words $w_{d,n}$), one is able to infer the per-word topic assignment $z_{d,n}$, the per-document topic proportions $\theta_d$, and the per-corpus topic distributions $\beta_k$, through a joint posterior distribution $p(\theta,z,\beta|w)$. Numerous inference algorithms are developed to handle data at scale, for example, variational methods~\citep{blei2003latent,teh2008collapsed,hoffman2013stochastic,mimno2012sparse,srivastava2016neural}, expectation propagation~\citep{minka2012expectation}, collapsed Gibbs sampling~\citep{griffiths2002probabilistic}, distributed sampling~\citep{newman2008distributed,ahmed2013network}, and spectral methods~\citep{arora2012learning,anandkumar2014tensor}. The posterior expectations can then be used to perform the task at hand: information retrieval, document similarity determination, exploration, and so on.

The standard LDA, however, may not work well with micro-text like tweets. In particular, each tweet usually concentrates on a single topic, and it is not reasonable to consider one tweet as a document in the traditional sense as there is limited data (e.g., word co-occurrences) from which the latent topics can be learned. To overcome this ``data sparsity'' issue, efforts have been made along on three major directions~\citep{qiang2020short}: 1) methods predicated on the assumption that each text (e.g., tweet) is sampled from only one latent topic; 2) methods utilizing global (i.e., the whole corpus) word co-occurrences structures; 3) methods based on aggregation/pooling of texts into `pseudo-documents' prior to topic inference. 

In this article, we apply the Twitter LDA model~ \citep[T-LDA,][]{zhao2011comparing}, for modeling topics at each time slice. T-LDA can be categorized along the directions 1) and 3) mentioned above. But we note that the proposed framework works with any topic model that outputs word distributions representing learned latent topics. We selected T-LDA since it has been widely used in many related applications, including aspect mining~\citep{yang2016aspect}, user modeling~\citep{qiu2013not}, and bursty topic detection~\citep{diao2012finding}. The generative model underlying T-LDA assumes that there are $K$ topics in the Tweets, each represented by a word distribution, denoted as $\beta_k$ for topic $k$ and $\beta_B$ for background words. Let $\theta_u$ denote the topic assignment distribution for user $u$. Let $\pi$ denote a Bernoulli distribution that governs the choice between background words and topic words. The generating process for a tweet is as follows: a user first chooses a topic based on its user-specific topic assignment distribution. Then the user chooses a bag of words one-by-one based on the chosen topic or the background model. The generation process is summarized in Algorithm~\ref{alg:t-lda}, and a plate notation comparison between the T-LDA and standard LDA is included in Appendix~\ref{supp:tlda}. Similarly to a standard LDA algorithm, parameters in each multinomial distribution are governed by symmetric Dirichlet priors. The model inference can be performed using collapsed Gibbs sampling (code available at \url{https://github.com/minghui/Twitter-LDA}). Due to space limitations we leave out derivation details and sampling formulas. More details on the implementation can be found in Appendix~\ref{supp:tlda}.

\begin{algorithm}
\begin{minipage}{0.9\textwidth}
\begin{algorithmic}
\caption{Generating process for T-LDA}
\label{alg:t-lda}
\REQUIRE Constants $\eta,\gamma$
\STATE Draw $\beta_B \sim \text{Dir}(\eta)$, $\pi \sim \text{Dir}(\gamma)$
\FOR{topic $k=1,\dots,K$} 
    \STATE Draw $\beta_k \sim \text{Dir}(\eta)$
\ENDFOR
\FOR{user $u=1,\dots,U$}
    \STATE Draw $\theta_u \sim \text{Dir}(\alpha)$
    \FOR{Tweet $s=1,\dots,S_u$}
        \STATE Draw $z_{u,s} \sim \text{Multi}(\theta_u)$
        \FOR{word $n=1,\dots,N_{u,s}$}
            \STATE Draw $y_{u,s,n} \sim \text{Multi}(\pi)$
            \IF{$y_{u,s,n}=0$}
                \STATE Draw $w_{u,s,n} \sim \text{Multi}(\beta_B)$
            \ELSE
                \STATE Draw $w_{u,s,n} \sim \text{Multi}(\beta_{z_{u,s}})$ 
            \ENDIF
        \ENDFOR
    \ENDFOR
\ENDFOR
\end{algorithmic}
\end{minipage}
\end{algorithm}

\paragraph{Weak Supervision with Word-level Prior Knowledge:} To encourage topic models to learn latent topics that correlate directly with word-level side information, we augment them with a weakly supervised signal in the form of seed words. Rather than fully guiding the model with labels in a supervised version, as in for example, \citet{mcauliffe2007supervised} and \citet{ramage2009labeled}, we use a set of seed words to deﬁne an asymmetric prior on the word-topic distributions. The reasons for modeling choice are twofold: 1) In one of the applications we concern, we have access to micro-text data from TalkLife and the corresponding labels for each post. However, the labeling is noisy (labels could be wrong/imperfect), limited (not every post is labeled), and have overlaps (a post could be tagged with multiple related labels). Directly applying supervised LDAs that assume perfect labeling may not be appropriate. 2) We hope to learn novel latent topics that have not yet been discovered and/or have been missed by domain experts. 

Using seed words as a form of word-level side information has been considered by a few researchers~\citep{lu2011multi,zhu2009multi,wang2010latent}, although their goal was normally aspect ratings and multi-aspect sentence labeling. In this work, we characterize our prior knowledge (seed words) for the original T-LDA model using a conjugate Dirichlet prior to the multinomial word-topic distributions. We define a combined conjugate prior for each word $n$ in the vocabulary $V$ as $\beta_k \sim \text{Dir}\Big(\{\eta + w_n\}_{n \in V}\Big)$ for each topic $k$, where $w_n$ can be interpreted as an equivalent sample size, i.e., the impact of our asymmetric prior is equivalent to adding $w_n$ pseudo counts to the
sufﬁcient statistics of the topic to which word $n$ belongs. In practice, $w_n$ can be obtained empirically as proportional to the volume of a word $n$ in the corpus. We set $w_n=0$ when we do not have prior knowledge of a word.

\subsection{Time evolution of topics and shortest paths}\label{sec:temporal}
Instead of explicitly building the temporal structures into the model as in the globally parameterized DMT and its variants, we propose a two-stage approach: 1) construct a new corpus at each time point via subsampling the documents and independently fitting a topic model to each new corpus; 2) link each of these time points together via shortest distance paths through topics. 

\paragraph{\textit{Temporal smoothing by subsampling}}
\label{sec:smoothing}
A subsample of tweets is constructed at each time point by conditional sampling of all the tweets with a sampling distribution that is inversely proportional to the temporal proximity of the tweet.  This produces subsamples that are local mixtures of tweets at nearby time points, accomplishing a degree of temporal smoothing prior to topic analysis. 

To clarify the subsampling procedure, we give a simple example. Assume that the corpus is composed of five tweets per day over a 5-day period. We write $d_t$ as the set of five tweets on day $t$.  On the first day ($t=1$), exponential weights are computed $w_{1}=\{1.000,0.7500,0.5625,0.4219,0.3164\}$ and normalized by their sum, defining the sampling distribution used to construct the subsample. That is, at time $1$, the subsample consists of $100\%$ of all the tweets from day $1$, $75\%$ of the tweets from day $2$, $75\%^2=56.25\%$ of the tweets from day 3 and so on. This subsample is denoted $C_{1} = \{s_1,\dots,s_5\}$. To construct the subsample on day $2$, we condition on the subsample $C_1$ on day $1$ and we construct exponential sampling weights for day $2$ of the form $w_{2}=\{0.7500,1.000,0.7500,0.5625,0.4219\}$.  We combine the subsample $C_1$ and weights $w_2$ to construct the subsample on day $2$, denoted $C_2$, by either randomly removing extra tweets if the sampling weights on a particular day $i$ decreased or randomly adding more tweets from $d_i - s_i$ if the weights on a particular day $i$ increased.  This algorithm ensures a (tuneable) degree of smoothness over the subsamples generated at each time in the sense that subsamples which are close in time are likely to contain similar tweets. The procedure is illustrated in Figure~\ref{fig:Tweets-smoothed}. Specifically, using notations developed above, $C_{1} = \{s_1,\dots,s_5\}$ and $C_{2} = \{s_1 - \text{Tweet 3}, s_2 + \text{Tweet 7}, s_3 + \text{Tweet 15}, s_4 + \text{Tweet 20}, s_5 + \text{Tweet 24}\}$, where \text{Tweet 3}, \text{Tweet 7}, \text{Tweet 15}, \text{Tweet 20}, and \text{Tweet 24} were randomly chosen.
\begin{figure}[!tbh]
    \centering
    \includegraphics[width=0.9\textwidth]{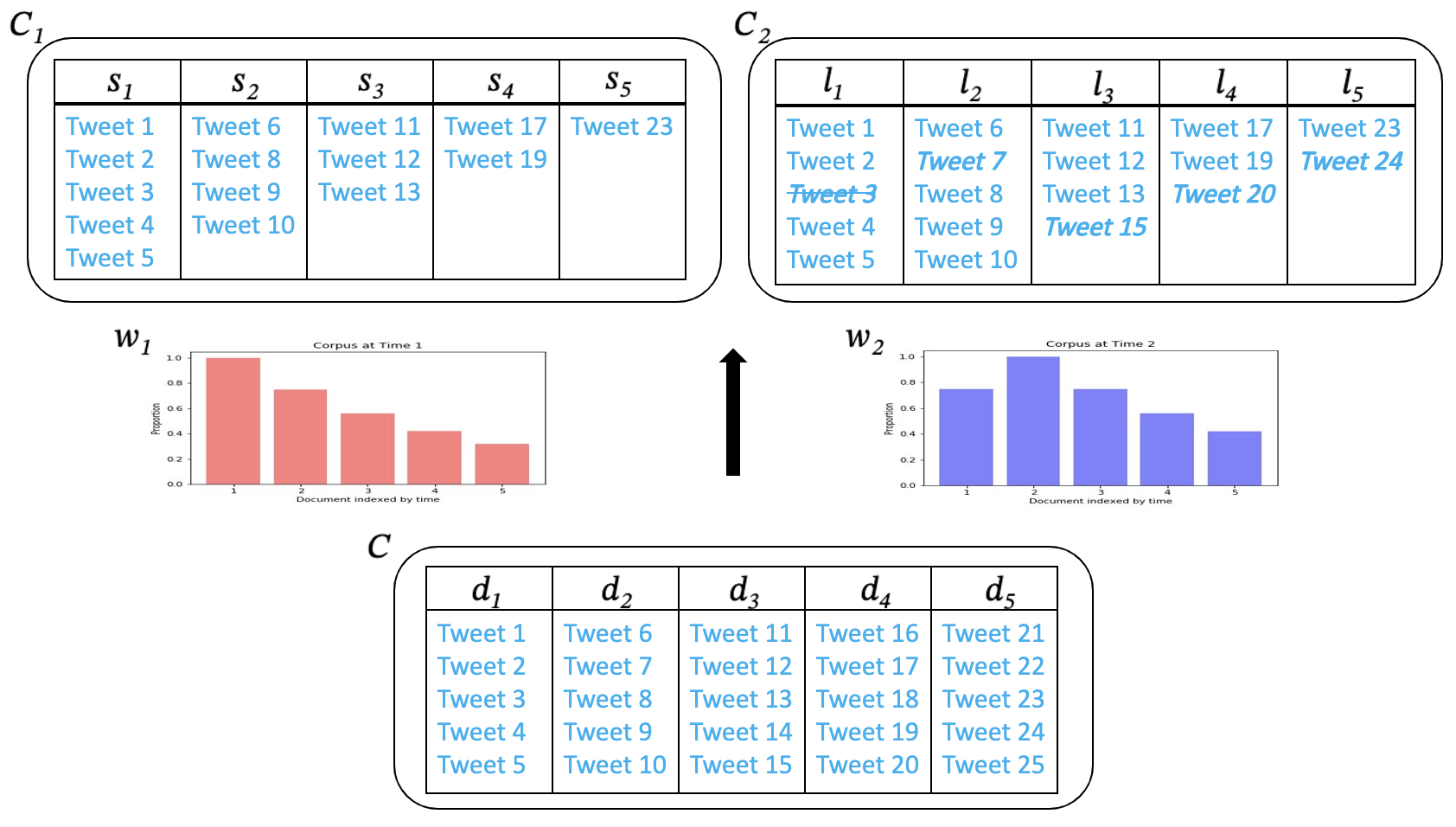}
    \caption{Conditional subsampling procedure using a hypothetical corpus composed of five documents each containing five tweets. For example, $C=\{d_1,\dots,d_5\}$, $d_1$ aggregates tweets from day $1$, $d_2$ aggregates tweets from day $2$, and so on. The subsampling weights for each document are shown in the bar plots and are exponentially decaying with a factor of $0.75$, centered at day $1$ (left, $w_{1}$) and day $2$ (right, $w_{2}$), respectively. Each newly generated corpus is a proportionally weighted random sample and a realization of these samples are shown in the tables ($C_{1}$ and $C_{2}$). Note that the two corpora differ only by those highlighted and italicized tweets.}
    \label{fig:Tweets-smoothed}
\end{figure}

After constructing the temporally smoothed corpus, any topic modeling algorithm, such as LDA or its (weakly) supervised variants, can be independently applied to each of the extracted corpora, using either the same set of parameters across all LDA runs, or seeding the next LDA with estimated parameter values (e.g., the posterior mean of the Markov chain Monte Carlo samples) from the current LDA, as described in~\citet{song2005modeling}.

Our proposed temporal smoothing technique is similar to smoothing approaches introduced in spatiotemporal statistics. For example, in time-series analysis the idea of exponential smoothing was proposed in the late 1950s~\citep{brown1959statistical,holt2004forecasting,winters1960forecasting}, and has motivated some of the most successful forecasting methods. Forecasts produced using exponential smoothing methods are weighted averages of past observations, with the weights decaying exponentially as the observations become older. As another example, kriging~\citep{krige1951statistical} or Gaussian process regression is a widely used method of interpolation in geostatistics. The basic idea of kriging is to predict the value of a function at a given point by computing a weighted average of the known values of the function in the neighboring points~\citep{cressie2015statistics}. Lastly, in nonparametric regression analysis, the locally estimated scatterplot smoothing (LOESS) is a widely used method that combines multiple regression models in a $k$-nearest neighbor based framework. Particularly, at each point in a data set a low-degree polynomial is fitted to a subset of the data, with explanatory variable values near the point whose response is being estimated. The subsets used for each of these polynomial fits are determined by a nearest neighbor algorithm. A user-specified `bandwidth' or smoothing parameter determines how much of the data is used to fit each local polynomial. This smoothing parameter is the fraction of the total number of data points that are used in each local fit. 

\paragraph{\textit{Dissimilarity between topic word distributions}}
After applying local LDA to each of the time localized subsamples of the smoothed corpus, we stitch together the local LDA results. The alignment of topics with different time stamps is accomplished by creating a weighted graph connecting all pairs of topics where the edge weights are a measure of topic similarity, to be described below.  Assume that each local model generates  $K$ topics, resulting in a total of $K \times T$ topics across $T$ time points. A weighted adjacency matrix is constructed from the similarities between $\binom{K \times T}{2}$ topic pairs. The similarities between topics will allow the alignment algorithm to relate topics together across time and enable us to track topic evolution. 

As each topic is characterized by the LDA word distribution, any metric that measures dissimilarity between discrete distributions could be used to construct a similarity measure.  It is well known that the Euclidean distance is not well adapted to measuring dissimilarity between probability distributions~\citep{amari2012differential}. As an alternative, we propose using the Hellinger metric on the space of distributions, which we justify as follows. The LDA word distribution is conditionally multinomial and it lies on a statistical manifold called an \textit{information geometry}, that is endowed with a natural distance metric, called the Fisher-Rao Riemannian metric. Unlike the Euclidean metric, this Riemannian metric characterizes intrinsic minimal (geodesic) distances between multinomial distributions and it depends on the Fisher information matrix $[\mathcal{I}(\theta)]$, $\theta$ is the multinomial probability vector. \citet{carter2009fine} showed that this metric can be well approximated by the Hellinger distance between multinomial distributions. The Hellinger distance between discrete probability distributions $P=(p_1,\dots,p_N)$ and $Q=(q_1,\dots,q_N)$ is defined as 
\begin{equation*}
    H(P,Q) = \frac{1}{\sqrt{2}}\sqrt{\sum_{n=1}^N(\sqrt{p_n}-\sqrt{q_n})^2}, \quad 0 \leq H(\cdot,\cdot) \leq 1.
\end{equation*}

The major advantages of the Hellinger distance are threefold: 1) it defines a true metric for probability distributions, as compared to, for example, the Kullback-Leibler divergence; 2) it is computationally simple, as compared to the Wasserstein distance; 3) and it is a special case of the $f$-divergence, which enjoys many geometric properties and has been used in many statistical applications. For example, \citet{liese2012phi} showed that $f$-divergence can be viewed as the integrated Bayes risk in hypothesis testing where the integral is with respect to a distribution on the prior; \citet{nguyen2009surrogate} linked $f$-divergence to the achievable accuracy in binary classification problems; \citet{jager2007goodness} used a subclass of $f$-divergences for goodness of fit testing; \citet{rao2020use} demonstrated the advantages of the Hellinger metric for graphical representations of contingency table data; \citet{srivastava2016functional} adopted the Hellinger distance to measure distances between functional and shape data; \citet{shemyakin2014hellinger} showed the connection of the Hellinger distance to Hellinger information, which is useful in nonregular statistical models when Fisher information is not available; and finally, \citet{servidea2006statistical} derived an identity between the Hellinger derivative and the Fisher information that is useful for studying the interplay between statistical physics and statistical computation.

We note that in previous work on aligning topics the $L2$ or cosine distance is commonly applied~\citep{chuang2013topic,chuang2015topiccheck,yuan2018multilingual}. As discussed above, these distances are practically and theoretically deficient for aligning distributions. A simulation study is presented in Appendix~\ref{supp:phate-simulaiton} that compares use of these distances to the Hellinger distance, showing that the latter better preserves topic trend coherence.
 
\paragraph{\textit{Nearest neighbor graphs and shortest paths}}
We use the topic graph with Hellinger weights to identify natural progressions from one topic to another over time. We use Dijkstra shortest paths through a nearest neighbor subgraph to identify these progressions.  These paths can be interpreted as trajectories of public discourse on the topics identified. This is of interest because we want to understand how conversations around a topic evolves over time. Shortest path analysis allows us to do this with minimal assumptions on the data. In particular, we do not assume or further encourage temporal smoothness in the data beyond the temporally smoothed corpora described in Section~\ref{sec:smoothing}.

Due to the noisy nature of social media data and the wide range of topics, we pay special attention to local neighborhoods of data points. Hence, instead of working with a fully connected graph induced by the full $N \times N$ Hellinger distance matrix of pairwise distances between topics, we build a $k$-nearest neighbor graph from it. Natural evolution of a topic over time can then be inferred by finding a shortest path of topics on the weighted $k$-nearest graph, where Hellinger distances represent edge weights. Here, a shortest path is a path between two vertices (i.e., two topics) in a weighted graph such that the total sum of edges weights is minimum, and can be computed efficiently using, for example,  Dijkstra's algorithm. The approach of using neighborhood graphs for estimating the intrinsic geometry of a data manifold is justifiable both empirically and theoretically. In manifold learning similar ideas are used to reconstruct lower dimensional geometry from data. For example, the isometric feature mapping~\citep[ISOMAP]{tenenbaum2000global} extends metric multidimensional scaling (MDS) by replacing the matrix of Euclidean distances in MDS with the matrix of shortest path distances between pairs of vertices in the Euclidean $k$ nearest neighbor graph. Using such embedding, ISOMAP is able determine lower dimensional structure in high-dimensional data and capture perceptually natural but highly nonlinear ``morphs'' of the corresponding high-dimensional observations (see figure 4 in \citet{tenenbaum2000global}). Such shortest path analysis is supported by substantial theory ~\citep{bernstein2000graph,costa2006determining,hwang2016shortest}. Under the assumption that the data points are random realizations on a compact and smooth Riemannian manifold, as the number of data points grows, the shortest paths over the $k$ nearest neighborhood graph converge to the true geodesic distance along the manifold.   

In the context of our topic alignment application, this theory suggests that the analogous Hellinger shortest paths should be able to achieve alignment if the empirical LDA word distributions can themselves be interpreted as random draws from an underlying distribution that varies continuously and smoothly over time along a statistical manifold. To illustrate, Figure~\ref{fig:covid-short-path} demonstrates how a COVID-19-related topic learned from the corpus on February 15, 2020 (far left), evolves to a COVID-19-related health care--focused topic learned from the corpus on May 15, 2020 (far right). The top row in the figure was constructed by computing the shortest Hellinger distance path on a $10$-nearest neighbor graph, whereas the bottom row was constructed using the full graph. As expected, the shortest path on the neighborhood graph captures perceptually natural but highly nonlinear `morphs' of the corresponding high-dimensional word distributions by transforming them approximately along geodesic paths. On the other hand, the shortest path on the full graph connects the two observations through a sequence of apparently unrelated and nonintuitive topics. In Appendix~\ref{supp:topicflow}, we compare the proposed Hellinger shortest path topic alignment method with  TopicFlow~\citep{malik2013topicflow}, a common method for topic alignment that uses local matching and Euclidean distances.

\begin{figure}[!tbh]
\centering
\begin{subfigure}{.2\textwidth}
  \centering
  \includegraphics[width=\linewidth]{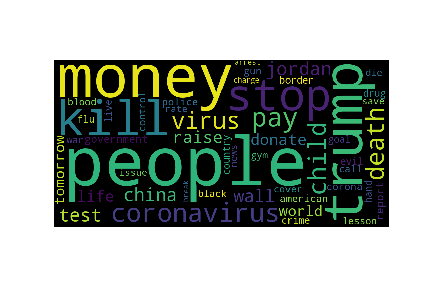}
\end{subfigure}%
{$\rightarrow$}%
\begin{subfigure}{.2\textwidth}
  \centering
  \includegraphics[width=\linewidth]{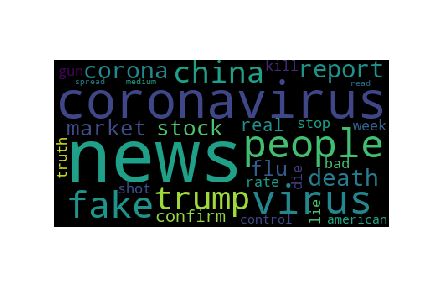}
\end{subfigure}
{$\rightarrow$}%
\begin{subfigure}{.2\textwidth}
  \centering
  \includegraphics[width=\linewidth]{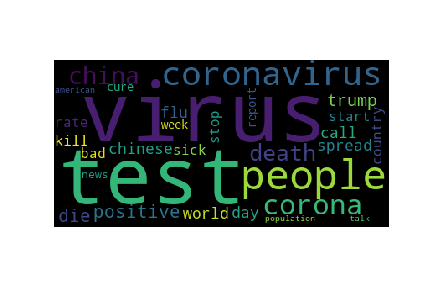}
\end{subfigure}
{$\rightarrow$}%
\begin{subfigure}{.2\textwidth}
  \centering
  \includegraphics[width=\linewidth]{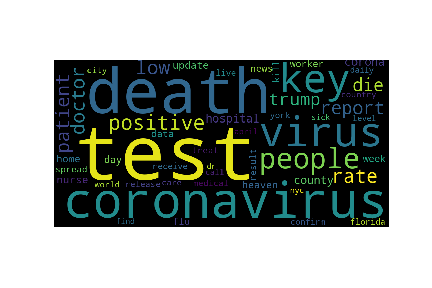}
\end{subfigure}
\\
\begin{subfigure}{.2\textwidth}
  \centering
  \includegraphics[width=\linewidth]{covid_short_path_b.png}
\end{subfigure}%
{$\rightarrow$}%
\begin{subfigure}{.2\textwidth}
  \centering
  \includegraphics[width=\linewidth]{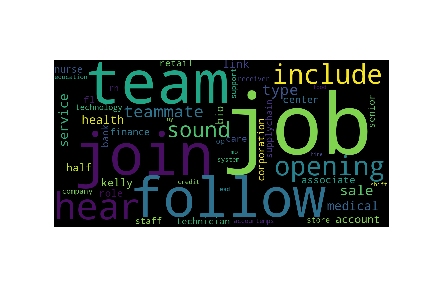}
\end{subfigure}
{$\rightarrow$}%
\begin{subfigure}{.2\textwidth}
  \centering
  \includegraphics[width=\linewidth]{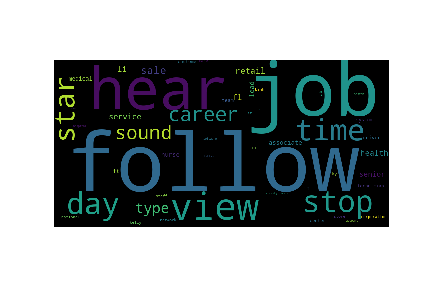}
\end{subfigure}
{$\rightarrow$}%
\begin{subfigure}{.2\textwidth}
  \centering
  \includegraphics[width=\linewidth]{covid_short_path_e.png}
\end{subfigure}
\caption{Evolution along the Hellinger shortest paths of a COVID-19 topic on February 15, 2020, to a COVID-19
topic on May 15, 2020. The paths are computed on a $10$-nearest neighbor graph (top) and a fully connected graph (bottom).  Each word cloud image represents a topic at a particular time, showing the word distribution encoded by font size (only the top $30$ words in each topic are shown). The middle two word clouds represent two intermediate topics on the respective paths and illustrate the benefit of using the $k$ nearest neighbor graph. The middle two topics on the top row seem naturally connected to the beginning and the end topics, in contrast to the bottom row.}
\label{fig:covid-short-path}
\end{figure}

The choice of $k$ for the neighborhood graph affects the approximation to the Hellinger geodesic path: choosing a $k$ that is too large creates short circuits in the graph, resulting in a noisy path like the bottom row of Figure~\ref{fig:covid-short-path}; choosing a  $k$ that is too small results in a graph that is disconnected for which there might not exist a path between two points of interest. The problem of selecting an optimal value of $k$ remains open, although several computational data-driven approaches have been proposed for ISOMAP \citep{tenenbaum2000global,samko2006selection,gao2011dynamical}. Here we use $k=10$, which exceeds the connectivity threshold, to induce the most natural approximation to the true geodesic path between topics of interest. In Appendix~\ref{supp:sensitivity} we establish that our results are robust to perturbations around this value of $k$.

We also note that the Hellinger shortest paths may differ in length, which is the number of topics that they connect over time. This variation is due to the occasional time skips in the path that occur when the shortest path algorithm does not find an adequate match between topics at successive time points. Such skipping can occur when a topic thread wanes temporarily, merges with another thread, or dies. In Appendix~\ref{supp:path-skips-summary} we provide statistics on the occurrences of skips for a subset of paths.

\subsection{Interpretation and visualization of topic trends via low-dimensional embedding}\label{sec:phate}
In LDA each latent topic is represented by a vector that lies on a simplex that constitutes a discrete probability distribution over words. This vector could be very high dimensional depending on the size of the vocabulary. Dimensionality reduction methods are useful for visualization, exploration, and interpretation of such high-dimensional data, as they enable extraction of critical information in the data while discarding noise. Many popular methods are available for visualizing high dimensional data, such as principle component analysis (PCA), MDS, uniform manifold approximation and projection~\citep[UMAP]{mcinnes2018umap}, and t-distributed stochastic neighbor embedding~\citep[t-SNE]{maaten2008visualizing}. These methods use spectral decompositions of the pairwise distance matrix to embed the data into lower dimension. PHATE~\citep{moon2019visualizing}, on the other hand, is designed to visualize high-dimensional time-varying data. As demonstrated by the authors, it is capable of uncovering hidden low-dimensional embedded temporal progression and  branching structure. 

Here we embed the estimated LDA word distributions into lower dimensions using a novel application of PHATE to the Hellinger distance matrix. For details on our implementation, see Appendix~\ref{supp:phate-details}. Here, using simulated data, we demonstrate the power of the proposed PHATE-Hellinger embedding for visualization of temporal evolution patterns as compared to other embedding methods. Specifically, we simulate $10$ trajectories of $100$-dimensional probability vectors using the model
\begin{equation*}
    X_t^j|X_{t-1}^j \sim \mathcal{N}_{100}(X_{t-1}^j,\sigma^2_jI)
\end{equation*} 
and
\begin{equation*}
    P_{t,i}^j = \frac{\exp(X_{t,i}^j)}{\sum_{i=1}^p \exp(X_{t,i}^j)}, \qquad i=1,\dots,100
\end{equation*}
for $j=1,\dots,10$ and $t=0,\dots,99$. Each trajectory starts at the same point $X_0 \in \mathbb{R}^{100}$ and differs from realization to realization  depending on $\sigma_j$. We project all $1000$ vectors onto a hypersphere by computing the element-wise square root of each probability vector and using the mapping $P_{t,1} + \cdots + P_{t,100} = 1 \Leftrightarrow (\sqrt{P_{t,1}})^2 + \cdots + (\sqrt{P_{t,100}})^2 = 1^2$. Figure~\ref{fig:phate-hellinger-sphere} presents the 2D embeddings of this synthetic dataset using PCA on the Euclidean distance matrix, and t-SNE, UMAP, and PHATE on the Hellinger distance matrix. Observe that, among all methods, only PHATE correctly captures the temporal progressions as distinct trajectories originating from a common initial point $X_0$. Additional simulation studies comparing PCA, t-SNE, UMAP, and PHATE with and without Hellinger distance are included in the Appendix~\ref{supp:phate-simulaiton}. In particular, the benefit of using the Hellinger distance  instead of the Euclidean distance is demonstrated.  

\begin{figure}[!tbh]
    \centering
    \includegraphics[width=\textwidth]{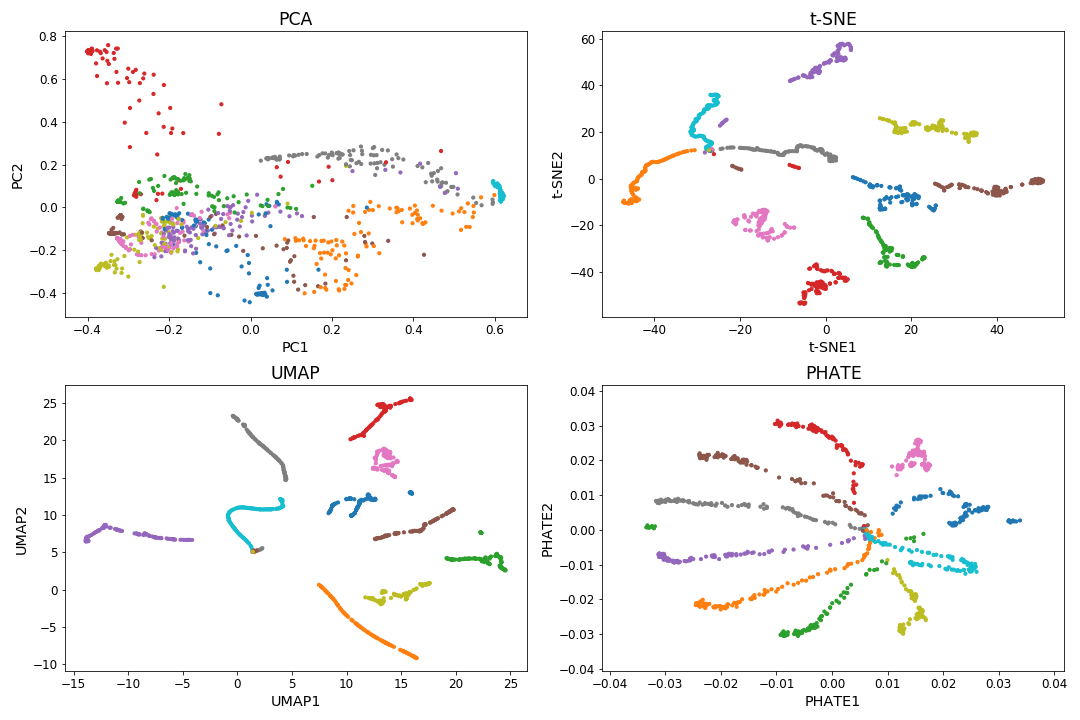}
    \caption{Comparison of principle component analysis (PCA), t-distributed stochastic neighbor embedding (t-SNE), uniform manifold approximation and projection (UMAP), and potential of heat-diffusion for affinity-based transition embedding (PHATE) for dimensionality reduction. The methods are applied to 2D embedding of simulated $10$ trajectories (identified by color) of $100$-dimensional probability vectors, all originating from a common initial point. Except for PCA, all these methods are applied to the matrix of Hellinger distances. Only PHATE  correctly captures the temporal progressions as distinct trajectories originating from a common initial point.}
    \label{fig:phate-hellinger-sphere}
\end{figure}


\section{Twitter Data Analysis}\label{sec:gdtm-results-twitter}
The entire pipeline for our analysis is described in Algorithm~\ref{alg:procedure}. The implementation requires setting several hyperparameters. For Twitter data, the temporal smoothing parameter was selected as $\gamma=0.75$, which corresponds to smoothing approximately one month of tweets into the current time point, in inverse proportion to temporal proximity; the parameter for the number of topics was set to $K=50$ at every time point; the number of neighbors was set to $k=10$ for the neighborhood graph to compute shortest paths. In Appendix~\ref{supp:sensitivity} we show relative insensitivity of our framework to the choice of these hyperparameters. Although not explored here, one could also vary $K$ over time, for example, selected by minimizing perplexities or Bayesian information criteria (BIC) scores at each time (see Appendix~\ref{supp:sensitivity} for a further discussion).

\begin{algorithm}[!tbh]
\begin{minipage}{0.9\linewidth}
\begin{algorithmic}
\caption{Longitudinal analysis of micro-text data.}
\label{alg:procedure}
\REQUIRE Raw micro-text data
    \STATE 1: Preprocess data and organize tweets into a temporally smoothed corpus as described in Section~\ref{sec:temporal}, with smoothing parameter $\gamma$.
    \STATE 2: Apply topic models described in Section~\ref{sec:lda} independently to each one of $T$ corpus with $K$ topics. This results in $TK$ word distributions. 
    \STATE 3: Compute pairwise Hellinger distances for word distributions.
    \STATE 4: Compute:
    \begin{ALC@g}
        \STATE a: $k$-nearest neighbor graph with from the $TK$-by-$TK$ Hellinger matrix and find the shortest path of interest on the neighborhood graph using the Djikstra algorithm.
        \STATE b: PHATE embedding of high-dimensional word distributions in 2D and 3D.
    \end{ALC@g}
\ENSURE Shortest paths and PHATE coordinates.
\end{algorithmic}
\end{minipage}
\end{algorithm}

\subsection{Data preparation}
We downloaded data via the Twitter Decahose Stream API (\url{https://developer.twitter.com/en/docs/Tweets/sample-realtime/overview/decahose}). The Decahose includes a random sample of $\sim 10\%$ of the tweets from each day, resulting in a sample of $300 - 500$ millions of tweets per day. Among all tweets that are sampled, between $\sim 0.1\%$ and $0.5\%$ (see Appendix~\ref{supp:raw_volume} for details) of them contain geographic location information, called geotags, that localize the tweet to within a neighborhood of the user's location when the tweet was generated. Note that Twitter's precise location service that uses GPS information has been turned off by default (\url{https://twitter.com/TwitterSupport/status/1141039841993355264}). We consider here the more common Twitter ``Place'' object that consists of $4$ longitude-latitude coordinates that define the general area from which the user is posting the tweet (\url{https://developer.twitter.com/en/docs/tutorials/filtering-Tweets-by-location}). Here, we focus on a time period from February 15, 2020, to May 15, 2020, where we expect there to be a large volume of tweets that are COVID-19 related. Figure~\ref{fig:volume_all} in Appendix~\ref{supp:raw_volume} shows the number of tweets for each day in the study period. The following filtering was used:
\begin{itemize}
    \item \textbf{U.S. geographic area}: Tweets that are geotagged and originated in the United States as indicated by the Twitter location service.
    \item \textbf{English language Tweets}: Tweets from users who selected English as their default language.
    \item \textbf{Non-retweets}: Tweets that contain original content from the users and are not a retweet of other tweets.
\end{itemize}
The following text preprocessing steps were undertaken: 1) we remove stop words (e.g., \textit{in, on, and}, etc., which do not carry semantic meaning); 2) we keep only common forms of words (lemmatization); 3) we remove words that occurred less than $5$ times in a document. As a result, the average vocabulary length per timestamp was been reduced from around $300000$ to $3000$. Further, the union of the unique words from each timestamp has been used as the common vocabulary with word frequencies zeroed out on days where those words do not occur.

\subsection{Hellinger-PHATE embedding for all topics}
Figure~\ref{fig:phate-overall-2d} shows the 2D Hellinger-PHATE embedding of $4500$ word distributions. We labeled the points on the plots with different colors, sizes, and styles for visualization and interpretation of various time points, tweet volumes, and shortest paths. The full labeling scheme is included in Appendix~\ref{supp:phate-dict}. Figure~\ref{fig:phate-overall-2d} also shows (as insets) two zoom-ins onto selected COVID-19 topics. We observe several interesting trajectory patterns in the PHATE embeddings. For example, the ``STAY HOME (executive order)'' cluster (bottom inset) is organized along a straight line, where the points are more dense at the beginning as well as at the end of the line while sparser in between. The COVID and COVID NEWS clusters (top inset) behave like a splitting between two branches of a tree, and the COVID NEWS (presidential election) path in those clusters exhibits a 'hook' or a 'U' shape. Within the COVID NEWS cluster, the presidential election path also splits and diverges from other points in the same cluster. The following two subsections will focus on these two clusters and paths therein to illustrate the advantages of the proposed framework. Additional visualizations for the SANITIZING (wash hands) and STAY HOME (executive order) paths are included in Appendix~\ref{supp:linear-trajectories}.

\begin{figure}[!tbh]
    \centering
    \includegraphics[width=\textwidth]{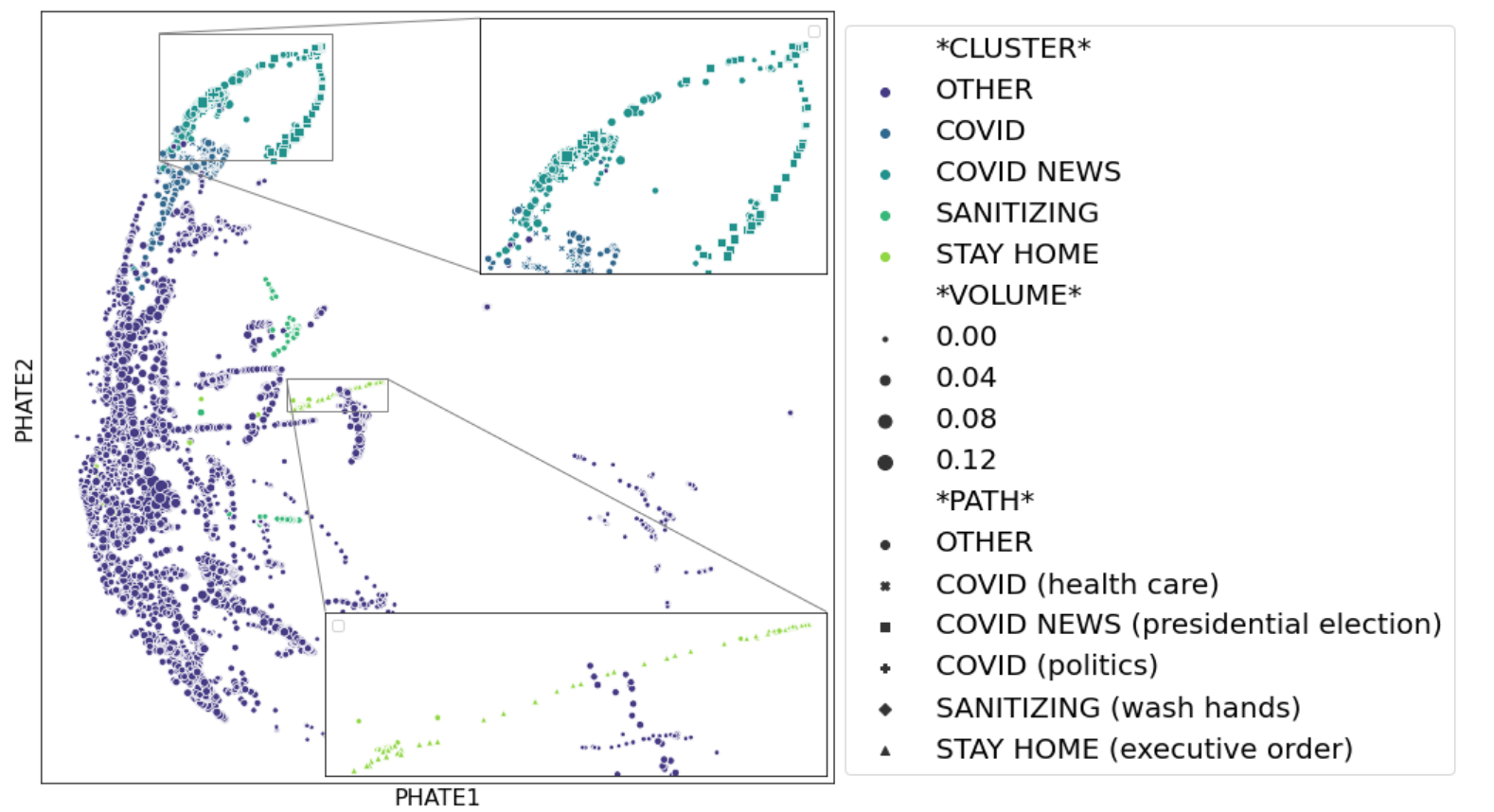}
    \caption{Potential of heat-diffusion for affinity-based transition embedding (PHATE) for all word distributions. Here the two bounding boxes and insets highlight two of the COVID-19-related topic clusters/paths (COVID/COVID NEWS and STAY HOME). The colors, sizes, and styles signify various clusters, tweet volumes, and shortest paths, as given in the dictionary in Appendix~\ref{supp:phate-dict}. Note that the embedding captures some important clustering/trajectory structures, for example, branching, splitting, merging, and so on.}
    \label{fig:phate-overall-2d}
\end{figure}

\subsection{Case study I: presidential election topic path} 
Here, we focus on a cluster of topics that is implicitly COVID-related but can be well understood from associated real-world events. We call this the presidential election topical path. The subset of topics lying on this shortest path is illustrated in Figure~\ref{fig:phate-election-2d}. Here continuous color scales are used to illustrate temporal evolution, which exhibits a smooth transition from the beginning to the end points on the path. The PHATE embedding exhibits three subclusters on the path: 1) an early March cluster that groups topics related to Super Tuesday; 2) an April cluster that groups topics related to or triggered by the ``Bernie Sanders dropped out of the presidential race'' event; 3) an early to mid-May cluster that groups topics converging to more general COVID-related political topics.

\begin{figure}[!tbh]
    \centering
    \includegraphics[width=0.45\textwidth]{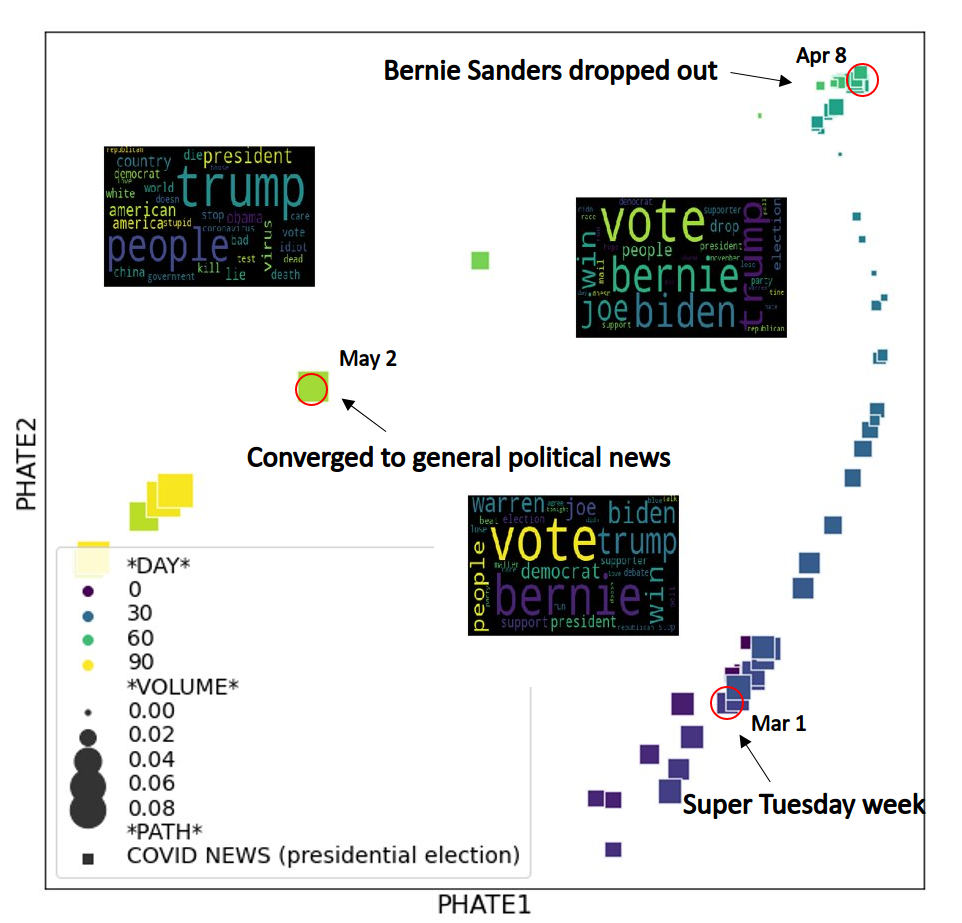} 
    \includegraphics[width=0.45\textwidth]{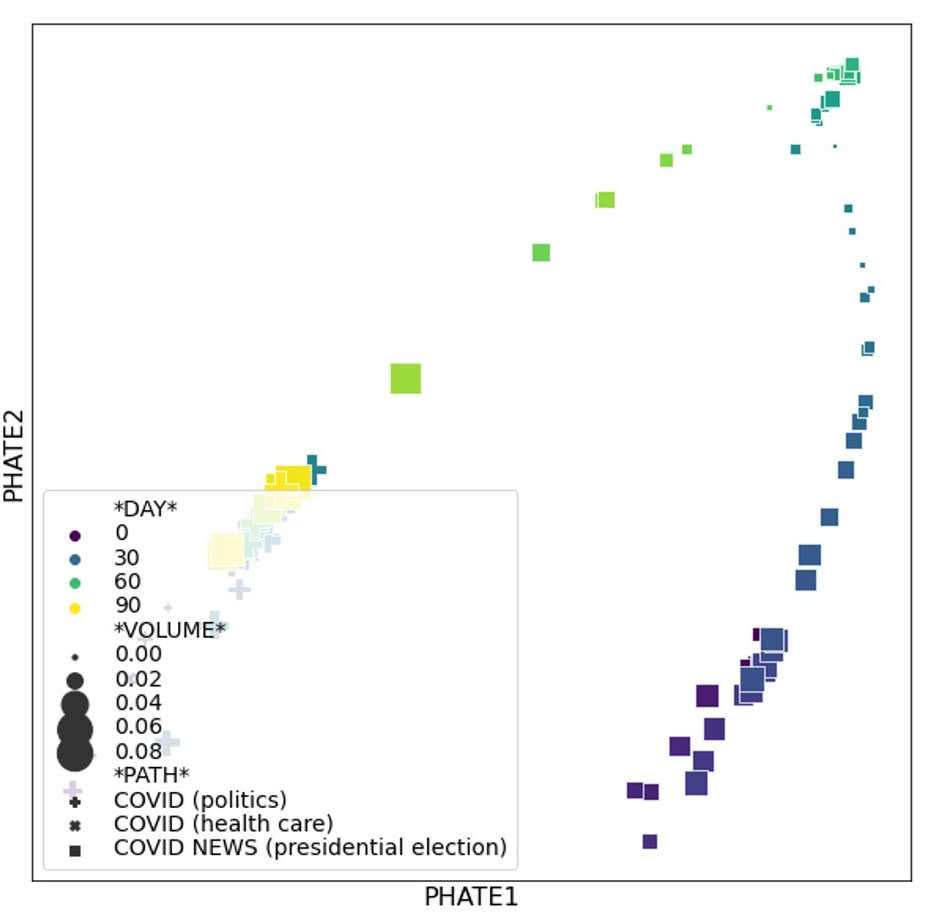}
    \caption{Potential of heat-diffusion for affinity-based transition embedding (PHATE) for subsets of topics in the COVID NEWS cluster (right) and the presidential election path (left) within the cluster. Colors and sizes highlight time and tweet volumes, respectively. Here three word clouds containing top $30$ words in corresponding topics are shown for the time points highlighted by red circles, showing important real-word events that are annotated. Note the plot at the bottom shows (near lower left) the merge and split of different paths (labeled by filled squares, crosses, and pluses) within the same cluster.}
    \label{fig:phate-election-2d}
\end{figure}

Additionally, in terms of tweet volume generated by COVID NEWS topic, there exists again a U-shaped trend: starting at a high level in mid-February the tweet volumes dropped down after the Super-Tuesday week and started to rise near the time when Bernie Sanders dropped out and eventually peaked in mid-May. We believe that this modulation of the presidential election path can be explained by the COVID-19 pandemic in the United States, which accelerated through March when many states issued stay-at-home orders. This then triggered public discourse around COVID-19, increasing the volume of COVID-19-related topics. However, starting in May, as many stay-at-home orders were lifted, more mainstream political news topics reentered the discourse. 

Following we present results of spatial analysis, showing county-level tweet volume in California, illustrating that the Hellinger-distance shortest path combined with PHATE is able to capture more granular-level variations in both space and time. In Figure~\ref{fig:ca_county_map_election} we plot smoothed choropleth maps for the same three topics that were highlighted in Figure~\ref{fig:phate-election-2d}, where the color changes with respect to tweet proportions (the estimated tweet volumes generated from the given topics normalized by the total tweet volumes for the given days for each county). Here raw tweet proportions have been smoothed using a simple Markov random field (MRF) smoother~\citep{wood2017generalized}, which regularizes neighboring counties (i.e., regions with contiguous boundaries, that is, sharing one or more boundary point) to have similar tweet proportions. This smoothing procedure is used to identify hot spots, or areas whose tweet volumes have a high likelihood of differing over neighboring locations. The regularization removes some of the variance one would normally see in a choropleth, and gives a bird's eye view of the entire state. For visualization of similar choropleth maps for other states, as well as a comparison of the maps between states, we include interactive maps at \url{https://wayneyw.shinyapps.io/mrf_smooth_map_app/}.

From the top row of Figure~\ref{fig:ca_county_map_election} we observe two `presidential election' hot spots in counties near the Bay Area and in counties near Los Angeles. The local trend in tweet volume for California is similar to the global trend overall in the United States, as indicated by Figure~\ref{fig:phate-election-2d} above as well as Figure~\ref{fig:states_election_path} in Appendix~\ref{supp:state-level-spatial}.

\begin{figure}[!tbh]
    \centering
    \includegraphics[width=0.7\textwidth]{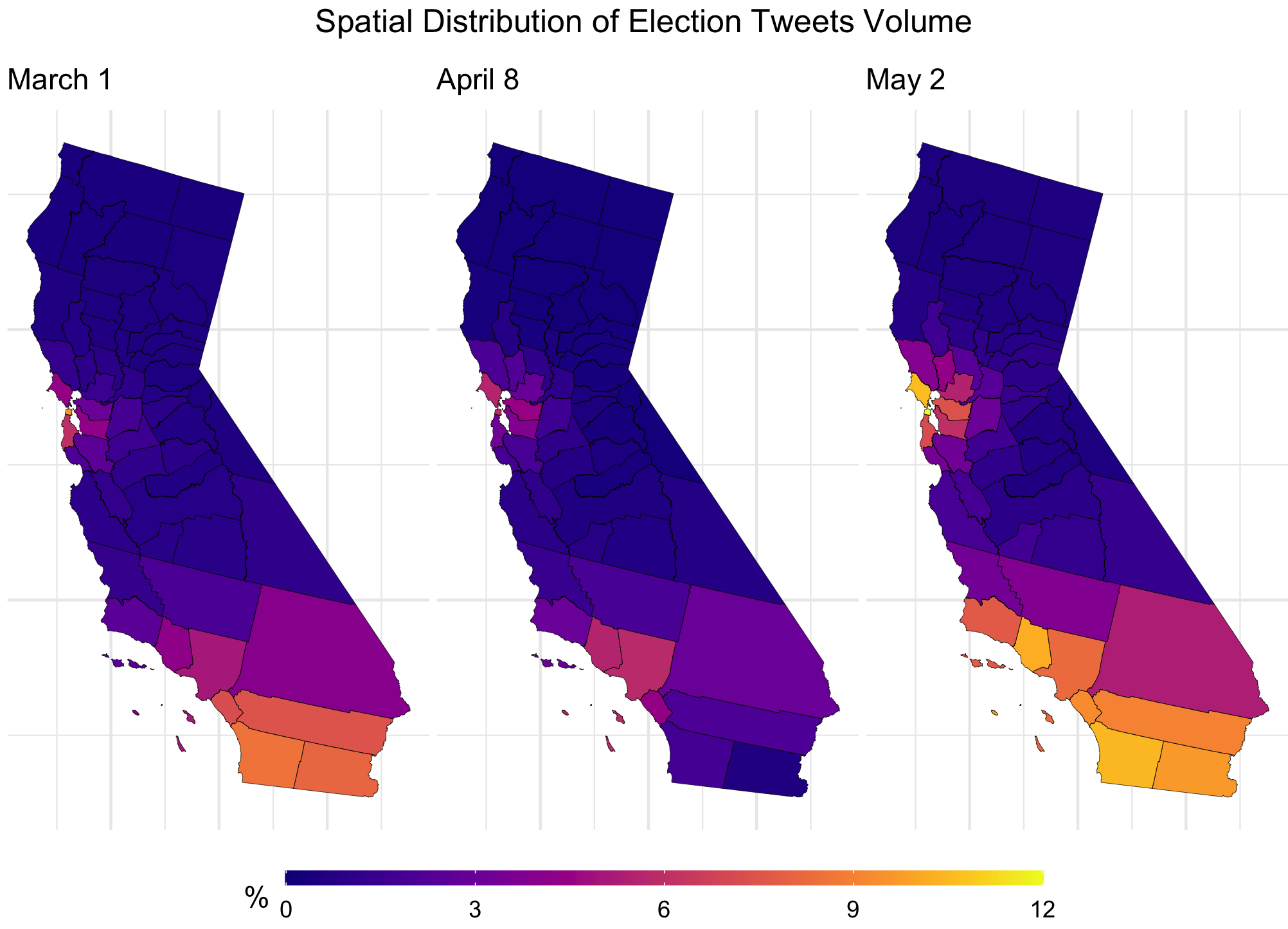}
    \caption{County-level maps for California. It shows the spatial distribution of proportional tweet volumes for the three time points on the COVID NEWS (presidential election) path.}
    \label{fig:ca_county_map_election}
\end{figure}

\subsection{Case study II: general COVID-19 topic path}
In this case study we focus on an explicit COVID-19 topic cluster and shortest paths therein. Figure~\ref{fig:phate-covid} shows the PHATE embedding for subsets of topics in the COVID cluster. The embedding identifies two paths that together exhibit splitting behavior, which can be considered as types of structures built into PHATE a priori. In this case, two similar discussions around COVID-19 split into a path that focused on health care, for example, testing, deaths, hospital, and so on, and a path that focused on politics, for example, government, Trump, president, and so on, respectively. The split of the two paths into two different sets of topics is revealed by naive clustering algorithms, such as hierarchical clustering. We emphasize here that such bifurcation behavior would be difficult to model explicitly, for example, using a time-varying global LDA-type model, but appears naturally in the PHATE embedding of the shortest paths using Hellinger distance. 

\begin{figure}[!tbh]
    \centering
    \includegraphics[width=0.45\textwidth]{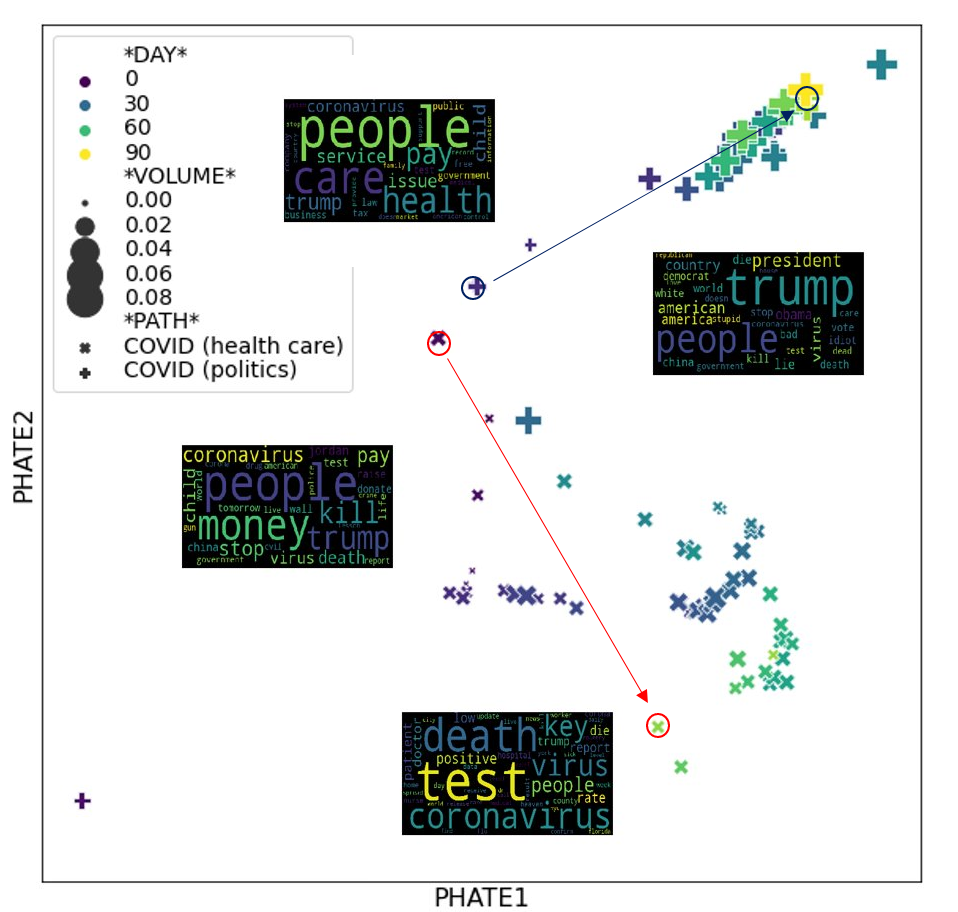} 
    \includegraphics[width=0.45\textwidth]{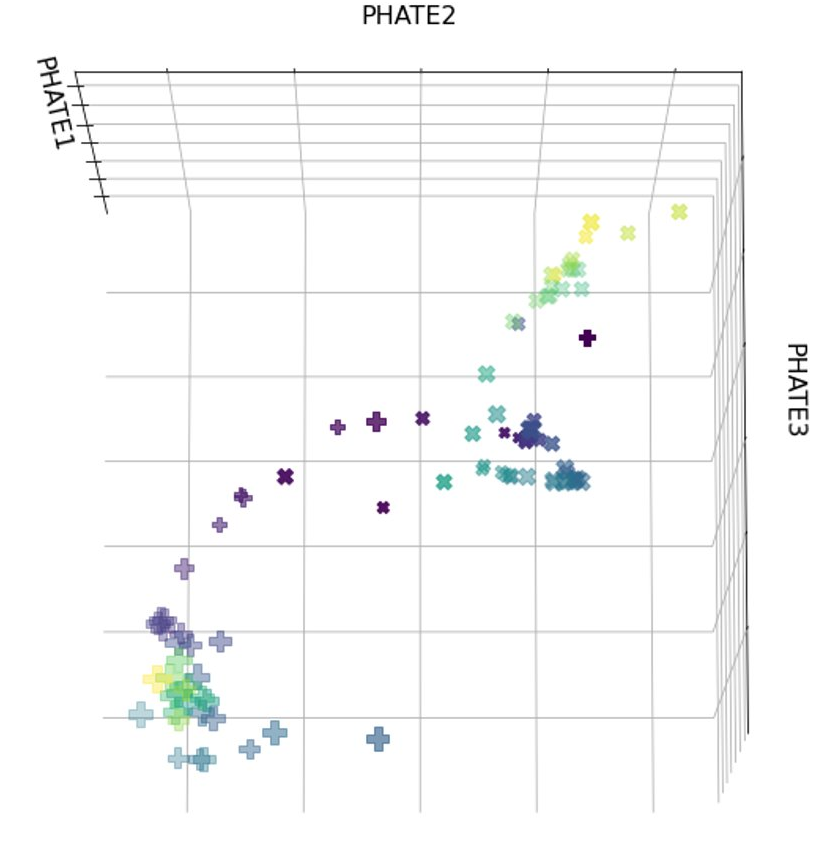}
    \caption{Potential of heat-diffusion for affinity-based transition embedding (PHATE) for subsets of topics in the COVID cluster. The plots demonstrate a 2D (left) and a 3D (right) embedding of two different paths (i.e., health care and politics). Colors and sizes highlight time and tweet volumes, respectively. Here four word clouds containing top $30$ words in corresponding topics are shown for the time points (with arrows connecting the beginning and the end topics on the same path) highlighted by red (health care) and black (politics) circles. Note the plots show divergent behavior of public discourse around COVID-19, where two similar discussions diverge to different discussions (indicated by the word clouds). The 3D embedding illustrates nonlinear paths, that is, spirals and loops, for this topic.}
    \label{fig:phate-covid}
\end{figure}

The two separated paths can be more clearly observed in the 3D view, where a `spiral' structure in the path labeled by filled circles is revealed. This spiral as well as the `loop' presented in Figure~\ref{fig:phate-election-2d} capture sharp transitions of discussions within a topic path, in contrast to more linear structures such as those exhibited in the SANITIZING (wash hands) and the STAY HOME (executive order) clusters, where the discussion is stable over time. In particular, the health care trajectory transitioned from a discussion on general concerns about the coronavirus to testing-focused discussions on a similar topic; the discussions along the presidential election trajectory transitioned from politicians in the presidential race to more general politics. On the other hand, as illustrated in Appendix~\ref{supp:linear-trajectories}, for more linear `wash hands' and `executive order' trajectories, discussions along the paths are quite stable in terms of the most relevant words. We conjecture, more formally, that linear paths geometrically constitute a one-dimensional subspace over which a single multinomial word distribution propagates over time, unaffected by nearby clusters. This represents stability in the discussions of the topic. Nonlinear paths like spirals, on the other hand, likely constitute a nonlinear subspace where the multinomial word distribution changes smoothly over time, affected by proximity to other clusters.

For county-level spatial analysis, three examples of events can be visualized in Figure~\ref{fig:ca_county_map_covid} (following the list of relevant events found at \url{https://en.wikipedia.org/wiki/COVID-19_pandemic_in_California}):
\begin{itemize}
    \item Spatial distribution of COVID tweet proportions on March 1, where the Bay Area is identified as a relative hot spot in the state. Around late February and early March, counties near the Bay Area were first hit by the coronavirus pandemic. For example, cases were reported in Alameda and Solano Counties on that day; a case was reported in Marin County, who was a passenger on the Grand Princess cruise.
    \item On March 11, the first death due to coronavirus was reported in LA County, and Ventura County reported their first case on the day before. These `light up' the two counties on the map as a hot spot.
    \item On March 20 to March 21, Los Angeles County, which is nationally the second-largest municipal health system, announced that it could no longer contain the virus and changed their guidelines for COVID-19 testing to not test symptomatic patients if a positive result would not change their treatment. Note that the Bay Area hot spot before started to `fade away' in terms of Tweets volume proportions.
\end{itemize}

\begin{figure}[!tbh]
    \centering
    \includegraphics[width=0.7\textwidth]{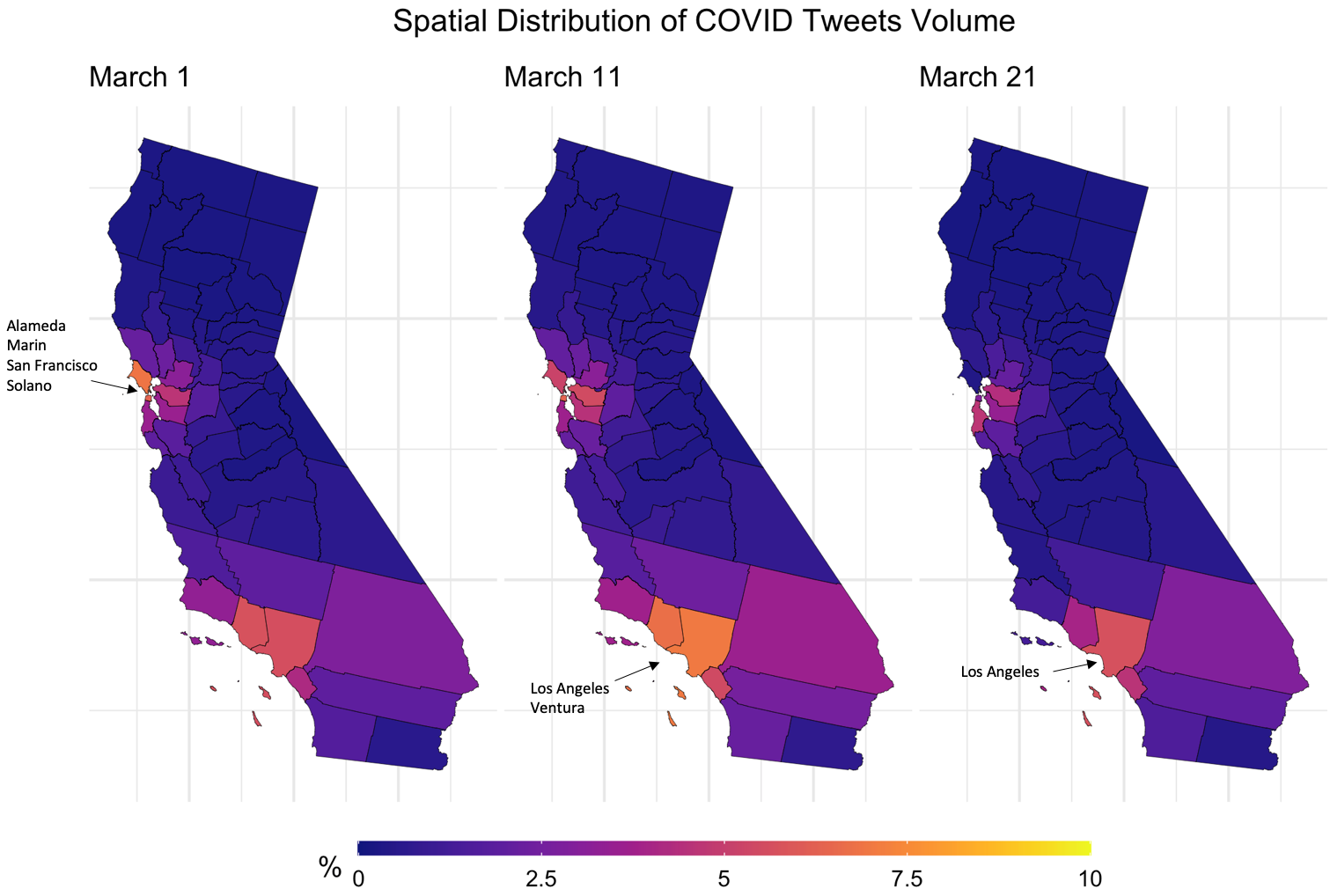}
    \caption{County-level maps for California. It shows the spatial distribution of proportional tweet volumes for three time points on the COVID (health care) path. Note that counties' names are given for spatial hot spots (in terms of tweet volume).}
    \label{fig:ca_county_map_covid}
\end{figure}

\section{TalkLife Data Analysis}\label{sec:gdtm-results-talklife}
For TalkLife data, a similar procedure detailed in Algorithm~\ref{alg:procedure} is applied with a weakly-supervised T-LDA using seed words. Here, we set $\gamma=0.5$ which corresponds to one month of posts for the current time point; the parameter for the number of topics was set to $K=15$ at every time point; the number of neighbors was set to $k=10$ for the neighborhood graph to compute shortest paths.

In contrast to the Twitter analysis, we do not have access to a time period of high-volume data and a diverse set of real-world events (e.g., COVID-19 pandemic, presidential election) that can be compared against each other. Although this results in a lack of ground truth for our dynamical analysis, we do have access to a set of labels generated by human experts. This motivates us to compare the learned topics with these labels in the following subsections.

\subsection{Data preparation}
We use data provided by the TalkLife platform (\url{https://www.talklife.com/research}). All posts from the year 2019 are extracted for further processing. Since the volume of daily posts is much lesser compared to that of Twitter data, we combined posts on a weekly basis and consider ``week'' as our time unit for the following analyses. Volumes are these weekly posts are comparable to the volume of the daily tweets (see Figure~\ref{fig:volume_geo}). Furthermore, similar text preprocessing steps as in the Twitter analysis were undertaken, and the union of the unique words
from each timestamp has been used as the common vocabulary with word frequencies zeroed out on weeks where those words do not occur.

A notable feature of the TalkLife data is the human-generated labels for the posts. There are $33$ labels in total, and each post is tagged by $\geq 0$ labels that describe the underlying/suspected mental health issues embedded in the post. Note that there are cases where more than one label is tagged to a post. In Table~\ref{tab:label_volume} of Appendix~\ref{supp:talklife_data} basic information including percentage volume of each label is shown. Additionally, the labels are used for constructing word-level features that are fed into the weakly supervised LDA model. Specifically, top words (measured by percentage volume) from posts associated with the labels are selected as ``seed words'' that guide the LDA discovery of latent topics. Using the notation from Section~\ref{sec:lda}, for each seed word $n \in V$, an incremental weight of $w_n$ that is proportional to the volume of the corresponding word has been employed in the Dirichlet prior to the multinomial word distribution. Table~\ref{tab:seed_words} of Appendix~\ref{supp:talklife_data} provides details of the seed word selection and the associated prior weights being used in the weakly-supervised T-LDA model. 
 
\subsection{Clustering of labels}
An issue with $33$ labels is that they overlap with each other in terms of the bag-of-words representations of the corresponding posts. For example, the posts with labels ``NauseaWithEatingDisorderSuspected'' and ``NauseaSuspected'' respectively may look similar from the perspective of bag-of-words. As our goal of using the labels is to compare them against the learned topics from an LDA model that is based on the bag-of-words approach, we further cluster the labels into meta-labels in order to reduce noise and redundancy. 

In particular, we extract the associated posts for each given label and construct a word distribution where the weights are computed as the percentage volumes. When using a potentially large set of features, one might expect that the true underlying clusters present in the data differ only with respect to a small fraction of the features, and will be missed if one clusters the observations using the full set of features~\citep{witten2010framework}. Here, since we are training a clustering model on a dataset with $33$ samples where each sample is a $3623$-dimensional vector ($3623 \gg 33$), i.e., a discrete distribution over the vocabulary, we sparsify each feature vector by zeroing out the weights of words that belong to the intersection of the $33$ samples. This procedure will de-emphasize words such as ``feel'', ``sad'', ``upset'', etc. that are common to most mental health-related posts in any clustering algorithm. Moreover, as noted by several researchers~\citep{gopal2014mises,batmanghelich2016nonparametric,meng2019weakly}, clustering text data with Euclidean metrics such as used by the kmeans algorithm or Gaussian mixture models are not appropriate as the data is usually normalized (e.g., term frequency, word counts) and lies on a unit-sphere manifold induced by the Hellinger metric over distribution pairs. Here, taking into consideration the above-mentioned issues, we employ a von Mises-Fisher (vMF) mixture model~\citep{banerjee2005clustering,gopal2014mises} to cluster the label using their sparse representations of word distributions.

We construct $10$ meta-clusters where the number of components is chosen via a Bayesian information criteria (BIC). Figure~\ref{fig:combined_label_topics} visualizes the top words from the merged words distributions. Table~\ref{tab:clustered_labels} and Figure~\ref{fig:sparse_labels} of Appendix~\ref{supp:label_cluster} shows the labels that belong to each of the cluster and a similar top-word visualization for the $33$ samples before clustering, respectively.

\begin{figure}[!tbh]
    \centering
    \includegraphics[width=\textwidth]{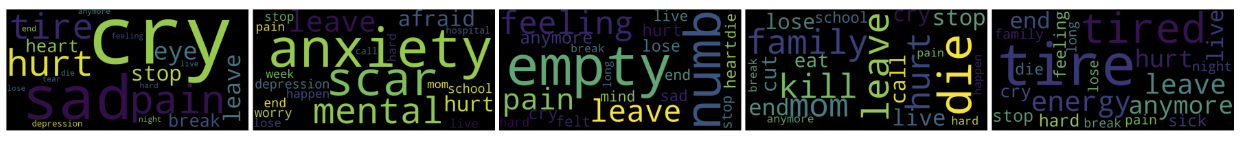}
    \includegraphics[width=\textwidth]{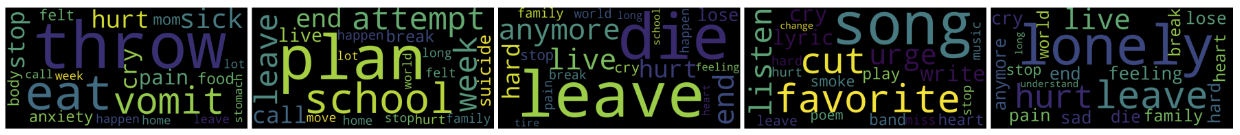}
    \caption{Top words from the merged/clustered words distributions.}
    \label{fig:combined_label_topics}
\end{figure}

\subsection{Learned topics vs. label topics}
In this study we focus on comparing the learned latent topics against the label topics (i.e., meta-clusters). At each timestamp, we compute the dot products between the learned topics with each of the label topic, that is, we compute the dot products between the corresponding weights over the vocabulary. Figure~\ref{fig:top_matches_eat} depicts topics that have the largest (i.e., most similar), smaller (i.e., not quite similar), and the smallest (i.e., least similar) dot products compared with a meta label topic that includes posts labeled by ``NauseaSuspected'' and/or ``NauseaWithEatingDisorderSuspected''. The topics with smaller dot products are clearly unrelated to eating disorder. This comparison is potentially helpful for discovering novel mental health related discussions/topics that have not yet been assigned a proper label by experts. Specifically, the learned topics that consistently result in small dot products with all label topics may indicate such novel discovery. For example, the last topic in the middle row of Figure~\ref{fig:top_matches_eat} may indicate a topic related to sleeping disorder (e.g., insomnia), which has not been included in the label set.

Furthermore, as one increases the level of supervision by increasing the prior weights on the seed words, we expect to see an increased similarity between the learned topics and a given label topic. This is confirmed in Figure~\ref{fig:top_matches_scores_eat} that compares two weighting schemes (with different seed words weights) with an unsupervised T-LDA.

\begin{figure}[!tbh]
    \centering
    \includegraphics[width=\textwidth]{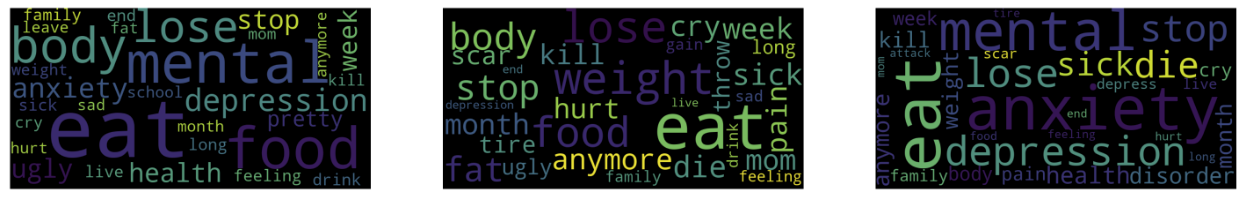}
    \includegraphics[width=\textwidth]{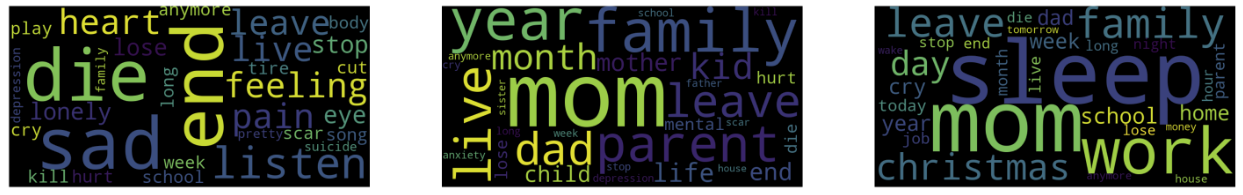}
    \includegraphics[width=\textwidth]{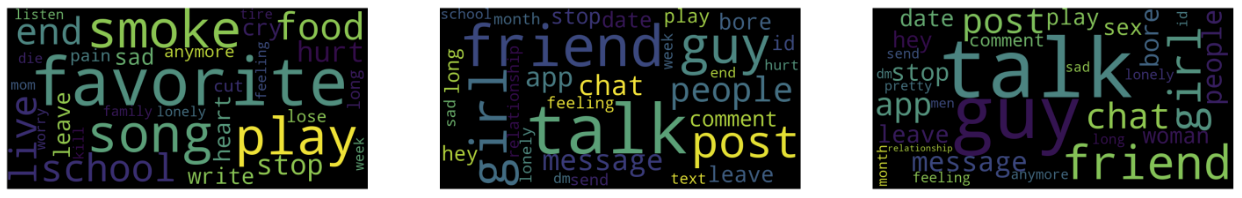}
    \caption{Examples of the most similar (top row), not quite similar (middle row), and the least similar (bottom row) learned topics compared to the label topic under labels ``NauseaSuspected'' and/or ``NauseaWithEatingDisorderSuspected'' at various timestamps. The top row clearly resembles the discussion expected from expert knowledge.}
    \label{fig:top_matches_eat}
\end{figure}

\begin{figure}[!tbh]
    \centering
    \includegraphics[width=\textwidth]{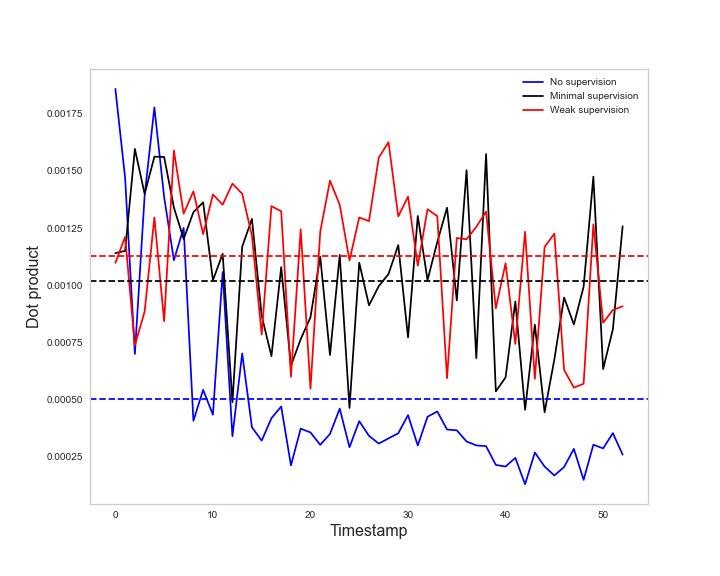}
    \caption{The dot product scores between the learned topics and the label topic under labels ``NauseaSuspected'' and/or ``NauseaWithEatingDisorderSuspected'' across timestamps (52 weeks in 2019), where the horizontal dotted line indicate the average over those timestamps. Here, scores computed from topics learned with no supervision, minimal supervision, and weak supervision are compared -- more supervision results in more similar topics compared with the labels.} 
    \label{fig:top_matches_scores_eat}
\end{figure}

\subsection{Case study: anxiety and suicide topic paths}
In this case study we focus on two critical mental health issues and the corresponding latent topic paths learned using our method: anxiety and suicide ideation. Figure~\ref{fig:phate-anxiety-suicide} depicts the PHATE embedding for these two topic paths. The embedding identifies two paths that together exhibit converging
behavior, which can be considered as types of structures built into PHATE a priori (similar to the splitting structure in case study II of Twitter analysis). Here, two dissimilar discussions on TalkLife around anxiety and suicide, respectively, merge into discussions centered on life-worthlessness and suicidal ideation. The convergence of the two paths is further revealed by clustering of the true labels -- Figure~\ref{tab:clustered_labels} in the appendix shows that the labels ``AnxietyPanicFearSuspected'' and ``SuicidalIdeationAndBehaviorSuspected'', as well as other related labels such as ``AgitationOrIrritationSuspected'', ``SelfHarmRelapseSuspected'' all belong to the same cluster.

Moreover, in terms of the shape of the embedding, we again observe similar curved and spiral structures that occurred in Figure~\ref{fig:phate-covid} and Figure~\ref{fig:phate-election-2d} in the Twitter analysis. For example, we believe the curvy structure appeared on the suicidal path (solid circle) is due to seasonal effect in suicide rates. In particular, a study by the Annenberg Public Policy Center~\citep{suiciderates} found that in 2018 the month with the lowest average daily suicide rate was December with the next-lowest rates in November and January (e.g, winter months). In the same year, the highest rates were in June, July, and August (summer months). Here, the ``inflection point'' in the suicidal path occurred around the beginning of summer with an increase in post volumes. The drop in suicidal rate during the winter months is signified by the convergence of the path with the anxiety path. On the anxiety path (solid cross), there is a similar inflection that occurred around early May, which indicates an increase in the suicidal rates and convergence of the path with the suicidal path. From a mental health point of view, a preexisting anxiety issue is a risk factor for the subsequent onset of suicidal ideation and attempts. This is consistent with published analysis~\citep{sareen2005anxiety}. Further, another change of direction occurred on the anxiety path towards the winter, which coincides with the decrease in suicidal events. After this second transition, we observe on the PHATE plot that the topics in the anxiety path are diverging from the suicidal path and converging with the earlier topics on anxiety.

\begin{figure}[!tbh]
    \centering
    \includegraphics[width=0.45\textwidth]{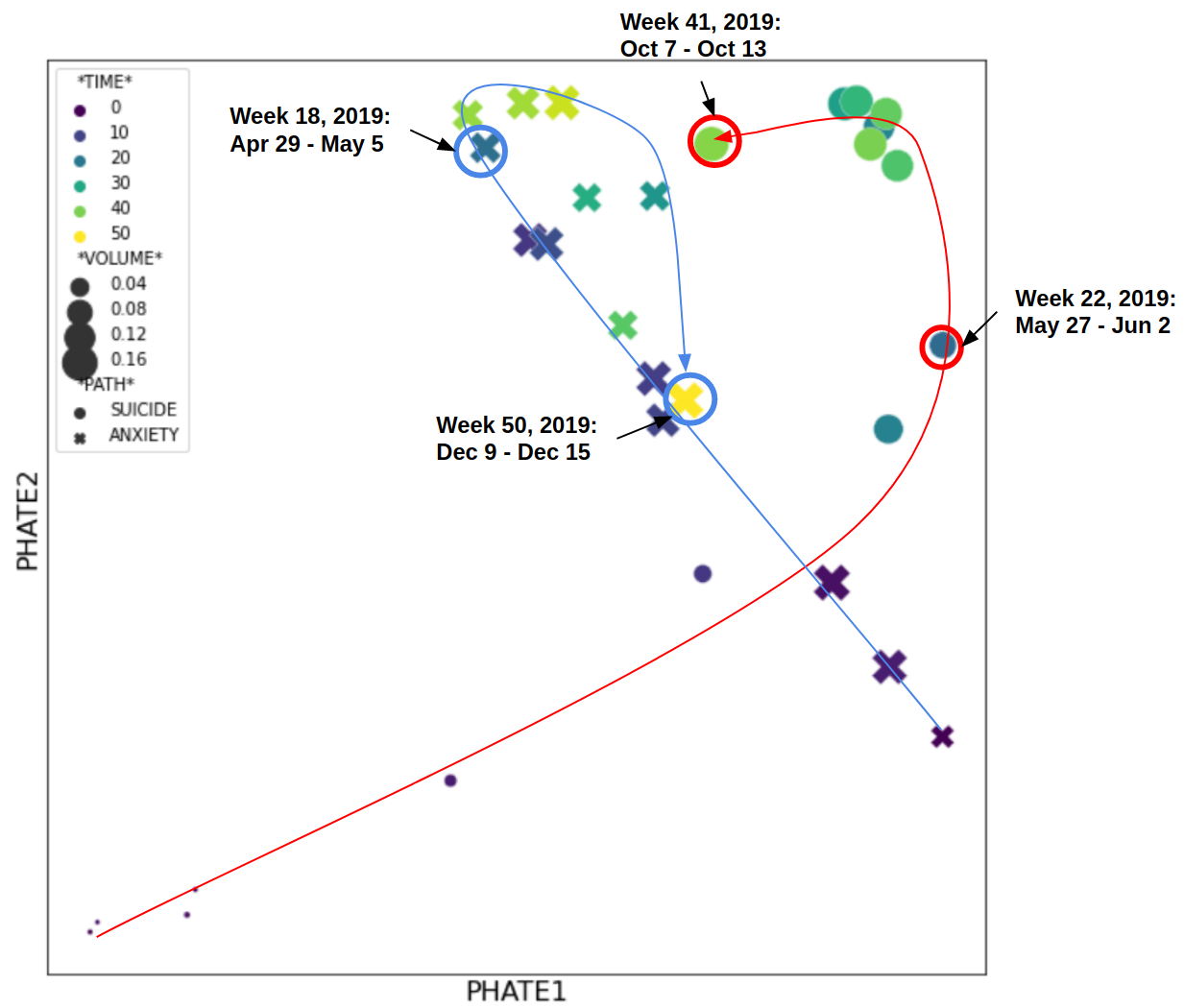} 
    \includegraphics[width=0.45\textwidth]{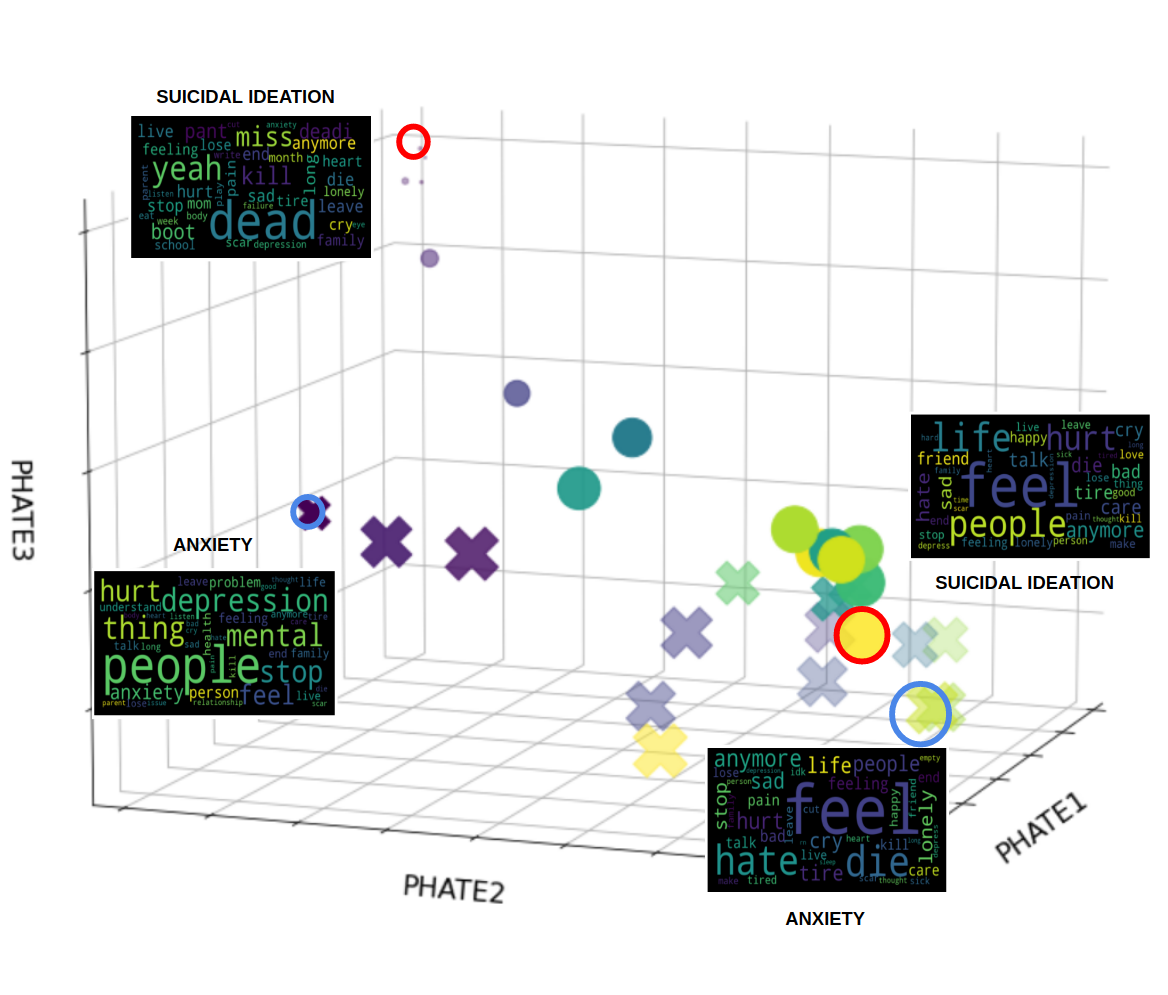}
    \caption{Potential of heat-diffusion for affinity-based transition embedding (PHATE) for two different topic paths. The plots demonstrate a 2D (left) and a 3D (right) embedding of two different paths -- anxiety and suicidal ideation/attempts. Colors and sizes highlight time (52 weeks in 2019) and posts volumes, respectively. Here four word clouds containing top $30$ words in corresponding topics are shown for the time points highlighted by red (suicide) and blue (anxiety) circles on the 3D plot. On the 2D plot, arrows are drawn connecting the beginning and the end topics on the same path with circles emphasizing several key time points. Note the plots show convergent behavior of these two temporal topic paths, where two dissimilar discussions converge to similar discussions (indicated by the word clouds). The 3D embedding further confirms this converging behavior.}
    \label{fig:phate-anxiety-suicide}
\end{figure}

\section{Conclusion}\label{sec:gdtm-conclusion}
We proposed a framework for longitudinal analysis of text data, combining tools from graph algorithms, statistics, and computational geometry. The proposed procedure works by linking together marginal topic-word distributions discovered by a `regularized' LDA model designed for micro-text, via Hellinger distances and shortest paths on neighborhood graphs. The resulting chain of topics can then be visualized by PHATE dimensionality reduction, which preserves the progressive nature of the input data. With this framework, we discovered and interpreted how certain conversations split and merged under the impact of the COVID-19 pandemic, which can be validated by associating with real-world events. Granular-level spatial analyses showed that our framework is able to capture both global (in the United States) and local variations of COVID-19-related discussions. We further extended the framework to incorporate side information via weak supervision in the form of seed words. With TalkLife data, this extension has been shown to be able to capture latent topics that coincide with expert knowledge. Finally, we believe that social media data could be used to supplement traditional health care or census data to provide fresh insights into the impact of events, like the pandemic, on society, as well as to aid study of mental health issues.

\subsection{Limitations}
There are several limitations of our analysis  that deserve additional attention. First, as with most statistical algorithms, there are user tuning parameters that must be selected. There are three tuning parameters that the user must provide: 1) the numbers of nearest neighbors $k$ in the $k$ nearest neighbor graph; 2) the data smoothing parameter $\gamma$; and 3) the number of topics $K$ for the T-LDA algorithm. We have shown that our results are robust to perturbations about the parameters we chose, but there may be better choices. These include comprehensive cross-validation methods which, with sufficient computational resources, can be used to reliably select parameters that minimize a loss function. Such methods have been proposed for selecting $k$. For selection of $K$, a promising option is the hierarchical Dirichlet process (HDP), a nonparametric Bayesian model for the number of topics that could vary over time and model birth and death of topics~\citep{teh2006hierarchical}. Our wrapper framework could easily incorporate an HDP in place of the T-LDA model, but at the expense of increased computation. Two more challenging limitations are those of selection bias and model bias. 

\paragraph{Selection biases} 
\label{section:bias}
The use of Twitter data for studying public discourse may be subject to selection bias as users of Twitter may not be representative of the U.S. population. Additionally, users of Twitter may be engaged in different types of public discourses around COVID-19 than users of other social media platforms, for example, Facebook and Reddit, which have different user demographics and privacy policies. Different types of subsampling of Tweets may create their own biases. For example, subsampling based on retweet status, geotag information, country, and time range (e.g., Feb 15 to May 15) are all subject to selection biases. Our subsampling procedure may leave out some important information. For example, we did not consider any retweets, which may contain information on how popular a particular topic might be. Retweets could possibly shed light on a particular topic, which can be measured, for example, by the longitudinal distribution of retweet frequencies for the topic. However, we could not perform a retweet analysis on our geotagged tweets since Twitter does not allow retweets to be geotagged. We also leave out tweets that are generated from U.S. users who are outside of the United States.

\paragraph{Model biases}
The LDA algorithm we have applied to topic modeling summarizes unstructured texts by themes or topics using a \textit{bag-of-words} approach. This particular approach is computationally scalable but it ignores the relative order of words. For example, a topic about `vaccines are not available' can be very close to a topic on `vaccines are available, but not to me.' The issue may be alleviated by using more sophisticated representations, for example,  bigrams or latent semantic analysis. This would result in higher computational burden--the length of unique phrases would increase exponentially as the word order dimension. Other approaches that attempt to model the semantic meaning of topics, such as deep neural networks, could also be used. Additionally, our construction of the smoothed corpora assumes temporal similarities between tweets generated at adjacent time points. Similar types of smoothing assumptions are common in other areas of spatiotemporal statistics as described in Section~\ref{sec:temporal}. The manifold hypothesis is also essential in our model for the shortest path algorithms to  recover the intrinsic similarities between topics over time.


%% file: Conclusion/conclusion.tex
\section{Summary}
This thesis focused on statistical methods concerning data with spatio-temporal structure. In particular, it touched on two separated research areas:
\begin{itemize}
    \item Tensor-variate Gaussian graphical models and its connections to and applications in spatio-temporal physical processes
    \item Temporal topic modeling in both unsupervised and weakly-supervised settings with applications to analyses of public opinions and mental health
\end{itemize}

These areas are connected in that in each of them there exists structured and graphical representations of the data, and it is imperative to utilize these interpretable representations to achieve scalable and valid inference. This dissertation advanced the state-of-the-art by introducing a new class of Gaussian graphical model for tensor-valued data, designing new estimation algorithms with fast convergence, introducing a new framework for filtering and data assimilation, and finally describing a new approach to temporal topic modeling. There are multiple fruitful extensions of these methodology that warrant further investigation, which we discuss next.

\section{Future Work}
\paragraph{Physical interpretability.} While the Kronecker products expansion used in Kronecker PCA captures dense structures in the covariance matrix of data generated from more complex spatio-temporal physical processes as illustrated in Chapter~\ref{ch:enkf}, it lacks physical interpretability. In contrast to the case of Sylvester graphical model and Poisson-AR(1) processes, it is not obvious whether the sum of Kronecker products structure corresponds to any true physical models. Recent development in quantum informatics~\citep{chu2021nonlinear} has demonstrated a link between estimation of the density matrix for entangled quantum states and the structured tensor approximation via $\sum_{i=1}^l \mat{A}_i \otimes \mat{B}_i$. Further characterizing and extending these connections to other classes of discretized PDEs is an interesting future direction. Furthermore, in both the blocked Poisson-AR(1) and convection-diffusion examples, a mixed Kronecker sum and Kronecker product structure emerges that can be related to the state inverse covariance of a dynamical system.

\paragraph{Heavy-tailed multiway covariance/precision models.} Most existing work on multiway covariance and inverse covariance models focus on modeling Gaussian variables. It would be interesting to explore whether the pseudo-likelihood framework we adopted for SyGlasso and SG-PALM can be extended to non-Gaussian heavy-tailed models, e.g., using copula's or  elliptically contoured distributions. This could have important practical applications, in particular to solar flare and active region prediction problems presented in Chapter~\ref{ch:sgpalm}. The images that characterize the active regions generally include a small number of pixels of extreme high-intensity. These pixels might not be captured by a Gaussian-like distribution. Recently, there have been advances~\citep{wei2017heavytailed,ke2019user} in covariance estimation for heavy-tailed, non-Gaussian vector-variate data. Multiway (inverse) covariance estimation is an open  problem. Furthermore, robust Kronecker structured covariance / correlation models where robust estimator of the correlation matrix with sparse Kronecker structure has been recently studied for high-dimensional matrix-variate data~\citep{niu2020robust}. But, there are still open problems such as theoretical guarantees (comparable to those of traditional methods) and efficient computational algorithms that warrant future development.

\paragraph{Kronecker-structured autoencoders.} Low-rank covariance models have close connections with variational autoencoders (VAEs). \citet{dai2018connections} studied the relationship between (robust) PCA and VAEs. Since the Kronecker product for matrices is a generalization of the outer product for vectors, KPCA can be considered as a generalization of a the low-rank approximation method of PCA. It is thus natural to exploit similar relationships between KPCA and VAEs. In this case VAE may be considered as a nonlinear/non-Gaussian extension to KPCA for low separation rank covariance models. Additionally, recent advances in efficient training of the VAE-type neural network architecture (e.g., using stochastic gradient descent) could improve the computational complexity of KPCA that is currently limited by an expensive singular value decomposition~\citep{tsiligkaridis2013covariance,greenewald2015robust}.

\paragraph{Model selection for Kronecker-structured models.} Each of the KP, KS, or Sylvester structure has its pros and cons and is appealing only under appropriate data generating processes. It is still unclear, for a given data problem with unknown underlying generative process, how to choose among various Kronecker-structured models. This problem has been attracting attentions only very recently -- \citet{guggenberger2022test} developed a procedure for testing for a covariance matrix to have Kronecker product structure. However, the method proposed relies on an expensive rank test procedure~\citep{kleibergen2006generalized} that is not scalable to modern big-data applications. Moreover, it is still an open problem to develop similar tests for Kronecker sum and Sylvester structures in either the covariance or its inverse.

\paragraph{Theoretical analysis of geometry-driven dynamic topic models.}
The ``nonparametric'' geometry-driven framework proposed in Chapter~\ref{ch:gdtm} shows promising results in recovering perceptually natural temporal dynamics that may exist among data. We demonstrated the ``closeness'' of the recovered chain of topics to a series of real events happened around the same time period. However, from a theoretical point of view, it is desirable to understand whether the estimated topic chain approximates well the truth. Just as statisticians have studied when least-squares regression can estimate the ``true'' regression model, it is natural and important for us to study the ability of the computational geometric algorithms to estimate the ``true'' topic path in a stochastic topic model. 

Researchers have explored the performance of nonparametric algorithms that are based on heuristics or insights on the underlying problems under certain statistical/stochastic models. For example, \citet{rohe2011spectral} showed the consistency of the spectral clustering algorithms in identifying clusters in network data under a true network generated from the Stochastic Blockmodel~\citep{holland1983stochastic}. \citet{bickel2009nonparametric} proved that, also under the Stochastic Blockmodel, a nonparametric community detection algorithm called the Newman–Girvan modularity~\citep{newman2004finding} are asymptotically consistent estimators of block partitions. 

Akin to these work of studying the performance of nonparametric methods on parametric tasks of estimating quantities in statistical models, we propose to study the consistency of the geometry-driven topic modeling algorithm in identifying the true topic path, under topics generated by the DTM model proposed in~\citet{blei2006dynamic}. More specifically, under DTM, the generative process at a time stamp $t$ is
\begin{enumerate}
    \item Draw $\beta_{t,k}|\beta_{t-1,k} \sim \mathcal{N}(\beta_{t-1,k},\sigma^2\mat{I}), \forall k$
    \item Draw $\alpha_t|\alpha_{t-1} \sim \mathcal{N}(\alpha_{t-1},\delta^2\mat{I})$
    \item For each document:
    \begin{enumerate}
        \item Draw $\eta_{t,d} \sim \mathcal{N}(\alpha_t,a^2\mat{I})$
        \item For each word:
        \begin{enumerate}
            \item Draw topic $Z_{t,d,n} \sim \text{Multi}(\pi(\eta_{t,d}))$
            \item Draw $W_{t,d,n} \sim \text{Multi}(\pi(\beta_{t,Z_{t,d,n}}))$,
        \end{enumerate}
    \end{enumerate}
\end{enumerate}
where $\pi(x)$ is a mapping from the natural parameterization $x$ to the mean parameterization. Here, define
\begin{itemize}
    \item[] $\alpha_{t}$ as the per-document topic distribution at time $t$.
    \item[] $\beta_{t,j}$ as the word distribution of topic $k$ at time $t$.
    \item[] $\eta_{t,d}$ as the topic distribution for document $d$ at time $t$.
    \item[] $z_{t,d,n}$ as the topic for the $n$th word in document $d$ in time $t$.
    \item[] $w_{t,d,n}$ as the word.
\end{itemize}

A plausible direction in proving performance of the nonparametric topic modeling approach proposed in Chapter~\ref{ch:gdtm} under DTM is to characterize the distances between the recovered topic paths to the true paths, i.e., $\beta_{t,k}$'s, for $t=1,\dots$ and show that these distances vanish as both the length of the documents and the number of latent topics grow to infinity. 

%% file: Appendices/Appendix_A.tex
In this Appendix,
\begin{itemize}
    \item[] Section~\ref{supp:syglasso_alg_derivation} provides the detailed derivation of the updates for Algorithm \ref{alg:nodewise_tensor_lasso};
    \item[] Section~\ref{supp:syglasso_convrg_proofs} provides the proofs of theorems stated in Section~\ref{sec:syglasso-thm}; 
    \item[] Section~\ref{supp:simulated_precision_matrix} provides details on the simulated data in Section~\ref{sec:syglasso-experiments}. 
\end{itemize}

\section{Derivation of the Nodewise Tensor Lasso Estimator}
\label{supp:syglasso_alg_derivation}

\subsection{Off-Diagonal updates}
\label{supp:syglasso_derivation_offdiag}
For $1 \leq i_k < j_k \leq m_k$, $T_{i_kj_k}(\mat{\Psi}_k^{\text{off}})$ can be computed in closed form:
\begin{equation}\label{eqn:update_offdiag}
    (T_{i_kj_k}(\mat{\Psi}_k))_{i_kj_k}^{\text{off}} = 
    \frac{S_{\frac{\lambda_k}{N}}\Big(F_{\tensor{X},\{\mat{\Psi}_k\}_{k=1}^K}\Big)}
    {
    (\frac{1}{N}\tensor{X}_{(k)}\tensor{X}_{(k)}^T)_{i_ki_k} + (\frac{1}{N}\tensor{X}_{(k)}\tensor{X}_{(k)}^T)_{j_kj_k}
    },
\end{equation}
where
\begin{align*}
        F_{\tensor{X},\{\mat{\Psi}_k\}_{k=1}^K} = - \frac{1}{N} &
        \Bigg(\Big((\tensor{W}_{(k)} \circ \tensor{X}_{(k)}) \tensor{X}_{(k)}^T \Big)_{i_kj_k} + \Big((\tensor{W}_{(k)} \circ \tensor{X}_{(k)}) \tensor{X}_{(k)}^T \Big)_{j_ki_k} \\
        & + \Big(\tensor{X}_{(k)}(\tensor{X} \times_k \mat{\Psi}_k^{\text{off},i_kj_k})^T_{(k)}\Big)_{j_ki_k} + \Big(\tensor{X}_{(k)}(\tensor{X} \times_k \mat{\Psi}_k^{\text{off},i_kj_k})^T_{(k)}\Big)_{i_kj_k} \\
        & \quad + \sum_{l \neq k} \Big(\tensor{X}_{(k)}(\tensor{X} \times_l \mat{\Psi}_l^{\text{off}})^T_{(k)}\Big)_{i_kj_k} + \sum_{l \neq k} \Big(\tensor{X}_{(k)}(\tensor{X} \times_l \mat{\Psi}_l^{\text{off}})^T_{(k)}\Big)_{j_ki_k}
        \Bigg).
\end{align*} 
Here the $\circ$ operator denotes the Hadamard product between matrices; $\mat{\Psi}_k^{\text{off},i_kj_k}$ is $\mat{\Psi}_k^{\text{off}}$ with the $(i_k,j_k)$ entry being zero; and $S_{\lambda}(x):=\text{sign}(x)(|x|-\lambda)_{+}$ is the soft-thresholding operator. 

\subsection{Diagonal updates}
\label{supp:syglasso_derivation_diag}
For $\tensor{W}$,
\begin{equation}
\label{eqn:update_diag}
(T(\tensor{W}))_{i_{[1:K]}}
    = \frac{-\Big(\tensor{X}_{(N)}^T\tensor{Y}_{(N)}\Big)_{i_{[1:K]}}+\sqrt{\Big(\tensor{X}_{(N)}^T\tensor{Y}_{(N)}\Big)_{i_{[1:K]}}^2+4\Big(\tensor{X}_{(N)}\tensor{X}_{(N)}^T\Big)_{i_{[1:K]}}}}
    {2 \Big(\tensor{X}_{(N)}\tensor{X}_{(N)}^T\Big)_{i_{[1:K]}}}.
\end{equation}
Here we define $\tensor{Y}:=\sum_{k=1}^K \Big(\tensor{X} \times_k \mat{\Psi}_k^{\text{off}} \Big)$. Equations \eqref{eqn:update_offdiag} and \eqref{eqn:update_diag} give necessary ingredients for designing a coordinate descent approach to minimizing the objective function in \eqref{eqn:objective}. The optimization procedure is summarized in Algorithm \ref{alg:nodewise_tensor_lasso}.
\subsection{Derivation of updates}
Note that for $1 \leq i_k < j_k \leq m_k$, $1 \leq k \leq K$,
\begin{align*}
    & Q_N(\{\mat{\Psi}_k\}_{k=1}^K) \\ 
    & = (N/2)\Big(\sum_{i_{[1:k-1,k+1:K]}} ({\tensor{X}_{i_{[1:K]}}^{i_k}}^2 +{\tensor{X}_{i_{[1:K]}}^{j_k}}^2)\Big)\Big((\mat{\Psi}_k)_{i_kj_k}\Big)^2 \\ 
    & + N F_{\tensor{X},\{\mat{\Psi}\}_{k=1}^K} (\mat{\Psi}_k)_{i_kj_k} + \lambda_k|(\mat{\Psi}_k)_{i_kj_k}| \\ 
    & + \text{terms independent of $(\mat{\Psi}_k)_{i_kj_k}$},
\end{align*} where

\begin{equation*}
  \begin{aligned}
        F_{\tensor{X},\{\mat{\Psi}\}_{k=1}^K} = - \sum_{i_{[1:k-1,k+1:K]}} &
        \Big( \tensor{W}_{i_{[1:K]}}^{i_k} \tensor{X}_{i_{[1:K]}}^{i_k}\tensor{X}_{i_{[1:K]}}^{j_k} + \tensor{W}_{i_{[1:K]}}^{j_k} \tensor{X}_{i_{[1:K]}}^{j_k}\tensor{X}_{i_{[1:K]}}^{i_k}\\
        & \quad + (\mat{\Psi}_k)_{i_k,\text{\textbackslash} \{i_k,j_k\}}^T \tensor{X}_{i_{[1:K]}}^{\text{\textbackslash} \{i_k,j_k\}} \tensor{X}_{i_{[1:K]}}^{j_k}  \\
        & \quad + (\mat{\Psi}_k)_{j_k,\text{\textbackslash} \{i_k,j_k\}}^T \tensor{X}_{i_{[1:K]}}^{\text{\textbackslash} \{i_k,j_k\}}
        \tensor{X}_{i_{[1:K]}}^{i_k}\\
        & \quad + \sum_{l \in [1:k-1,k+1:K]}
        (\mat{\Psi}_l)_{i_l,\text{\textbackslash} i_l}^T \tensor{X}_{i_{[1:K]}}^{i_k,\text{\textbackslash} i_l} \tensor{X}_{i_{[1:K]}}^{j_k} \\
        & \quad + \sum_{l \in [1:k-1,k+1:K]}
        (\mat{\Psi}_l)_{i_l,\text{\textbackslash} i_l}^T \tensor{X}_{i_{[1:K]}}^{j_k,\text{\textbackslash} i_l} \tensor{X}_{i_{[1:K]}}^{i_k}
         \Big).
    \end{aligned}
\end{equation*} Here $\tensor{X}_{i_{[1:K]}}^{i_k}$ denotes the element of $\tensor{X}$ indexed by $i_{[1:K]}$ except that the $k$th index is replaced by $i_k$ and $\tensor{X}_{i_{[1:K]}}^{i_k,j_l}$ denotes the element of $\tensor{X}$ indexed by $i_{[1:K]}$ except that the $k,l$th indices are replaced by $i_k,j_l$. Note the following equivalence:
\begin{align*}
    & \sum_{i_{[1:k-1,k+1:K]}} \tensor{W}_{i_{[1:K]}}^{i_k} \tensor{X}_{i_{[1:K]}}^{i_k}\tensor{X}_{i_{[1:K]}}^{j_k}=\Big((\tensor{W}_{(k)} \circ \tensor{X}_{(k)}) \tensor{X}_{(k)}^T \Big)_{i_kj_k} \\ & \sum_{i_{[1:k-1,k+1:K]}}\tensor{X}_{i_{[1:K]}}^{i_k}\tensor{X}_{i_{[1:K]}}^{j_k} = (\tensor{X}_{(k)}\tensor{X}_{(k)}^T)_{i_kj_k} \\ & \sum_{i_{[1:k-1,k+1:K]}}(\mat{\Psi}_l)_{i_l,.}^T \tensor{X}_{i_{[1:K]}}^{i_k,.} \tensor{X}_{i_{[1:K]}}^{j_k} = \Big(\tensor{X}_{(k)}(\tensor{X}\times_l\mat{\Psi}_l)_{(k)}^T\Big)_{j_ki_k},
\end{align*}
where $\tensor{W}$ is a tensor of the same dimensions of $\tensor{X}$, formed by tensorize values in $\tensor{W}$, and in the case of $N>1$ the last mode of $\tensor{W}$ is the observation mode similarly to $\tensor{X}$ but with exact replicates. Using the tensor notation and standard sub-differential method, Equation \eqref{eqn:update_offdiag} then follows. 

For $\tensor{W}_{i_{[1:K]}}$, using similar tensor operations,
\begin{align*}
    & \frac{\partial}{\partial \tensor{W}_{i_{[1:K]}}} Q_N(\tensor{W},\{\mat{\Psi}_k^{\text{off}}\}_{k=1}^K)  = 0 \\
    & \iff -\frac{1}{\tensor{W}_{i_{[1:K]}}} + \tensor{W}_{i_{[1:K]}}^2 \tensor{X}_{i_{[1:K]}}^2 + \tensor{W}_{i_{[1:K]}}\Big(\tensor{X}_{i_{[1:K]}}\sum_{k=1}^K(\tensor{X} \times_k \mat{\Psi}_k^{\text{off}})_{i_{[1:K]}})\Big) = 0 \\
    & \iff \tensor{W}_{i_{[1:K]}}^2 \Big(\tensor{X}_{(N)}^T\tensor{X}_{(N)}\Big)_{i_{[1:K]}} + \tensor{W}_{i_{[1:K]}} \Big(\tensor{X}_{(N)}^T\sum_{k=1}^K(\tensor{X} \times_k \mat{\Psi}_k^{\text{off}})\Big)_{i_{[1:K]}} - 1 = 0
\end{align*} which is a quadratic equation in $\tensor{W}_{i_{[1:K]}}$ and since $\tensor{W}_{i_{[1:K]}}>0$, so the positive root has been retained as the solution. Note that the estimation for one entry of $\tensor{W}$ is independent of the other entries. So during the estimation process we update all the entries at once by noting that $\diag\Big(\tensor{X}_{(N)}^T\tensor{X}_{(N)}\Big)=\Big(\Big(\tensor{X}_{(N)}^T\tensor{X}_{(N)}\Big)_{i_{[1:K]}}, \forall i_{[1:K]} \Big)$.

\section{Proofs of Main Theorems}\label{supp:syglasso_convrg_proofs}
We first list some properties of the loss function.

\begin{lemma}
The following is true for the loss function:
\begin{enumerate}[label=(\roman*)]
    \item There exist constants $0 < \Lambda_{\min}^L \leq \Lambda_{\max}^L < \infty$ such that for $\mathcal{S}_{k}:=\{(i_k,j_k):1 \leq i_k < j_k \leq m_k\},k=1,\dots,K$,
    \begin{equation*}
        \Lambda_{\min}^L \leq \lambda_{\min}(\bar{L}^{\prime\prime}_{\mathcal{S}_{k},\mathcal{S}_{k}}(\bar{\bm{\beta}})) \leq \lambda_{\max}(\bar{L}^{\prime\prime}_{\mathcal{S}_{k},\mathcal{S}_{k}}(\bar{\bm{\beta}})) \leq \Lambda_{\max}^L
    \end{equation*}
    \item There exists a constant $K(\bar{\bm{\beta}})<\infty$ such that for all $1 \leq i_k < j_k \leq m_k$, $\bar{L}^{\prime\prime}_{i_kj_k,i_kj_k}(\bar{\bm{\beta}}) \leq K(\bar{\bm{\beta}})$
    \item There exist constant $M_1(\bar{\bm{\beta}}), M_2(\bar{\bm{\beta}})<\infty$, such that for any $1 \leq i_k < j_k \leq m_k$
    \begin{equation*}
        \Var_{\bar{\tensor{W}},\bar{\bm{\beta}}}(L^{\prime}_{i_kj_k}(\bar{\tensor{W}},\bar{\bm{\beta}},\tensor{X})) \leq M_1(\bar{\bm{\beta}}), \ \Var_{\bar{\tensor{W}},\bar{\bm{\beta}}}(L^{\prime\prime}_{i_kj_k,i_kj_k}(\bar{\tensor{W}},\bar{\bm{\beta}},\tensor{X})) \leq M_2(\bar{\bm{\beta}})
    \end{equation*}
    \item There exists a constant $0 < g(\bar{\bm{\beta}}) <\infty$, such that for all $(i,j) \in \mathcal{A}_{k}$
    \begin{equation*}
        \bar{L}^{\prime\prime}_{ij,ij}(\bar{\tensor{W}},\bar{\bm{\beta}}) - \bar{L}^{\prime\prime}_{ij,\mathcal{A}_{k}^{ij}}(\bar{\tensor{W}},\bar{\bm{\beta}})[\bar{L}^{\prime\prime}_{\mathcal{A}_{k}^{ij},\mathcal{A}_{k}^{ij}}(\bar{\tensor{W}},\bar{\bm{\beta}})]^{-1}\bar{L}^{\prime\prime}_{\mathcal{A}_{k}^{ij},ij}(\bar{\tensor{W}},\bar{\bm{\beta}}) \geq g(\bar{\bm{\beta}}),
    \end{equation*} where $\mathcal{A}_{k}^{ij}:=\mathcal{A}_{k}/\{(i,j)\}$.
    \item There exists a constant $M(\bar{\bm{\beta}})<\infty$, such that for any $(i,j) \in \mathcal{A}_{k}^c$
    \begin{equation*}
        \|\bar{L}^{\prime\prime}_{ij,\mathcal{A}_{k}}(\bar{\tensor{W}},\bar{\bm{\beta}})[\bar{L}^{\prime\prime}_{\mathcal{A}_{k},\mathcal{A}_{k}}(\bar{\tensor{W}},\bar{\bm{\beta}})]^{-1}\|_2 \leq M(\bar{\bm{\beta}}).
    \end{equation*}
\end{enumerate}
\end{lemma}

\begin{proof}[proof of Lemma A.2.1.]
We prove $(i)$. $(ii-v)$ are then direct consequences, and the proofs follow from the proofs of B1.1-B1.4 in \citet{peng2009partial}, with the modifications being that the indexing is now with respect to each $k$ for $1 \leq k \leq K$.

Consider the loss function in matrix form as in \eqref{eqn:syglasso_objective_matrix}. Then $\bar{L}^{\prime\prime}_{\mathcal{S}_{k},\mathcal{S}_{k}}(\bar{\bm{\beta}})$ is equivalent to 
\begin{align*}
& \frac{\partial^2}{\partial\mat{\Psi}_k^{\text{off}} \partial\mat{\Psi}_k^{\text{off}} } L(\tensor{W},\{\mat{\Psi}_k^{\text{off}}\}_{k=1}^K) \\
& = \frac{\partial^2}{\partial\mat{\Psi}_k^{\text{off}} \partial\mat{\Psi}_k^{\text{off}}} \Bigg(\tr(\mat{\Psi}_k^T \mat{S} \mat{\Psi}_k) + \text{first order terms in $\mat{\Psi}_k$} + \text{terms independent of $\mat{\Psi}_k$} \Bigg) \\
& = \frac{\partial^2}{\partial\mat{\Psi}_k^{\text{off}} \partial\mat{\Psi}_k^{\text{off}}} \Bigg(\tr((\mat{\Psi}_k^{\text{off}} + \text{diag}(\mat{\Psi}_k))^T \mat{S} (\mat{\Psi}_k^{\text{off}} + \text{diag}(\mat{\Psi}_k))) + \text{first order terms in $\mat{\Psi}_k^{\text{off}}$} \\
& \qquad \qquad \qquad \quad + \text{terms independent of $\mat{\Psi}_k^{\text{off}}$} \Bigg) \\ 
& = \frac{\partial^2}{\partial\mat{\Psi}_k^{\text{off}} \partial\mat{\Psi}_k^{\text{off}}} \Bigg(\tr((\mat{\Psi}_k^{\text{off}})^T \mat{S} \mat{\Psi}_k^{\text{off}}) + \text{first order terms in $\mat{\Psi}_k^{\text{off}}$} \\
& \qquad \qquad \qquad \quad + \text{terms independent of $\mat{\Psi}_k^{\text{off}}$} \Bigg) \\
& = \mat{S} = \frac{1}{N}\vecto(\tensor{X})^T\vecto(\tensor{X}).
\end{align*}
Thus $\bar{L}^{\prime\prime}_{\mathcal{S}_{k},\mathcal{S}_{k}}(\bm{\beta})=E_{\tensor{W},\bm{\beta}}(\mat{S})$. Then for any non-zero $\mat{a} \in \mathbb{R}^p$, we have
\begin{equation*}
    \mat{a}^T \bar{L}^{\prime\prime}_{\mathcal{S}_{k},\mathcal{S}_{k}}(\bar{\bm{\beta}}) \mat{a} = \mat{a}^T \mat{\bar{\Sigma}} \mat{a} \geq \|\mat{a}\|_2^2 \lambda_{\min}(\bar{\mat{\Sigma}}).
\end{equation*}Similarly, $\mat{a}^T \bar{L}^{\prime\prime}_{\mathcal{S}_{k},\mathcal{S}_{k}}(\bar{\bm{\beta}}) \mat{a} \leq \|\mat{a}\|_2^2 \lambda_{\max}(\bar{\mat{\Sigma}})$. By (A2), $\bar{\mat{\Sigma}}$ has bounded eigenvalues, thus the lemma is proved.

\end{proof}

\begin{lemma}
Suppose conditions (A1-A2) hold, then for any $\eta>0$, there exist constant $c_{0,\eta},c_{1,\eta},c_{2,\eta},c_{3,\eta}$, such that for any $u \in \mathbb{R}^{q_{k}}$ the following events hold with probability at least $1-O(\exp(-\eta \log p))$ for sufficiently large $N$:
\begin{enumerate}[label=(\roman*)]
    \item $\|L_{N,\mathcal{A}_{k}}^{\prime}(\bar{\tensor{W}},\bar{\bm{\beta}},\tensor{X})\|_2 \leq c_{0,\eta}\sqrt{q_{k}\frac{\log p}{N}}$
    \item $|u^T L_{N,\mathcal{A}_{k}}^{\prime}(\bar{\tensor{W}},\bar{\bm{\beta}},\tensor{X})| \leq c_{1,\eta}\|u\|_2\sqrt{q_{k}\frac{\log p}{N}}$
    \item $|u^T L_{N,\mathcal{A}_{k}\mathcal{A}_{k}}^{\prime\prime}(\bar{\tensor{W}},\bar{\bm{\beta}},\tensor{X})u - u^T \bar{L}^{\prime\prime}_{\mathcal{A}_{k}\mathcal{A}_{k}}(\bar{\bm{\beta}})u| \leq c_{2,\eta}\|u\|_2^2q_{k}\sqrt{\frac{\log p}{N}}$
    \item $|L_{N,\mathcal{A}_{k}\mathcal{A}_{k}}^{\prime\prime}(\bar{\tensor{W}},\bar{\bm{\beta}},\tensor{X})u - \bar{L}^{\prime\prime}_{\mathcal{A}_{k}\mathcal{A}_{k}}(\bar{\bm{\beta}})u| \leq c_{3,\eta}\|u\|_2^2q_{k}\sqrt{\frac{\log p}{N}}$
\end{enumerate}
\end{lemma}

\begin{proof}[proof of Lemma A.2.2]
$(i)$ By Cauchy-Schwartz inequality,
\begin{equation*}
    \|L^{\prime}_{N,\mathcal{A}_{k}}(\bar{\tensor{W}},\bar{\bm{\beta}},\tensor{X})\|_2 \leq \sqrt{q_{k}} \max_{i \in \mathcal{A}_{k}} |L^{\prime}_{N,i}(\bar{\tensor{W}},\bar{\bm{\beta}},\tensor{X})|.
\end{equation*} Then note that 
\begin{align*}
    & L^{\prime}_{N,i}(\tensor{W},\bm{\beta},\tensor{X}) \\
    & = \sum_{i_{[1:k-1,k+1:K]}} (e_{i_{[1:k-1]},p,i_{[k+1:K]}}(\tensor{W},\bm{\beta})\tensor{X}_{i_{[1:k-1]},q,i_{[k+1:K]}} \\
    & \quad + e_{i_{[1:k-1]},q,i_{[k+1:K]}}(\tensor{W},\bm{\beta})\tensor{X}_{i_{[1:k-1]},p,i_{[k+1:K]}}),
\end{align*}where $e_{i_{[1:k-1]},p,i_{[k+1:K]}}\tensor{X}_{i_{[1:k-1]},q,i_{[k+1:K]}}(\tensor{W},\bm{\beta})$ is defined by
\begin{equation*}
\begin{aligned}
    w_{i_{[1:k-1]},p,i_{[k+1:K]}}\tensor{X}_{i_{[1:k-1]},p,i_{[k+1:K]}} & + \sum_{j_k \neq p} (\mat{\Psi}_k)_{p,j_k}\tensor{X}_{i_{[1:k-1]},j_k,i_{[k+1:K]}} \\
    & + \sum_{l \neq k} \sum_{j_l \neq i_l} (\mat{\Psi}_l)_{i_l,j_l}\tensor{X}_{i_{[1:k-1]},p,i_{[k+1:K]}}.
\end{aligned}
\end{equation*}
Then evaluated at the true parameter values $(\bar{\tensor{W}},\bar{\bm{\beta}})$, we have $e_{i_{[1:k-1]},p,i_{[k+1:K]}}(\bar{\tensor{W}},\bar{\bm{\beta}})$ uncorrelated with $\tensor{X}_{i_{[1:k-1]},\text{\textbackslash} p,i_{[k+1:K]}}$ and $E_{(\bar{\tensor{W}},\bar{\bm{\beta}})}(e_{i_{[1:k-1]},p,i_{[k+1:K]}}(\bar{\tensor{W}},\bar{\bm{\beta}}))=0$. Also, since $\tensor{X}$ is subgaussian and $\Var(L^{\prime}_{N,i}(\bar{\tensor{W}},\bar{\bm{\beta}},\tensor{X}))$ is bounded by Lemma C.1. $\forall i$, $L^{\prime}_{N,i}(\bar{\tensor{W}},\bar{\bm{\beta}},\tensor{X})$ has subexponential tails. Thus, by Bernstein inequality,
\begin{align*}
    & P(\|L_{N,\mathcal{A}_{k}}^{\prime}(\bar{\tensor{W}},\bar{\bm{\beta}},\tensor{X})\|_2 \leq c_{0,\eta}\sqrt{q_{k}\frac{\log p}{N}}) \\
    & \geq P(\sqrt{q_{k}} \max_{i \in \mathcal{A}_{k}} |L^{\prime}_{N,i}(\bar{\tensor{W}},\bar{\bm{\beta}},\tensor{X})| \leq c_{0,\eta}\sqrt{q_{k}\frac{\log p}{N}}) \geq 1 - O(\exp(-\eta \log p)).
\end{align*}

$(iii)$ By Cauchy-Schwartz,
\begin{align*}
    & |u^T L_{N,\mathcal{A}_{k}\mathcal{A}_{k}}^{\prime\prime}(\bar{\tensor{W}},\bar{\bm{\beta}},\tensor{X})u - u^T \bar{L}^{\prime\prime}_{\mathcal{A}_{k}\mathcal{A}_{k}}(\bar{\bm{\beta}})u| \\
    & \leq \|u\|_2 \|u^T L_{N,\mathcal{A}_{k}\mathcal{A}_{k}}^{\prime\prime}(\bar{\tensor{W}},\bar{\bm{\beta}},\tensor{X}) - u^T \bar{L}^{\prime\prime}_{\mathcal{A}_{k}\mathcal{A}_{k}}(\bar{\bm{\beta}})\|_2 \\
    & \leq \|u\|_2 \sqrt{q_{k}} \max_i |u^T L_{N,\mathcal{A}_{k},i}^{\prime\prime}(\bar{\tensor{W}},\bar{\bm{\beta}},\tensor{X}) - u^T \bar{L}^{\prime\prime}_{\mathcal{A}_{k},i}(\bar{\bm{\beta}})| \\
    & = \|u\|_2 \sqrt{q_{k}} |u^T L_{N,\mathcal{A}_{k},i_{\max}}^{\prime\prime}(\bar{\tensor{W}},\bar{\bm{\beta}},\tensor{X}) - u^T \bar{L}^{\prime\prime}_{\mathcal{A}_{k},i_{\max}}(\bar{\bm{\beta}})| \\
    & = \|u\|_2 \sqrt{q_{k}} |\sum_{j=1}^{q_{k}} (u_j L_{N,j,i_{\max}}^{\prime\prime}(\bar{\tensor{W}},\bar{\bm{\beta}},\tensor{X}) - u_j \bar{L}^{\prime\prime}_{j,i_{\max}}(\bar{\bm{\beta}}))| \\
    & \leq \|u\|_2 q_{k} |u_{j_{\max}}|| L_{N,j_{\max},i_{\max}}^{\prime\prime}(\bar{\tensor{W}},\bar{\bm{\beta}},\tensor{X}) -  \bar{L}^{\prime\prime}_{j_{\max},i_{\max}}(\bar{\bm{\beta}}))| \\ 
    & \leq \|u\|_2^2 q_{k} | L_{N,j_{\max},i_{\max}}^{\prime\prime}(\bar{\tensor{W}},\bar{\bm{\beta}},\tensor{X}) -  \bar{L}^{\prime\prime}_{j_{\max},i_{\max}}(\bar{\bm{\beta}}))|.
\end{align*}Then by Bernstein inequality,
\begin{align*}
    & P(|u^T L_{N,\mathcal{A}_{k}\mathcal{A}_{k}}^{\prime\prime}(\bar{\tensor{W}},\bar{\bm{\beta}},\tensor{X})u - u^T \bar{L}^{\prime\prime}_{\mathcal{A}_{k}\mathcal{A}_{k}}(\bar{\bm{\beta}})u| \leq c_{2,\eta}\|u\|_2^2q_{k}\sqrt{\frac{\log p}{N}}) \\
    & \geq P(\|u\|_2^2 q_{k} | L_{N,j_{\max},i_{\max}}^{\prime\prime}(\bar{\tensor{W}},\bar{\bm{\beta}},\tensor{X}) -  \bar{L}^{\prime\prime}_{j_{\max},i_{\max}}(\bar{\bm{\beta}}))| \leq c_{2,\eta}\|u\|_2^2q_{k}\sqrt{\frac{\log p}{N}}) \\
    & \geq 1 - O(\exp(-\eta \log p)).
\end{align*}

$(ii)$ and $(iv)$ can be proved using similar arguments.
\end{proof}

Lemma A.2.3. and A.2.4. are used later to prove Theorem 1.

\begin{lemma}
Assuming conditions of Theorem 1. Then there exists a constant $C_1(\bar{\bm{\beta}})>0$ such that for any $\eta>0$, there exists a global minimizer of the restricted problem \eqref{eqn:restricted_problem} within the disc:
\begin{equation*}
    \{\bm{\beta}: \|\bm{\beta}-\bar{\bm{\beta}}\|_2 \leq C_1(\bar{\bm{\beta}}) \sqrt{K} \max_k\sqrt{q_{k}}\lambda_{N,k} \}
\end{equation*} with probability at least $1 - O(\exp(-\eta \log p))$ for sufficiently large $N$.
\end{lemma}

\begin{proof}[proof of Lemma A.2.3.]
Let $\alpha_N = \max_{k}\sqrt{q_{k}}\lambda_{N,k}$. Further for $1 \leq k \leq K$ let $C_k>0$ and $u^k \in \mathbb{R}^{m_k(m_k-1)/2}$ such that $u_{\mathcal{A}_{k}^c}^k=0$, $\|u^k\|_2=C_k$, and  $u=(u_1,\dots,u_K)$ with $\sqrt{K}\min_kC_k \leq \|u\|_2 \leq \sqrt{K}\max_kC_k$.

Then by Cauchy-Schwartz and triangle inequality, we have
\begin{equation*}
    \|\bar{\bm{\beta}}^k + \alpha_N u^k - \alpha_N u^k\|_1
    \leq \|\bar{\bm{\beta}}^k + \alpha_N u^k\|_1 + \alpha_N\|u^k\|_1,
\end{equation*} and
\begin{equation*}
    \|\bar{\bm{\beta}}^k\|_1 - \|\bar{\bm{\beta}}^k + \alpha_N u^k\|_1 \leq \alpha_N\|u^k\|_1 \leq \alpha_N \sqrt{q_{k}} \|u^k\|_2 = C_k \alpha_N \sqrt{q_{k}}.
\end{equation*} Thus,
\begin{align*}
    & Q_N(\bar{\bm{\beta}} + \alpha_N u, \tensor{X},\{\lambda_{N,k}\}_{k=1}^K) - Q_N(\bar{\bm{\beta}}, \tensor{X},\{\lambda_{N,k}\}_{k=1}^K) \\
    & = L_N(\bar{\bm{\beta}} + \alpha_N u, \tensor{X}) - L_N(\bar{\bm{\beta}} , \tensor{X}) - \sum_{k=1}^K \lambda_{N,k} \big(\|\bar{\bm{\beta}}^k\|_1 - \|\bar{\bm{\beta}}^k + \alpha_N u^k\|_1 \big) \\
    & \geq L_N(\bar{\bm{\beta}} + \alpha_N u, \tensor{X}) - L_N(\bar{\bm{\beta}} , \tensor{X}) - \sum_{k=1}^K \lambda_{N,k} C_k \alpha_N\sqrt{q_{k}} \\
    & \geq L_N(\bar{\bm{\beta}} + \alpha_N u, \tensor{X}) - L_N(\bar{\bm{\beta}} , \tensor{X}) - \alpha_N K \max_k C_k\sqrt{q_{k}}\lambda_{N,k} \\
    & \geq L_N(\bar{\bm{\beta}} + \alpha_N u, \tensor{X}) - L_N(\bar{\bm{\beta}} , \tensor{X}) - K \alpha_N^2\max_kC_k.
\end{align*}Next, 
\begin{align*}
    & L_N(\bar{\bm{\beta}} + \alpha_N u, \tensor{X}) - L_N(\bar{\bm{\beta}},\tensor{X}) = \alpha_N u^T_{\mathcal{A}} L_{N,\mathcal{A}}^{\prime}(\bar{\bm{\beta}},\tensor{X}) + \frac{1}{2}\alpha_N^2 u^T_{\mathcal{A}}L_{N,\mathcal{A}\mathcal{A}}^{\prime\prime}(\bar{\bm{\beta}},\tensor{X})u_{\mathcal{A}} \\
    & = \alpha_N \sum_{k=1}^K (u^k_{\mathcal{A}_{k}})^T L_{N,\mathcal{A}_{k}}^{\prime}(\bar{\bm{\beta}},\tensor{X}) + \frac{1}{2}\alpha_N^2 \sum_{k=1}^K (u^k_{\mathcal{A}_{k}})^T L_{N,\mathcal{A}_{k}\mathcal{A}_{k}}^{\prime\prime}(\bar{\bm{\beta}},\tensor{X})u^k_{\mathcal{A}_{k}} \\
    & = \alpha_N \sum_{k=1}^K (u^k_{\mathcal{A}_{k}})^T L_{N,\mathcal{A}_{k}}^{\prime}(\bar{\bm{\beta}},\tensor{X}) + \frac{1}{2}\alpha_N^2 \sum_{k=1}^K (u^k_{\mathcal{A}_{k}})^T (L_{N,\mathcal{A}_{k}\mathcal{A}_{k}}^{\prime\prime}(\bar{\bm{\beta}},\tensor{X}) - \bar{L}_{N,\mathcal{A}_{k}\mathcal{A}_{k}}^{\prime\prime}(\bar{\bm{\beta}},\tensor{X}))u^k_{\mathcal{A}_{k}} \\
    & + \frac{1}{2}\alpha_N^2 \sum_{k=1}^K (u^k_{\mathcal{A}_{k}})^T \bar{L}_{N,\mathcal{A}_{k}\mathcal{A}_{k}}^{\prime\prime}(\bar{\bm{\beta}},\tensor{X}) u^k_{\mathcal{A}_{k}} \\
    & \geq \frac{1}{2}\alpha_N^2 \sum_{k=1}^K (u^k_{\mathcal{A}_{k}})^T \bar{L}_{N,\mathcal{A}_{k}\mathcal{A}_{k}}^{\prime\prime}(\bar{\bm{\beta}},\tensor{X}) u^k_{\mathcal{A}_{k}} - \alpha_N K(\max_k c_{1,\eta}\|u^k_{\mathcal{A}_{k}}\|_2\sqrt{q_{k}\frac{\log p}{N}}) \\
    & - \frac{1}{2} \alpha_N^2 K(\max_k c_{2,\eta}\|u^k_{\mathcal{A}_{k}}\|_2^2q_{k}\sqrt{\frac{\log p}{N}}).
\end{align*} Here the first equality is due to the second order expansion of the loss function and the inequality is due to Lemma A.2.2 For sufficiently large $N$, by assumption that $\lambda_{N,k}\sqrt{N/\log p} \rightarrow \infty$ if $m_k \rightarrow \infty$ and $\sqrt{\log p / N}=o(1)$, the second term in the last line above is $o(\alpha_N \sqrt{q_{k}} \lambda_{N,k}) = o(\alpha_{N}^2)$; the last term is $o(\alpha_N^2)$. Therefore, for sufficiently large $N$
\begin{align*}
    & Q_N(\bar{\bm{\beta}} + \alpha_N u, \tensor{X},\{\lambda_{N,k}\}_{k=1}^K) - Q_N(\bar{\bm{\beta}}, \tensor{X},\{\lambda_{N,k}\}_{k=1}^K) \\
    & \geq \frac{1}{2}\alpha_N^2 \sum_{k=1}^K (u^k_{\mathcal{A}_{k}})^T \bar{L}_{N,\mathcal{A}_{k}\mathcal{A}_{k}}^{\prime\prime}(\bar{\bm{\beta}},\tensor{X}) u^k_{\mathcal{A}_{k}} \\
    & - K \alpha_N^2\max_kC_k \\
    & \geq  \frac{1}{2}\alpha_N^2 K \min_k \big((u^k_{\mathcal{A}_{k}})^T \bar{L}_{N,\mathcal{A}_{k}\mathcal{A}_{k}}^{\prime\prime}(\bar{\bm{\beta}},\tensor{X}) u^k_{\mathcal{A}_{k}}\big) \\
    & - K \alpha_N^2\max_kC_k,
\end{align*} with probability at least $1-O(N^{-\eta})$. 

By Lemma A.2.1., $(u^k_{\mathcal{A}_{k}})^T \bar{L}_{N,\mathcal{A}_{k}\mathcal{A}_{k}}^{\prime\prime}(\bar{\bm{\beta}},\tensor{X}) u^k_{\mathcal{A}_{k}} \geq \Lambda_{\min}^L \|u^k_{\mathcal{A}_{k}}\|_2^2=\Lambda_{\min}^L(C_k)^2$, for each $k$.
So, if we choose $\min_k C_k$ and $\max_k C_k$ such that the upper bound is minimized, then for $N$ sufficiently large, the following holds 
\begin{equation*}
    \inf_{u:u_{(\mathcal{A}_{k})^c} = 0,\|u^k\|_2=C_k,k=1,\dots,K} Q_N(\bar{\bm{\beta}} + \alpha_N u, \tensor{X},\{\lambda_{N,k}\}_{k=1}^K) > Q_N(\bar{\bm{\beta}}, \tensor{X},\{\lambda_{N,k}\}_{k=1}^K),
\end{equation*} with probability at least $1-O(\exp(-\eta \log p))$, which means any solution to the problem defined in \eqref{eqn:restricted_problem} is within the disc $\{\bm{\beta}: \|\bm{\beta}-\bar{\bm{\beta}}\|_2 \leq \alpha_N \|u\|_2 \leq \alpha_N \sqrt{K} \max_kC_k\}$ with probability at least $1-O(\exp(-\eta \log p))$.

\end{proof}

\begin{lemma}
Assuming conditions of Theorems 1. Then there exists a constant $C_2(\bar{\bm{\beta}})>0$, such that for any $\eta>0$, for sufficiently large $N$, the following event holds with probability at least $1-O(\exp(-\eta \log p))$: if for any $\bm{\beta} \in S=\{\bm{\beta}: \|\bm{\beta}-\bar{\bm{\beta}}\|_2 \geq C_2(\bar{\bm{\beta}})\sqrt{K}\max_k\sqrt{q_{k}}\lambda_{N,k},\bm{\beta}_{\mathcal{A}_{N}^c}=0\}$, then $\|L^{\prime}_{N,\mathcal{A}_{N}}(\bar{\tensor{W}},\bar{\bm{\beta}},\tensor{X})\|_2 > \sqrt{K}\max_k\sqrt{q_{k}}\lambda_{N,k}$.
\end{lemma}

\begin{proof}[proof of Lemma A.2.4.]
Let $\alpha_N = \max_k \sqrt{q_{k}}\lambda_{N,k}$. For $\bm{\beta} \in S$, we have $\bm{\beta}=\bar{\bm{\beta}}+\alpha_N u$, with $u_{(\mathcal{A})^c}$ and $\|u\|_2 \geq C_2(\bar{\bm{\beta}})$. Note that by Taylor expansion of $L^{\prime}_{N,\mathcal{A}}(\bar{\tensor{W}},\bm{\beta},\tensor{X})$ at $\bar{\bm{\beta}}$
\begin{align*}
    L^{\prime}_{N,\mathcal{A}}(\bar{\tensor{W}},\bm{\beta},\tensor{X}) & = L^{\prime}_{N,\mathcal{A}}(\bar{\tensor{W}},\bm{\beta},\tensor{X}) + \alpha_NL^{\prime\prime}_{N,\mathcal{A}\mathcal{A}}(\bar{\tensor{W}},\bm{\beta},\tensor{X})u_{\mathcal{A}} \\
    & = L^{\prime}_{N,\mathcal{A}}(\bar{\tensor{W}},\bm{\beta},\tensor{X}) + \alpha_N\big(L^{\prime\prime}_{N,\mathcal{A}\mathcal{A}}(\bar{\tensor{W}},\bm{\beta},\tensor{X})-\bar{L}^{\prime\prime}_{N,\mathcal{A}\mathcal{A}}(\bar{\bm{\beta}})\big)u_{\mathcal{A}} \\
    & \quad + \alpha_N\bar{L}^{\prime\prime}_{N,\mathcal{A}\mathcal{A}}(\bar{\bm{\beta}})u_{\mathcal{A}}.
\end{align*}
By triangle inequality and similar proof strategies as in Lemma A.2.3., for sufficiently large $N$ 
\begin{align*}
    \|L^{\prime}_{N,\mathcal{A}}(\bar{\tensor{W}},\bm{\beta},\tensor{X})\|_2 & \geq  \|L^{\prime}_{N,\mathcal{A}}(\bar{\tensor{W}},\bm{\beta},\tensor{X})\|_2 + \alpha_N\|L^{\prime\prime}_{N,\mathcal{A}\mathcal{A}}(\bar{\tensor{W}},\bm{\beta},\tensor{X})u_{\mathcal{A}}-\bar{L}^{\prime\prime}_{N,\mathcal{A}\mathcal{A}}(\bar{\bm{\beta}})u_{\mathcal{A}}\|_2 \\
    & \quad + \alpha_N\|\bar{L}^{\prime\prime}_{N,\mathcal{A}\mathcal{A}}(\bar{\bm{\beta}})u_{\mathcal{A}}\|_2 \\
    & \geq \alpha_N\|\bar{L}^{\prime\prime}_{N,\mathcal{A}\mathcal{A}}(\bar{\bm{\beta}})u_{\mathcal{A}}\|_2 + o(\alpha_N)
\end{align*} with probability at least $1-O(\exp(-\eta \log p))$. By Lemma A.2.1., $\|\bar{L}^{\prime\prime}_{N,\mathcal{A}\mathcal{A}}(\bar{\bm{\beta}})u_{\mathcal{A}}\|_2 \geq \Lambda_{\min}^L(\bar{\bm{\beta}}) \|u_{\mathcal{A}}\|_2$. Therefore, taking $C_2(\bar{\bm{\beta}})$ to be $1/\Lambda_{\min}^L(\bar{\bm{\beta}}) + \epsilon$ completes the proof.
\end{proof}

\begin{proof}[proof of Theorem 1]
By the Karush-Kuhn-Tucker condition, for any solution $\hat{\bm{\beta}}$ of \eqref{eqn:restricted_problem}, it satisfies $\|L_{N,\mathcal{A}_{k}}^{\prime}(\tensor{W},\hat{\bm{\beta}},\tensor{X})\|_{\infty} \leq \lambda_{N,k}$. Thus,
\begin{align*}
    \|L_{N,\mathcal{A}_{N}}^{\prime}(\tensor{W},\hat{\bm{\beta}},\tensor{X})\|_2 & \leq \sqrt{K}\max_k\|L_{N,\mathcal{A}_{k}}^{\prime}(\tensor{W},\hat{\bm{\beta}},\tensor{X})\|_2 \\
    & \leq \sqrt{K}\max_k\sqrt{q_{k}}\|L_{N,\mathcal{A}_{k}}^{\prime}(\tensor{W},\hat{\bm{\beta}},\tensor{X})\|_{\infty} \\
    & \leq \sqrt{K}\max_k\sqrt{q_{k}}\lambda_{N,k}.
\end{align*}Then by Lemmas A.2.4., for any $\eta >0$, for $N$ sufficiently large, all solutions of \eqref{eqn:restricted_problem} are inside the disc $\{\bm{\beta}: \|\bm{\beta}-\bar{\bm{\beta}}\|_2 \leq C_2(\bar{\bm{\beta}})\max_k\sqrt{q_{k}}\lambda_{N,k},\bm{\beta}_{\mathcal{A}_{N}^c}=0\}$ with probability at least $1-O(\exp(-\eta \log p))$. If we further assume that $\min_{(i,j) \in \mathcal{A}_{k}}|\bar{\bm{\beta}}_{i,j}| \geq 2C(\bar{\bm{\beta}})\max_{k}\sqrt{q_{k}}\lambda_{N,k}$ for each $k$, then
\begin{align*}
    & 1-O(\exp(-\eta \log p)) \\
    & \leq P_{\bar{\tensor{W}},\bar{\bm{\beta}}}(\|\hat{\bm{\beta}}^{\mathcal{A}}-\bar{\bm{\beta}}^{\mathcal{A}}\|_2 \leq C_2(\bar{\bm{\beta}})\max_k\sqrt{q_{k}}\lambda_{N,k},\min_{(i,j) \in \mathcal{A}_{k}}|\bar{\bm{\beta}}_{i,j}| \geq 2C(\bar{\bm{\beta}})\max_{k}\sqrt{q_{k}}\lambda_{N,k},\forall k) \\
    & \leq P_{\bar{\tensor{W}},\bar{\bm{\beta}}}(\text{sign}(\hat{\bm{\beta}}_{i_kj_k}^{\mathcal{A}_{k}})=\text{sign}(\bar{\bm{\beta}}_{i_kj_k}^{\mathcal{A}_{k}}),\forall (i_k,j_k) \in \mathcal{A}_{k},\forall k).
\end{align*}
\end{proof}

\begin{proof}[proof of Theorem 2]
Let $\mathcal{E}_{N,k}=\{\text{sign}(\hat{\bm{\beta}}_{i_kj_k}^{\mathcal{A}_{k}})=\text{sign}(\bar{\bm{\beta}}_{i_kj_k}^{\mathcal{A}_{k}})\}$. Then by Theorem 1, $P_{\bar{\tensor{W}},\bar{\bm{\beta}}}(\mathcal{E}_{N,k}) \geq 1-O(\exp(-\eta \log p))$ for large $N$. On $\mathcal{E}_{N,k}$, by the KKT condition and the Taylor's expansion of $L_{N,\mathcal{A}_{k}}^{\prime}(\bar{\tensor{W}},\hat{\bm{\beta}}^{\mathcal{A}_{k}},\tensor{X})$ at $\bar{\bm{\beta}}^{\mathcal{A}_{k}}$
\begin{align*}
    - \lambda_{N,k} & \text{sign}(\bar{\bm{\beta}}^{\mathcal{A}_{k}}) \\
    & = L_{N,\mathcal{A}_{k}}^{\prime}(\bar{\tensor{W}},\hat{\bm{\beta}}^{\mathcal{A}_{k}},\tensor{X})\\
    & = L_{N,\mathcal{A}_{k}}^{\prime}(\bar{\tensor{W}},\bar{\bm{\beta}}^{\mathcal{A}_{k}},\tensor{X}) + L_{N,\mathcal{A}_{k}\mathcal{A}_{k}}^{\prime\prime}(\bar{\tensor{W}},\bar{\bm{\beta}},\tensor{X}) v_{N,k} \\
    & = \bar{L}^{\prime\prime}_{\mathcal{A}_{k}\mathcal{A}_{k}} v_{N,k}+ L_{N,\mathcal{A}_{k}}^{\prime}(\bar{\tensor{W}},\bar{\bm{\beta}}^{\mathcal{A}_{k}},\tensor{X}) + (L_{N,\mathcal{A}_{k}\mathcal{A}_{k}}^{\prime\prime}(\bar{\tensor{W}},\bar{\bm{\beta}},\tensor{X})-\bar{L}^{\prime\prime}_{\mathcal{A}_{k}\mathcal{A}_{k}}) v_{N,k},
\end{align*}where $v_{N,k}=\hat{\bm{\beta}}^{\mathcal{A}_{k}}-\bar{\bm{\beta}}^{\mathcal{A}_{k}}$. By rearranging the terms
\begin{equation}\label{eqn:thm2_exps1}
\begin{aligned}
    & v_{N,k} = \\
    & -\lambda_{N,k}[\bar{L}^{\prime\prime}_{\mathcal{A}_{k}\mathcal{A}_{k}}]^{-1}\text{sign}(\bar{\bm{\beta}}^{\mathcal{A}_{k}}) - [\bar{L}^{\prime\prime}_{\mathcal{A}_{k}\mathcal{A}_{k}}]^{-1}[L_{N,\mathcal{A}_{k}}^{\prime}(\bar{\tensor{W}},\bar{\bm{\beta}}^{\mathcal{A}_{k}},\tensor{X})+D_{N,\mathcal{A}_{k}\mathcal{A}_{k}}(\bar{\tensor{W}},\bar{\bm{\beta}}^{\mathcal{A}_{k}})v_{N,k}],
\end{aligned} 
\end{equation}where $D_{N,\mathcal{A}_{k}\mathcal{A}_{k}}=L_{N,\mathcal{A}_{k}\mathcal{A}_{k}}^{\prime\prime}(\bar{\tensor{W}},\bar{\bm{\beta}},\tensor{X})-\bar{L}^{\prime\prime}_{\mathcal{A}_{k}\mathcal{A}_{k}}$. Next, for fixed $(i,j) \in \mathcal{A}_{k}^c$, by expanding $L_{N,\mathcal{A}_{k}}^{\prime}(\bar{\tensor{W}},\hat{\bm{\beta}}^{\mathcal{A}_{k}},\tensor{X})$ at $\bar{\bm{\beta}}^{\mathcal{A}_{k}}$
\begin{equation}\label{eqn:thm2_exps2}
    L_{N,ij}^{\prime}(\bar{\tensor{W}},\hat{\bm{\beta}}^{\mathcal{A}_{k}},\tensor{X}) = L_{N,ij}^{\prime}(\bar{\tensor{W}},\bar{\bm{\beta}}^{\mathcal{A}_{k}},\tensor{X}) + L_{N,ij,\mathcal{A}_{k}}^{\prime\prime}(\bar{\tensor{W}},\bar{\bm{\beta}}^{\mathcal{A}_{k}},\tensor{X}) v_{N,k}. 
\end{equation} Then combining \eqref{eqn:thm2_exps1} and \eqref{eqn:thm2_exps2} we get
\begin{equation}\label{eqn:thm2_bouding}
\begin{aligned}
& L_{N,ij}^{\prime}(\bar{\tensor{W}},\hat{\bm{\beta}}^{\mathcal{A}_{k}},\tensor{X}) \\
& = -\lambda_{N,k}\bar{L}^{\prime\prime}_{ij,\mathcal{A}_{k}}(\bar{\bm{\beta}}^{\mathcal{A}_{k}})[\bar{L}^{\prime\prime}_{\mathcal{A}_{k}\mathcal{A}_{k}}]^{-1}\text{sign}(\bar{\bm{\beta}}^{\mathcal{A}_{k}}) - \bar{L}^{\prime\prime}_{ij,\mathcal{A}_{k}}(\bar{\bm{\beta}}^{\mathcal{A}_{k}})[\bar{L}^{\prime\prime}_{\mathcal{A}_{k}\mathcal{A}_{k}}]^{-1}L_{N,\mathcal{A}_{k}}^{\prime}(\bar{\tensor{W}},\bar{\bm{\beta}}^{\mathcal{A}_{k}},\tensor{X}) \\
& + [D_{N,ij,\mathcal{A}_{k}}(\bar{\tensor{W}},\bar{\bm{\beta}}^{\mathcal{A}_{k}})-\bar{L}^{\prime\prime}_{ij,\mathcal{A}_{k}}(\bar{\bm{\beta}}^{\mathcal{A}_{k}})[\bar{L}^{\prime\prime}_{\mathcal{A}_{k}\mathcal{A}_{k}}]^{-1}D_{N,\mathcal{A}_{k}\mathcal{A}_{k}}(\bar{\tensor{W}},\bar{\bm{\beta}}^{\mathcal{A}_{k}})] v_{N,k} \\
& + L_{N,ij}^{\prime}(\bar{\tensor{W}},\bar{\bm{\beta}}^{\mathcal{A}_{k}},\tensor{X}).
\end{aligned}
\end{equation} By the incoherence condition outlined in condition (A3), for any $(i,j) \in \mathcal{A}_{k}$,
\begin{equation*}
    |\bar{L}_{ij,\mathcal{A}_{k}}^{''}(\bar{\tensor{W}},\bar{\bm{\beta}})[\bar{L}_{\mathcal{A}_{k},\mathcal{A}_{k}}^{''}(\bar{\tensor{W}},\bar{\bm{\beta}})]^{-1} \text{sign}(\bar{\bm{\beta}}_{\mathcal{A}_{k}})| \leq \delta < 1.
\end{equation*}Thus, following straightforwardly (with the modification that we are considering each $\mathcal{A}_{k}$ instead of $\mathcal{A}$) from the proofs of Theorem 2 of \citet{peng2009partial}, the remaining terms in \eqref{eqn:thm2_bouding} can be shown to be all $o(\lambda_{N,k})$, and $\max_{(i,j) \in \mathcal{A}_{k}^c}|L_{N,ij}^{\prime}(\bar{\tensor{W}},\hat{\bm{\beta}}^{\mathcal{A}_{k}},\tensor{X})| < \lambda_{N,k}$ with probability at least $1-O(\exp(-\eta \log p))$ for sufficiently large $N$. Thus, it has been proved that for sufficiently large $N$, no wrong edge will be included for each true edge set $\mathcal{A}_{k}$ and hence, no wrong edge will be included in $\mathcal{A} = \cup_k \mathcal{A}_{k}$.
\end{proof}

\begin{proof}[proof of Theorem 3]
By Theorem 1 and Theorem 2, with probability tending to $1$, any solution of the restricted problem is also a solution of the original problem. On the other hand, by Theorem 2 and the KKT condition, with probability tending to $1$, any solution of the original problem is also a solution of the restricted problem. Therefore, Theorem 3 follows. 
\end{proof}

\section{Simulated Precision Matrix}
\label{supp:simulated_precision_matrix}
\begin{enumerate}
  \item \textbf{AR1($\rho$)}: The covariance matrix of the form $\mat{A} = (\rho^{|i-j|})_{ij}$ for $\rho \in (0,1)$.
  \item \textbf{Star-Block (SB):} A block-diagonal covariance matrix, where each block's precision matrix corresponds to a star-structured graph with $(\mat{\Psi}_k)_{ij} = 1$. Then, for $\rho \in (0,1)$, we have that $\mat{A}_{ij} = \rho$ if $(i,j) \in E$ and $\mat{A}_{ij}=\rho^2$ for $(i,j) \not \in E$, where $E$ is the corresponding edge set.
  \item \textbf{Erdos-Renyi random graph (ER):} The precision matrix is initialized at $\mat{A} = 0.25 \mat{I}$, and $d$ edges are randomly selected. For the selected edge $(i,j)$, we randomly choose $\psi \in [0.6, 0.8]$ and update $\mat{A}_{ij} = \mat{A}_{ji} \rightarrow \mat{A}_{ij} - \psi$ and $\mat{A}_{ii} \rightarrow \mat{A}_{ii} + \psi$, $\mat{A}_{jj} \rightarrow \mat{A}_{jj} + \psi$.
\end{enumerate}

%% file: Appendices/Appendix_B.tex
In this Appendix, 
\begin{itemize}
    \item[] Section~\ref{supp:pseudolik} provides detailed derivation of the log-pseudolikelihood function.
    \item[] Section~\ref{supp:bb-step-size} provides justifications for the Barzilai-Borwein step sizes implemented in Algorithm~\ref{alg:sg-palm}.
    \item[] Section~\ref{supp:proofs} provides detailed proofs of Theorems~\ref{thm:statistical} and \ref{thm:sg-palm-main}.    
    \item[] Section~\ref{supp:nonconvex} discusses extensions of Algorithm~\ref{alg:sg-palm} and its convergence properties to non-convex cases.
    \item[] Section~\ref{supp:additional_experiments} provides additional details of the solar flare experiments.
\end{itemize}

\section{Derivation of the Log-Pseudolikelihood}\label{supp:pseudolik}
By rewriting the Sylvester tensor equation defined in \eqref{eqn:sylvester} element-wise, we first observe that
\begin{equation}
\label{eqn:elementwise_sylvester}
\begin{aligned}
    & \left( \sum_{k=1}^K (\mat{\Psi}_k)_{i_k,i_k} \right) \tensor{X}_{i_{[1:K]}} \\
    & = -\sum_{k=1}^K \sum_{j_k \neq i_k} (\mat{\Psi}_k)_{i_k,j_k} \tensor{X}_{i_{[1:k]},j_k,i_{[k+1:K]}} + \tensor{T}_{i_{[1:K]}}.
\end{aligned}
\end{equation} 
Note that the left-hand side of \eqref{eqn:elementwise_sylvester} involves only the summation of the diagonals of the $\mat{\Psi}_k$'s and the right-hand side is composed of columns of $\mat\Psi_k$'s that exclude the diagonal terms. Equation \eqref{eqn:elementwise_sylvester} can be interpreted as an autogregressive model relating the $(i_1,\dots,i_K)$-th element of the data tensor (scaled by the sum of diagonals) to other elements in the fibers of the data tensor. The columns of $\mat{\Psi}_k$'s act as regression coefficients. The formulation in \eqref{eqn:elementwise_sylvester} naturally leads to a pseudolikelihood-based estimation procedure \citep{besag1977efficiency} for estimating $\mat\Omega$ (see also \citet{khare2015convex} for how this procedure applied to vector-variate Gaussian graphical model estimation). It is known that inference using pseudo-likelihood is consistent and enjoys the same $\sqrt{N}$ convergence rate as the MLE in general \citep{varin2011overview}. This procedure can also be more robust to model misspecification (e.g., non-Gaussianity) in the sense that it assumes \textit{only that the sub-models/conditional distributions (i.e., $\tensor{X}_i|\tensor{X}_{-i}$) are Gaussian}. Therefore, in practice, even if the data is not Gaussian, the Maximum Pseudolikelihood Estimation procedure is able to perform reasonably well. \citet{wang2020sylvester} also studied a different model misspecification scenario where the Kronecker product/sum and Sylvester structures are mismatched for SyGlasso.

From \eqref{eqn:elementwise_sylvester} we can define the sparse least-squares estimators for $\mat\Psi_k$'s as the solution of the following convex optimization problem:
\begin{equation*}
  \begin{aligned}
    & \min_{\substack{\mat{\Psi}_k \in \bbR^{d_k \times d_k}\\k=1,\dots K}} -N \sum_{i_1,\dots,i_K} \log \tensor{W}_{i_{[1:K]}} \\ 
    & \qquad + \frac{1}{2} \sum_{i_1,\dots,i_K} \|(I) + (II)\|_2^2 + \sum_{k=1}^K P_{\lambda_k}(\mat{\Psi}_k).
  \end{aligned} \
\end{equation*}
where $P_{\lambda_k}(\cdot)$ is a penalty function indexed by the tuning parameter $\lambda_k$ and 
\begin{align*}
  (I) & = \tensor{W}_{i_{[1:K]}}\tensor{X}_{i_{[1:K]}} \\
  (II) & = \sum_{k=1}^K \sum_{j_k \neq i_k} (\mat{\Psi}_k)_{i_k,j_k} \tensor{X}_{i_{[1:k]},j_k,i_{[k+1:K]}},
\end{align*}
with $\tensor{W}_{i_{[1:K]}} := \sum_{k=1}^K (\mat{\Psi}_k)_{i_k,i_k}$.

The optimization problem above can be put into the following matrix form:
\begin{equation*}
    \begin{aligned}
    \min_{\substack{\mat{\Psi}_k \in \bbR^{d_k \times d_k}\\ k=1,\dots K}} 
    & -\frac{N}{2} \log|(\text{diag}(\mat{\Psi}_1) \oplus \dots \oplus \text{diag}(\mat{\Psi}_K))^2| \\ \nonumber
    + & \frac{N}{2} \tr(\mat{S}(\mat{\Psi}_1 \oplus \dots \oplus \mat{\Psi}_K)^2) + \sum_{k=1}^K P_{\lambda_k}(\mat{\Psi}_k) \nonumber
    \end{aligned}
\end{equation*}
where $\mat{S} \in \bbR^{d \times d}$ is the sample covariance matrix, i.e., $\mat{S}=\frac{1}{N} \sum_{i=1}^N \vecto(\tensor{X}^i) \vecto(\tensor{X}^i)^T$. Note that this is equivalent to the negative log-pseudolikelihood function that approximates the $\ell_1$-penalized Gaussian negative log-likelihood in the log-determinant term by including only the Kronecker sum of the diagonal matrices instead of the Kronecker sum of the full matrices.

\section{The Barzilai-Borwein Step Size}\label{supp:bb-step-size}
The BB method has been proven to be very successful in solving nonlinear optimization problems. In this section we outline the key ideas behind the BB method, which is motivated by quasi-Newton methods. Suppose we want to solve the unconstrained minimization problem
\begin{equation*}
    \min_x f(x),
\end{equation*}
where $f$ is differentiable. A typical iteration of quasi-Newton methods for solving this problem is
\begin{equation*}
    x_{t+1} = x_{t} - B_t^{-1} \nabla f(x_t),
\end{equation*}
where $B_t$ is an approximation of the Hessian matrix of $f$ at the current iterate $x_t$. Here, $B_t$ must satisfy the so-called secant equation: $B_t s_t = y_t$, where $s_t = x_t - x_{t-1}$ and $y_t = \nabla f(x_t) - \nabla f(x_{t-1})$ for $t \geq 1$. It is noted that in to get $B_t^{-1}$ one needs to solve a linear system, which may be computationally expensive when $B_t$ is large and dense.

One way to alleviate this burden is to use the BB method, which replaces $B_t$ by a scalar matrix $(1/\eta_t)\mat{I}$. However, it is hard to choose a scalar $\eta_t$ such that the secant equation holds with $B_t = (1/\eta_t)\mat{I}$. Instead, one can find $\eta_t$ such that the residual of the secant equation, i.e., $\|(1/\eta_t)s_t - y_t\|_2^2$, is minimized, which leads to the following choice of $\eta_t$:
\begin{equation*}
    \eta_t = \frac{\|s_t\|_2^2}{s_t^Ty_t}.
\end{equation*}
Therefore, a typical iteration of the BB method for solving the original problem is
\begin{equation*}
    x_{t+1} = x_{t} - \eta_t \nabla f(x_t),
\end{equation*}
where $\eta_t$ is computed via the previous formula.

For convergence analysis, generalizations and variants of the BB method, we refer the interested readers to~\citet{raydan1993barzilai,raydan1997barzilai,dai2002r,fletcher2005barzilai} and references therein. BB method has been successfully applied for solving problems arising from emerging applications, such as compressed sensing~\citep{wright2009sparse}, sparse reconstruction~\citep{wen2010fast} and image processing~\citep{wang2007projected}.

\section{Proofs of Theorems}\label{supp:proofs}
\subsection{Proof of Theorem~\ref{thm:statistical}}\label{supp:thm_statistical}

We first state the regularity conditions needed for establishing convergence of the SG-PALM estimators $\{\hat{\mat\Psi}_k\}_{k=1}^K$ to their true value $\{\bar{\mat\Psi}_k\}_{k=1}^K$.

\noindent \textbf{(A1 - Subgaussianity)} The data $\tensor{X}^1,\dots,\tensor{X}^N$ are i.i.d subgaussian random tensors, that is, $\vecto(\tensor{X}^i) \sim \mat{x}$, where $\mat{x}$ is a subgaussian random vector in $\mathbb{R}^d$, i.e., there exist a constant $c>0$, such that for every $\mat{a} \in \mathbb{R}^d$, $\mathbb{E}e^{\mat{a}^T x} \leq e^{c\mat{a}^T \bar{\mat{\Sigma}} \mat{a}}$, and there exist $\rho_j > 0$ such that $\mathbb{E}e^{tx_j^2} \leq +\infty$ whenever $|t| < \rho_j$, for $1 \leq j \leq d$.

\noindent \textbf{(A2 - Bounded eigenvalues)} There exist constants $0 < \Lambda_{\min} \leq \Lambda_{\max} < \infty$, such that the minimum and maximum eigenvalues of $\mat{\Omega}$ are bounded with $\lambda_{\min}(\bar{\mat{\Omega}}) = (\sum_{k=1}^K \lambda_{\max}(\mat{\Psi}_k))^{-2} \geq \Lambda_{\min}$ and $\lambda_{\max}(\bar{\mat{\Omega}}) = (\sum_{k=1}^K \lambda_{\min}(\mat{\Psi}_k))^{-2} \leq \Lambda_{\max}$.

\noindent \textbf{(A3 - Incoherence condition)} There exists a constant $\delta < 1$ such that for $k=1,\dots,K$ and all $(i,j) \in \mathcal{A}_{k}$
\begin{equation*}
    |\bar{\mathcal{L}}_{ij,\mathcal{A}_{k}}^{''}(\bar{\mat{\Psi}})[\bar{\mathcal{L}}_{\mathcal{A}_{k},\mathcal{A}_{k}}^{''}(\bar{\mat{\Psi}})]^{-1} \text{sign}(\bar{\mat{\Psi}}_{\mathcal{A}_{k},\mathcal{A}_{k}})| \leq \delta,
\end{equation*} where for each $k$ and $1 \leq i < j \leq d_k$, $1 \leq k < l \leq d_k$,
\begin{equation*}
    \bar{\mathcal{L}}_{ij,kl}^{''}(\bar{\mat{\Psi}}) := E_{\bar{\mat{\Psi}}} \Bigg(\frac{\partial^2 \mathcal{L}(\mat{\Psi})}{\partial(\mat{\Psi}_k)_{i,j} \partial(\mat{\Psi}_k)_{k,l}}|_{\mat{\Psi}=\bar{\mat{\Psi}}} \Bigg),
\end{equation*}
and
\begin{equation*}
    \mathcal{L}(\mat\Psi)
    = -\frac{N}{2} \log | (\bigoplus_{k=1}^K \diag(\mat\Psi_k))^2|  + \frac{N}{2} \tr(\mat{S} \cdot (\bigoplus_{k=1}^K \mat\Psi_k)^2).
\end{equation*}

Given assumptions (A1-A3), the theorem follows from Theorem 3.3 in \citet{wang2020sylvester}.

\subsection{Proof of Theorem~\ref{thm:sg-palm-main}}\label{supp:thm_sg-palm-main}
We next turn to convergence of the iterates $\{\mat\Psi^t\}$ from SG-PALM to a global optimum of \eqref{eqn:objective}. The proof leverages recent results in the convergence of alternating minimization algorithms for non-strongly convex objective~\citep{bolte2014proximal,karimi2016linear,li2018calculus,zhang2020new}. We outline the proof strategy:
\begin{enumerate}
    \item We establish Lipschitz continuity of the blockwise gradient $\nabla_kH(\mat\Psi)$ for $k=1,\dots,K$.
    \item We show that the objective function $\mathcal{L}_{\mat\lambda}$ satisfies the Kurdyka - \L ojasiewicz (KL) property. Further, it has a KL exponent of $\frac{1}{2}$ (defined later in the proofs).
    \item The KL property (with exponent $\frac{1}{2}$) is equivalent to a generalized Error Bound (EB) condition, which enables us to establish linear iterative convergence of the objective function~\eqref{eqn:objective} to its global optimum.
\end{enumerate}

\begin{definition}[Subdifferentials]\label{def:subdiff}
Let $f: \bbR^d \rightarrow (-\infty,+\infty]$ be a proper and lower semicontinuous function. Its domain is defined by
\begin{equation*}
    \text{dom}f := \{x \in \bbR^d: f(x) < + \infty\}.
\end{equation*}
If we further assume that $f$ is convex, then the subdifferential of $f$ at $x \in \text{dom}f$ can be defined by
\begin{equation*}
    \partial f(x) := \{v \in \bbR^d: f(z) \geq f(x) + <v,z-x>, \forall z \in \bbR^d\}.
\end{equation*}
The elements of $\partial f(x)$ are called subgradients of $f$ at $x$.
\end{definition}

Denote the domain of $\partial f$ by $\text{dom}\partial f := \{x \in \bbR^d: \partial f(x) \neq \emptyset\}$. Then, if $f$ is proper, semicontinuous, convex, and $x \in \text{dom}f$, then $\partial f(x)$ is a nonempty closed convex set. In this case, we denote by $\partial^0 f(x)$ the unique least-norm element of $\partial f(x)$ for $x \in \text{dom}\partial f$, along with $\|\partial^0 f(x)\|=+\infty$ for $x \notin \text{dom}\partial f$. Points whose subdifferential contains $0$ are critical points, denoted by $\textbf{crit}f$. For convex $f$, $\textbf{crit}f=\argmin f$.

\begin{definition}[KL property]\label{def:kl}
Let $\Gamma_{c_2}$ stands for the class of functions $\phi:[0,c_2] \rightarrow \mathbb{R}_{+}$ for $c_2>0$ with the properties:
\begin{enumerate}[label=(\roman*)]
    \item $\phi$ is continuous on $[0,c_2]$;
    \item $\phi$ is smooth concave on $(0,c_2)$;
    \item $\phi(0)=0, \phi'(s)>0, \forall s \in (0,c_2)$.
\end{enumerate}
Further, for $x \in \bbR^d$ and any nonempty $Q \subset \bbR^d$, define the distance function $d(x,Q):=\inf_{y \in Q} \|x-y\|$. Then, a function $f$ is said to have the Kurdyka - \L ojasiewicz (KL) property at point $x_0$, if there exist $c_1 > 0$, a neighborhood $B$ of $x_0$, and $\phi \in \Gamma_{c_2}$ such that for all
\begin{equation*}
    x \in B(x_0,c_1) \cap \{x:f(x_0)<f(x)<f(x_0)+c_2\},
\end{equation*}
the following inequality holds
\begin{equation*}
    \phi'\Big(f(x)-f(x_0)\Big)\text{dist}(0,\partial f(x)) \geq 1.
\end{equation*}
If $f$ satisfies the KL property at each point of $\text{dom}\partial f$ then $f$ is called a KL function.
\end{definition}

We first present two lemmas that characterize key properties of the loss function.

\begin{lemma}[Blockwise Lipschitzness]\label{lemma:lip}
The function $H$ is convex and continuously differentiable on an open set containing $\text{dom} G$ and its gradient, is block-wise Lipschitz continuous with block Lipschitz constant $L_k>0$ for each $k$, namely for all $k=1,\dots,K$ and all $\mat\Psi_k, \mat\Psi_k' \in \mathbb{R}^{d_k \times d_k}$
\begin{equation*}
\begin{aligned}
    & \|\nabla_k H(\mat\Psi_{i < k},\mat\Psi_k,\mat\Psi_{i > k}) - \nabla_k H(\mat\Psi_{i < k},\mat\Psi_k',\mat\Psi_{i > k})\| \\
    & \leq L_k \|\mat\Psi_k - \mat\Psi_k'\|,
\end{aligned}
\end{equation*}
where $\nabla_k H$ denotes the gradient of $H$ with respect to $\mat\Psi_k$ with all remaining $\mat\Psi_i$, $i \neq k$ fixed. Further, the function $G_k$ for each $k=1,\dots,K$ is a proper lower semicontinuous (lsc) convex function.
\end{lemma}
\begin{proof}
For simplicity of notation, in this and the following proofs we use $\mat\Psi$ (i.e., omitting the subscript) to denote the set $\{\mat\Psi_k\}_{k=1}^K$ or the $K$-tuple $(\mat\Psi_1,\dots,\mat\Psi_K)$ whenever there is no confusion. Recall the blockwise gradient of the smooth part of the objective function $H$ with respect to $\mat\Psi_k$, for each $k=1,\dots,K$, is given by
\begin{equation*}
\begin{aligned}
    \nabla_k H(\mat\Psi) &= \diag\Big(\Big[\tr\{(\diag((\mat\Psi_k))_{ii} + \bigoplus_{j \neq k}\diag(\mat\Psi_j))^{-1}\} \quad i=1:d_k \Big]\Big) \\
    & \quad + \mat{S}_k\mat\Psi_k + \mat\Psi_k\mat{S}_k + 2\sum_{j \neq k}\mat{S}_{j,k}.
\end{aligned}
\end{equation*}
Then for $\mat\Psi_k,\mat\Psi'_k$, 
\begin{equation*}
\begin{aligned}
    & \|\mat{S}_k\mat\Psi_k + \mat\Psi_k\mat{S}_k + 2\sum_{j \neq k}\mat{S}_{j,k} - (\mat{S}_k\mat\Psi'_k + \mat\Psi'_k\mat{S}_k + 2\sum_{j \neq k}\mat{S}_{j,k})\| \\
    & = \|\mat{S}_k\mat\Psi_k + \mat\Psi_k\mat{S}_k - \mat{S}_k\mat\Psi'_k - \mat\Psi'_k\mat{S}_k\| \\
    & \leq 2\|\mat{S}_k\| \|\mat\Psi_k - \mat\Psi'_k\|.
\end{aligned}
\end{equation*}
To prove Lipschitzness of the remaining parts, we consider the case of $K=2$ for simplicity of notations. The arguments easily carry over cases of $K>2$. In this case, denote $\mat{A}=(a_{ij}):=\mat\Psi_1$ and $\mat{B}=(b_{kl}):=\mat\Psi_2$. Let $f(\mat{A}):=\frac{\partial}{\partial \mat{A}}\log|\diag(\mat{A} \oplus \mat{B})|$, then
\begin{equation*}
    f(\mat{A}) - f(\mat{A}') = \diag \Big(\Big[ \sum_{i=1}^{m_2}(a_{jj}+b_{ii})^{-1} - \sum_{i=1}^{m_2}(a'_{jj}+b_{ii})^{-1} \quad j=1,\dots,m_1 \Big] \Big)
\end{equation*}
and
\begin{equation*}
    \begin{aligned}
        \|f(\mat{A}) - f(\mat{A}')\|_F &= \Bigg(\sum_{j=1}^{m_1}\Big(\sum_{i=1}^{m_2}(a_{jj}+b_{ii})^{-1} - \sum_{i=1}^{m_2}(a'_{jj}+b_{ii})^{-1}\Big)^2\Bigg)^{1/2} \\
        &\leq \Bigg(\sum_{j=1}^{m_1}m_2\sum_{i=1}^{m_2}\Big((a_{jj}+b_{ii})^{-1} - (a'_{jj}+b_{ii})^{-1}\Big)^2\Bigg)^{1/2} \\
        &= \Big(m_2\sum_{j=1}^{m_1}\sum_{i=1}^{m_2}(c_{ji}^{-1}-(c'_{ji})^{-1})^2\Big)^{1/2} \\
        &= \Big(m_2\sum_{j=1}^{m_1}\sum_{i=1}^{m_2}(c'_{ji})^{-2}(c'_{ji}-c_{ji})^2c_{ji}^{-2}\Big)^{1/2} \\
        &= \Big(m_2\sum_{j=1}^{m_1}(a_{jj}-a'_{jj})^2\sum_{i=1}^{m_2}(c'_{ji}c_{ji})^{-2}\Big)^{1/2} \\
        &\leq \Big(Cm_2\sum_{j=1}^{m_1}\sum_{i=1}^{m_2}(c'_{ji}c_{ji})^{-2}\Big)^{1/2} \|\mat{A} - \mat{A}'\|_F,
    \end{aligned}
\end{equation*}
where the first inequality is due to Cauchy-Schwartz inequality; the third line is due to $c_{ji}:=a_{jj}+b_{ii}$; and in the last inequality we upper-bound each $(a_{jj}-a'_{jj})^2$ by its maximum over all $j$, which is absorbed in  a constant $C$. Note that the first term in the last line above is finite as long as the summations of the diagonal elements of the factors $\mat{A}$ and $\mat{B}$ are finite, which is implied if the precision matrix $\mat{\Omega}$ defined by the Sylvester generating equation as $(\mat{A} \oplus \mat{B})^2$ has finite diagonal elements. This follows from Theorem 3.1 of \citet{oh2014optimization}, who proved that if a symmetric matrix $\mat\Omega$ satisfying $\mat\Omega \in \mathcal{C}_0$, where
\begin{equation*}
    \mathcal{C}_0 = \Big\{\mat\Omega| \mathcal{L}_{\mat\lambda}(\mat\Omega) \leq \mathcal{L}_{\mat\lambda}(\mat\Omega^{(0)})=M \Big\},
\end{equation*}
and $\mat\Omega^{(0)}$ is an arbitrary initial point with a finite function value $\mathcal{L}_{\mat\lambda}(\mat\Omega^{(0)}):=M$, the diagonal elements of $\mat\Omega$ are bounded above and below by constants which depend only on $M$, the regularization parameter $\mat\lambda$, and the sample covariance matrix $\mat{S}$. Therefore, we have
\begin{equation*}
    \|f(\mat{A}) - f(\mat{A}')\|_F \leq \Tilde{C}\|\mat{A}-\mat{A}'\|_F
\end{equation*}
for some constant $\Tilde{C} \in (0,+\infty)$. Similarly, we can establish such an inequality for $\mat{B}$, proving that the first term in $\nabla_k H$ is Lipschitz continuous.
\end{proof}

As a consequence of Lemma~\ref{lemma:lip}, the gradient of $H$, $\nabla H = (\nabla_1 H,\dots,\nabla_K H)$ is Lipschitz continuous on bounded subsets $\mathbb{B}_1 \times \cdots \times \mathbb{B}_K$ of $\mathbb{R}^{d_1 \times d_1} \times \cdots \times \mathbb{R}^{d_K \times d_K}$ with some constant $L>0$, such that for all $(\mat\Psi_k,\mat\Psi_k') \in \mathbb{B}_k \times \mathbb{B}_k$,
\begin{equation*}
\begin{aligned}
    & \|(\nabla_1 H(\mat\Psi_1,\mat\Psi_{i > 1})-\nabla_1 H(\mat\Psi_1',\mat\Psi_{i > 1}'),\dots,\\
    & \nabla_K H(\mat\Psi_{i < K}',\mat\Psi_K')-\nabla_K H(\mat\Psi_{i < K}',\mat\Psi_K'))\| \\ 
    &\leq L\|(\mat\Psi_1-\mat\Psi_1',\dots,\mat\Psi_K-\mat\Psi_K')\|,
\end{aligned}
\end{equation*}
and we have $L \leq \sum_{k=1}^K L_k$.

\begin{lemma}[KL property of $\mathcal{L}_{\mat\lambda}$]\label{lemma:kl_loss}
The objective function $\mathcal{L}_{\mat\lambda}(\mat\Psi)$ defined in~\eqref{eqn:objective} satisfies the KL property. Further, $\phi$ in this case can be chosen to have the form $\phi(s) = \alpha s^{1/2}$, where $\alpha$ is some positive real number. Functions satisfying the KL property with this particular choice of $\phi$ is said to have a KL exponent of $\frac{1}{2}$.
\end{lemma}
\begin{proof}
This can be established in a few steps:
\begin{enumerate}
    \item It can be shown that the function (of $\mat{X}$) $\tr(\mat{S}\mat{X}^2) + \|\mat{X}\|_{1,\text{off}}$
    satisfies the KL property with exponent $\frac{1}{2}$~\citep{karimi2016linear}. We then apply the calculus rules of the KL exponent (compositions and separable summations) studied in~\citet{li2018calculus} to prove that $\tr(\mat{S}(\bigoplus_j\mat\Psi_j)^2)$ and $\sum_j\|\mat\Psi_j\|_{1,\text{off}}$ are also KL functions with exponent $\frac{1}{2}$. 
    \item The $-\log\det\Big(\bigoplus_j\diag(\mat\Psi_j)\Big)$ term can be shown to be KL with exponent $\frac{1}{2}$ using a transfer principle studied in~\citet{lourencco2019generalized}.
    \item Finally, using the calculus rules of KL exponent one more time, we combine the first two results and establish that $\mathcal{L}_{\mat\lambda}$ has KL exponent of $\frac{1}{2}$.
\end{enumerate}

 \citet{karimi2016linear} proved that the following function, parameterized by some symmetric matrix $\mat{X}$, satisfies the KL property with KL exponent $\frac{1}{2}$:
\begin{equation*}
    \tr(\mat{S}\mat{X}^2) + \|\mat{X}\|_{1,\text{off}} = \|\mat{A}\mat{X}\|_F^2 + \|\mat{X}\|_{1,\text{off}}
\end{equation*}
for $\mat{S}=\mat{A}\mat{A}^T$, even when $\mat{A}$ is not of full rank.   

We apply the calculus rules of the KL exponent studied in~\citet{li2018calculus} to prove that $\tr(\mat{S}(\bigoplus_j\mat\Psi_j)^2)$ and $\sum_j\|\mat\Psi_j\|_{1,\text{off}}$ are KL functions with exponent $\frac{1}{2}$. 
Particularly, we observe that $\tr\Big(\mat{S}(\bigoplus_j \mat\Psi_j)^2\Big)$ is the composition of functions $\mat{X} \rightarrow \tr(\mat{S}\mat{X})$ and $(\mat{X}_1,\dots,\mat{X}_K) \rightarrow \bigoplus_j \mat{X}_j$; and $\sum_j\|\mat\Psi_j\|_{1,\text{off}}$ is a separable block summation of functions $\mat{X}_j \rightarrow \|\mat{X}_j\|_{1,\text{off}}$. 

Thus, by Theorem 3.2. (exponent for composition of KL functions) in \citet{li2018calculus}, since the Kronecker sum operation is linear and hence continuously differentiable, the trace function is KL with exponent $\frac{1}{2}$, and the mapping $(\mat{X}_1,\dots,\mat{X}_K) \rightarrow \bigoplus_j \mat{X}_j$ is clearly one to one, the function $\tr(\mat{S}(\bigoplus_j\mat\Psi_j)^2)$ has the KL exponent of $\frac{1}{2}$. By Theorem 3.3. (exponent for block separable sums of KL functions) in \citet{li2018calculus}, since the function $\|\cdot\|_{1,\text{off}}$ is proper, closed, continuous on its domain, and is KL with exponent $\frac{1}{2}$, the function $\|\mat{X}_j\|_{1,\text{off}}$ is KL with an exponent of $\frac{1}{2}$.

It remains to prove that $-\log\det\Big(\bigoplus_j\diag(\mat\Psi_j)\Big)$ is also a KL function with an exponent of $\frac{1}{2}$. By Theorem 30 in \citet{lourencco2019generalized}, if we have $f:\mathbb{R}^r \rightarrow \mathbb{R}$ a symmetric function and $F:\mathcal{E} \rightarrow \mathbb{R}$ the corresponding spectral function, the followings hold
\begin{enumerate}[label=(\roman*)]
    \item $F$ satisfies the KL property at $\mat{X}$ iff $f$ satisfies the KL property at $\lambda(\mat{X})$, i.e., the eigenvalues of $\mat{X}$.
    \item $F$ satisfies the KL property with exponent $\alpha$ iff $f$ satisfies the KL property with exponent $\alpha$ at $\lambda(\mat{X})$.
\end{enumerate}
Here, take $f(\lambda(\mat{X})):=-\sum_{i=1}^r\log(\lambda_i(\mat{X}))$, and $F(\mat{X}):=-\log\det(\mat{X})$ the corresponding spectral function. Then, the function $f$ is symmetric since its value is invariant to permutations of its arguments, and it is a strictly convex function in its domain, so it satisfies the KL property with an exponent of $\frac{1}{2}$. Therefore, $F$ satisfies the KL property with the same KL exponent of $\frac{1}{2}$. Now, we apply the calculus rules for KL functions again. As both the Kronecker sum and the $\diag$ operators are linear, we conclude that $-\log\det\Big(\bigoplus_j\diag(\mat\Psi_j)\Big)$ is a KL function with an exponent of $\frac{1}{2}$.

Overall, we have that the negative log-pseudolikelihood function $\mathcal{L}(\mat\Psi)$ satisfies the KL property with an exponent of $\frac{1}{2}$.
\end{proof}

Now we are ready to prove Theorem~\ref{thm:sg-palm-main}. We follow \citet{zhang2020new} and divide the proof into three steps.

\paragraph{Step 1.} We obtain a sufficient decrease property for the loss function $\mathcal{L}$ in terms of the squared distance of two successive iterates:
\begin{equation}\label{eqn:sufficient_descent}
    \mathcal{L}(\mat\Psi^{(t)}) - \mathcal{L}(\mat\Psi^{(t+1)}) \geq \frac{L_{\min}}{2} \|\mat\Psi^{(t)} - \mat\Psi^{(t+1)}\|^2.
\end{equation}

Here and below, $\mat\Psi^{(t+1)}:=(\mat\Psi_1^{(t+1)},\dots,\mat\Psi_K^{(t+1)})$ and $L_{\min}:=\min_k L_k$. First note that at iteration $t$, the line search condition is satisfied for step size $\frac{1}{\eta_k^{(t)}} \geq L_k$, where $L_k$ is the Lipschitz constant for $\nabla_k H$. Further, it follows that for SG-PALM with backtracking one has for every $t \geq 0$ and each $k=1,\dots,K$,
\begin{equation*}
    \frac{1}{\eta_k^{(0)}} \leq \frac{1}{\eta_k^{(t)}} \leq c L_k,
\end{equation*}
where $c>0$ is the backtracking constant.

Then by Lemma 3.1 in \citet{shefi2016rate}, we get 
\begin{equation*}
    \begin{aligned}
        \mathcal{L}(\mat\Psi^{(t)}) - \mathcal{L}(\mat\Psi^{(t+1)}) &\geq \frac{1}{2\eta_{\min}^{(t+1)}} \|\mat\Psi^{(t)} - \mat\Psi^{(t+1)}\|^2 \\
        &\geq  \frac{L_{\min}}{2} \|\mat\Psi^{(t)} - \mat\Psi^{(t+1)}\|^2
    \end{aligned}
\end{equation*}
for $\eta_{\min}^{(t)}:=\min_k \eta_{k}^{(t)}$.

\paragraph{Step 2.} By Lemma~\ref{lemma:kl_loss}, $\mathcal{L}$ satisfies the KL property with an exponent of $\frac{1}{2}$. Then from Definition~\ref{def:kl}, this suggests that at $x=\mat\Psi^{t+1}$ and $f(x_0)=\min\mathcal{L}$
\begin{equation}\label{eqn:res-obj-EB}
    \|\partial^0\mathcal{L}(\mat\Psi^{t+1})\| \geq \alpha \sqrt{\mathcal{L}(\mat\Psi^{t+1}) - \min\mathcal{L}},
\end{equation}
where $\alpha>0$ is a fixed constant defined in Lemma~\ref{lemma:kl_loss}. This property is equivalent to the error bound condition, ($\partial^0 \mathcal{L}, \alpha, \Omega$)-(res-obj-EB), defined in Definition 5 in \citet{zhang2020new}, for $\Omega \subset \text{dom}\partial\mathcal{L}$. This is strictly weaker than strong convexity (see Section 4 in \citet{zhang2020new}).

At iteration $t+1$, there exists $\xi_k^{(t+1)} \in \partial G_k(\mat\Psi_k^{(t+1)})$ satisfying the optimality condition:
\begin{equation*}
    \nabla_k H(\mat\Psi_{i<k}^{(t+1)},\mat\Psi_{i \geq k}^{(t)}) + \frac{1}{\eta_k^{(t+1)}}(\mat\Psi_k^{(t+1)} - \mat\Psi_k^{(t)}) + \xi_k^{(t+1)} = 0.
\end{equation*}
Let $\xi^{(t+1)}:=(\xi_1^{(t+1)},\dots,\xi_K^{(t+1)})$. Then,
\begin{equation*}
    \nabla H(\mat\Psi^{(t+1)}) + \xi^{(t+1)} \in \partial\mathcal{L}(\mat\Psi^{(t+1)})
\end{equation*}
and hence the error bound condition becomes
\begin{equation*}
    \mathcal{L}(\mat\Psi^{(t+1)}) - \min\mathcal{L} \leq \frac{\|\partial^0\mathcal{L}(\mat\Psi^{(t+1)})\|^2}{\alpha^2} \leq \frac{\|\nabla H(\mat\Psi^{(t+1)}) + \xi^{(t+1)}\|^2}{\alpha^2}.
\end{equation*}
It follows that
\begin{equation*}
    \begin{aligned}
        & \|\nabla H(\mat\Psi^{(t+1)}) + \xi^{(t+1)}\|^2 \\
        & = \sum_{k=1}^K \|\nabla_k H(\mat\Psi^{(t+1)}) -\nabla_k H(\mat\Psi_{i<k}^{(t+1)},\mat\Psi_{i \geq k}^{(t)}) - \frac{1}{\eta_k^{(t+1)}}(\mat\Psi_k^{(t+1)} - \mat\Psi_k^{(t)})\|^2 \\
        &\leq \sum_{k=1}^K 2\|\nabla_k H(\mat\Psi^{(t+1)}) -\nabla_k H(\mat\Psi_{i<k}^{(t+1)},\mat\Psi_{i \geq k}^{(t)})\|^2 + \sum_{k=1}^K \frac{2}{(\eta_k^{(t+1)})^2}\|\mat\Psi_k^{(t+1)} - \mat\Psi_k^{(t)}\|^2 \\
        &\leq \sum_{k=1}^K 2\|\nabla H(\mat\Psi^{(t+1)}) -\nabla H(\mat\Psi_{i<k}^{(t+1)},\mat\Psi_{i \geq k}^{(t)})\|^2 + \sum_{k=1}^K \frac{2}{(\eta_k^{(t+1)})^2}\|\mat\Psi_k^{(t+1)} - \mat\Psi_k^{(t)}\|^2 \\
        &\leq \sum_{k=1}^K 2\Big(\sum_{j=1}^K\frac{1}{\eta_j^{(t+1)}}\Big)^2 \|\mat\Psi^{(t+1)}_{i \geq k} - \mat\Psi^{(t)}_{i \geq k}\|^2 + \sum_{k=1}^K \frac{2}{(\eta_k^{(t+1)})^2}\|\mat\Psi_k^{(t+1)} - \mat\Psi_k^{(t)}\|^2 \\
        &\leq \Bigg(2Kc^2\Big(\sum_{j=1}^K L_j\Big)^2 + 2c^2L_{\max}\Bigg)\|\mat\Psi^{(t+1)} - \mat\Psi^{(t)}\|^2.
    \end{aligned}
\end{equation*}
Therefore, we get
\begin{equation}\label{eqn:obj_gap}
    \mathcal{L}(\mat\Psi^{(t+1)}) - \min\mathcal{L} \leq \frac{\Bigg(2Kc^2\Big(\sum_{j=1}^K L_j\Big)^2 + 2c^2L_{\max}\Bigg)}{\alpha^2} \|\mat\Psi^{(t+1)} - \mat\Psi^{(t)}\|^2.
\end{equation}

\paragraph{Step 3.} Combining \eqref{eqn:sufficient_descent} and \eqref{eqn:obj_gap}, we have
\begin{equation*}
    \begin{aligned}
        \mathcal{L}(\mat\Psi^{(t)}) - \min\mathcal{L} &= \Big(\mathcal{L}(\mat\Psi^{(t)}) - \mathcal{L}(\mat\Psi^{(t+1)})\Big) + \Big(\mathcal{L}(\mat\Psi^{(t+1)}) - \min\mathcal{L}\Big) \\
        &\geq \frac{L_{\min}}{2} \|\mat\Psi^{(t)} - \mat\Psi^{(t+1)}\|^2 + \Big(\mathcal{L}(\mat\Psi^{(t+1)}) - \min\mathcal{L}\Big) \\
        &\geq \Bigg(\frac{\alpha^2L_{\min}}{4Kc^2(\sum_{j=1}^K L_j)^2 + 4c^2L_{\max}} + 1\Bigg) \Big(\mathcal{L}(\mat\Psi^{(t+1)}) - \min\mathcal{L}\Big). 
    \end{aligned}
\end{equation*}
This completes the proof.

\section{SG-PALM with Non-Convex Regularizers}\label{supp:nonconvex}
The estimation algorithm for non-convex regularizer is largely the same as Algorithm~\ref{alg:sg-palm}, except with an additional term added to the gradient term. Specifically, the updates are of the form
\begin{equation*}
    \mat\Psi_k^{(t+1)} = \text{prox}_{\eta^t_k\lambda_k}^{\|\cdot\|_{1,\text{off}}}\Big(\mat\Psi_k^t - \eta^t_k \nabla_k \bar{H}(\mat\Psi_{i < k}^{t+1},\mat\Psi_{i \geq k}^t)\Big),
\end{equation*}
where 
\begin{equation*}
    \bar{H}(\mat\Psi) = H(\mat\Psi) + \sum_{k=1}^K \sum_{i \neq j} \Big(g_{\lambda_k}([\mat\Psi_k]_{i,j})-\lambda_k|[\mat\Psi_k]_{i,j}|\Big).
\end{equation*}
Here, the formulation covers a range of non-convex regularizations. Particularly, the SCAD penalty~\citep{fan2001variable} with parameter $a>2$ is given by
\begin{equation*}
    g_\lambda(t) =
    \begin{cases}
    \lambda|t|, \quad \text{if} \quad |t|<\lambda \\
    -\frac{t^2-2a\lambda|t|+\lambda^2}{2(a-1)}, \quad \text{if} \quad \lambda < |t| < a\lambda \\
    \frac{(a+1)\lambda^2}{2}, \quad \text{if} \quad a\lambda < |t|,
    \end{cases}
\end{equation*}
which is linear for small $|t|$, constant for large $|t|$, and a transition between the two regimes for moderate $|t|$. 

The MCP penalty~\citep{zhang2010nearly} with parameter $a>0$ is given by
\begin{equation*}
    g_\lambda(t) = \text{sign}(t)\lambda\int_0^{|t|}\Big(1-\frac{z}{a\lambda}\Big)_+dz,
\end{equation*}
which gives a smoother transition between the approximately linear region and the constant region ($t>a\lambda$) as defined in SCAD.

The updates can also be written as
\begin{equation*}
    \mat\Psi_k^{(t+1)} = \text{prox}_{\eta^t_k\lambda_k}^{\|\cdot\|_{1,\text{off}}}\Bigg(\mat\Psi_k^t - \eta^t_k \nabla_k \Big(H(\mat\Psi_{i < k}^{t+1},\mat\Psi_{i \geq k}^t) + Q_{\lambda_k}'(\mat\Psi_k)\Big)\Bigg),
\end{equation*}
where $q_{\lambda}'(t):=\frac{d}{dt}(g_\lambda(t)-\lambda|t|)$ for $t \neq 0$ and $q_{\lambda}'(0)=0$ and $Q_\lambda'$ denotes $q_\lambda'$ applied elementwise to a matrix argument. These updates can be inserted into the framework of Algorithm~\ref{alg:sg-palm}. The details are summarized in Algorithm~\ref{alg:sg-palm-noncvx}.

\begin{algorithm}[!tbh]
\begin{algorithmic}
\caption{SG-PALM with non-convex regularizer}\label{alg:sg-palm-noncvx}
\REQUIRE Data tensor $\tensor{X}$, mode-$k$ Gram matrix $\mat{S}_k$, regularizing parameter $\lambda_k$, backtracking constant $c \in (0,1)$, initial step size $\eta_0$, initial iterate $\mat\Psi_k$ for each $k=1,\dots,K$.
\WHILE{not converged}
    \FOR{$k=1,\dots,K$}
        \STATE \textit{Line search:} Let $\eta^t_k$ be the largest element of $\{c^j \eta_{k,0}^t\}_{j=1,\dots}$ such that condition~\eqref{eqn:linesearch-cond} is satisfied for $\mat\Psi_k^{t+1} = \text{prox}_{\eta^t_k\lambda_k}^{\|\cdot\|_{1,\text{off}}}\Bigg(\mat\Psi_k^t - \eta^t_k \nabla_k \Big(H(\mat\Psi_{i < k}^{t+1},\mat\Psi_{i \geq k}^t) + Q_{\lambda_k}'(\mat\Psi_k)\Big)\Bigg)$.
        
        \STATE \textit{Update:} 
        
        $\mat\Psi_k^{t+1} \longleftarrow \text{prox}_{\eta^t_k\lambda_k}^{\|\cdot\|_{1,\text{off}}}\Bigg(\mat\Psi_k^t - \eta^t_k \nabla_k \Big(H(\mat\Psi_{i < k}^{t+1},\mat\Psi_{i \geq k}^t) + Q_{\lambda_k}'(\mat\Psi_k)\Big)\Bigg)$.
    \ENDFOR
    \STATE \textit{Next initial stepsize:} Compute Barzilai-Borwein stepsize $\eta_0^{t+1}=\min_k \eta^{t+1}_{k,0}$, where $\eta^{t+1}_{k,0}$ is computed via~\eqref{eqn:bb-step}.
\ENDWHILE
\ENSURE Final iterates $\{\mat\Psi_k\}_{k=1}^K$.
\end{algorithmic}
\end{algorithm}

\subsection{Convergence Property}
Consider a sequence of iterate $\{\mat{x}^t\}_{t \in \mathbb{N}}$ generated by a generic PALM algorithm for minimizing some objective function $f$. Specifically, assume 
\begin{itemize}
    \item[] $(\mathcal{H}_1)$ $\inf f > -\infty$.
    \item[] $(\mathcal{H}_2)$ The restriction of the function to its domain is a continuous function.
    \item[] $(\mathcal{H}_3)$ The function satisfies the KL property.
\end{itemize}

Then, as in Theorem 2 of \citet{attouch2009convergence}, if this objective function satisfying $(\mathcal{H}_1),(\mathcal{H}_2),(\mathcal{H}_3)$ in addition satisfies the KL property with
\begin{equation*}
    \phi(s) = \alpha s^{1-\theta},
\end{equation*}
where $\alpha>0$ and $\theta \in (0,1]$. Then, for $\mat{x}^{\ast}$ some critical point of $f$, the following estimations hold
\begin{enumerate}[label=(\roman*)]
    \item If $\theta=0$ then the sequence of iterates converges to $\mat{x}^{\ast}$ in a finite number of steps.
    \item If $\theta \in (0,\frac{1}{2}]$ then there exist $\omega>0$ and $\tau \in [0,1)$ such that $\|\mat{x}^t-\mat{x}^{\ast}\| \leq \omega \tau^t$.
    \item If $\theta \in (\frac{1}{2},1)$ then there exist $\omega>0$ such that $\|\mat{x}^t-\mat{x}^{\ast}\| \leq \omega t^{-\frac{1-\theta}{1\theta-1}}$.
\end{enumerate}

In the case of SG-PALM with non-convex regularizations, so long as the non-convex $\mathcal{L}$ satisfies the KL property with an exponent in $(0,\frac{1}{2}]$, the algorithm remains linearly convergent (to a critical point). We argue that this is true for SG-PALM with MCP or SCAD penalty. \citet{li2018calculus} showed that penalized least square problems with such penalty functions satisfy the KL property with an exponent of $\frac{1}{2}$. The proof strategy for the convex case can be easily adopted, incorporating the KL results for MCP and SCAD in \citet{li2018calculus}, to show that the new $\mathcal{L}$ still has KL exponent of $\frac{1}{2}$. Therefore, SG-PALM with MCP or SCAD penalty converges linearly in the sense outlined above.

\section{Additional Details of the Solar Flare Experiments}\label{supp:additional_experiments}
\subsection{HMI and AIA Data}
The Solar Dynamics Observatory (SDO)/Helioseismic \& Magnetic Imager (HMI) data characterize solar variability including the Sun's interior and the various components of magnetic activity; the SDO/Atmospheric Imaging Assembly (AIA) data contain a set of measurements of the solar atmosphere spectrum at various wavelengths. In general, HMI produces data that is particularly useful in determining the mechanisms of solar variability and how the physical processes inside the Sun that are related to surface magnetic field and activity. AIA contains structural information about solar flares, and the the high AIA pixel values are correlated with the flaring intensities. We are interested in examining if combination of multiple instruments enhances our understanding of the solar flares, comparing to the case of single instrument. Both HMI and AIA produce multi-band (or multi-channel) images, for this experiment we use all three channels of the HMI images and $9.4, 13.1, 17.1, 19.3$ nm wavelength channels of the AIA images. For a detailed descriptions of the instruments and all channels of the images, see \url{https://en.wikipedia.org/wiki/Solar_Dynamics_Observatory} and the references therein. Furthermore, for training and testing involved in this study, we used the data described in~\citep{galvez2019machine}, which are further pre-processed HMI and AIA imaging data for machine learning methods.

\subsection{Classification of Solar Flares/Active Regions (AR)}
The classification system for solar flares uses the letters A, B, C, M or X, according to the peak flux in watts per square metre ($W/m^2$) of X-rays with wavelengths $100$ to $800$ picometres ($1$ to $8$ angstroms), as measured at the Earth by the GOES spacecraft (\url{https://en.wikipedia.org/wiki/Solar_flare#Classification}). Here, A usually refers to a ``quite'' region, which means that the peak flux of that region is not high enough to be classified as a real flare; B usually refers a ``weak'' region, where the flare is not strong enough to have impact on spacecrafts, earth, etc; and M or X refers to a ``strong'' region that is the most detrimental. Differentiating between a weak and a strong flare/region ahead of time is a fundamental task in space physics and has recently attracted attentions from the machine learning community~\citep{chen2019identifying,jiao2019solar,sun2019interpreting}. In our study, we also focus on B and M/X flares and attempt to predict the videos that lead to either one of these two types of flares.

\subsection{Run Time Comparison}
We compare run times of the SG-PALM algorithm for estimating the precision matrix from the solar flare data with SyGlasso. Table~\ref{tab:solar_flare_run_time} illustrates that the SG-PALM algorithm converges faster in wallclock time. Note that in this real dataset, which is potentially non-Gaussian, the convergence behavior of the algorithms is different compare to synthetic examples. Nonetheless, SG-PALM enjoys an order of magnitude speed-up over SyGlasso.

\begin{table}[!tbh]
\centering
\caption{Run time (in seconds) comparisons between SyGlasso and SG-PALM on solar flare data for different regularization parameters. Note that the SG-PALM is an order of magnitude faster that SyGlasso.}
\label{tab:solar_flare_run_time}
\begin{tabular}{|c||c|c|}
 \multicolumn{3}{c}{} \\
 \hline
 \multirow{2}{*}{$\lambda$} & \textbf{SyGlasso} & \textbf{SG-PALM} \\
 \cline{2-3} 
 & \textbf{iter} \quad \textbf{sec} & \textbf{iter} \quad \textbf{sec} \\
 \hline
 $0.28$ & $47$ \quad $5772.1$ & $89$ \quad $583.7$ \\
 $0.41$ & $43$ \quad $5589.0$ & $86$ \quad $583.4$\\
 $0.54$ & $45$ \quad $5673.7$ & $85$ \quad $568.8$ \\
 $0.67$ & $42$ \quad $5433.0$ & $77$ \quad $522.6$ \\
 $0.79$ & $39$ \quad $4983.2$ & $82$ \quad $511.4$ \\
 $0.92$ & $40$ \quad $5031.9$ & $72$ \quad $498.0$ \\
 $1.05$ & $39$ \quad $4303.7$ & $76$ \quad $452.2$ \\
 $1.18$ & $41$ \quad $4234.7$ & $64$ \quad $437.6$ \\
 $1.30$ & $40$ \quad $4039.5$ & $58$ \quad $406.9$ \\
 $1.43$ & $35$ \quad $3830.7$ & $64$ \quad $364.9$ \\
 \hline
\end{tabular}
\end{table}

\subsection{Examples of Predicted Magnetogram Images}
Figure~\ref{fig:hmi_predicted_vs_real_img} depicts examples of the predicted HMI channels by SG-PALM. We observe that the proposed method was able to reasonably capture various components of the magnetic field and activity. Note that the spatial behaviors of the HMI components are quite different from those of AIA channels, that is, the structures tend to be less smooth and continuous (e.g., separated holes and bright spots) in HMI.

\begin{figure}[tbh!] 
\centering
\begin{tabular}{@{}c@{}}
    Predicted HMI examples - B vs. M/X \\
    \rotatebox{90}{\qquad AR B}
    \includegraphics[width=0.6\textwidth]{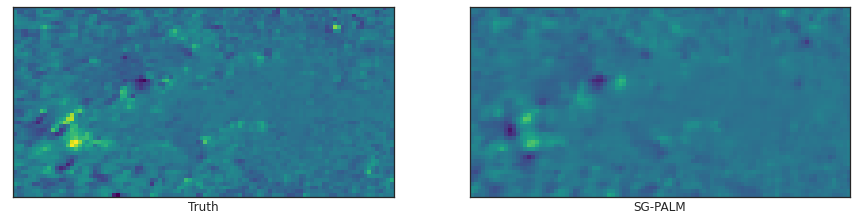}  \\
    \rotatebox{90}{\qquad  AR B}
    \includegraphics[width=0.6\textwidth]{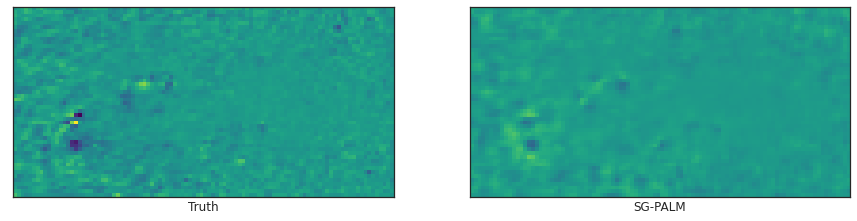} \\
    \rotatebox{90}{\qquad  AR B}
    \includegraphics[width=0.6\textwidth]{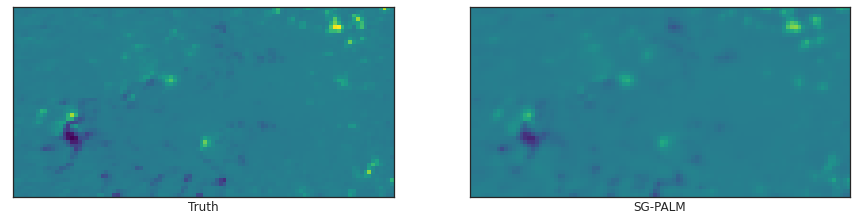} \\
    \rotatebox{90}{\quad  AR M/X}
    \includegraphics[width=0.6\textwidth]{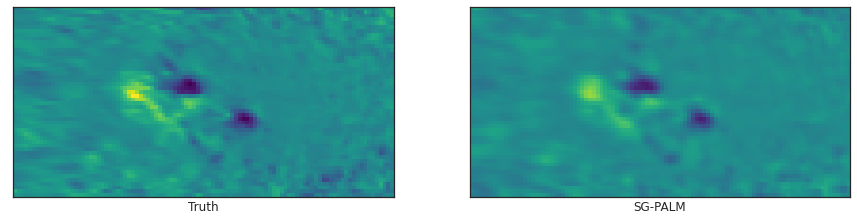} \\
    \rotatebox{90}{\quad  AR M/X}
    \includegraphics[width=0.6\textwidth]{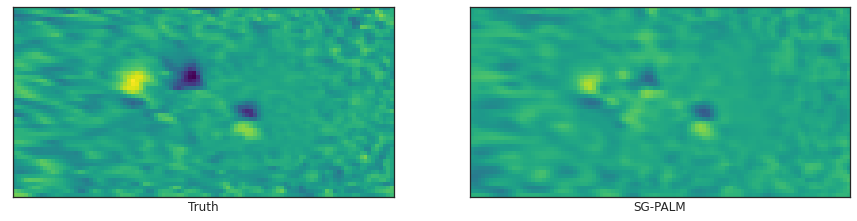} \\
    \rotatebox{90}{\quad  AR M/X}
    \includegraphics[width=0.6\textwidth]{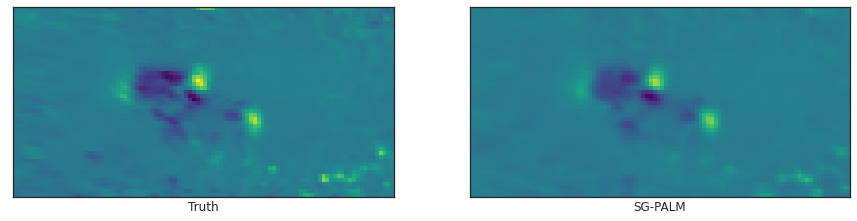}
\end{tabular}
\caption{Examples of one-hour ahead prediction of the first three channels (HMI components) of ending frames of $13$-frame videos, leading to B- (first three rows) and MX-class (last three rows) flares, produced by the SG-PALM, comparing to the real image (left column). Similarly to AIA predictions, linear forward predictors tend to underestimate the contrast ratio of the images. Nonetheless, the SG-PALM algorithm was able to both capture the spatial structures of the underlying magnetic fields. HMI images tend to be harder to predict, as indicated by the increased number and decreased degree of smoothness of features, signifying the underlying magnetic activity on the solar surface.}
\label{fig:hmi_predicted_vs_real_img}
\end{figure}

\subsection{Multi-instrument vs. Single Instrument Prediction}
To illustrate the advantages of multi-instrument analysis, we compare the NRMSEs between an AIA-only (i.e., last four channels of the dataset) and an HMI\&AIA (i.e., all seven channels of the dataset) study in predicting the last frames of $13$-frame AIA videos, for each flare class, respectively, using the proposed SG-PALM. The results are depicted in Figure~\ref{fig:hmi_vs_aia}, where the average, standard deviation, and range of the NRMSEs across pixels are also shown for each error image. By leveraging the cross-instrument correlation structure, there is a $0.5\%-1\%$ drop in the averaged error rates and a $2\%-4\%$ drop in the range of the errors.

\begin{figure}[tbh!] 
\centering
\begin{tabular}{@{}c@{}}
    \qquad Avg. NRMSE $=0.0379$ (w/ HMIs) \quad Avg. NRMSE $=0.0479$ (w/o HMIs) \\
    \rotatebox{90}{\qquad \qquad AR B}
    \includegraphics[width=0.75\textwidth]{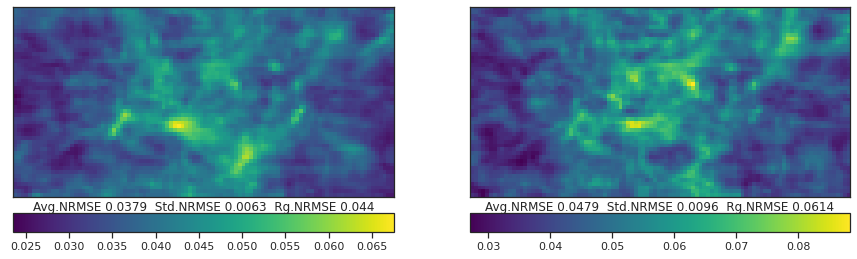}  \\
    \qquad Avg. NRMSE $=0.0620$ (w/ HMIs) \quad Avg. NRMSE $=0.0674$ (w/o HMIs) \\
    \rotatebox{90}{\qquad \qquad AR MX}
    \includegraphics[width=0.75\textwidth]{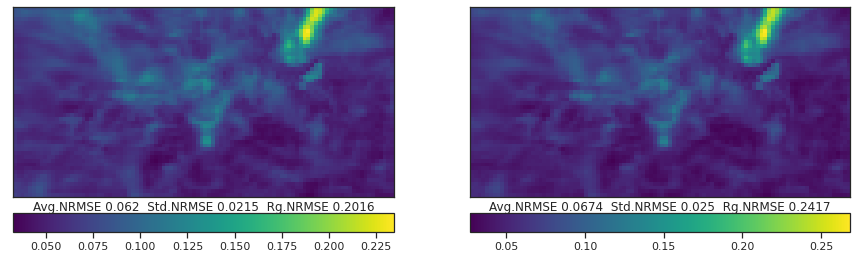} 
\end{tabular}
\caption{Comparison of the SG-PALM performance measured by NRMSE in predicting the AIA channels (i.e., last four channels) of the ending frame of $13$-frame videos leading to B- and MX-class solar flares, by using all HMI\&AIA channels (left column) and AIA-only channels (right column). The NRMSEs are computed by averaging across both testing samples and channels for each pixel. Note that there are improvements in both the averaged errors rates and the uncertainty in those errors (i.e., range of the errors) by including multi-instrument image channels.}
\label{fig:hmi_vs_aia}
\end{figure}

\subsection{Illustration of the Difficulty of Predictions for Two Flares Classes}
We demonstrate the difficulty of forward predictions of video frames. Figure~\ref{fig:mx_predicted_vs_prev_img} depicts two different channels of multiple frames from two videos leading to MX-class solar flares. Note that the current frame is the $13$th frame in the sequence that we are trying to predict. We observe that the prediction task is particularly difficult if there is a sudden transition of either the brightness or spatial structure of the frames near the end of the video. These sudden transitions are more frequent for MX flares than for B flares. In addition, as MX flares are generally considered as rare events (i.e., less frequent than B flares), it is harder for SG-PALM or related methods to learn a common correlation structures from training data.

On the other hand, typical image sequences leading to B flares exhibit much smoother transitions from frame to frame. As shown in Figure~\ref{fig:b_predicted_vs_prev_img}, the SG-PALM was able to produce remarkably good predictions of the current frames. 

\begin{figure}[tbh!] 
\centering
\begin{tabular}{@{}c@{}}
    Predicted examples - M/X \\
    \includegraphics[width=0.85\textwidth]{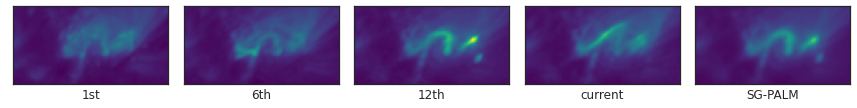}  \\
    \includegraphics[width=0.85\textwidth]{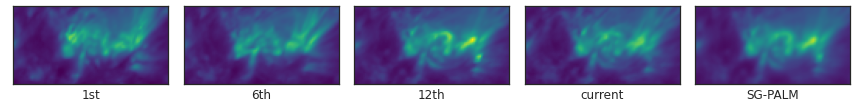} \\
    \includegraphics[width=0.85\textwidth]{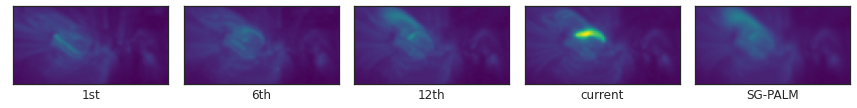} \\
    \includegraphics[width=0.85\textwidth]{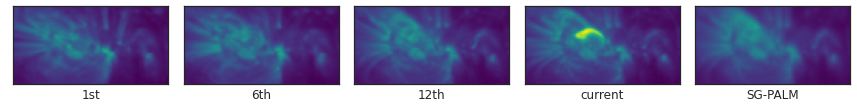}
\end{tabular}
\caption{Examples of frames at various timestamps of videos preceding the predictions of the last frames (last column) that lead to MX flares. Here, the first two rows correspond to the same video as the last two rows in Figure~\ref{fig:predicted_vs_real_img}. Note that the prediction tasks are difficult in these two extreme cases, where there are dramatic changes from the $12$th to the current ($13$th) frames.}
\label{fig:mx_predicted_vs_prev_img}
\end{figure}

\begin{figure}[tbh!] 
\centering
\begin{tabular}{@{}c@{}}
    Predicted examples - B  \\
    \includegraphics[width=0.85\textwidth]{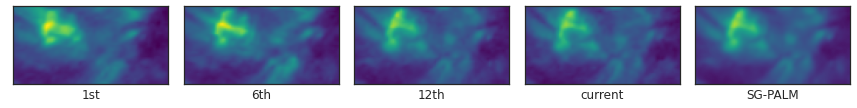} \\
    \includegraphics[width=0.85\textwidth]{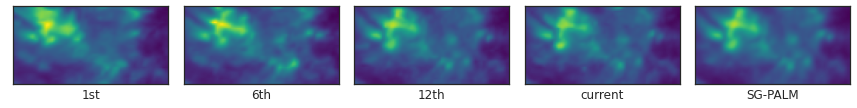} \\
    \includegraphics[width=0.85\textwidth]{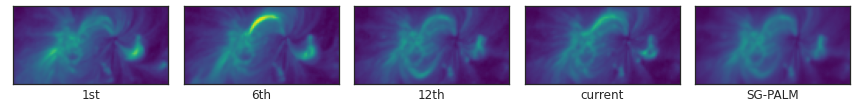} \\
    \includegraphics[width=0.85\textwidth]{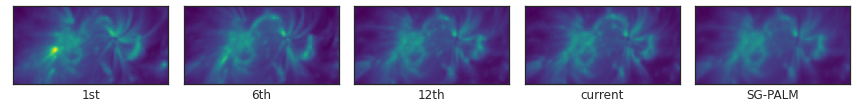}
\end{tabular}
\caption{Examples of frames at various timestamps of videos preceding the predictions of the last frames (last column) that lead to B flares. Here, the first two rows correspond to the same video as the first two rows in Figure~\ref{fig:predicted_vs_real_img}. Note that the prediction tasks are easier than those illustrated in Figure~\ref{fig:mx_predicted_vs_prev_img}, since the transitions near the end of the videos are much smoother.}
\label{fig:b_predicted_vs_prev_img}
\end{figure}

\subsection{Illustration of the Estimated Sylvester Generating Factors}
Figure~\ref{fig:psih_all} illustrates the patterns of the estimated Sylvester generating factors ($\mat\Psi_k$'s) for each flare class. Here, the videos from both classes appear to form Markov Random Fields, that is, each pixel only depends on its close neighbors in space and time given all other pixels. This is demonstrated by observing that the temporal or each of the spatial generating factor, which can be interpreted as conditional dependence graph for the corresponding mode, has its energies concentrate around the diagonal and decay as the nodes move far apart (in space or time).

The spatial patterns are similar for different flares. Although the exact spatial patterns are different from one frame to another, they always have their energies being concentrated at certain region (i.e., the brightest spot) that is usually close to the center of the images. This is due to the way how these images were curated and pre-processed before analysis. On the other hand, the temporal structures are quite different. Specifically, B flares tend to have longer range dependencies, as the frames leading to these types flares are smooth, which is consistent with results from the previous section.

\begin{figure}[tbh!] 
\centering
\begin{subfigure}[t]{0.47\linewidth}
\centering
\includegraphics[width=\textwidth]{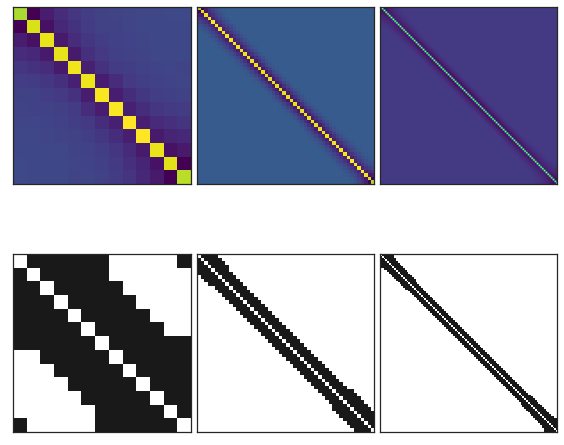}
\caption{Estimated precision matrices - B flares}
\end{subfigure}
\begin{subfigure}[t]{0.47\linewidth}
\centering
\includegraphics[width=\textwidth]{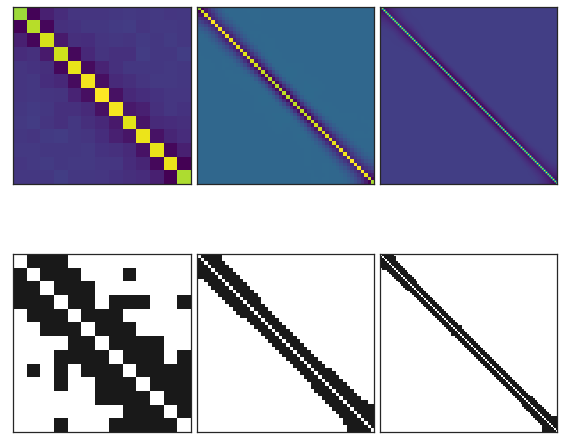}
\caption{Estimated precision matrices - M/X flares}
\end{subfigure}
\caption{Estimated spatial and two (longitude and latitude) temporal Sylvester generating factors for B and MX solar flares, along with their off-diagonal sparsity patterns (second row in each subplot). Both classes exhibit autoregressive dependence structures (across time or space). Note the significant difference in the temporal components, where the B flares exhibit longer range dependency. This is consistent with the smooth transition property of the corresponding videos as illustrated previously.}
\label{fig:psih_all}
\end{figure}

%% file: Appendices/Appendix_C.tex
In this Appendix, we discuss the blocked versions of the Possion-AR(1) and convection-diffusion processes.

\paragraph{Poisson-AR(1) Process.}
The first extension, which we call the Poisson-AR(1) process, imposes an autoregressive temporal model of order 1 on the source function $f$ in the Poisson equation (\ref{eq:Poisson}). Specifically, we say a sequence of discretized spatial observations $\{\bU^k \in \bbR^{d_1\times d_2}\}_k$ indexed by time step $k=1,\cdots,T$ is from a Poisson-AR(1) process if
\begin{align}
    &(\mat{A}_{d_1} \oplus \mat{A}_{d_2}) \vecto(\mat{U}^k) = \vecto(\mat{Z}^k), \\
    &\vecto(\mat{Z}^k) = a \vecto(\mat{Z}^{k-1}) + \vecto(\mat{W}^k),\quad |a|<1, \label{eq:ar1}
\end{align}
where $\{\mat{W}^k \in \bbR^{d_1\times d_2}\}_k$ is spatiotemporal white noise, i.e., $W_{i,j}^k \sim \mathcal{N}(0, \sigma^2_w)$, i.i.d.
Assuming $\mat{Z}^0 = \mathbf{0}$ and defining the $T$-by-$T$ matrix
\begin{equation*}
    \mat{B} = \begin{bmatrix}
        1 & -a &  &  \\
         & 1 & \ddots  &  \\
        & & \ddots & -a \\
        & & & 1
    \end{bmatrix},
\end{equation*}
the above linear system of equations can be written as $(\mat{A}_{d_1} \oplus \mat{A}_{d_2}) \mat{U} \mat{B} = \mat{W}$, or equivalently,
\begin{align}
    \left( \mat{B}^T \otimes (\mat{A}_{d_1} \oplus \mat{A}_{d_2}) \right) \vecto(\mat{U}) = \vecto(\mat{W}),
    \label{eq:poisson_ar_discrete}
\end{align}
where $\mat{U} = [\vecto(\mat{U}^1) \vecto(\mat{U}^2) \dots \vecto(\mat{U}^T)] \in \Reals^{d_1d_2 \times T}$ and $\mat{W}$ is defined likewise. The inverse covariance of $\mat{U}$, despite having a large size of $d_1d_2T \times d_1d_2T$, is sparse and has a mixed Kronecker sum and product structure.

\paragraph{Convection-diffusion Process.} 
The second time-varying extension of the Poisson PDE model (\ref{eq:Poisson}) is based on the convection-diffusion process~\cite{chandrasekhar1943stochastic}
\begin{equation}
    \pdv{u}{t} = \theta \sum_{i=1}^2 \pdv[2]{u}{x_i} - \epsilon \sum_{i=1}^2 \pdv{u}{x_i}.
\end{equation}
Here, $\theta > 0$ is the diffusivity; and $\epsilon \in \Reals$ is the convection velocity of the quantity along each coordinate. Note that for simplicity of discussion here, we assume these coefficients do not change with space and time (see, \citet{stocker2011introduction}, for example, for a detailed discussion). These equations are closely related to the Navier-Stokes equation commonly used in stochastic modeling for weather and climate prediction~\citep{chandrasekhar1943stochastic,stocker2011introduction}. Coupled with Maxwell's equations, these equations can be used to model magneto-hydrodynamics~\citep{roberts2006slow}, which characterize solar activities including flares.


A solution of Equation~\eqref{eqn:convec-diff} can be approximated similarly as in the Poisson equation case, through a finite difference approach. Denote the discrete spatial samples of $u(\mat{x},t)$ at time $t_k$ as a matrix $\mat{U}^k\in\bbR^{d_1 \times d_2}$. We obtain a discretized update propagating $u(\mat{x},t)$ in space and time, which locally satisfies
\begin{equation}
\begin{aligned}
    \frac{U_{i,j}^k - U_{i,j}^{k-1}}{\Delta t} = &\ \theta \left(\frac{U_{i+1,j}^k + U_{i-1,j}^k + U_{i,j+1}^k + U_{i,j-1}^k - 4U_{i,j}^k}{h^2}\right) \\
    &- \epsilon \left(\frac{U_{i+1,j}^k - U_{i-1,j}^k + U_{i,j+1}^k - U_{i,j-1}^k}{2h}\right),
\end{aligned}
\end{equation}
where $\Delta t = t_{k+1} - t_{k}$ is the time step and $h$ is the mesh step (spatial grid spacing). 
Similarly to the Poisson-AR(1) process, in the following, we consider a ``blocked'' version of the convection-diffusion process. 

We define the first-order and second-order discretized differential operators, denote by $\mat{D}$ and $\mat{A}$, respectively: 
\begin{equation*}
\mat{D} = 
    \begin{bmatrix}
    1   &     &       &   \\
    -1  &   1   & &   \\
        & \ddots& \ddots& \\
        &       &   -1  & 1
    \end{bmatrix}, \quad
\mat{A} = 
    \begin{bmatrix}
    2   &   -1  &       &   \\
    -1  &   2   & \ddots&   \\
        & \ddots& \ddots& -1\\
        &       &   -1  & 2
    \end{bmatrix}.
\end{equation*}
Then, Equation~\eqref{eqn:convec-diff-discrete} can be written as
\begin{equation}
\begin{aligned}
    \frac{1}{\Delta t}(\mat{D} \otimes \mat{I} \otimes \mat{I}) \vecto{\mat{U}} =& \frac{\theta}{h^2}(\mat{I} \otimes \mat{A} \otimes \mat{I} + \mat{I} \otimes \mat{I} \otimes \mat{A}) \vecto{\mat{U}} \\
    &- \frac{\epsilon}{2h} (\mat{I} \otimes \mat{D} \otimes \mat{I} + \mat{I} \otimes \mat{I} \otimes \mat{D}) \vecto{\mat{U}},
\end{aligned}
\end{equation}
where $\mat{U} = [\vecto(\mat{U}^1) \vecto(\mat{U}^2) \dots \vecto(\mat{U}^T)] \in \Reals^{d_1d_2 \times T}$. Assuming the process is driven by some white noise $\mat{W}$, similarly defined as in the Poisson-AR equation, the inverse covariance of $\mat{U}$ is again sparse and has a mixed Kronecker sum and product structure.

We consider a spatio-temporal process (2D space + time) on a $8 \times 8$ spatial grid, and generated instances of state trajectories, which we call true states, according to the Poisson-AR(1) and the convection-diffusion dynamics for $T = 50$ time steps. Several realizations of the true state variables are shown in Figure~\ref{fig:enkf_states} to illustrate how the states evolve over time under each model. 

We generated $N = 50$ independent realizations of random tensors of dimension $64 \times 50$ and estimated the state covariance / inverse covariance (with $K = 2$) using several sparse (multiway) inverse covariance estimation methods described in Section~\ref{sec:enkf-background} of Chapter~\ref{ch:enkf}, including Glasso, KPCA, 
Tlasso, 
TeraLasso, 
SG-PALM. Note that none of the above-mentioned models operate under the true generative processes (i.e., there is model mismatch with the data).
Here, the sparsity-regularized methods are all implemented with an $\ell_1$ penalty function, and the penalty parameters were selected similarly and guided by the theoretical results in Table~\ref{tab:enkf_guarantees}. For example, for SG-PALM, we use a penalty parameter of $\lambda_k = C \sqrt{\frac{d_k \log d}{N}}$ where $C$ is chosen by optimizing a normalized Frobenius norm error between the estimate and the truth, over a range of $\lambda$ values parameterized by $C$. For the KPCA alorithm, both the nuclear norm penalty parameter and the separation rank are selected by optimizing a normalized Frobenius norm error via grid search.

Summary of the estimation accuracy in terms of the recovery of the matrix entries measured normalized Frobenius norm error as well as the recovery of the sparsity patterns measured by Mathews Correlation Coefficient~\citep{matthews1975comparison} are reported in Table~\ref{tab:sythetic_perf}. In Figure~\ref{fig:poisson_inv_cov_struct_compare} and \ref{fig:convec_diff_inv_cov_struct_compare} we show the true and the estimated inverse covariance matrices obtained for all the methods except KPCA, under both the Poisson-AR (panel (a)) and the Convection-Diffusion processes (panel (b)). The inverse covariances under both generating processes admit structures with a mix of Kronecker sums and Kronecker products of sparse matrices. In both of the cases, the SG-PALM produces the estimates with the closest and richest structures, which we believe is due to the nature of the Sylvester graphical model that imposes a squared KS structure on the precision matrix. Tlasso has comparable performances and achieves the best matrix approximation error under the convection-diffusion generating process. This might be due to the fact that the KP model corresponds to an underlying spatio-temporal autoregressive process. TeraLasso seems to produce the biggest model mismatch as indicated by the MCC scores. Although Glasso works reasonable well given that it ignores any multiway structure, this also leads to an increased computational cost for the vector-variate estimating algorithm. In Figure~\ref{fig:cov_struct_compare}, we also show compare the true covariance matrix and the estimate obtained by KPCA. Here, although the KPCA model does not match the underlying generating process, the estimates were able to capture certain blocking patterns that similarly exist in the true covariance.

\begin{figure}[!tbh]
    \centering
    \includegraphics[width=0.3\textwidth]{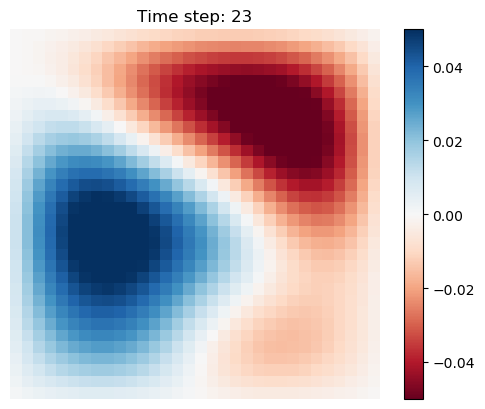}
    \includegraphics[width=0.3\textwidth]{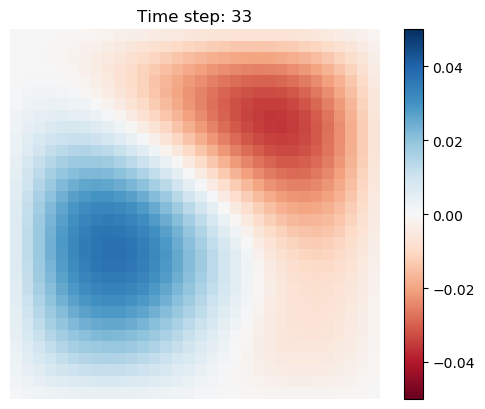}
    \includegraphics[width=0.3\textwidth]{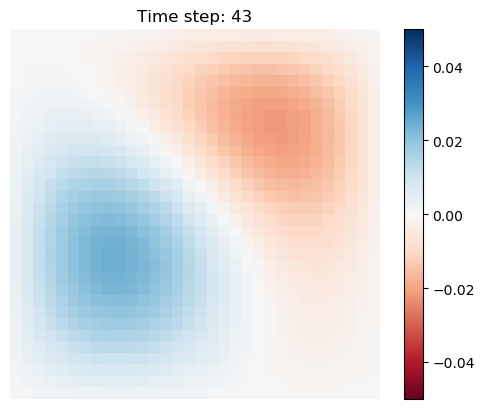}
    \includegraphics[width=0.3\textwidth]{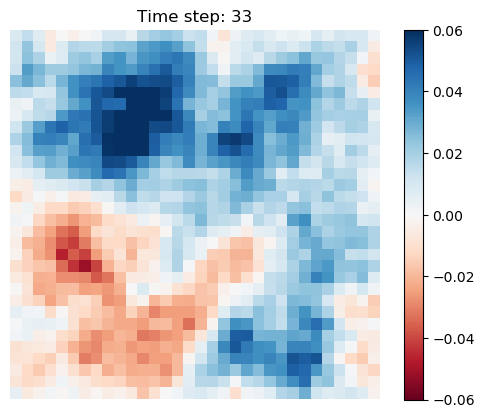}
    \includegraphics[width=0.3\textwidth]{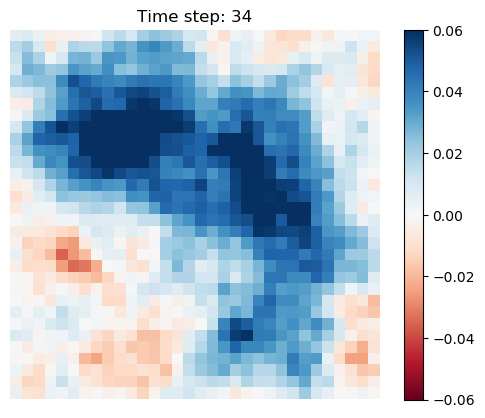}
    \includegraphics[width=0.3\textwidth]{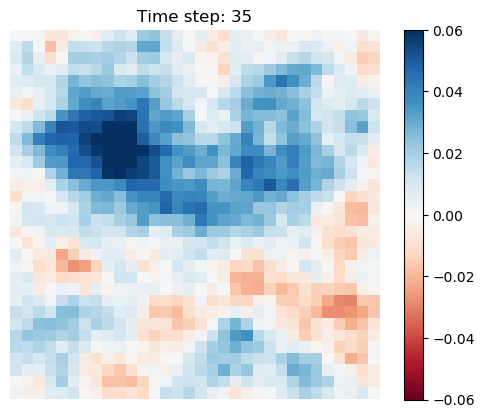}
    \caption{2D Convection-diffusion (top) and Poisson-AR(1) state variables at three different time steps.}
    \label{fig:enkf_states}
\end{figure}

\begin{figure}
    \centering   
    \begin{subfigure}{\textwidth}
    \centering 
    \includegraphics[width=0.8\textwidth]{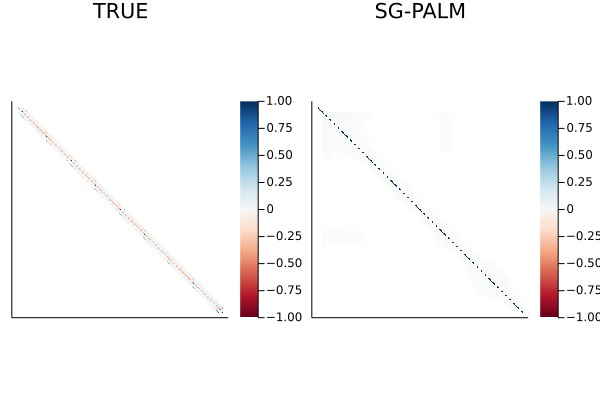}
    \caption{Poisson-AR inverse covariance (left) and the SG-PALM estimate (right).}
    \end{subfigure}
    \begin{subfigure}{\textwidth}
    \centering 
    \includegraphics[width=0.9\textwidth]{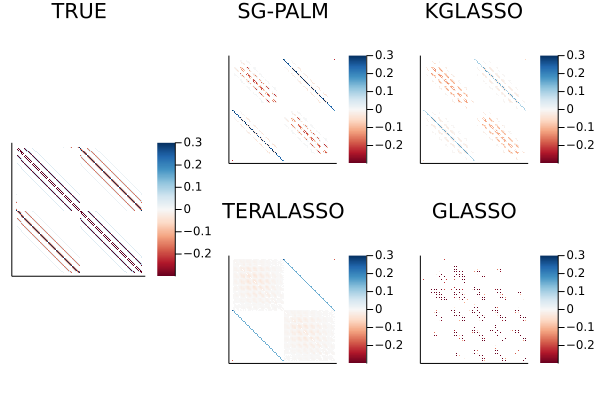}
    \caption{Zoomed-in (middle $128$ rows and columns) Poisson-AR inverse covariance structure (left) and the estimate obtained by SG-PALM, KGlasso, Glasso, TeraLasso (right, clockwise).}
    \end{subfigure}
    \caption{Inverse covariance structures for Poisson-AR(1) and its estimates. Here, white entries indicate zeros in the inverse covariance matrices. The zoomed-in plots show two temporal blocks (each of size $64 \times 64$) of spatial inverse correlation structures with the diagonal elements removed for clearer visualization. SG-PALM and the associated Sylvester graphical model produce the richest structures.
    }
    \label{fig:poisson_inv_cov_struct_compare}
\end{figure}

\begin{figure}
    \centering   
    \begin{subfigure}{\textwidth}
    \centering 
    \includegraphics[width=0.8\textwidth]{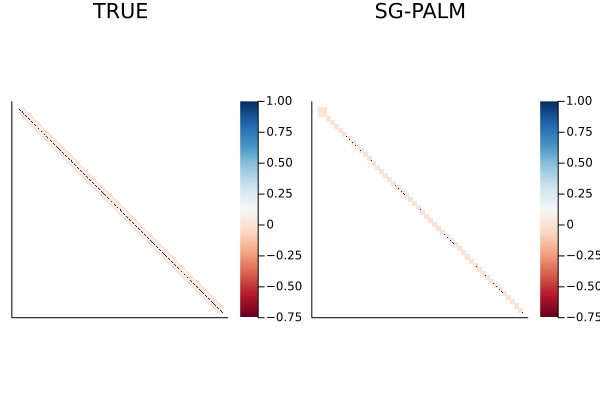}
    \caption{Convection-diffusion inverse covariance (left) and the SG-PALM estimate (right).}
    \end{subfigure}
    \begin{subfigure}{\textwidth}
    \centering 
    \includegraphics[width=0.9\textwidth]{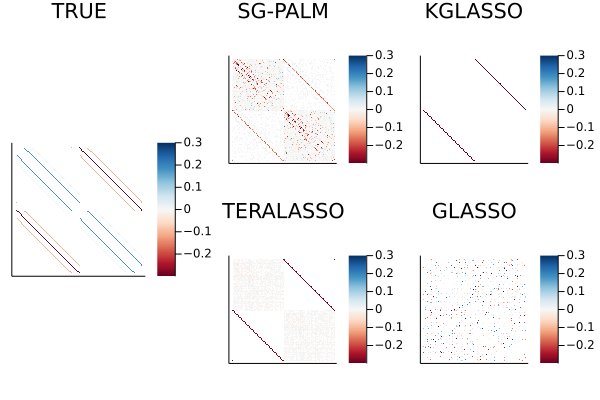}
    \caption{Zoomed-in (middle $128$ rows and columns) convection-diffusion inverse covariance (left) and the estimates by SG-PALM, KGlasso, Glasso, TeraLasso (right, clockwise).}
    \end{subfigure}
    \caption{Inverse covariance structures for the Convection-Diffusion and its estimates. Here, white entries indicate zeros in the inverse covariance matrices. The zoomed-in plots show two temporal blocks ($64 \times 64$) of spatial inverse correlation structures with the diagonal elements removed for clearer visualization. SG-PALM and the associated Sylvester graphical model produce the richest structures.
    }
    \label{fig:convec_diff_inv_cov_struct_compare}
\end{figure}

\begin{figure}
    \centering   
    \begin{subfigure}{\textwidth}
    \centering\includegraphics[width=0.8\textwidth]{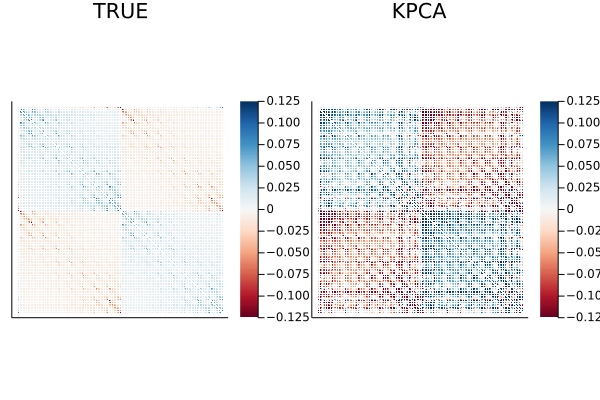}
    \caption{Poisson-AR covariance structure (left) and the estimate obtained by KPCA (right).}
    \end{subfigure}
    \begin{subfigure}{\textwidth}
    \centering\includegraphics[width=0.8\textwidth]{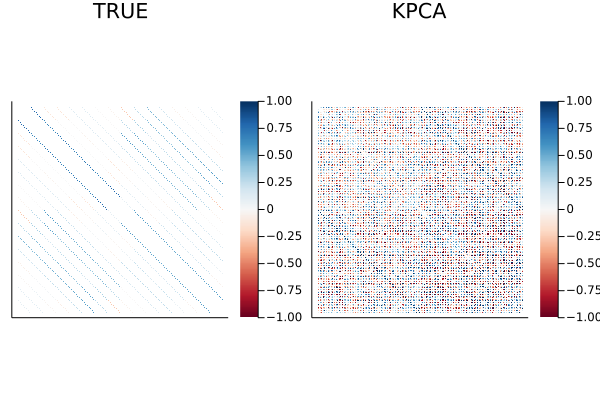}
    \caption{Convection-Diffusion covariance (left) and the estimate obtained by KPCA (right).}
    \end{subfigure}
    \caption{Visualizations of the middle $128$ rows and columns of the covariance structures for Poisson-AR(1) and Convection-Diffusion dynamics and their estimates, which show two temporal blocks of spatial correlation structures, each of size $64 \times 64$, with the diagonal elements removed for clearer visualization of the pattern. Here, white entries indicate zeros in the covariance matrices. Since the covariances are not sparse in general, all matrices are thresholded for clearer inspections of patterns.}
    \label{fig:cov_struct_compare}
\end{figure}

\begin{table}[!tbh]
\centering
\caption{Comparisons of performances measured by $\log(\|\widehat{\bSigma} - \bSigma\|_F \text{\textbackslash} \|\bSigma\|_F)$ for KPCA as well as $\log(\|\widehat{\bOmega} - \bOmega\|_F \text{\textbackslash} \|\bOmega\|_F)$ and the Mathews Correlation Coefficient (MCC) for SG-PALM, Tlasso, TeraLasso, Glasso. The MCC is a measure of the quality of sparsity recovery considered as a binary classification problem, where $\pm 1$ indicates perfect agreement or disagreement between the truth and the estimation. Here the Frobenius norm errors are included in the first row under each generating type while the MCCs are in the second row. Note that the best performers under each type/criteria are highlighted.}
\label{tab:sythetic_perf}
\begin{tabular}{|c|c||c|c|c|c|c|}
 \multicolumn{7}{c}{} \\
 \hline
 Type & Metric & \textbf{SG-PALM} & \textbf{KGlasso} & \textbf{TeraLasso} & \textbf{Glasso} & \textbf{KPCA}\\
 \cline{1-7} 
 \multirow{2}{*}{P-AR} & Fnorm & $\mathbf{-0.2622}$ & $1.1777$ & $0.6312$ & $0.9775$ & $0.3289$\\
 & MCC & $\mathbf{0.4300}$ & $0.3395$ & $0.2061$ & $0.0560$ & N/A\\
\hline
 \multirow{2}{*}{C-D} & Fnorm & $\mathbf{-0.0420}$ & $1.4919$ & $-0.0208$ & $2.2041$ & $0.0642$\\
 & MCC & $\mathbf{0.2122}$ & $0.1884$ & $0.2018$ & $0.0349$ & N/A\\
 \hline
\end{tabular}
\end{table}

Computational efficiencies of the various covariance/precision estimation algorithms are also vitally important in practice to facilitate real-time tracking of physical systems. Table~\ref{tab:blocked_enkf_runtime} shows the runtime of different covariance and inverse covariance estimation algorithms for the synthetic experiments. It shows that by recognizing and exploiting multiway structures in the data, sparse multiway inverse covariance estimation methods, TeraLasso, Tlasso, and SG-PALM significantly reduce the runtime complexity of Glasso that ignores such special multiway structures. Remark that KPCA runs considerably slower than other methods as it involves expensive singular value decomposition of a large-dimensional re-arranged sample covariance matrix of the data.

\begin{table}[tbh!]
\centering
\caption{Runtime (in seconds) of estimating spatio-temporal (inverse) covariance matrices of size $d \times 50$, where $d$ is varying, using various algorithms. Comparisons under various problem sizes (i.e., different $d$ and $N$) are shown. Note the sparse multiway precision models (SG-PALM, KGlasso, TeraLasso) are comparably fast and are all faster than Glasso (for large problems) and KPCA.}
\label{tab:blocked_enkf_runtime}
\begin{tabular}{|p{0.4cm}|p{0.5cm}||r|r|r|r|r|}
 \multicolumn{7}{c}{} \\
 \hline
 \multirow{2}{*}{$d$} & \multirow{2}{*}{$N$} & \textbf{Glasso} & \textbf{SG-PALM} & \textbf{TeraLasso} & \textbf{KGlasso} & \textbf{KronPCA}\\
 \cline{3-7} 
 && \textbf{sec} & \textbf{sec} & \textbf{sec} & \textbf{sec} & \textbf{sec} \\
 \hline
 \multirow{3}{*}{$8^2$} & $25$ &
$0.40 (0.20)$ & $0.46 (0.15)$ & $0.15 (0.35)$ & $0.65(0.11)$ & $37.22 (0.20)$ \\
 \cline{3-7}
 & $50$ & 
$0.48 (0.21)$ & $0.47 (0.08)$ & $0.22 (0.50)$ & $0.70 (0.10)$ & $38.22 (0.55)$ \\
 \cline{3-7}
 & $100$ & 
$0.76 (0.05)$ & $0.44 (0.13)$ & $0.26 (0.28)$ & $0.69 (0.30)$ & $39.09 (1.05)$\\
 \cline{1-7}
 \multirow{3}{*}{$16^2$} & $25$ &
$6.43(1.45)$ & $3.37 (1.09)$ & $5.38 (0.58)$ & $5.14 (1.99)$ & $495.47 (2.69)$ \\
 \cline{3-7}
 & $50$ & 
$9.12 (0.98)$ & $3.27 (1.81)$ & $4.62 (1.98)$ & $3.39 (2.00)$ & $516.64 (2.19)$ \\
 \cline{3-7}
 & $100$ & 
 $11.84 (2.01)$ & $4.85 (1.10)$ & $6.71 (0.72)$ & $5.67 (0.57)$ & $498.04 (4.01)$ \\
 \hline
\end{tabular}
\end{table}

%% file: Appendices/Appendix_D.tex
\section{Additional details of the Twitter latent Dirichlet allocation (T-LDA) algorithm}\label{supp:tlda}
The generation processes for a T-LDA and an LDA are illustrated side-by-side in Figure~\ref{fig:plate}. Here, the key differences exhibited in T-LDA are aggregation (pooling tweets from users) and regularization (restricting a tweet to be generated from only one topic). In our study, the aggregation is done by pooling tweets generated from the same day.

\begin{figure}[thb!]
\centering
    \begin{subfigure}{0.45\textwidth}
    \begin{tikzpicture}[x=1.5cm,y=1cm]
      \node[obs]                   (W)      {$w_{u,s,n}$}; %
      \node[latent, above=of W]    (Z)      {$z_{u,s}$}; %
      \node[latent, above=of Z]     (theta) {$\theta_u$};
      \node[latent, left=of W]    (beta)   {$\beta_k$};
    
      \node[const, above=of theta] (alpha)  {$\alpha$}; %
      \node[const, left=of beta] (eta) {$\eta$};
    
      \factor[left=of W]     {W-f}     {above:Multi} {} {} ; %
      \factor[above=of Z]     {Z-f}     {left:Multi} {} {} ; %
      \factor[above=of theta]   {theta-f}   {left:Dir} {} {} ; %
      \factor[left=of beta]   {beta-f}   {above:Dir} {} {} ; %
    
      \factoredge {Z} {W} {};
      \factoredge {alpha}  {theta-f} {theta}; %
      \factoredge {theta}  {Z-f} {Z}; %
      \factoredge {eta} {beta-f} {beta}; %
      \factoredge {beta}  {W-f} {W}; %
    
      \plate {plate1} {
        (W)
      } {$N$}
      \plate {plate2} { %
        (plate1)
        (W) %
        (Z)
      } {$S$}; %
      \plate {} { %
        (plate2) %
        (Z)%
        (theta)
      } {$U$} ; %
      \plate {} { %
        (beta) %
      } {$K$} ; %
    \end{tikzpicture}
    \end{subfigure}
    \begin{subfigure}{0.45\textwidth}
    \begin{tikzpicture}[x=1.5cm,y=1.3cm]
      \node[obs]                   (W)      {$w_{d,n}$}; %
      \node[latent, above=of W]    (Z)      {$z_{d,n}$}; %
      \node[latent, above=of Z]     (theta) {$\theta_d$};
      \node[latent, left=of W]    (beta)   {$\beta_k$};
      
      \node[const, above=of theta] (alpha)  {$\alpha$}; %
      \node[const, left=of beta] (eta) {$\eta$};
    
      \factor[left=of W]     {W-f}     {above:Multi} {} {} ; %
      \factor[above=of Z]     {Z-f}     {left:Multi} {} {} ; %
      \factor[above=of theta]   {theta-f}   {left:Dir} {} {} ; %
      \factor[left=of beta]   {beta-f}   {above:Dir} {} {} ; %

      \factoredge {Z} {W} {};
      \factoredge {alpha}  {theta-f} {theta}; %
      \factoredge {theta}  {Z-f} {Z}; %
      \factoredge {eta} {beta-f} {beta}; %
      \factoredge {beta}  {W-f} {W}; %
    
      \plate {plate1} { %
        (W) %
        (Z) %
      } {$N$}; %
      \plate {} { %
        (plate1) %
        (theta)%
      } {$D$} ; %
      \plate {} { %
        (beta) %
      } {$K$} ; %
    \end{tikzpicture} 
    \end{subfigure}
\caption{Plate notation comparison for the Twitter Latent Dirichlet Allocation (T-LDA) (left) and the standard Latent Dirichlet Allocation (LDA) (right) models. Here nodes are random variables; edges indicate dependence through probability distributions (e.g., Dirichlet or multinomial). Shaded nodes are observed; unshaded nodes are latent. Plates indicate replicated variables. Note that the T-LDA model aggregates tweets from each user into a document and constrains each tweet to be drawn from only one topic.}
\label{fig:plate}
\end{figure}
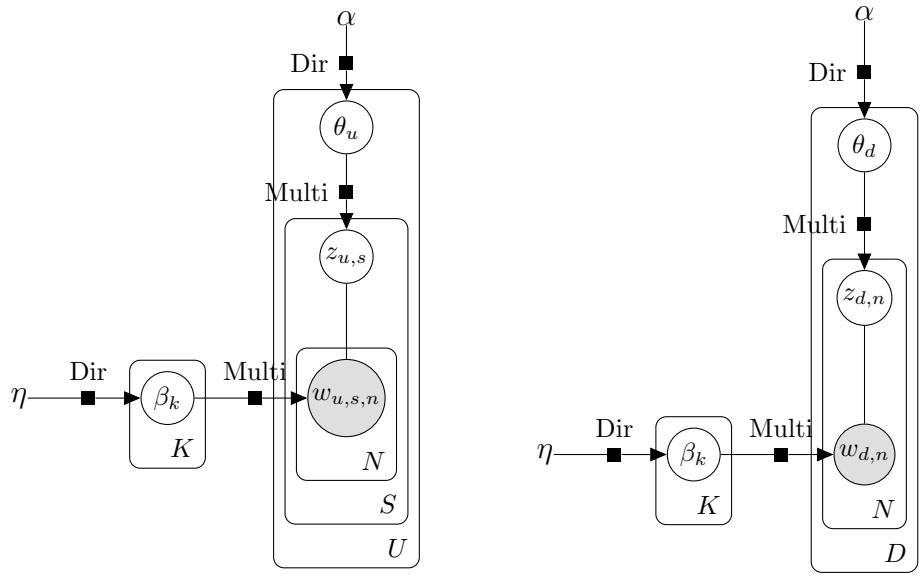

All numerical results presented in the article were produced with the following implementation details of the T-LDA algorithm: the collapsed Gibbs sampler has been run for $2000$ iterations with the first $1000$ samples discarded as burn-in. The latent variable $\beta$ is assumed to be symmetric Dirichlet with hyperparameter $\eta=0.01$ for all topics; and $\theta$ is assumed to be symmetric Dirichlet with hyperparameter $\alpha=0.5$ for all time stamps. Additionally, for weakly-supervised T-LDA implemented on the TalkLife data, additional weights are added to $\eta$ such that $\eta_n = 0.01 + w_n$ for each seed word $n$ and $\eta_n = 0.01$ otherwise, where the details of the weights can be found in Section~\ref{supp:talklife_data}.

\clearpage

\section{Summary statistics characterizing shortest paths}\label{supp:path-skips-summary}
To characterize the smoothness and continuity of the learned shortest paths of topics, we present summaries of the `skips' (days where there are no topics connected to either a topic immediately before or after the current timestamp) they made. Table~\ref{tab:path-skips-summary} depicts the number of skips and the length of the skips for four topic paths (see Appendix~\ref{supp:phate-dict} for details on the path names). We note that the length of a whole path (number of topics connected) could be different because 1) the different numbers of skips, and 2) the different time span as some topics appeared only for a certain time range (e.g., the wash hands topic). The lengths of those paths shown in the table are: COVID NEWS (presidential election), $70$; COVID (health care), $58$; STAY HOME (executive order), $59$; SANITIZING (wash hands) $19$. Clearly, longer paths could make longer skips. However, the paths remain fairly continuous (small numbers of short skips) during their time span. This is partly due to the corpora smoothing being applied--the topics learned at time $t$ should usually be very similar to those learned at nearby timestamps.

\begin{table}[!tbh]
\caption{Summary of the number of skips along with the length of those skips for four different topic paths. The paths are discovered by the shortest path algorithm using $10$-nearest neighbor weighted graph. Note that all paths exhibit small numbers of short-length skips.}
\label{tab:path-skips-summary}
\centering
\begin{tabular}{@{} *{5}{c} @{}}
\headercell{Path Name} & \multicolumn{4}{c@{}}{Days Skipped}\\
\cmidrule(l){2-5}
& 1 & 2 & 3 & 4 \\ 
\midrule
  COVID (health care)  & 10 & 0 &  1 &  1  \\
  COVID NEWS (presidential election)  & 3 & 1 & 2 & 2  \\
  SANITIZING (wash hands)  & 1 & 1 & 0 & 0  \\
  STAY HOME (executive order) & 4 & 0 & 0 & 0  \\
\end{tabular}
\end{table}

\clearpage

\section{Shortest path on MDS, ISOMAP, and PHATE}\label{supp:mds-isomap-phate-comparison}
We desire a low-dimensional embedding that preserves the trajectory structures of shortest paths, so that we can visualize and interpret any results computed using methods described in Section~\ref{sec:temporal}. Here we compare PHATE with MDS and ISOMAP. MDS does not take any local structural information into account when building the embedding; ISOMAP applies MDS using shortest path distances computed on neighborhood graphs; finally, PHATE applies MDS on potential distances computed on neighborhood graphs while striking a balance between local and global trajectory structures. Figure~\ref{fig:mds-isomap-phate} shows that MDS failed to identify any path between two points. ISOMAP identifies a cleaner structure but there are interrupting background points on the path. PHATE identifies a clean path that is also well separated from background points. The comparison also highlights the importance of working with neighborhood graphs, instead of the fully connected graph, when trying to identify local structures in data.

\begin{figure}[!tbh]
    \centering
    \begin{subfigure}{0.5\textwidth}
      \centering
      \includegraphics[width=\textwidth]{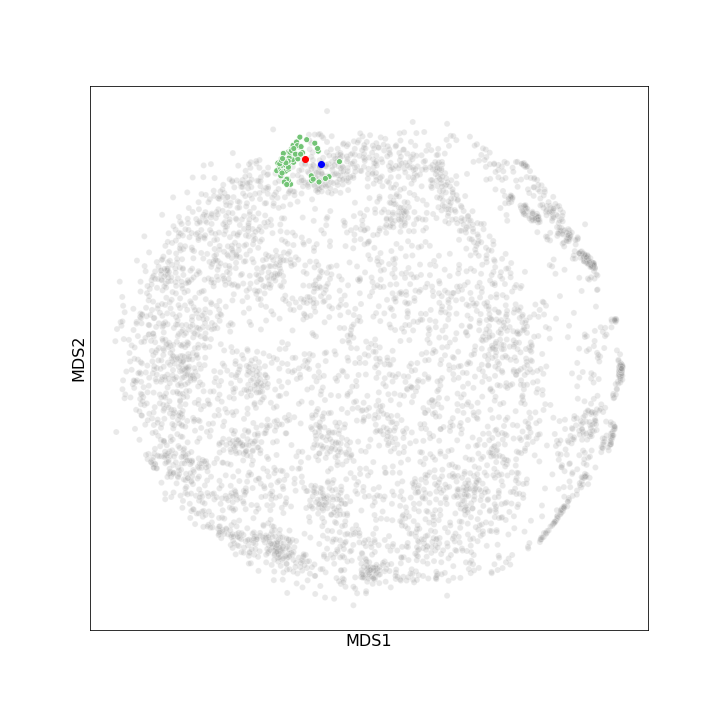}
    \end{subfigure}
    \hspace{-20pt}
    \begin{subfigure}{0.5\textwidth}
      \centering
      \includegraphics[width=\textwidth]{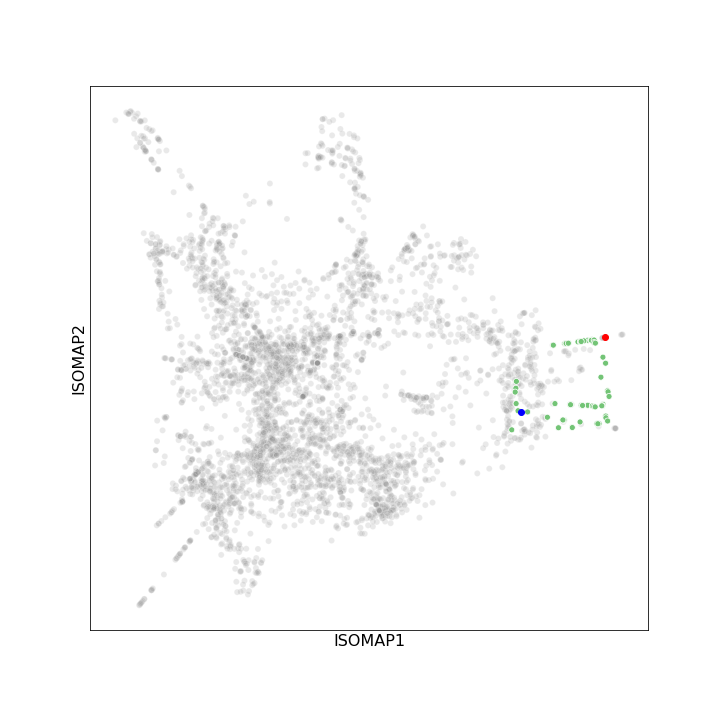}
    \end{subfigure}
    \vspace{-20pt}
    \begin{subfigure}{0.5\textwidth}
      \centering
      \includegraphics[width=\textwidth]{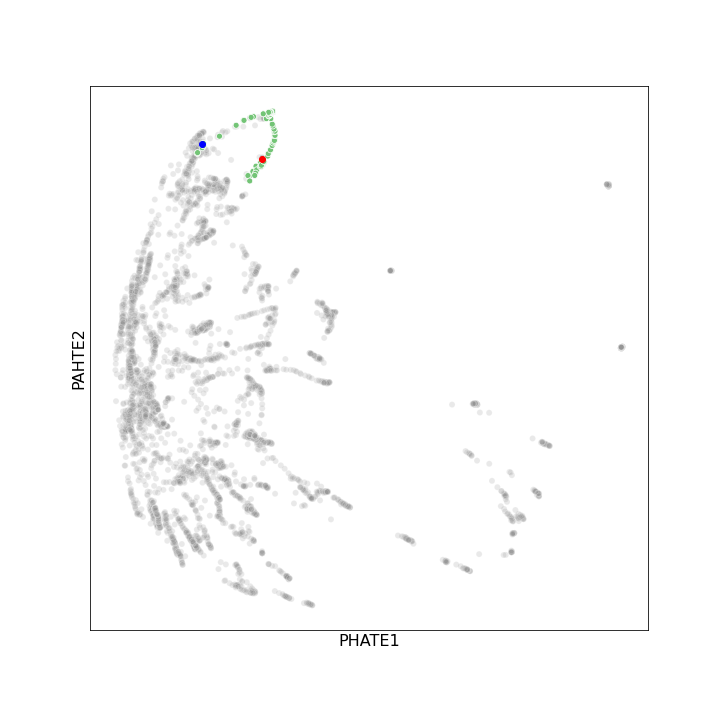}
    \end{subfigure}
    \caption{Multidimensional scaling (MDS), isometric feature mapping (ISOMAP), and potential of heat-diffusion for affinity-based transition embedding (PHATE) for the same set of word distributions. A shortest path computed on $10$ nearest neighbors graph is highlighted on each embedding with red and blue points indicating the starting and ending points of the path. Note that PHATE identifies the cleanest path connecting the red and blue points, with minimal background noises (grey points) included in between.}
    \label{fig:mds-isomap-phate}
\end{figure}

\clearpage

\section{Detailed descriptions of PHATE}\label{supp:phate-details}
Algorithm~\ref{alg:phate} outlines the steps for obtaining a low-dimensional embedding using PHATE with the Hellinger distance metric.

\begin{algorithm}
\begin{minipage}{0.9\linewidth}
\begin{algorithmic}[1]
\caption{PHATE with Hellinger distance}
\label{alg:phate}
\REQUIRE $N$ observations of some objects
    \STATE Compute pairwise Hellinger distance matrix (denoted as $D$) from all pairs of multinomial topic distributions (stored as columns in a matrix $X$).
    \STATE Compute $k$-nearest neighbor distance (denoted as $\epsilon_k(x)$) from each column of $X$.
    \STATE Compute local affinity matrix $K_{k,\alpha}$ from $D$ and $\epsilon_k$.
    \STATE Form a diffusion operator $P$, which is a Markov transition matrix computed by normalizing $K_{k,\alpha}$.
    \STATE Compute time scale via Von Neumann Entropy. The time scale is then used to diffuse $P$ to obtain $P^t$.
    \STATE Compute potential representation of the diffusion matrix as $U_t=-\log(P^t)$ and compute potential distance matrix $D_{U,t}$ from $U_t$.
    \STATE Apply MDS on $D_{U,t}$ to embed the data in lower dimension.
\ENSURE An $N \times L$ matrix that contains $L$-dimensional coordinates for each observation.
\end{algorithmic}
\end{minipage}
\end{algorithm}

In Algorithm~\ref{alg:phate} we use the Hellinger distance to compute $D$ and $\epsilon_k$. This ensures that PHATE is being used to perform dimension reduction on a statistical manifold~\citep{amari2012differential}.

The PHATE construction is based on computing local similarities between data points, and then diffusing through the data using a Markovian random-walk diffusion process to infer more global relations. The local similarities between points are computed by first computing pairwise distances and then transforming the distances into similarities, via a kernel named the $\alpha$-decaying kernel with locally adaptive bandwidth. It is defined as
\begin{equation}
    K_{k,\alpha}(x,y) = \frac{1}{2}\exp\Big(-\Big(\frac{\|x-y\|}{\epsilon_k(x)}\Big)^\alpha\Big) + \frac{1}{2}\exp\Big(-\Big(\frac{\|x-y\|}{\epsilon_k(y)}\Big)^\alpha\Big).
\end{equation}
Here the $k$-nearest neighbor distance $\epsilon_k$ is used to ensure that the bandwidth is locally adaptive and varies based on the local density of the data. The exponent $\alpha$ controls the rate of decay of the tails in the kernel $K_{k,\alpha}$. Setting $\alpha=2$ is equivalent to the use of a Gaussian kernel and choosing $\alpha>2$ results in lighter tails in the kernel. The kernel is then normalized by row-sums that results in a row-stochastic matrix $P=P_{k,\alpha}$ (the diffusion operator), which is used for following steps.

In Step 5, the diffusion operator is powered by a time scale $t$. In particular, for a data point $x$ and diffusion operator $P$, and let $\delta_x$ be the Dirac delta that is defined to be a row vector of length $N$ (length of the data) with a one at entry corresponding to $x$ and zero elsewhere. The $t$-step distribution of $x$ is the row in $P^t$ corresponding to $x$:
\begin{equation}
    p^t_x := \delta_xP^t = [P^t]_{(x,\cdot)}.
\end{equation}
This distribution captures multiscale (where $t$ serves as the scale) local neighborhoods of data points, where the local neighborhoods are explored by randomly walking or diffusing over the intrinsic manifold geometry of the data.  The scale parameter $t$ affects the embedding. It can be selected based on any prior knowledge of the data or, as proposed in~\citet{moon2019visualizing}, by quantifying the information in the powered diffusion operator with different values of $t$, via computing the Von Neumann Entropy~\citep{neumann2013mathematische,anand2011shannon} of the diffusion affinity, and choosing the one that explains the maximum amount of variability in the data.

Finally, a new type of distance, called the potential distance in~\citet{moon2019visualizing}, is recovered in the end from the powered diffusion operator, which is obtained by taking the negative log of the transition probabilities. This transforms these transition probabilities into the heat-potential context.

\clearpage

\section{Additional simulation studies for PHATE}\label{supp:phate-simulaiton}
To illustrate the idea of probability vectors on a sphere, in Figure~\ref{fig:hellinger-sphere-demo-3d} we present a simple example of a sphere in 3D and probability vectors (simulated as in Section~\ref{sec:phate}) lying on the sphere. The trajectories in this simulated example exhibit different progressive structures. In particular, the trajectory in dark blue evolves smoothly and remains roughly on the same path; the trajectory in brown exhibits a sharp turn in the direction at a certain position; finally, the trajectory in light blue behaves more chaotically and exhibits clustering structures. The PHATE embedding presented in Figure~\ref{fig:phate-hellinger-sphere} of Section~\ref{sec:phate} was able to uncover all these types of structures in low dimension.

\begin{figure}[!tbh]
    \centering
    \includegraphics[width=0.8\textwidth]{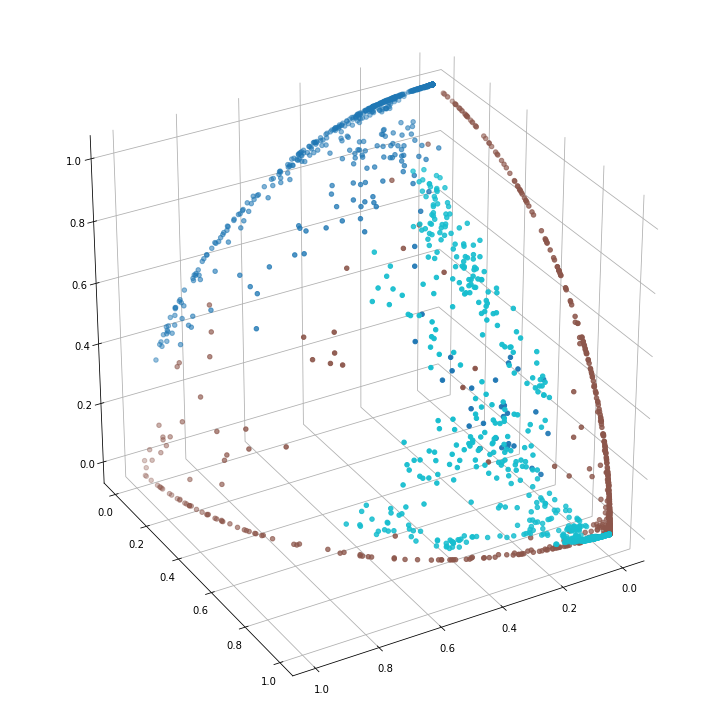}
    \caption{Three simulated trajectories of probability vectors on a sphere. Each color signifies a trajectory simulated using a specific $\sigma$ in the random-walk structure described in Section~\ref{sec:phate}. Here, three trajectories started at the same point exhibit different progressive structures: stable (dark blue), chaotic and clustering (light blue), and sharp transition (brown).}
    \label{fig:hellinger-sphere-demo-3d}
\end{figure}

To further demonstrate the advantage of PHATE over traditional methods for uncovering progressive structures, we present a similar example to that in~\citet{moon2019visualizing}, which uses artificial tree-structured data and compare principle component analysis (PCA), t-distributed stochastic neighbor embedding (t-SNE) and PHATE in constructing low-dimensional embedding. In particular, we generate tree-structured data with $10$ branches and $200$ dimensions, and each branch has length $300$. Thus, we have $3000$ observations of $200$-dimensional data, and the goal is to find a $2$-dimensional embedding for visualization. Figure~\ref{fig:pca-tsne-phate} shows the results of embedding for three different methods. PCA is good for finding an optimal linear transformation that gives the major axes of variation in the data. However, the underlying data structure in this case is nonlinear in which case PCA is not ideal. t-SNE is able to embed nonlinear data; however, it is optimized for cluster structure and as a result will destroy any continuous progression structure in the data. PHATE for this example separates the clusters and is able to clearly represent the trajectory structure of the data. Additionally, PHATE neatly captures the branching/splitting points of different trajectories. This feature is vital for our study of tweeting behaviors as we are interested in learning how different conversations converge to a similar one or diverge to different topics. 

\begin{figure}[!tbh]
    \centering
    \includegraphics[width=0.9\textwidth]{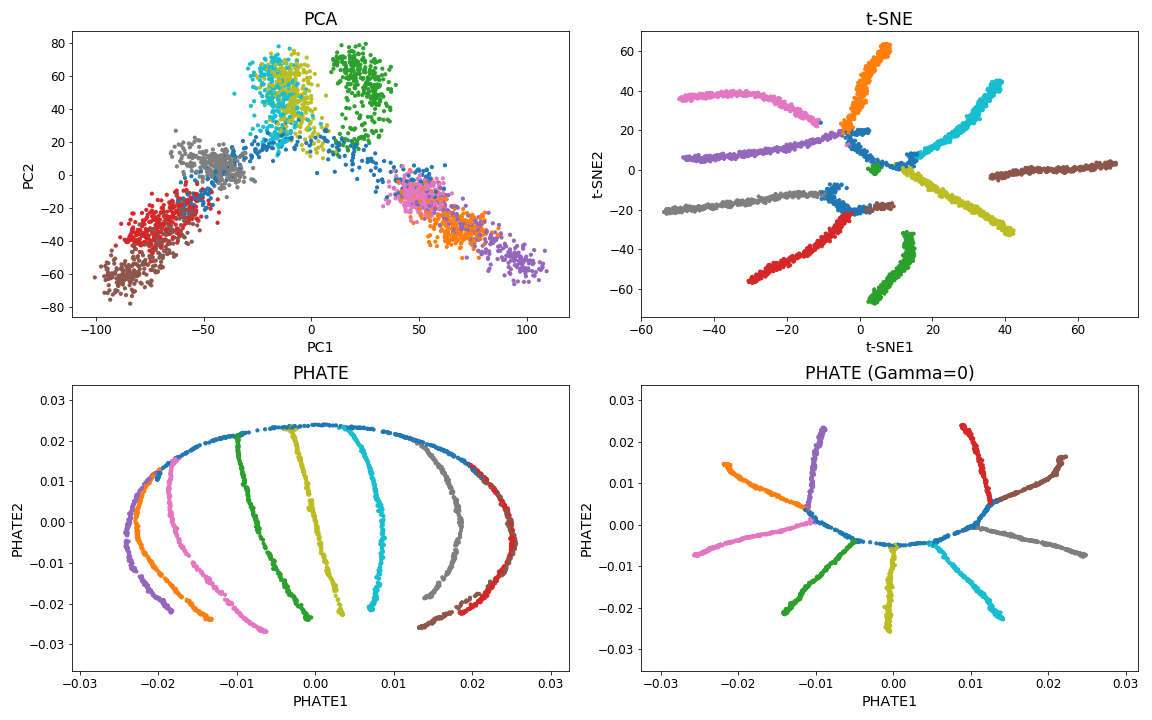}
    \caption{Comparison of principle component analysis (PCA), t-distributed stochastic neighbor embedding (t-SNE), and potential of heat-diffusion for affinity-based transition embedding (PHATE). Two versions of PHATE with different tuning parameters are illustrated. The data are $3000$ tree-structured observations with $10$ branches. Various branches are colored differently. Note that for this truly trajectory-based data, PHATE gives the clearest low-dimensional representation of the data.}
    \label{fig:pca-tsne-phate}
\end{figure}

Additionally, we also demonstrate that PHATE does not `create' spurious trajectories, although it does not preclude the existence of such structures. Here, $3000$ independent data points were simulated from a $3$-component (with weights $0.6,0.3,0.1$) $10$-dimensional Gaussian mixture model and transformed through softmax (i.e., $z_j \rightarrow \frac{\exp(z_j)}{\sum_{i=1}^{10} \exp(z_i)}, j=1,\dots,10$). Figure~\ref{fig:phate-random-no-trajectory} depicts $2$-dimensional embedding computed by PCA, t-SNE, uniform manifold approximation and projection (UMAP), and PHATE using Hellinger distance. Clearly, PHATE did not artificially `trajectorize' the data; t-SNE seems to perform the best in terms of clustering as it often tries to separate data as much as possible; UMAP separated the clusters well but generated artificial segments and trajectories in the embedding. 

\begin{figure}[!tbh]
    \centering
    \includegraphics[width=0.9\textwidth]{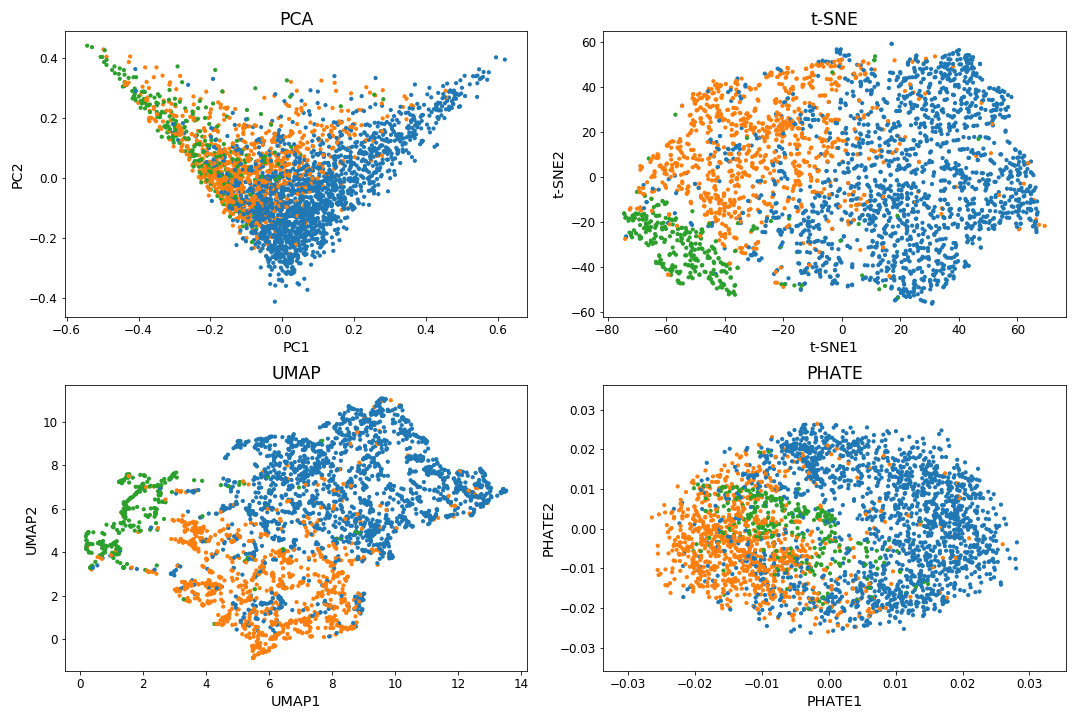}
    \caption{Comparison of principle component analysis (PCA), t-distributed stochastic neighbor embedding (t-SNE), uniform manifold approximation and projection (UMAP), and potential of heat-diffusion for affinity-based transition embedding (PHATE). Here $3,000$ independent data points were generated from a $3$-component (with weights $0.6,0.3,0.1$) $10$-dimensional Gaussian mixture model. Here, data were transformed via softmax to resemble a probability vector. Note that for this random nonstructured data, PHATE did not `create' spurious trajectories in the low-dimensional embedding.}
    \label{fig:phate-random-no-trajectory}
\end{figure}

Lastly, we compare PHATE (and other) embeddings using different distance metrics. In particular, we compute $2$-dimensional embeddings for the data generated in Section~\ref{sec:phate} using Euclidean and cosine distances/similarities. Figure~\ref{fig:phate-random-no-trajectory} depicts the results comparing PCA, t-SNE, UMAP, and PHATE. It shows that the Hellinger metric (for t-SNE, UMAP, and PHATE) outperforms the other two in terms of generating the clearest low-dimensional embedding that preserves the true data geometry. 

\begin{figure}[!tbh]
    \centering
    \begin{subfigure}{0.7\textwidth}
        \centering
        \includegraphics[width=\textwidth]{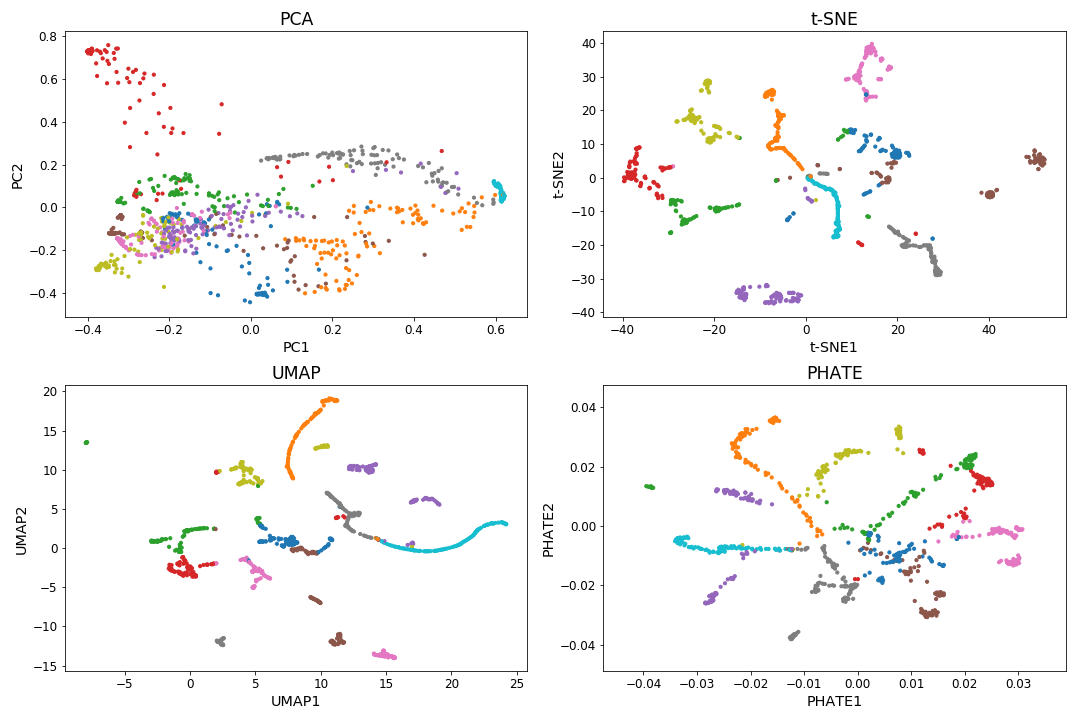}
        \caption{Embedding using Euclidean metric.}
    \end{subfigure}
    \begin{subfigure}{0.7\textwidth}
        \centering
        \includegraphics[width=\textwidth]{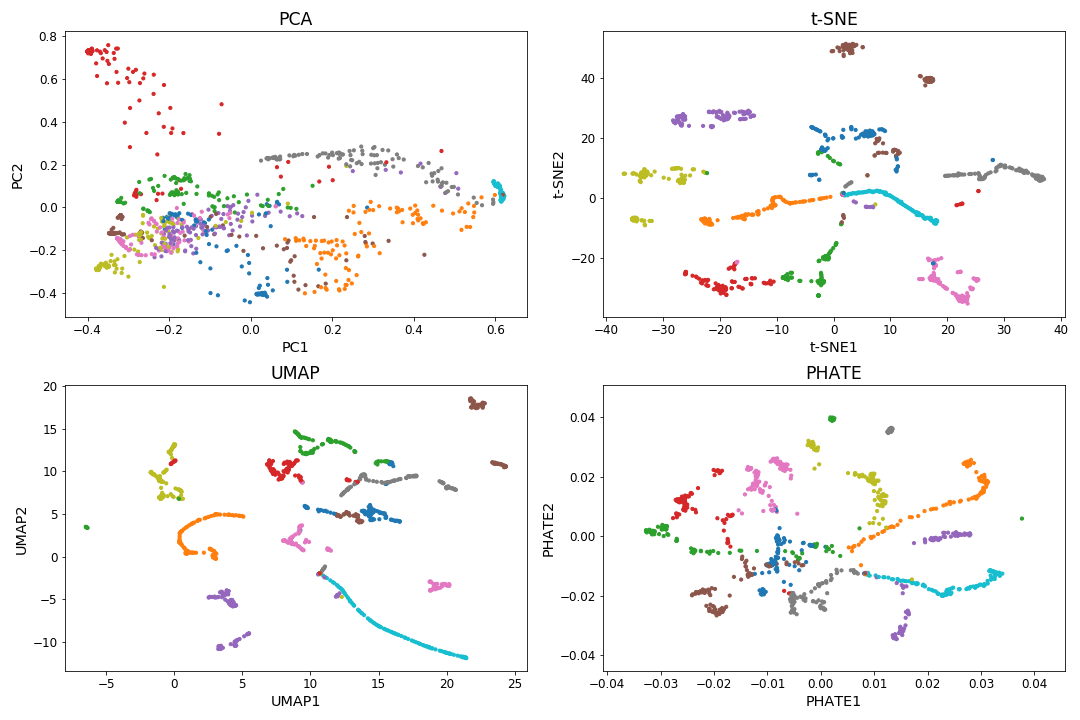}
        \caption{Embedding using Cosine metric.}
    \end{subfigure}
    \caption{Principle component analysis (PCA), t-distributed stochastic neighbor embedding (t-SNE), uniform manifold approximation and projection (UMAP), and potential of heat-diffusion for affinity-based transition embedding (PHATE) using Euclidean and cosine metrics. Here $10$ trajectories of $100$-dimensional probability vectors are generated, where the trajectories are colored differently. PHATE gives the clearest 2D representation of the inputs that preserves their high-dimensional progressive structures, regardless of the distance metric used. Comparing with Figure~\ref{fig:phate-hellinger-sphere}, the Hellinger metric outperforms the other two metrics in recovering the data geometry.}
    \label{phate-non-hellinger-demo}
\end{figure}

\clearpage

\section{Comparison with TopicFlow for topic trend mining}\label{supp:topicflow}
TopicFlow~\citep{malik2013topicflow} is an analysis framework for Twitter data over adjacent time slices, binned topic models, and alignment, which is an application of LDA to timestamped documents at independent time intervals and alignment of the resulting topics. The key differences between TopicFlow and the proposed framework are: 1) a different similarity measure between topics, that is, cosine similarity metric for TopicFlow; 2) a different mechanism for topic alignment and connection--TopicFlow connects every pair of adjacent topics that has similarity above a certain threshold. The advantages of Hellinger metric over other metrics for comparing/embedding word distributions have been made clear in the previous section. Here, we demonstrate the advantages of the proposed shortest path mechanism over TopicFlow for obtaining natural temporal evolution of topics.

We analyze a particular topic cluster--the presidential election cluster discussed in Section~\ref{sec:gdtm-results-twitter}--and compare the connections computed by the proposed shortest path algorithm and the TopicFlow algorithm. Here, for a fair and direct comparison, we fix the bins and the topic detection algorithms to be the same for both frameworks--using the smoothed temporal corpus and the T-LDA; the shortest path is performed on a $10$-nearest neighbor weighted graph and the TopicFlow is performed with a connection threshold of $0.2$. For the latter, we obtain a path by localizing the connection that has the largest cosine similarity at each pair of adjacent timestamp. For illustration, in Table~\ref{tab:path-comparison}, we highlight a time segment that exhibits differences between two paths. In particular, the shortest path skipped 3 days, March 23 to March 26, while the TopicFlow remain continuously connected. The top row of Figure~\ref{fig:path-comparison} depicts the top word clouds of topics at timestamps March 23, 24, 27, and May 15 on the TopicFlow path. It shows a sharp transition from a voting/election topic to general political topics and finally to a relatively nonpolitical topic. On the other hand, the shortest path automatically skipped the timestamps where these new topics emerged and maintained the major theme of the path, which is voting/election and later on general politics. This offers a more natural and much smoother transition. 

This comparison demonstrates particularly that the mechanism for topic trend discovery used by TopicFlow is restrictive as it potentially results in nonsmooth and nonintuitive transitions. Although one could tune the connection threshold, it increases the computational burden and there is no obvious objective (e.g., prediction score, loss, etc.) that could help with the tuning process.

\begin{table}[!tbh]
\caption{A portion of connected presidential election topics via the shortest path mechanism (left column) and the TopicFlow mechanism (right column). Here topics are indicated by their indices, e.g., $0-49$, at each timestamp (row index). \textit{NA} indicates that no connection has been made by the algorithm.}
\label{tab:path-comparison}
\[\begin{array}{c|cc}
& \textbf{SP topic index} & \textbf{TF topic index}\\
\hline
\textbf{Feb 15} & 37 & 37  \\
\textbf{\vdots} & \vdots & \vdots  \\
\textbf{Mar 23} & 1 & 1  \\
\textbf{Mar 24} & NA & 44  \\
\textbf{Mar 25} & NA & 27  \\
\textbf{Mar 26} & NA & 26  \\
\textbf{Mar 27} & 22 & 26  \\
\textbf{\vdots} & \vdots & \vdots  \\
\textbf{May 15} & 2 & 15  \\
\end{array}\]
\end{table}

\begin{figure}[!tbh]
\centering
\begin{subfigure}{.2\textwidth}
  \centering
  \includegraphics[width=\linewidth]{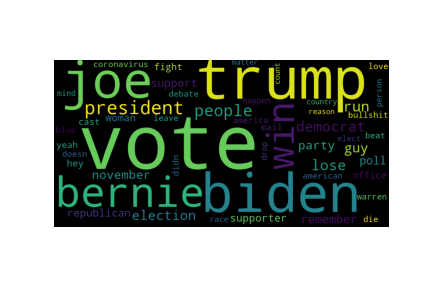}
\end{subfigure}%
{$\rightarrow$}%
\begin{subfigure}{.2\textwidth}
  \centering
  \includegraphics[width=\linewidth]{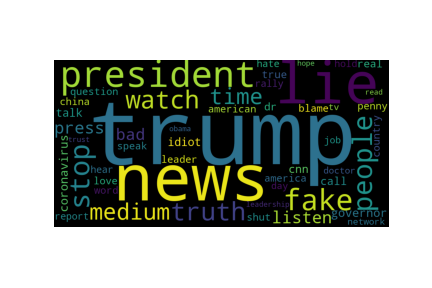}
\end{subfigure}
{$\rightarrow$}%
\begin{subfigure}{.2\textwidth}
  \centering
  \includegraphics[width=\linewidth]{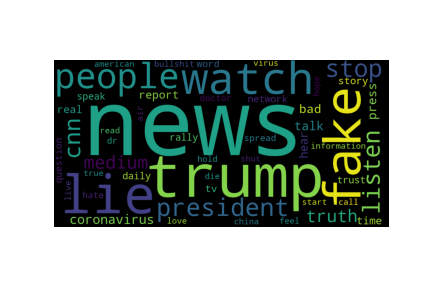}
\end{subfigure}
{$\rightarrow$}%
\begin{subfigure}{.2\textwidth}
  \centering
  \includegraphics[width=\linewidth]{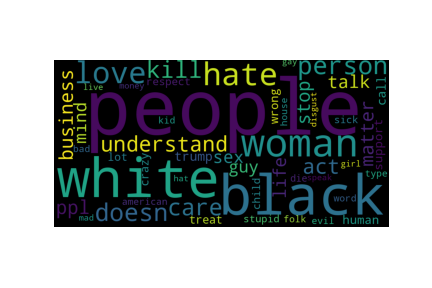}
\end{subfigure}
\\
\begin{subfigure}{.2\textwidth}
  \centering
  \includegraphics[width=\linewidth]{path_election_idx_38.png}
\end{subfigure}%
{$\rightarrow$}%
\begin{subfigure}{.2\textwidth}
  \centering
  \includegraphics[width=\linewidth]{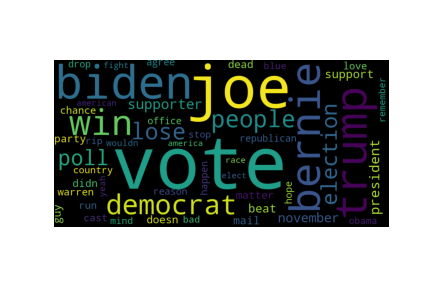}
\end{subfigure}
{$\rightarrow$}%
\begin{subfigure}{.2\textwidth}
  \centering
  \includegraphics[width=\linewidth]{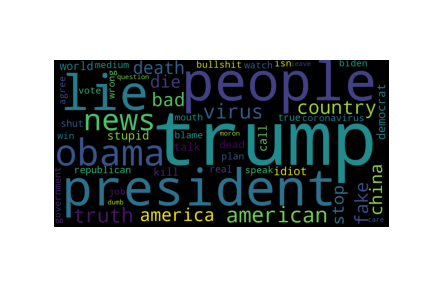}
\end{subfigure}
\caption{Top word clouds showing evolution of topics on the presidential election topic paths computed via the shortest path algorithm (bottom) and the TopicFlow (top) algorithm. The sample timestamps at which the topics are learned are March 23, 24, 27, and May 15 (top); March 23, 27, and May 15 (bottom). Note that the shortest path algorithm produces much smoother and more intuitive transitions among topics within a general theme.}
\label{fig:path-comparison}
\end{figure}

To further investigate the two different topic aligning methods, we fix the distance metric to be Hellinger, and compare the shortest path mechanism and the TopicFlow mechanism for the same set of topics. Table~\ref{tab:path-comparison-hellinger} depicts a similar pattern for the time range March 23 to 27, where the restrictive TopicFlow mechanism for topic connection exhibits a sharp transition as shown in Figure~\ref{fig:path-comparison}. Similar to Table~\ref{tab:path-comparison}, from February to March 23, the two paths are mostly the same. However, we observe that the two paths also exhibit similar topics near the end of the time period. This again demonstrates the superiority of Hellinger distance for measuring topic similarity.

\begin{table}[!tbh]
\caption{A portion of connected presidential election topics via the shortest path mechanism (left column) and the TopicFlow mechanism (right column) using the same distance metric (Hellinger). Here topics are indicated by their indices, e.g., $0-49$, at each timestamp timestamps (row index). \textit{NA} indicates that no connection has been made by the algorithm. Note that the restriction imposed by TopicFlow impacts the topic path similar (from March 23 to 27) to that in Table~\ref{tab:path-comparison}}
\label{tab:path-comparison-hellinger}
\[\begin{array}{c|cc}
& \textbf{SP topic index} & \textbf{TF topic index}\\
\hline
\textbf{Feb 15} & 37 & 37  \\
\textbf{\vdots} & \vdots & \vdots  \\
\textbf{Mar 23} & 1 & 1  \\
\textbf{Mar 24} & NA & 44  \\
\textbf{Mar 25} & NA & 27  \\
\textbf{Mar 26} & NA & 26  \\
\textbf{Mar 27} & 22 & 26  \\
\textbf{\vdots} & \vdots & \vdots  \\
\textbf{May 12} & 8 & 8  \\
\textbf{May 13} & 42 & 42  \\
\textbf{May 14} & 0 & 0  \\
\textbf{May 15} & 2 & 2  \\
\end{array}\]
\end{table}

\clearpage

\section{Volume plots of raw Twitter Decahose data}\label{supp:raw_volume}
Figure~\ref{fig:volume_geo} shows the Decahose Twitter volume plots before (top) and after (bottom) processing. Although Twitter officially claims the percentage of geotagged tweets to be around $1$-$2\%$ of the total tweets (\url{https://developer.twitter.com/en/docs/tutorials/Tweet-geo-metadata}), we found the percentage to much smaller. Note that there are several time points where the data is either incomplete (i.e., low volumes) or missing (i.e., $0$ volumes). 

\begin{figure}[!tbh]
    \centering
    \begin{subfigure}{0.8\textwidth}
    \centering
    \includegraphics[width=\textwidth]{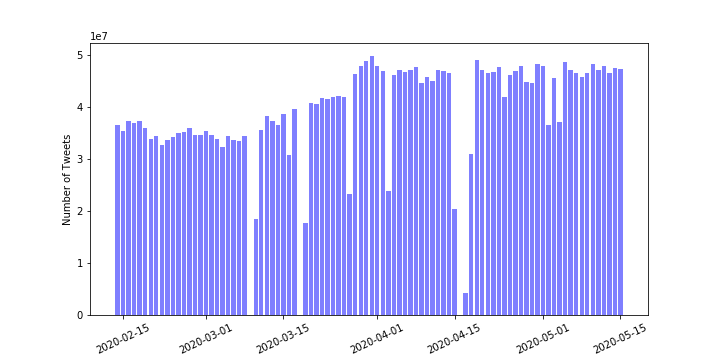}
    \caption{Raw Decahose tweets volume (on a scale of $10^7$ tweets) from Feb 15 to May 15.}
    \label{fig:volume_all}
    \end{subfigure}
    \begin{subfigure}{0.8\textwidth}
    \centering
    \includegraphics[width=\textwidth]{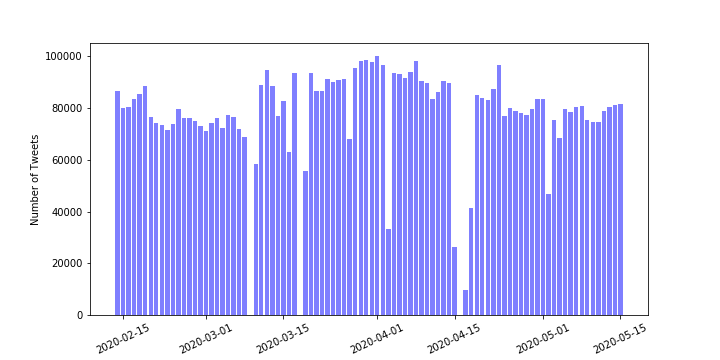}
    \caption{Geotagged U.S., non-retweet, English Decahose tweets volume from Feb 15 to May 15.}
    \label{fig:volume_geo}
    \end{subfigure}
    \caption{Volume of all and geotagged Decahose tweets for each day during the study period. The Decahose stream generates around $30-50$ million raw tweets and $50-100$ thousand geotagged English language tweets per day, except for several missing/incomplete cases with $0$ or abnormally small volumes.}
\end{figure}

\clearpage

\section{Sensitivity analysis for hyperparameters}\label{supp:sensitivity}
In this section, we perform sensitivity analyses for the hyperparameters $k$ and $\gamma$ in Algorithm~\ref{alg:procedure}, namely the number of neighbors in the nearest neighbor graph and the smoothing parameter for constructing new corpora. Further, we perform model selection for varying choices of $K$, the number of topics in T-LDA.

In Figure~\ref{fig:covid-short-path-neighbor}, the two shortest paths computed using neighborhood graphs of $k=8$ and $12$ are illustrated. For comparison, the same starting and ending topics as well as the two intermediate topics at the same time points as those in Figure~\ref{fig:covid-short-path} are used. It is clear from the word clouds that the shortest paths are not sensitive to the choice of $k$ in the neighborhood of $10$.

\begin{figure}[!tbh]
\centering
\begin{subfigure}{.2\textwidth}
  \centering
  \includegraphics[width=\linewidth]{covid_short_path_b.png}
\end{subfigure}%
{$\rightarrow$}%
\begin{subfigure}{.2\textwidth}
  \centering
  \includegraphics[width=\linewidth]{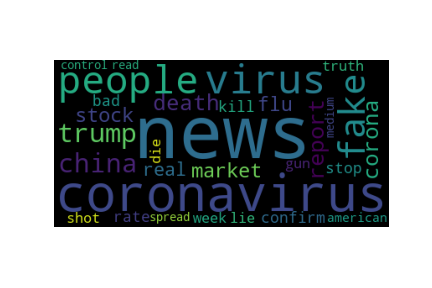}
\end{subfigure}
{$\rightarrow$}%
\begin{subfigure}{.2\textwidth}
  \centering
  \includegraphics[width=\linewidth]{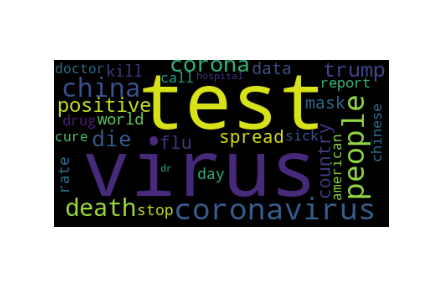}
\end{subfigure}
{$\rightarrow$}%
\begin{subfigure}{.2\textwidth}
  \centering
  \includegraphics[width=\linewidth]{covid_short_path_e.png}
\end{subfigure}
\\
\begin{subfigure}{.2\textwidth}
  \centering
  \includegraphics[width=\linewidth]{covid_short_path_b.png}
\end{subfigure}%
{$\rightarrow$}%
\begin{subfigure}{.2\textwidth}
  \centering
  \includegraphics[width=\linewidth]{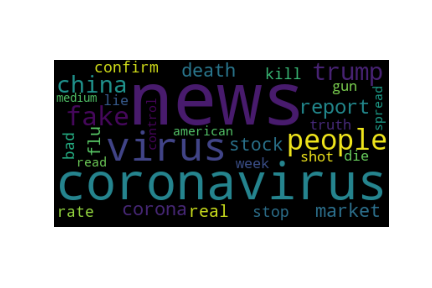}
\end{subfigure}
{$\rightarrow$}%
\begin{subfigure}{.2\textwidth}
  \centering
  \includegraphics[width=\linewidth]{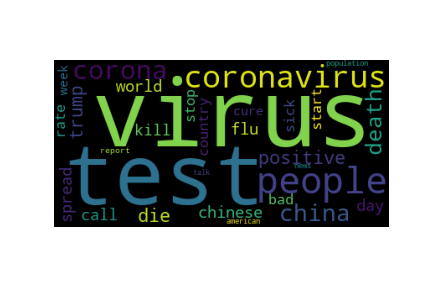}
\end{subfigure}
{$\rightarrow$}%
\begin{subfigure}{.2\textwidth}
  \centering
  \includegraphics[width=\linewidth]{covid_short_path_e.png}
\end{subfigure}
\caption{Evolution along the shortest paths of a COVID-19 topic on the first day to a COVID-19 health care focused topic on the last day illustrated as top word clouds. The paths are computed on a $8$- (top) and a $12$- (bottom) nearest neighbor graph. The middle two word clouds are illustrations of two of the topics on the paths at the same time points as those in Figure~\ref{fig:covid-short-path}. Note that the intermediate topics in both cases represent natural transformations from the beginning to the end topics, confirming that the shortest path is not sensitive to small perturbations of $k$ around $10$.}
\label{fig:covid-short-path-neighbor}
\end{figure}

Additionally, we quantify the similarities between any two shortest paths computed on different neighborhood graphs by computing the average Hellinger distance between topics (at the same time point) on the paths. Particularly, in Table \ref{tab:hdist-sensitivity-neighbor} we show the average Hellinger distances. For this particular cluster of topics, the average Hellinger distances are negligible and are stable across all pairs of different paths, which suggests that the shortest path is not sensitive to different $k$ in the neighborhood of $k=10$. 

\begin{table}[!tbh]
\caption{Average Hellinger distances between any two topics paths generated using various neighborhood parameters $k$ as the column/row indices. Examples are shown for the COVID (health care) topics. Note that the average Hellinger distances are identically $0$ across all pairs of paths, indicating that the shortest paths are stable under different choices of $k$.}
\label{tab:hdist-sensitivity-neighbor}
\[\begin{array}{c|ccc}
 & \textbf{8} & \textbf{10} & \textbf{12}\\
\hline
\textbf{8} & 0 & 0 & 0  \\
\textbf{10} & 0 & 0 & 0  \\
\textbf{12} & 0 & 0 & 0  \\
\end{array}\]
\end{table}

Figure~\ref{fig:volume_subsampled} shows the contributions (in terms of the number of tweets) from each document to the temporally smoothed corpus constructed for March $31$, using smoothing parameters of $0.65, 0.75, 0.85$. With $0.75$, the contents span the whole study period (Feb $15$ to May $15$) but concentrate on tweets within a month, centered at March $31$.

\begin{figure}[!thb]
    \centering
    \begin{subfigure}{0.8\textwidth}
        \centering
        \includegraphics[width=\textwidth]{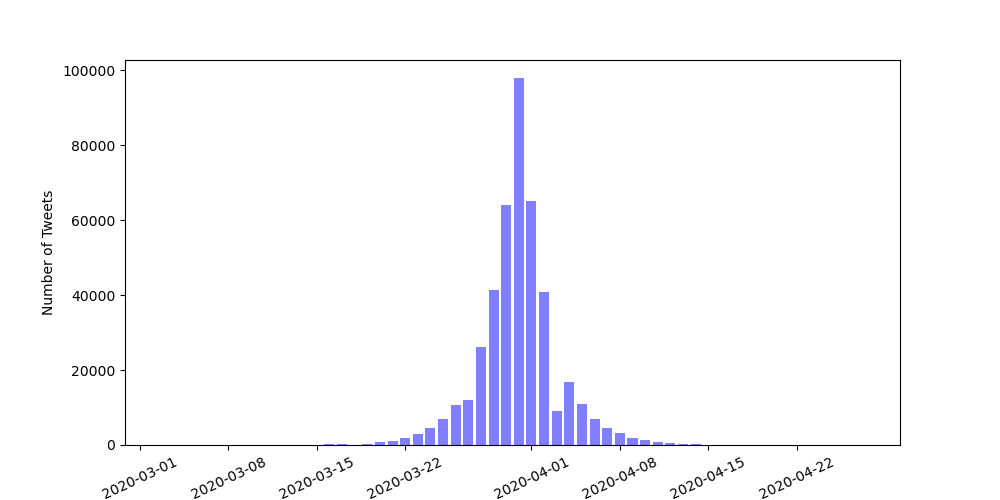}
    \end{subfigure}
    \begin{subfigure}{0.8\textwidth}
        \centering
        \includegraphics[width=\textwidth]{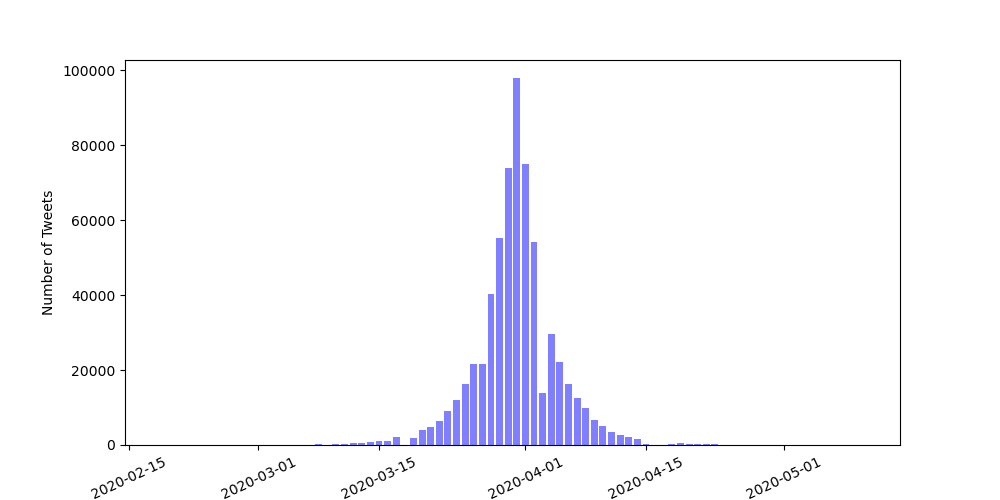}
    \end{subfigure}
    \begin{subfigure}{0.8\textwidth}
        \centering
        \includegraphics[width=\textwidth]{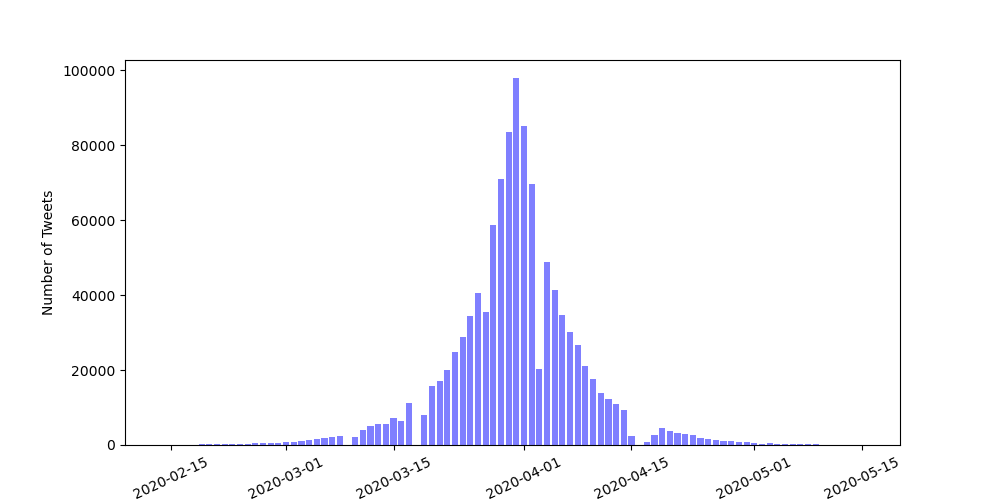}
    \end{subfigure}
    \caption{Contributions of tweet volume from various time points for temporally smoothed corpora. The examples are constructed for March $31$, using smoothing parameters $0.65, 0.75, 0.85$ (from top to bottom). Although the plots exhibit different resolutions and spans of the histograms, the shapes of the contribution distributions are similar in all cases. This illustrates robustness of the proposed method to the choice of smoothing parameters.}
    \label{fig:volume_subsampled}
\end{figure}

Moreover, Figure~\ref{fig:phate-overall-2d-0.65-0.85} shows the PHATE embedding of all topics learned by T-LDA, using corpus constructed with smoothing parameters $0.65$ and $0.85$. Here we highlight two clusters (COVID and COVID NEWS) and one shortest path (presidential election) similar to Figure~\ref{fig:phate-overall-2d}. Comparing the three PHATE plots, the overall structures are similar and the highlighted trajectories remain relatively stable (i.e., presidential election paths exhibit similar `U' shapes in all cases). Note that the `split-and-merge' behaviors within the COVID NEWS cluster are being captured in all cases as well. The only notable difference in the PHATE produced with different temporal smoothing is the length of the trajectories, with those in the embedding produced using smoothing parameter $0.85$ being the longest. This is reasonable because a larger smoothing parameter assumes a longer range temporal dependence structure of the data.

\begin{figure}[!tbh]
    \centering
    \includegraphics[width=0.5\textwidth]{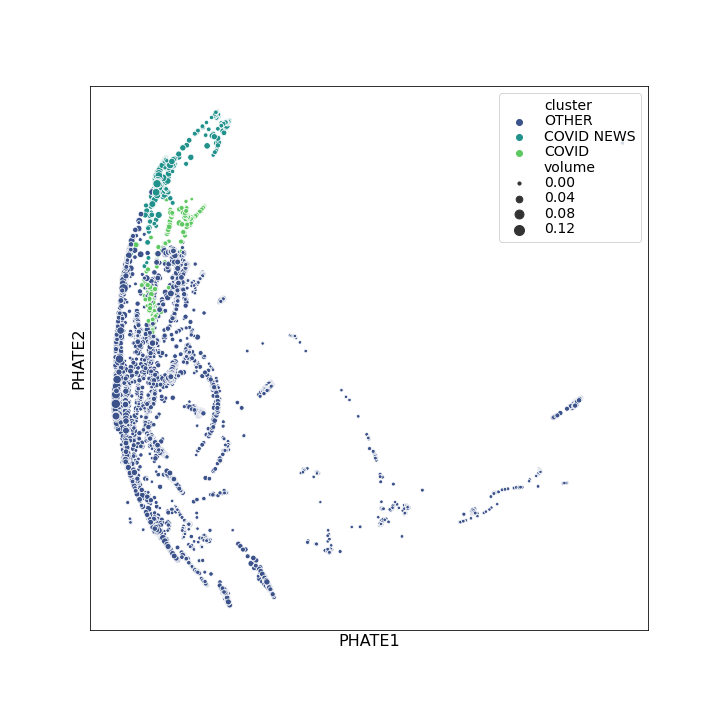}
    \includegraphics[width=0.5\textwidth]{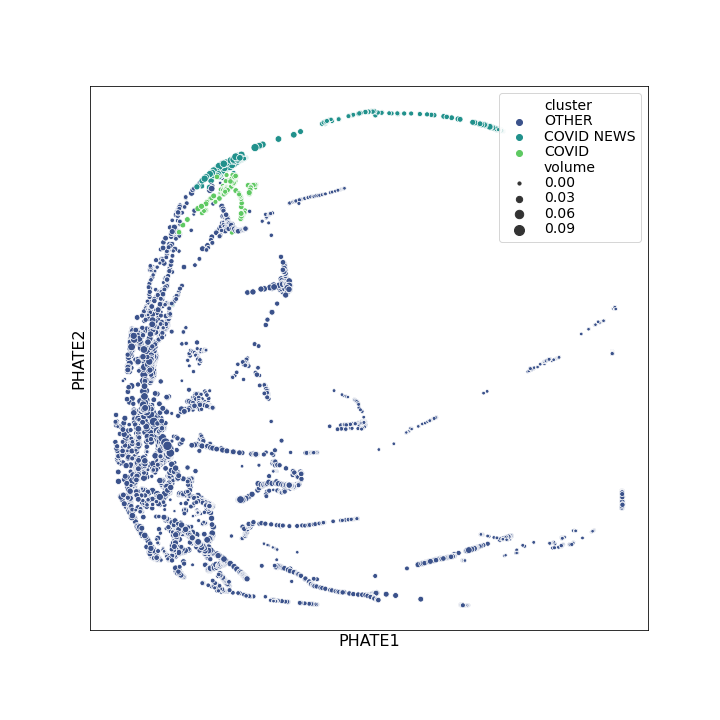}
    \caption{Potential of heat-diffusion for affinity-based transition embedding (PHATE) for all word distributions. The topics here are learned by T-LDA on tweet collections constructed with smoothing parameters $0.65$ (top) and $0.85$ (bottom). Here two clusters and one shortest path are highlighted for comparison with Figure~\ref{fig:phate-overall-2d}. Note that the overall structures as well as the trajectories for highlighted points are similar in all three cases, while the lengths of the trajectories are different, which are the result of different assumptions on the range of the temporal dependence (i.e., a smoothing using $0.85$ assumes longer range dependence by including more old tweets).}
    \label{fig:phate-overall-2d-0.65-0.85}
\end{figure}

Additionally, we quantify the similarities between any two shortest paths from different smoothed corpora by computing the average Hellinger distance of the topics on the paths. Particularly, in Table \ref{tab:hdist-sensitivity} we show the average Hellinger distances between any two paths computed under different smoothing conditions, for the COVID NEWS (presidential election) and COVID (health care) topics. In these two cases, the average Hellinger distances are around $0.35$ and are stable across all pairs, which suggests that the shortest paths of key topics of interest are not sensitive to different smoothing parameters. 

\begin{table}[!tbh]
\caption{Average Hellinger distances between any two topics paths generated from corpora with various smoothing parameters as the column/row indices. Examples are shown for the COVID NEWS (presidential election) and the COVID (health care) topics in the top and bottom tables, respectively. Note that the average Hellinger distances are both relatively small and stable in the sense that all pairwise distances are similar in magnitude, indicating that the shortest paths are stable under different choices of smoothing parameters.}
\label{tab:hdist-sensitivity}
\[\begin{array}{c|ccc}
 & \textbf{0.65} & \textbf{0.75} & \textbf{0.85}\\
\hline
\textbf{0.65} & 0 & 0.3520 & 0.3578  \\
\textbf{0.75} & 0.3520 & 0 & 0.3056  \\
\textbf{0.85} & 0.3578 & 0.3056 & 0  \\
\end{array}\]
\vspace{\baselineskip}
\[\begin{array}{c|ccc}
 & \textbf{0.65} & \textbf{0.75} & \textbf{0.85}\\
\hline
\textbf{0.65} & 0 & 0.3697& 0.4112  \\
\textbf{0.75} & 0.3697 & 0 & 0.3652  \\
\textbf{0.85} & 0.4112 & 0.3652 & 0  \\
\end{array}\]
\end{table}

Lastly, for the choice of the number of topics for T-LDA, we propose to compute a Bayesian Information Criteria (BIC) score at each timestamp defined as
\begin{equation*}
    -\text{log-likelihood} + \frac{C\log(D)}{2}
\end{equation*}
where the model complexity is computed by $C := Kp + (K-1)D$ with $p$ denoting the length of the vocabulary. The log-likelihood of the T-LDA model is defined as
\begin{equation*}
\begin{aligned}
     & \prod_{k=1}^K\text{Dirichlet}(\beta_k;\eta) \times \prod_{d=1}^D \text{Dirichlet}(\theta_d;\alpha) \\
     & \times \prod_{s=1}^{S_d}\text{Categorical}(z_{s,d};\theta_d) \times \prod_{n=1}^{N_s}\text{Categorical}(w_{n,s,d};\beta_{z_{s,d}}).   
\end{aligned}
\end{equation*}
Here, the categorical distribution is a special case of the multinomial distribution, in that it gives the probabilities of potential outcomes of a single drawing rather than multiple drawings; $S_d$ denotes the number of tweets in document $d$ ; and $N_s$ denotes the number of words in a tweet $s$. This criteria is similar to the topic model selection criteria proposed in \citet{taddy2012estimation}.

In Figure~\ref{fig:topic-num-selection-bic}, we show the computed scores across all timestamps for various choices of the numbers of topics. The model with $K=50$ consistently produces the lowest scores for the first half of the time range and is comparable to the model with $K=100$ for the second half.

\begin{figure}[!tbh]
    \centering
    \includegraphics[width=0.9\textwidth]{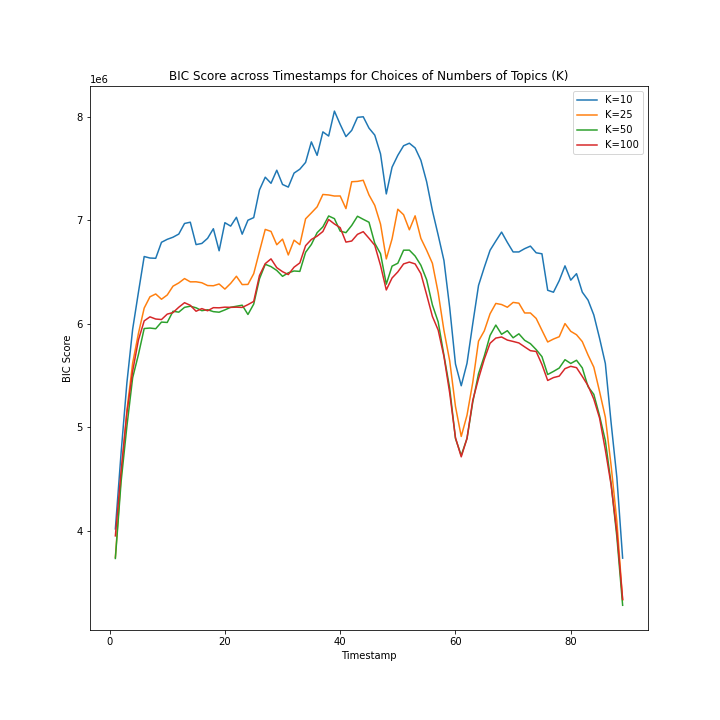}
    \caption{Bayesian information criteria (BIC) scores across timestamps for different choices of the numbers of topics.}
    \label{fig:topic-num-selection-bic}
\end{figure}

\clearpage

\section{PHATE dictionary of clusters and trajectories}\label{supp:phate-dict}
We explain the labeling of the PHATE plots for visualization and interpretation:
\begin{itemize}
    \item Colors signify clusters of topics. Clusters are computed by a hierarchical clustering algorithm using Hellinger distance between topics. Only selected COVID-19 topics are colored differently, and all others are grouped into a single color. Selected COVID-19 topics are: 
    \begin{itemize}
        \item COVID: topics where the top words are mostly general COVID-19 terms such as coronavirus, virus, covid, etc.
        \item COVID NEWS: topics where the top words are related to government officials or politicians discussing COVID-19 related issues. Typical top words include: Trump, government, news, covid, etc.
        \item SANITIZING: topics where the top words are mostly wash hands, sanitizing, virus, etc.
        \item STAY HOME: topics where the top words are mostly stay home, safe, covid, etc.
    \end{itemize}
    \item Sizes represent normalized number of tweets that is generated from each topic. 
    \item Shapes highlight selected COVID-19 related shortest paths computed on the neighborhood graph. Different shapes represent
    \begin{itemize}
        \item COVID (health care): a subset of topics in the COVID topic cluster that are all on a shortest path starting from a general COVID topic at the first time point and finishing at a health care focused COVID topic (e.g., testing, death).
        \item COVID (politics): a shortest path that starts from a topic that is second-closest in distance to the starting topic of the COVID (health care) and finishing at a politics focused COVID topic (e.g., president, news, etc.).
        \item COVID NEWS (presidential election): a subset of topics in the COVID NEWS cluster that are all closely related to presidential election and are on a shortest path starting from a election-related topic at the first time point.
        \item SANITIZING (wash hands): a subset of topics in the SANITIZING cluster that are on a shortest path starting from a topic related to washing hands due to COVID-19.
        \item STAY HOME (executive order): a subset of topics in the STAY HOME cluster that are on a shortest path starting from a topic related to stay home executive order due to COVID-19.
        \item General: topics that are not on selected shortest paths of interest.
    \end{itemize}
\end{itemize}

\clearpage

\section{Additional PHATE trajectories}\label{supp:linear-trajectories}
Figure~\ref{fig:linear-trajectories} shows two linear trajectories, the SANITIZING (wash hands) and the STAY HOME (executive order), on the PHATE embedding. In contrast to nonlinear trajectories presented in Figure~\ref{fig:phate-election-2d} and Figure~\ref{fig:phate-covid}, topics on linear trajectories exhibit no obvious deviation in terms of the top words.

\begin{figure}[!tbh]
    \centering
    \includegraphics[width=0.55\textwidth]{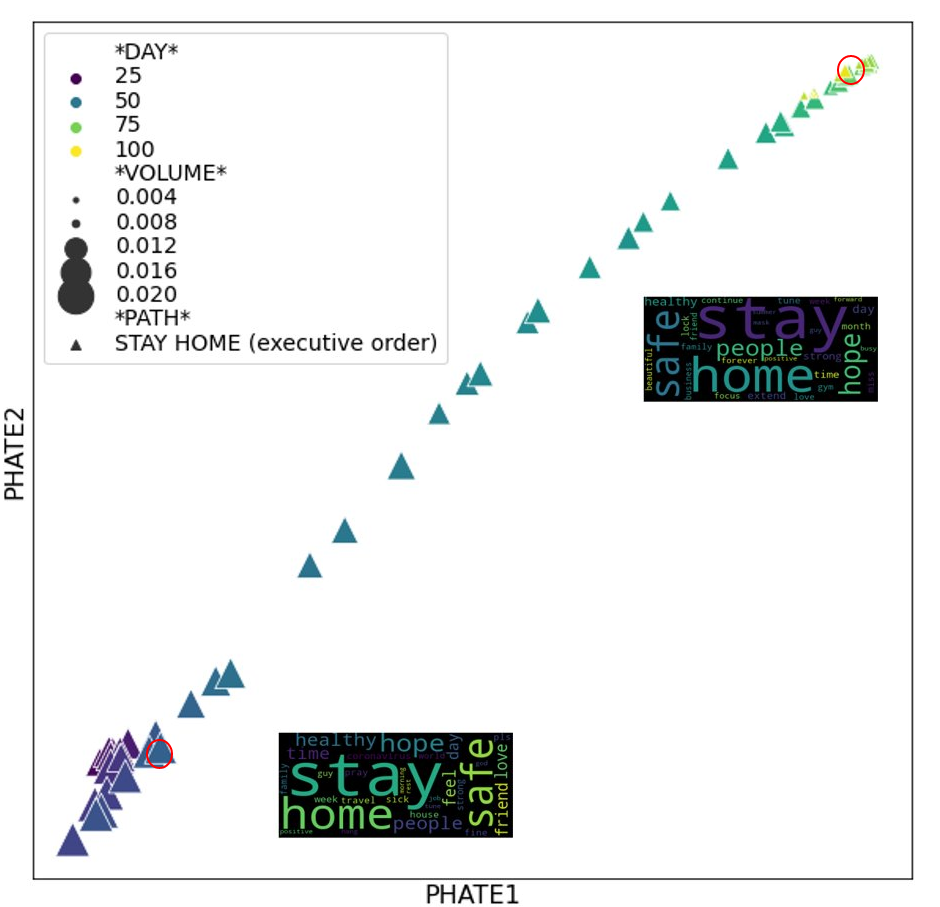}
    \includegraphics[width=0.55\textwidth]{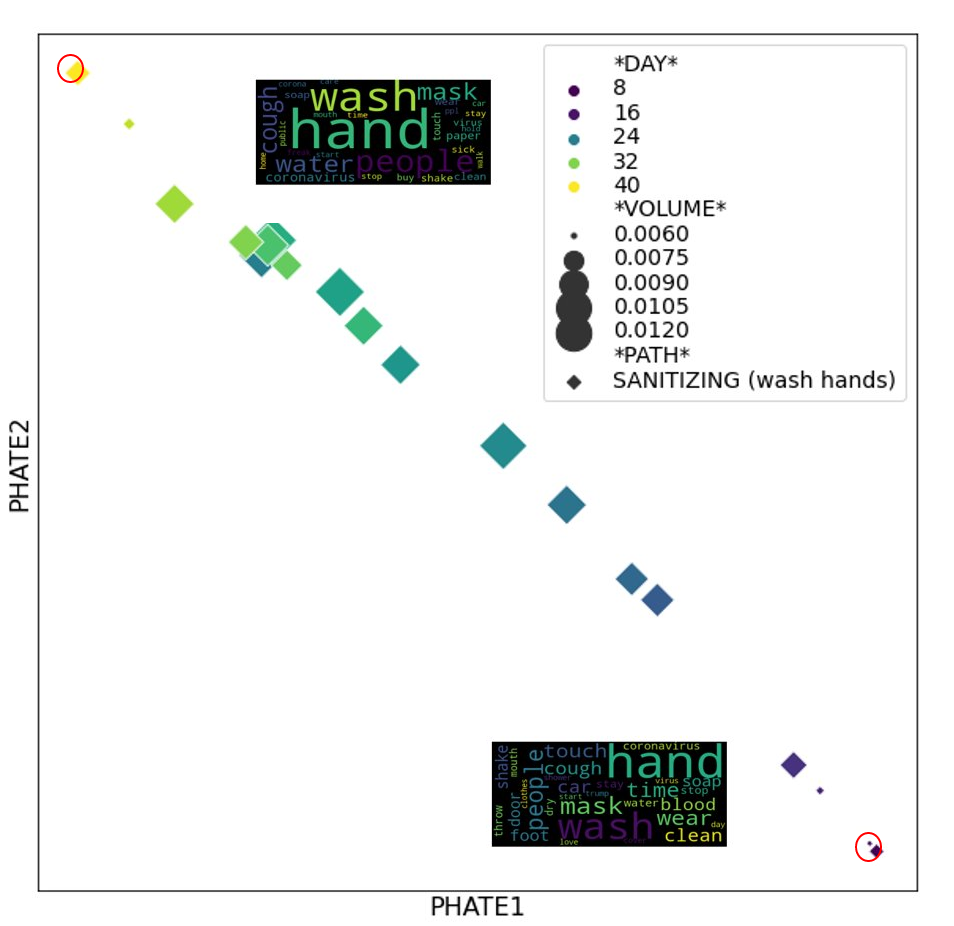}
    \caption{Potential of heat-diffusion for affinity-based transition embedding (PHATE) for subsets of topics lie on the executive order path (top) and the wash hands path (bottom). Colors and sizes of points highlight time and tweet volume, respectively. Here two word clouds containing top $30$ words in corresponding topics are shown for the time points highlighted by red circles in each path. Note that in both cases, the topic near the beginning of the study period is similar to that near the end of the study period. This shows the stability of topics on linear trajectories.}
    \label{fig:linear-trajectories}
\end{figure}

\clearpage

\section{State-level trend in Tweet proportions}\label{supp:state-level-spatial}
Here, we illustrate state-level variations in estimated tweet volumes generated by topics on the presidential election path, normalized by total tweet volume at each time point. From Figure~\ref{fig:states_election_path}, we see that although tweet proportions vary state by state, the overall trend is clear with peaks roughly correspond to the time points of key event highlighted. 

\begin{figure}[bh]
    \centering
    \includegraphics[width=0.8\textwidth]{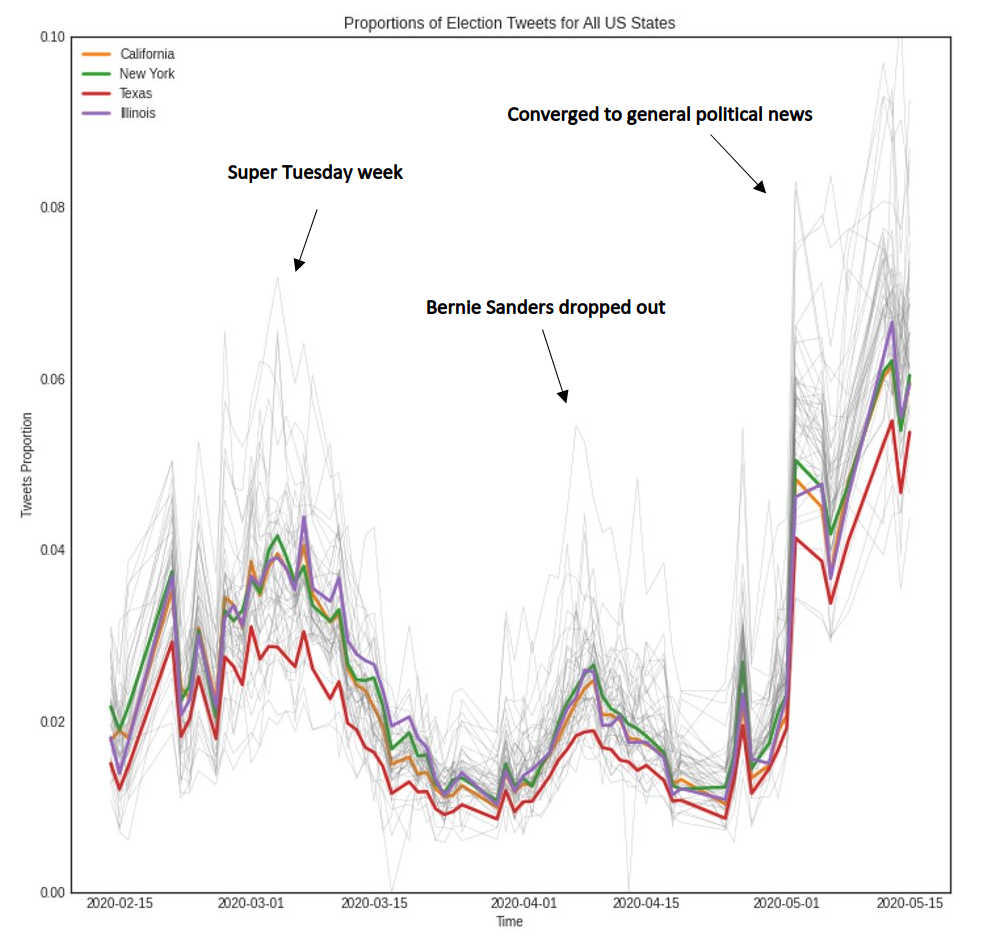}
    \caption{State-level spatial distribution of tweet proportions generated from all topics on the COVID NEWS (presidential election) path. California, New York, Texas, and Illinois are highlighted for illustration, while all other states are ll plotted in grey. Note that similar three events (annotated using texts) as in Figure~\ref{fig:phate-election-2d} correspond roughly to the three peaks in the time-course plot, indicating validations of the quality of the shortest path using real-world events.}
    \label{fig:states_election_path}
\end{figure}

For the COVID (health care) topic path, at the state level, tweet proportions follow global trends at the beginning of the study period in February and March but become chaotic starting in April. One possible explanation is that the COVID-19 pandemic in the United States started in several hot spots but quickly spread into other states, which then started to implement state-specific control measures. In addition, the overall new cases and death toll in the country reached a few record highs in April, starting with New York, which became an epicenter of the pandemic, with a record $12274$ new cases reported on April 4 (\url{https://en.wikipedia.org/wiki/COVID-19_pandemic_in_New_York_(state)}). This explains the difference in tweet proportions trend in New York, compared with the other three highlighted states.

\begin{figure}[!tbh]
    \centering
    \includegraphics[width=0.8\textwidth]{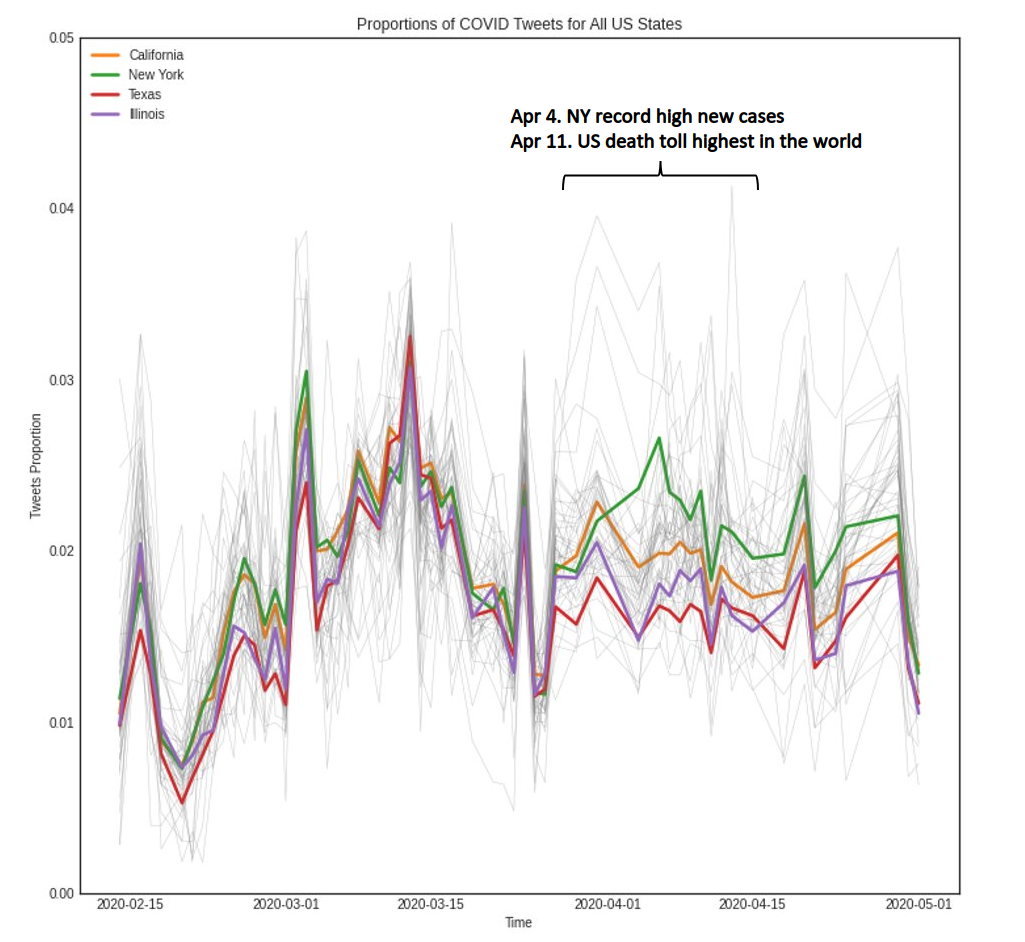}
    \caption{State-level spatial distribution of Tweet proportions generated from all topics on the COVID (health care) path. California, New York, Texas, and Illinois are highlighted for illustration, while all other states are plotted in grey. Note that a time period in April is annotated with relevant events explaining the surge in tweet proportions in many states. This validates the quality of this shortest path using real-world events.}
    \label{fig:states_covid_path}
\end{figure}

\clearpage

\section{Additional details of TalkLife data}\label{supp:talklife_data}

\begin{table}[!tbh]
\caption{Label names and corresponding percentage volume in all posts generated in 2019. Note the label ``Other'' indicates a post is not labeled by any other labels.}
\label{tab:label_volume}
\centering
\begin{tabularx}{\textwidth}{Xl}
\cline{1-2}
\textbf{Label }                                 & \textbf{Percentage volume}    \\
\cline{1-2}
Other                                  & 0.5840   \\
DepressedMoodSuspected                 & 0.0799  \\
AgitationOrIrritationSuspected         & 0.07421   \\
LonelinessSuspected                    & 0.0682  \\
FamilyIssuesSuspected                  & 0.0674  \\
BehavorialSymptomsSuspected            & 0.0618 \\
DistortedThinkingSuspected             & 0.0607  \\
AnxietyPanicFearSuspected              & 0.0566 \\
SelfHarmSuspectedTakeTwo               & 0.0506  \\
BodyImageEatingDisordersSuspected      & 0.0495  \\
SuicidalIdeationAndBehaviorSuspected   & 0.0433  \\
NumbnessEmptinessSuspected             & 0.0371  \\
NssiIdeationAndBehaviorSuspected       & 0.0327   \\
SelfHarmRelapseSuspected               & 0.0325  \\
TiredFatiguedLowEnergySuspected        & 0.0257 \\
MentalHealthTreatmentSuspected         & 0.0241  \\
CryingSuspected                        & 0.0229 \\
DeathOfOtherSuspected                  & 0.0201  \\
AlcoholAndSubstanceAbuseSuspected      & 0.0191   \\
HelplessnessHopelessnessSuspected      & 0.0163  \\
SelfHarmRemissionSuspected             & 0.0126   \\
EmotionalExhaustionSuspected           & 0.0110 \\
FinalTiredFatiguedLowEnergySuspected   & 0.0105 \\
FailureSuspected                       & 0.0094  \\
SongLyricsSuspected                    & 0.0078  \\
SuicidalPlanningSuspected              & 0.0078 \\
InpatientOutPatientMedicationSuspected & 0.0059  \\
EmptinessSuspected                     & 0.0054 \\
NumbnessSuspected                      & 0.0050 \\
SuicideAttemptSuspected                & 0.0017 \\
NssiUrgeSuspected                      & 0.0017\\
NauseaWithEatingDisorderSuspected      & 0.0015  \\
NauseaSuspected                        & 0.0012  \\
SelfHarmRemissionOrRelapseSuspected    & 0.0011 \\
\bottomrule
\end{tabularx}%
\end{table}

\begin{table}[!tbh]
\caption{Seed words and the associated weights that are used in the weakly-supervised LDA algorithm. Weights are computed as natural log of the volume (number of occurrences) of the corresponding word in the entire year of 2019, multiplied by a tune-able constant (equals $10$ here).}
\label{tab:seed_words}
\centering
\begin{tabularx}{\textwidth}{XXXX}
\cline{1-4}
\textbf{Word}       & \textbf{Weight}             & \textbf{Word}        & \textbf{Weight}              \\
\cline{1-4}
afraid     & 101.72 & listen     & 105.46 \\
anxiety    & 106.67 & live       & 113.41 \\
anymore    & 112.89 & lonely     & 108.59 \\
attempt    & 86.54  & long       & 111.18  \\
band       & 80.62 & lose       & 112.61 \\
body       & 104.47 & lyric      & 79.44 \\
clean      & 93.92  & medication & 86.72  \\
cry        & 112.81  & mental     & 101.90 \\
cut        & 104.78  & mom        & 109.04  \\
dad        & 103.89 & month      & 108.48 \\
dead       & 102.16 & numb       & 93.79  \\
death      & 98.71  & pain       & 110.16  \\
depression & 105.84 & parent     & 105.64 \\
die        & 114.56 & plan       & 97.73  \\
drink      & 99.31  & play       & 104.71 \\
drug       & 93.76 & pretty     & 103.28 \\
drunk      & 92.33  & relapse    & 81.82  \\
eat        & 107.72 & sad        & 111.77 \\
emptiness  & 77.17  & scar       & 106.90  \\
empty      & 95.84  & school     & 109.36 \\
end        & 113.06 & sick       & 103.66 \\
energy     & 93.67  & smoke      & 94.36  \\
exhaust    & 91.71    & song       & 102.18 \\
eye        & 104.10 & stop       & 114.58 \\
fail       & 96.32  & suicide    & 99.23  \\
failure    & 90.22 & tear       & 97.94 \\
family     & 110.45 & throw      & 97.05 \\
favorite   & 96.82  & tire       & 108.81 \\
feeling    & 112.08  & tired      & 101.14 \\
food       & 99.55  & ugly       & 100.59  \\
harm       & 96.07  & urge       & 86.28 \\
health     & 97.70 & vomit      & 71.62 \\
heart      & 112.15 & weak       & 92.34 \\
hospital   & 91.96  & week       & 107.86 \\
hurt       & 115.45 & worry      & 101.11 \\
kill       & 108.86 & write      & 101.68  \\
leave      & 115.83  &            &      \\ 
\bottomrule
\end{tabularx}%
\end{table}

\clearpage

\section{Additional details of clustering of TalkLife labels}\label{supp:label_cluster}

\begin{table}[!tbh]
\caption{Clustered labels.}
\label{tab:clustered_labels}
\centering
\resizebox{0.95\textwidth}{!}{%
\begin{tabularx}{\textwidth}{lX}
\cline{1-2}
\textbf{Cluster} & \textbf{Labels}                                                                                                        \\ \cline{1-2} 
1       & BehavorialSymptomsSuspected, CryingSuspected, DepressedMoodSuspected                                          \\
2       & InpatientOutPatientMedicationSuspected, MentalHealthTreatmentSuspected             \\
3       & EmptinessSuspected, NumbnessSuspected                                                                         \\
4 &
  AgitationOrIrritationSuspected, AlcoholAndSubstanceAbuseSuspected,
  AnxietyPanicFearSuspected,
  BodyImageEatingDisordersSuspected, DeathOfOtherSuspected, FamilyIssuesSuspected, NssiIdeationAndBehaviorSuspected, SelfHarmRelapseSuspected, SelfHarmRemissionSuspected, SelfHarmSuspectedTakeTwo, SuicidalIdeationAndBehaviorSuspected \\
5       & FinalTiredFatiguedLowEnergySuspected, TiredFatiguedLowEnergySuspected                                         \\
6       & NauseaSuspected, NauseaWithEatingDisorderSuspected                                                            \\
7       & SuicidalPlanningSuspected, SuicideAttemptSuspected                                                            \\
8       & DistortedThinkingSuspected, EmotionalExhaustionSuspected, FailureSuspected, HelplessnessHopelessnessSuspected \\
9       & NssiUrgeSuspected, SelfHarmRemissionOrRelapseSuspected, SongLyricsSuspected                                   \\
10      & LonelinessSuspected, NumbnessEmptinessSuspected             \\
\bottomrule
\end{tabularx}%
}
\end{table}

\begin{figure}[!tbh]
    \centering
    \includegraphics[width=\textwidth]{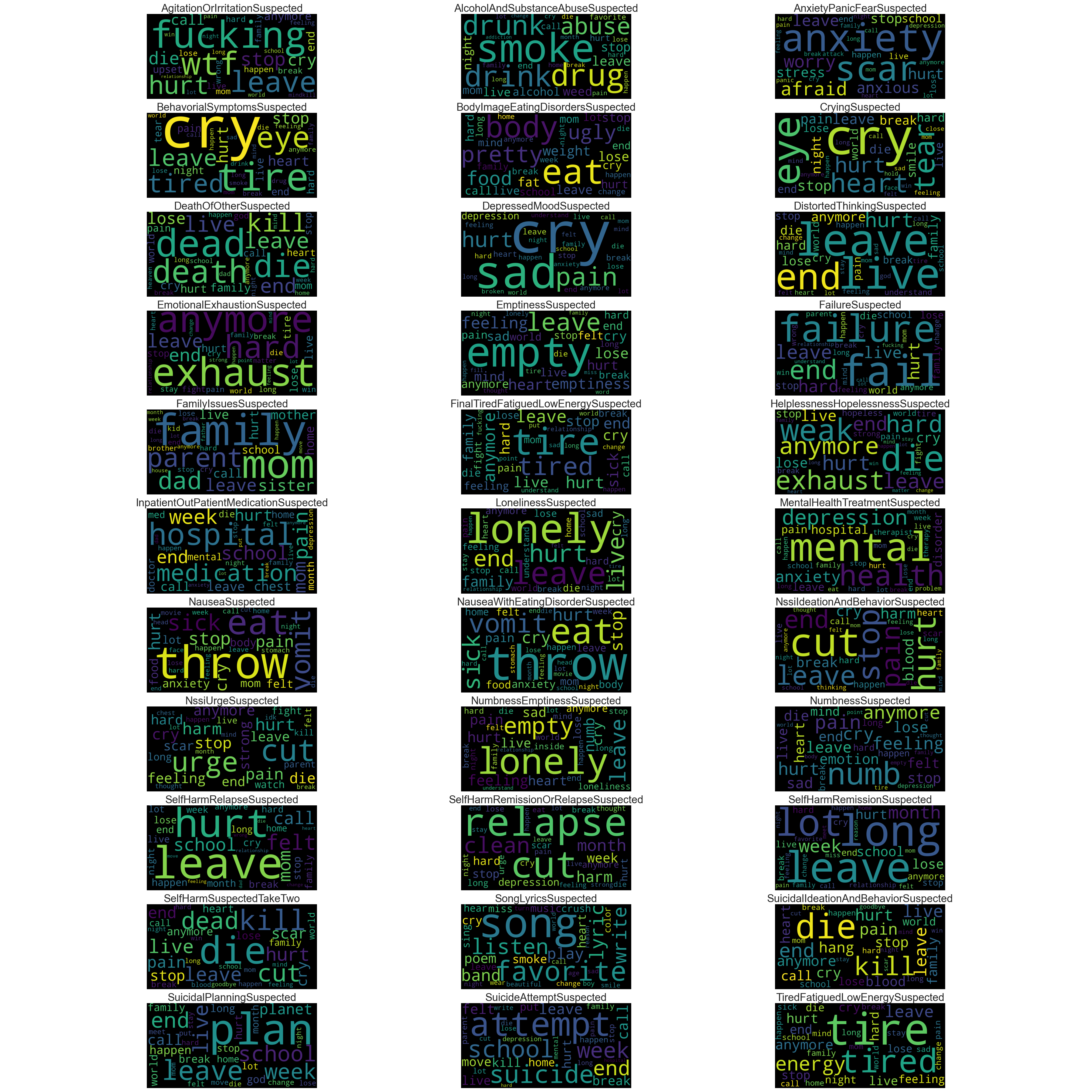}
    \caption{Top words visualization of the sparse word distributions of each label before clustering.}
    \label{fig:sparse_labels}
\end{figure}